%% file: main.tex
\documentclass[10pt]{article} % For LaTeX2e
\usepackage[accepted]{tmlr}
\usepackage{hyperref}
\usepackage{url}

\usepackage{graphicx}

\PassOptionsToPackage{noend}{algorithmic}

\usepackage{amsmath}
\usepackage{amssymb}
\usepackage{mathtools}
\usepackage{amsthm}
\usepackage{etoolbox}

\usepackage[capitalize,noabbrev]{cleveref}

\usepackage{algorithm}
\let\classAND\AND
\let\AND\relax
\usepackage{algorithmic}
\let\AND\classAND
\AtBeginEnvironment{algorithmic}{\let\AND\algoAND}

\usepackage[utf8]{inputenc} % allow utf-8 input
\usepackage[T1]{fontenc}    % use 8-bit T1 fonts
\usepackage{booktabs}       % professional-quality tables
\usepackage{amsfonts}       % blackboard math symbols
\usepackage{nicefrac}       % compact symbols for 1/2, etc.
\usepackage{microtype}      % microtypography
\usepackage{xcolor}         % colors
\newcommand\hl[1]{#1}	% highlight command
\newcommand\tmlrhl[1]{#1}	% highlight command

\usepackage{amsmath} % assumes amsmath package installed
\usepackage{mathtools}
\usepackage{amsthm}
\usepackage{amssymb}  % assumes amsmath package installed
\usepackage{twoopt}
\usepackage{float}
\usepackage{subcaption}
\usepackage{tikz}
\usepackage{tikzscale}
\usepackage{pgfplots}
\pgfplotsset{compat=1.7}

\usetikzlibrary{shapes,arrows,calc,decorations.pathreplacing,calligraphy}

\usepackage{makecell}

\usepackage{appendix}
\usepackage[bibliography=common]{apxproof}

\usepackage[export]{adjustbox}% http://ctan.org/pkg/adjustbox
\usepackage{diagbox}

\input{defs}

\renewcommand{\appendixsectionformat}[2]{Supplementary material for Section~#1 (#2)}
\title{Projected Randomized Smoothing for Certified Adversarial Robustness}

\author{\name Samuel Pfrommer \email sam.pfrommer@berkeley.edu \\
      \addr Department of Electrical Engineering and Computer Sciences \\
      University of California, Berkeley
      \AND
      \name Brendon G.\ Anderson \email bganderson@berkeley.edu \\
      \addr Department of Mechanical Engineering \\
      University of California, Berkeley
      \AND
      \name Somayeh Sojoudi \email sojoudi@berkeley.edu \\
      \addr Department of Electrical Engineering and Computer Sciences \\
      Department of Mechanical Engineering \\
      University of California, Berkeley}

\def\month{09}  % Insert correct month for camera-ready version
\def\year{2023} % Insert correct year for camera-ready version
\def\openreview{\url{https://openreview.net/forum?id=FObkvLwNSo}} % Insert correct link to OpenReview for camera-ready version

\begin{document}

\maketitle

\begin{abstract}
Randomized smoothing is the current state-of-the-art method for producing provably robust classifiers. While randomized smoothing typically yields robust $\ell_2$-ball certificates, recent research has generalized provable robustness to different norm balls as well as anisotropic regions. This work considers a classifier architecture that first projects onto a low-dimensional approximation of the data manifold and then applies a standard classifier. By performing randomized smoothing in the low-dimensional projected space, we characterize the certified region of our smoothed composite classifier back in the high-dimensional input space and prove a tractable lower bound on its volume. We show experimentally on CIFAR-10 and SVHN that classifiers without the initial projection are vulnerable to perturbations that are normal to the data manifold and yet are captured by the certified regions of our method. We compare the volume of our certified regions against various baselines and show that our method improves on the state-of-the-art by many orders of magnitude.
\footnote{Source code for reproducing our results is available on \href{https://github.com/spfrommer/projected_randomized_smoothing}{\textcolor{blue}{GitHub}}.}.
\end{abstract}

\input{intro}

\input{method}

\input{cert}

\input{experiment}

\input{conclusion}

\bibliography{main}
\bibliographystyle{tmlr}

\end{document}

%% file: defs.tex
\newcommand{\R}{{\mathbb R}}

\newcommand{\I}{I}
\newcommand{\normalchar}{{N}}
\newcommand{\normal}[1][]{{\normalchar(0, \sigma^2 \I_{#1})}}
\DeclareMathOperator*{\E}{\mathbb{E}}
\newcommand{\T}{^\intercal}
\newcommand{\Range}{\mathcal{R}}
\newcommand{\Null}{\mathcal{N}}
\newcommand{\defeq}{\vcentcolon=}
\newcommand{\vol}{{\protect\ooalign{\hfil$V$\hfil\cr\kern0.08em--\hfil\cr}}}
\DeclareMathOperator{\sign}{sign}

\DeclareMathOperator*{\argmax}{arg\,max}

\newcommand{\Cd}{C^d}
\newcommand{\Rd}{\R^d}
\newcommand{\Rp}{\R^p}
\newcommand{\Ce}{C^d_{\epsilon}}

\newcommand{\e}{\epsilon}

\newcommand{\classes}{[0,1]^c}
\newcommand{\labels}{\mathcal{Y}}

\newcommand{\xtil}{\tilde{x}}
\newcommand{\deltil}{\tilde{\delta}}
\newcommand{\p}{P}
\newcommand{\pout}{\tilde{P}}
\newcommand{\baseall}{f}
\newcommand{\net}{f_{\theta}}
\newcommand{\netout}{\tilde{f}_{\theta}}

\newcommand{\smoothout}{\tilde{f}^s_{\theta}}
\newcommand{\smoothsoft}{f^s}
\newcommand{\smoothhard}{g}
\newcommandtwoopt{\cert}[2][][R]{\Delta_{#1}^U(#2)}
\newcommandtwoopt{\certt}[2][][R]{\tilde{\Delta}_{#1}^U(#2)}

\newcommand{\SUx}{S^{\Null(U\T)}_{d-p}(x)}

\renewcommand{\xi}{x^{(i)}} 
\newcommand{\xip}{x^{(i+1)}} 
\newcommand{\xo}{x^{(1)}}

\newcommand{\di}{\delta_V^{(i)}}
\newcommand{\dip}{\delta_V^{(i+1)}}

\newcommand{\pgd}{\textsc{PGD}}
\newcommand{\spgd}{\textsc{SubspacePGD}}

\makeatletter
\def\thm@space@setup{\thm@preskip=8pt
\thm@postskip=0pt}
\makeatother
\newtheoremrep{theorem}{Theorem}
\newtheoremrep{proposition}{Proposition}
\newtheoremrep{lemma}{Lemma}
\newtheoremrep{corollary}{Corollary}

\theoremstyle{definition}
\newtheorem{definition}{Definition}

%% file: intro.tex
\section{Introduction}
\label{sec:intro}

Despite their state-of-the-art performance on a variety of machine learning tasks, neural networks are vulnerable to adversarial inputs---inputs with small (often human-imperceptible) noise that is maliciously crafted to induce failure \citep{biggio2013evasion,szegedy2014intriguing,nguyen2015deep}. This sensitive behavior is unacceptable in contemporary safety-critical applications of neural networks, such as autonomous driving \citep{bojarski2016end,wu2017squeezedet} and the operations of power systems \citep{kong2017short}. The works \citet{eykholt2018robust} and \citet{liu2019perceptual} highlight the validity and eminence of these threats, wherein both physical and digital adversarial perturbations are shown to cause image classification models to misclassify vehicle traffic signs.

Heuristics have been proposed to defend against various adversarial attacks, only to be defeated by stronger attack methods, leading to an ``arms race'' in the literature \citep{carlini2017adversarial,kurakin2017adversarial,athalye2018obfuscated,uesato2018adversarial,madry2018towards}. This has motivated researchers to consider certifiable robustness---theoretical proof that models perform reliably when subject to arbitrary attacks of a bounded norm \citep{wong2018provable,weng2018towards,raghunathan2018semidefinite,anderson2020tightened,ma2021sequential}. Randomized smoothing, popularized in \citet{lecuyer2019certified,li2019certified,cohen2019certified}, remains one of the state-of-the-art methods for generating classifiers with certified robustness guarantees. Instead of directly classifying a given input, randomized smoothing intentionally corrupts the input with random noise and returns the most probable class, which, intuitively, ``averages out'' any potential adversarial perturbations in the data.

The seminal work \citet{cohen2019certified} certifies that no adversarial perturbation within a certain $\ell_2$-ball can cause the misclassification of a smoothed model using isotropic Gaussian noise of a fixed variance. Recent works have attempted to certify larger regions of the input space by turning to randomized smoothing with optimized variances \citep{zhai2020macer}, input-dependent variances \citep{alfarra2020data,wang2021pretrain}, anisotropic distributions \citep{eiras2021ancer}, and semi-infinite linear programming \citep{anderson2022towards}. However, for a fixed variance, the certified radius is upper-bounded by a constant in the dimension $d$ of the input \citep{kumar2020curse}, implying that the volume of the certified $\ell_2$-ball degrades factorially fast as $O(K^d \Gamma(\frac{d}{2}+1)^{-1})$, where $\Gamma$ is Euler's gamma function and $K$ is some positive constant \citep{folland1999real}. Current input-dependent and anisotropic smoothing approaches have similarly been shown to suffer from the curse of dimensionality \citep{sukenik2021intriguing}.

The small certified regions of randomized smoothing in high dimensions corroborate empirical findings that show increased robustness when precomposing classifiers with dimensionality reduction, e.g., principal component analysis projections \citep{bhagoji2018enhancing} and autoencoders \citep{sahay2019combatting}. These findings align with the manifold hypothesis, which posits that real datasets lie on a low-dimensional manifold in a high-dimensional feature space \citep{fefferman2016testing}, and related results showing that perturbation directions most useful to an adversary are ones normal to this manifold \citep{jha2018detecting,zhang2020principal}. Thus, projecting inputs onto the manifold, or at least a low-dimensional subspace containing the manifold, should increase classification robustness. Methods taking this approach, such as \citet{mustafa2019image} and \citet{alemany2022dilemma}, have worked well as heuristics, but lack theoretical robustness guarantees. Motivated by these works, we aim to enlarge the certifiably robust regions of randomized smoothing by performing the smoothing in a low-dimensional space in which adversarial access to the data's statistically insignificant yet vulnerable features has been eliminated.

\subsection{Contributions}

We propose \emph{projected randomized smoothing}, whereby inputs are projected onto a low-dimensional linear subspace in which randomized smoothing is applied before classification. Our method combines the empirical successes of dimension-reducing projection methods with the theoretical guarantees of randomized smoothing to achieve the following contributions:
\begin{enumerate}
	\item We theoretically characterize the geometry of the certified region in the input space and prove a tractable lower bound on the volume of this certified region.
	\item We empirically demonstrate that classifiers can be attacked along subspaces spanned by statistically insignificant features that contribute nothing to classification accuracy, which are vulnerabilities that projected randomized smoothing certifiably eliminates.
    \item Experiments on CIFAR-10 \citep{krizhevsky2009learning} and SVHN \citep{netzer2011reading} show that our method yields certified regions with order-of-magnitude larger volumes than prior smoothing schemes.
\end{enumerate}

\subsection{Related works}

\paragraph{Robustification via dimensionality reduction.}
The work \citet{bhagoji2018enhancing} was the first to consider linearly projecting inputs onto the top principal components of the training data before classification as a means to improve empirical (not certified) robustness. The authors of \citet{sahay2019combatting} nonlinearly preprocess test data using denoising and dimension-reducing autoencoders, and find a substantial increase in classification accuracy when the inputs are subject to the popular fast gradient sign method attack.
 \hl{The work \citet{bafna2018thwarting} projects an input onto its top-$k$ discrete cosine transform components to defend against ``$\ell_0$''-attacks, but this empirical defense was later broken using adapative ``$\ell_0$''-attacks \citep{tramer2020adaptive}, which directly motivates our approach for certified projection-based robustness. The work \citet{sanyal2018robustness} introduces a low-rank regularizer to encourage neural network feature representations to reside in a low-dimensional linear subspace, which is found to enhance empirical robustness.}
In \citet{mustafa2019image}, the authors use super-resolution to project images onto the natural data manifold and obtain high empirical robustness for convolutional neural networks. \citet{alemany2022dilemma} shows that decreasing the codimension of data, i.e., decreasing the difference between the intrinsic dimension of the data manifold and the dimension of the input space in which it is embedded, generally leads to increased robustness of models defined on that input space.

\hl{\citet{shamir2021dimpled} posits that learned decision boundaries tend to align with and ``dimple'' around the natural data manifold, and that adversarial perturbations are normal to this manifold. This finding supports our approach for certifiably eliminating off-manifold perturbations by projecting onto a low-dimensional approximation of the data manifold. The authors of \citet{awasthi2021adversarially} reformulate principal component analysis to find projections that are robust with respect to projection error---a method that naturally complements our framework---and give robustness guarantees for the Bayes optimal projection-based classifier in the special case of binary Gaussian-distributed data.} The work \citet{zeng2021certified} precomposes classifiers with orthogonal encoders and performs randomized smoothing in the encoder's low-dimensional latent space as a means to speed up the sample-based smoothing procedure. To the best of our knowledge, \citet{zeng2021certified} is the only work that provides certified robustness guarantees for general models and data distributions when using dimensionality reduction at the input---all of the other referenced works are heuristic---and their choice of orthogonal encoders ensures that the certified $\ell_2$-ball in the input space has the same radius as that in the latent space. \hl{Notably, their approach is highly conservative in estimating the input-space certified set as it relies on Lipschitzness of the orthogonal encoding layers, and is thus employed primarily as a means to speed up randomized smoothing. On the other hand, the method we propose uses a robustification-motivated projection for which we prove more general (anisotropic) certicates that capture off-manifold perturbations.}

% We may want to also mention that orthogonal encoders are nontrivial to train and may be model-dependent (i.e., dependent on the classifier), whereas the use of PCA projection is entirely model agnostic and dependent only on the dataset.

\paragraph{Certification via randomized smoothing.}

The work \citet{cohen2019certified} develops randomized smoothing using an isotropic Gaussian distribution with input-independent variance to obtain certified $\ell_2$-balls. A subsequent line of works attempts to generalize randomized smoothing to other classes of certified regions, e.g., Wasserstein, ``$\ell_0$''-, $\ell_1$-, and $\ell_\infty$-balls \citep{levine2020wasserstein,lee2019tight,teng2019ell_1,yang2020randomized}. Various approaches have been taken to enlarge the certified regions. For example, \citet{salman2019provably} unifies adversarial training with randomized smoothing to obtain state-of-the-art certified $\ell_2$-radii. The authors of \citet{zhai2020macer} incorporate the certified $\ell_2$-radius into the model's training objective as a means to enlarge certified regions. The method in \citet{zhang2020black} optimizes over base classifiers to increase the size of more general $\ell_p$-balls. \tmlrhl{\citet{li2022double} employs a second smoothing distribution to tighten robustness certificates.}

Optimizing the certified region pointwise in the input space has also been considered, but generally these methods require locally constant smoothing distributions to ensure that the resulting certificates are mathematically valid \citep{alfarra2020data,wang2021pretrain,sukenik2021intriguing,anderson2022certified}. To further strengthen the robustness guarantees of randomized smoothing, the recent works \citet{eiras2021ancer,erdemir2021adversarial,tecot2021robustness} have turned to certifying anisotropic regions of the input space. For example, \citet{eiras2021ancer} maximizes the volume of certified ellipsoids and generalized cross-polytopes of the form $\{x\in \Rd : \|Ax\|_p \le b\}$ for $p\in\{1,2\}$, allowing for the certification of perturbations that are potentially larger in magnitude than the minimum adversarial perturbation. We show in Section \ref{sec:experiments} that our proposed method is able to outperform these methods by leveraging dimensionality reduction. \hl{As is standard practice in the randomized smoothing literature \citep{cohen2019certified,yang2020randomized,jeong2021smoothmix,zhai2020macer,lee2019tight}, our emphasis is on certified robustness and not empirical robustness---we refer the reader to \citet{maho2022randomized} for connections between certified and empirical robustness under randomized smoothing, and in particular the difficulty in constructing and evaluating suitable empirical attacks.}

We also emphasize that volume (Lebesgue measure) is the natural scalar measure for the size of anisotropic certified regions of the input space and is the standard notion considered by prior works \citep{liu2019certifying,eiras2021ancer,tecot2021robustness}.

\subsection{Notation}
We denote the set of real numbers by $\R$. The $\ell_2$-norm of a vector $x\in \R^n$ is denoted by $\|x\|$, whereas the general $\ell_p$-norm is given an explicit subscript $\|x\|_p$. The range and nullspace of a matrix $U\in \R^{m\times n}$ are denoted by $\Range(U) \subseteq \R^m$ and $\Null(U) \subseteq \R^n$, respectively. The $n\times n$ identity matrix is written as $I_n$. For a random variable $X$ with distribution $\mathcal{D}$ and a measurable function $f$, the expectation of $f(X)$ is denoted by $\E_{X\sim \mathcal{D}}f(X)$. The multivariate normal distribution with mean $\mu \in \R^n$ and covariance $\Sigma \in \R^{n \times n}$ is given by $\normalchar(\mu, \Sigma)$. The cardinality of a set $S$ is written as $|S|$. For a Lebesgue-measurable set $S\subseteq \R^n$ contained in a $k$-dimensional affine subspace, we write $\vol_k(S)$, termed the $k$-dimensional volume of $S$, to mean the Lebesgue measure of $S$ within that affine subspace. For sets $S,T\subseteq\R^n$, we denote their Minkowski sum by $S+T=\{x+y : x\in S, ~ y\in T\}$. Euler's gamma function is denoted by $\Gamma$. Recall that $\Gamma(n) = (n-1)!$ when $n$ is a positive integer.

%% file: method.tex
\section{Classifier architecture}

Consider the task of classifying inputs from a zero-centered cube $\Cd=[-1/2,1/2]^d \subseteq \Rd$ into $c$ distinct classes $\labels = \{1,2,\dots,c\}$.\footnote{The zero-centered cube is used without loss of generality instead of $[0, 1]^d$ for notational convenience and compatibility with results from the mathematical literature.} Under the randomized smoothing framework, we begin with a given classifier $\net\colon \Rd\to\classes$, parameterized by $\theta$, that maps into the probability simplex over $c$ classes. The problem at hand is to increase the robustness of $\net$ with certifiable guarantees.

\paragraph{Vanilla randomized smoothing.}
We give a brief overview of how this would be accomplished using vanilla randomized smoothing \citep{cohen2019certified}. Randomized smoothing takes the \textit{base classifier} $\net$ and smooths it with Gaussian noise on the input to yield the associated smoothed soft and hard classifiers
\[
    \smoothsoft(x) = \E_{\e \sim \normal[d]} \net(x + \e), \quad \smoothhard(x) = \argmax_{y\in\labels} \smoothsoft(x)_y,
\]
where $f^s(x)_y$ denotes the $y$th component of the vector $f^s(x)$ and $\sigma$ is a hyperparameter. \citet[Theorem~1]{cohen2019certified} then gives, under certain conditions, a certified $\ell_2$-ball for a particular input $x \in \Rd$; namely, that $g(x + \delta) = g(x)$ for all $\|\delta\| < R$, where $R > 0$ is determined by the confidence of the smoothed classifier at $x$. We leverage this result for our approach and refer interested readers to \citet{cohen2019certified} for additional details on the computation of the smoothing expectation and precise formula for $R$.

\paragraph{Projected randomized smoothing.}
Motivated by the relationships between robustness and dimensionality described in Section \ref{sec:intro}, we consider $p<d$ and let $\p\colon\Rd\to\Rp$ be a projection into $\Rp$ defined by $\p(x) = U\T x$, where $U\in \R^{d\times p}$ is a semi-orthogonal matrix satisfying $U\T U = \I_p$. Similarly, we let the reconstruction $\pout\colon \Rp\to\Rd$ be defined by $\pout(\xtil) = U \xtil$. Throughout, we let $v_1,\dots,v_{d-p} \in \Rd$ be an orthonormal basis for $\Null(U\T)$ and let $v_{d-p+1},\dots,v_d\in\Rd$ denote the orthonormal columns of $U$. In practice, we instantiate the columns of $U$ as the first $p$ principal components of a random subset of the training dataset, \tmlrhl{although our method and theory hold for any orthonormal set of vectors}. With the dimension-reducing projection $P$ in place, we consider the classifier architecture consisting of the composition 
\[
    f = \net\circ\pout\circ\p.
\]
In particular, $f$ first uses $\p$ to project inputs into the low-dimensional space $\Rp$ and then reconstructs the inputs in a lossy way using $\pout$ before feeding them through the classifier $\net$. We generally finetune $\net$ to account for the slight image corruption associated with the projection step.

We now propose \emph{projected randomized smoothing}, wherein randomized smoothing is performed in the compressed space $\Rp$. To do so, we define $\netout\colon \Rp \to \classes$ by $\netout = \net \circ \pout$ so that $\baseall = \netout \circ \p$, and we smooth $\netout$ by adding Gaussian noise in its low-dimensional input space to obtain a new classifier $\smoothout\colon\Rp \to \classes$ defined by
\begin{equation}
    \smoothout(\xtil) = \E_{\e \sim \normal[p]} \netout(\xtil + \epsilon).
    \label{eq:smoothout}
\end{equation}
The new overall smoothed soft classifier is then given by
\begin{equation}
\smoothsoft = \smoothout \circ \p,
\label{eq:smoothsoft}
\end{equation}
and its structure is illustrated in Figure~\ref{fig:architecture}. The corresponding hard classifier is then given by the $\argmax$ of the soft classifier:\footnote{For ease of exposition, we assume throughout that all $\argmax$ yield singleton sets and therefore equality signs may be used unambiguously.}
\begin{equation}
    \smoothhard(x) = \argmax_{y\in\labels} \smoothsoft(x)_y.
    \label{eq:smoothhard}
\end{equation}
A graphical illustration of our approach for $d=2$ is shown in Figure~\ref{fig:projected_rs_illustration}. To summarize, classifying an input $x\in \Rd$ using projected randomized smoothing amounts to applying the mapping $x\mapsto \smoothhard(x)$ defined by \eqref{eq:smoothout} through \eqref{eq:smoothhard}, and it is for $\smoothhard$ that we seek to derive certified regions of the input space.

\begin{figure}
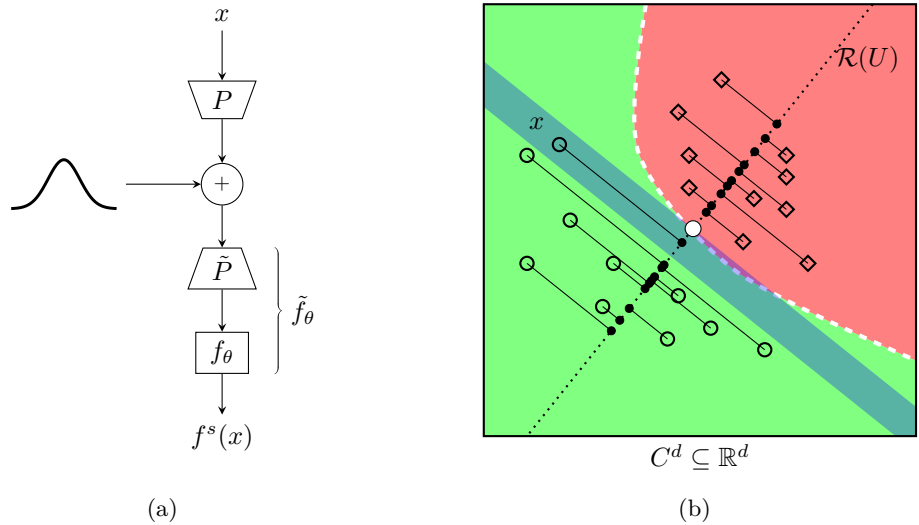

    \centering
    \hfil
    \begin{subfigure}[b]{0.35\linewidth}
	\centering
	\includegraphics[height=1.1\linewidth]{res/method_fig.tikz}
	\caption{~}
        \label{fig:architecture}
    \end{subfigure}
    \hfil
    \begin{subfigure}[b]{0.35\linewidth}
    %\vspace*{0.1cm}
	\centering
	\includegraphics[height=1.1\linewidth]{res/projected_rs_illustration.tikz}
	\caption{~}
        \label{fig:projected_rs_illustration}
    \end{subfigure}
    \hfil
	\caption{
        (\subref{fig:architecture}) Projected randomized smoothing architecture. Inputs $x$ are projected into low-dimensional space by $\p$, smoothed with Gaussian noise, and then reconstructed by $\pout$ and classified by $\net$.
        (\subref{fig:projected_rs_illustration}) Illustration of projected randomized smoothing for a binary classification task (circles vs. squares). The base classifier decision regions are shown in green and red. The white circle represents the smoothed decision boundary in $\Rp$, $p=1$, with the projected subspace depicted by the dotted line and projected points depicted as solid dots. The blue area represents the certified region around $x$ in $\Rd$ of the projected randomized smoothing classifier $\smoothhard$.
	}
\end{figure}

%% file: cert.tex
\section{Robustness certificates}
In this section, we construct certified regions for $g$ around arbitrary inputs $x$ in the high-dimensional space $\Rd$. The key idea is that $\smoothout$ is $\ell_2$-ball robust in the low-dimensional space $\Rp$, and the preimage of this ball in the original input space is then ``large'' as it includes the inputs in $\Null(U\T)$. We formalize the geometry of the certified region in Section~\ref{sec:geometry} and introduce our metric of interest as the volume of the certified region restricted to the unit cube of feasible inputs. In Section~\ref{sec:vol}, we provide a lower bound on this volume in high-dimensional spaces that involves solving an $\ell_\infty$-norm linear regression. Section~\ref{sec:asymptotic} compares the asymptotic behavior of the volume of the certified region of $g$ with the standard $\ell_2$-ball certificates as the input dimension grows large. Finally, we discuss runtime and limitations in Section~\ref{sec:runtime}. For ease of exposition, all proofs are deferred to the appendices.

\subsection{Characterizing the certified region geometry}
\label{sec:geometry}

In the following two propositions, we characterize the geometry of the projected randomized smoothing classifier $g$ in the high-dimensional input space $\Rd$ based on the certified $\ell_2$-robustness of the classifier $\smoothout$ in the low-dimensional projected space $\Rp$.

\begin{definition}
	\label{def:certified}
	Let $\xtil\in\Rp$ and $R\ge 0$. The classifier $\smoothout\colon\Rp\to\classes$ is said to be \emph{certified at $\xtil$ with radius $R$} if 
    \[
        \argmax_{y\in\labels}\smoothout(\xtil+\deltil)_y = \argmax_{y\in\labels}\smoothout(\xtil)_y
    \]
    for all $\deltil\in\Rp$ satisfying $\|\deltil\|\le R$.
\end{definition}

\begin{propositionrep}
	\label{prop:certset}
    Let $x \in \Rd$ and $R\ge 0$. If $\smoothout$ is certified at $\p(x) = U\T x$ with radius $R$, then $\smoothhard(x + \delta) = \smoothhard(x)$ for all $\delta \in \cert \subseteq \Rd$, where
    \[
        \cert \defeq \{ \delta \in \Rd : \| U\T \delta \| \le R \}
    \]
\end{propositionrep}
\begin{proof}
    Let $\delta \in \cert$. Then
    \begin{equation*}
        \smoothhard(x + \delta) = \argmax_{y\in\labels} \smoothout(\p(x + \delta))_y = \argmax_{y\in\labels} \smoothout(\p(x) + U\T \delta)_y.
    \end{equation*}
    Since $\|U\T \delta\| \le R$ by definition of $\cert$ and $\smoothout$ is certified at $\p(x)$ with radius $R$, we have that
    \begin{equation*}
        \smoothhard(x + \delta) = \argmax_{y\in\labels} \smoothout(\p(x))_y = g(x).
    \end{equation*}
\end{proof}

\begin{propositionrep}
\label{prop:certsetgeo}
    Let $R\ge 0$. The certified region $\cert$ can be expressed as the Minkowski sum $\cert = B_p^U(R) + \Null(U\T)$, where $B_p^U(R) \subseteq \Rd$ is a $p$-dimensional ball embedded into $\Range(U)$:
    \begin{equation*}
        B_p^U(R) \defeq \left\{ \beta_1 v_{d-p+1} + \cdots + \beta_{p} v_d : \|\beta\| \le R, ~ \beta\in \R^p \right\}.
    \end{equation*}
    \vspace*{-0.7cm}
\end{propositionrep}
\begin{proof}
    Let $y = y_1 + y_2$ with $y_1 \in B_p^U(R)$ and $y_2 \in \Null(U\T)$. Then
    \[
        \| U\T y \| = \| U\T y_1 \| = \| \beta \| \le R,
    \]
    so $y\in\cert$.

    On the other hand, let $y \in \cert$ as defined in Proposition~\ref{prop:certset}. We can decompose $y = y_1 + y_2$ for $y_1 \in \Range(U)$ and $y_2 \in \Null(U\T)$. Then there exists $\beta\in\R^p$ such that $y_1 = U\beta = \sum_{i=d-p+1}^n \beta_{i-d+p} v_i$, so $\| U\T y_1 \| = \| \beta \|$ and therefore $\| \beta \| \le R$.
\end{proof}
Propositions~\ref{prop:certset} and \ref{prop:certsetgeo} characterize the geometry of the certified region of our classifier $g$. Proposition~\ref{prop:certset} provides an easy-to-check condition for an input to lie in the certified region, while Proposition~\ref{prop:certsetgeo} formalizes the same geometry as a hypercylinder consisting of a low-dimensional sphere that is ``extruded'' along the nullspace of the projection $\p$, allowing us to certify adversarial off-manifold inputs of potentially very large magnitude that are projected back onto the natural data manifold. Intuitively, the certified region $\cert$ is potentially much larger than an $\ell_2$-ball of radius $R$ in $\Rd$, as it captures perturbations in the nullspace of $U\T$ whose dimensionality is large when $p \ll d$.

\tmlrhl{We note that the above characterization of the decision region geometry holds analogously for other norm ball certificates in the projected space (i.e., the $\ell_1$-ball certificates of \citet{levine2021improved}). While the following theory is presented for the concrete case of $\ell_2$-ball certificates, it also applies to this more general setting. Concrete experiments with other certificates is an exciting line of future work.}

\subsection{Lower-bounding the certified region volume}
\label{sec:vol}
To compare a standard $\ell_2$-ball certificate with our certified region $\cert$, which does not immediately come equipped with a notion of ``radius,'' we adopt the perspective of recent works, e.g., \citet{liu2019certifying,eiras2021ancer,tecot2021robustness}, by considering our metric of interest to be the volume of the certified region. One immediate issue is that the volume of $\cert$ is infinite since $\Null(U\T)$ is an unbounded subspace. To enable meaningful comparisons, we restrict ourselves to measuring the volume of $\cert$ contained in the cube $\Cd = [-1/2, 1/2]^d$ of possible inputs. This amounts to computing the volume
\begin{equation}
    \vol_d \left( \Cd \cap \cert[x] \right),
    \label{eqn:vol}
\end{equation}
where we recall that $\vol_d$ measures $d$-dimensional volume in Euclidean space, and $\cert[x] \defeq \{ x + \delta : \delta \in \cert \}$, with $R$ chosen such that $\smoothout$ is certified at $P(x)$ with radius $R$ so that $g(x') = g(x)$ for all $x'\in \cert[x]$ by Proposition~\ref{prop:certset}. Computing the volume in \eqref{eqn:vol} is highly nontrivial, especially in high-dimensional input spaces. Instead, we develop a tractable lower bound on $\vol_d(\Cd\cap\cert[x])$ throughout the remainder of this section. Since $\cert[x]$ contains affine subspaces, this derivation rests heavily on theory regarding cube-subspace intersections in high dimensions. The most important result for our purposes comes from \citet{vaaler1979geometric}, which showed the following.

% \begin{theorem}[\citep{vaaler1979geometric}]
\begin{theorem}
	\label{thm:vol}
    Let $S_k$ be a $k$-dimensional linear subspace of $\Rd$. Then $\vol_k (\Cd \cap S_k) \geq 1$.
    \vspace*{-0.1cm}
\end{theorem}

This result proved Good's conjecture and generalized a previous result for the $k=d-1$ case \citep{hensley1979slicing}. We begin with an extension of Theorem~\ref{thm:vol} to cubes of non-unit side length, and then to intersections with affine subspaces which do not necessarily contain the origin.

\begin{corollaryrep}
	\label{cor:vol}
    Let $S_k$ be a $k$-dimensional linear subspace of $\Rd$ and $r \Cd$ be a zero-centered cube of side length $r > 0$. Then $\vol_k (r\Cd \cap S_k) \geq r^k$.
\end{corollaryrep}
\begin{proof}
    Note that 
    \begin{align*}
        r\Cd \cap S_k &= \{ x \in \Rd : \| x \|_{\infty} \leq r/2,~ x \in S_k \} \\
                      &= \{ rx \in \Rd : \| rx \|_{\infty} \leq r/2, ~ rx \in S_k \} \\
                      &= \{ rx \in \Rd : \| x \|_{\infty} \leq 1/2, ~ x \in S_k \},
    \end{align*}
    since $x \in S_k$ if and only if $rx \in S_k$, by linearity of $S_k$. This is now equivalent to the set $r(\Cd \cap S_k)$, and we have scaled our $k$-dimensional subset by a uniform factor $r$. Therefore, $\vol_k(r \Cd \cap S_k) = \vol_k(r(\Cd \cap S_k)) = r^k \vol_k(\Cd \cap S_k)$ by \citet[Theorem~2.44]{folland1999real}. Thus, by Theorem~\ref{thm:vol}, we have ${\vol_k(r\Cd\cap S_k) \ge r^k}$.
\end{proof}

\begin{corollaryrep}
	\label{cor:volarbitrary}
    Let $x\in\Rd$ and let $S_k(x) \subseteq \Rd$ be the $k$-dimensional affine subspace 
    \[
        S_k(x) = \bigg\{ x + \sum_{i=1}^k \alpha_i v_i : \alpha \in\R^k \bigg\}
    \]
    spanned by arbitrary vectors $v_1,\dots,v_k$ and passing through $x$. Let $t\ge 0$ be the minimal $\ell_\infty$-norm of a point in $S_k(x)$:
    \begin{equation}
        t \defeq \inf_{x'\in S_k(x)}\|x'\|_\infty = \inf_{\alpha \in \R^{k}} ~ \left\lVert  x + \sum\nolimits_{i=1}^k \alpha_i v_i \right\rVert_{\infty}.
	    \label{eqn:mindist}
    \end{equation}
    Then, for all $r > 2t$, it holds that $\vol_k (r \Cd \cap S_k(x)) \geq (r - 2t)^k$.
    \vspace*{-0.1cm}
\end{corollaryrep}
\begin{proof}
First, notice that the infimum in \eqref{eqn:mindist} is attained since $\|\cdot\|_\infty$ is continuous and coercive, and $S_k(x)$ is closed in the standard topology on $\Rd$ \citep{bertsekas2016nonlinear}. Let $x^*\in S_k(x)$ be a point that attains the infimum in \eqref{eqn:mindist} so that $\|x^*\|_{\infty} = t$. If $r > 2t$, then $x^*$ is contained in the interior of $r\Cd$. In this case, we can construct a nonempty cube centered at $x^*$ with side lengths $r - 2t>0$ that is contained in $r\Cd$. Now, the plane $S_k(x)$ passes through $x^*$, and therefore Corollary~\ref{cor:vol} yields the result since volume is preserved under translation \citep[Theorem~2.42]{folland1999real}.
\end{proof}

Corollary~\ref{cor:volarbitrary} generalizes Corollary~\ref{cor:vol} to affine subspaces. If $S_k(x)$ contains the origin, $t = 0$ and the bound from Corollary~\ref{cor:vol} is recovered. We are now ready to present the main result of this section.

\begin{theoremrep}
\label{thm:certified_volume}
    Let $x\in\Cd$, let $t$ be defined as in \eqref{eqn:mindist} with $k=d-p$, and let $R \in [0,1/2 - t]$. If $\smoothout$ is certified at $\p(x)=U\T x$ with radius $R$, then
    \begin{equation}
        \vol_d(\Cd \cap \cert[x]) \geq \frac{\pi^{p/2}}{\Gamma(\frac{p}{2} + 1)} R^p (1 - 2R - 2t)^{d-p}.
	\label{eqn:projrsvol}
    \end{equation}
    \vspace*{-0.5cm}
\end{theoremrep}
\begin{proof}
    The characterization of $\cert$ in Proposition~\ref{prop:certsetgeo} yields
\begin{equation*}
        \cert[x] = B_p^U(R) + \SUx,
\end{equation*}
where
\begin{equation*}
	\SUx \defeq \{x\} + \Null(U\T)
\end{equation*}
is the affine subspace of $\Rd$ spanned by $\Null(U\T)$ and passing through $x$, which has dimension $d-p$. Therefore, the following is an inner-approximation of $\cert[x]$:
\begin{equation*}
        \certt[x] \defeq B_p^U(R) + \left((1 - 2R) \Cd \cap \SUx \right) \subseteq B_p^U(R) + \SUx = \cert[x].
\end{equation*}
    If we can show that $\certt[x] \subseteq \Cd$, then $\certt[x] \subseteq \Cd \cap \cert[x]$, in which case the volume of $\certt[x]$ will lower-bound the volume of $\Cd \cap \cert[x]$. To prove that this holds, let $y = y_1 + y_2 \in \certt[x]$ with $y_1\in B_p^U(R)$ and $y_2\in (1-2R)\Cd\cap\SUx$. Then
    \[
        \| y \|_{\infty} \leq \| y_1 \|_{\infty} + \| y_2 \|_{\infty} \leq R + \frac{1 - 2R}{2} = \frac{1}{2},
    \]
    by the fact that $\|y_1\|_\infty \le \|y_1\| = \|U\beta\| = \|\beta\|$ for some $\beta\in\Rp$ with $\|\beta\|\le R$ due to the semi-orthogonality of $U$, and by the fact that $y_2\in (1-2R)\Cd$. Therefore, indeed it holds that $\certt[x]\subseteq\Cd$. Thus, all that remains is to lower-bound $\vol_d(\certt[x])$. To this end, notice that $B_p^U(R)\subseteq \Range(U)$ and $(1-2R)\Cd\cap\SUx\subseteq \{x\} + \Null(U\T)$, so $B_p^U(R)$ and $(1-2R)\Cd\cap \SUx$ are contained in orthogonal affine subspaces, and therefore $\vol_d(\certt[x]) = \vol_p(B_p^U(R))\vol_{d-p}((1-2R)\Cd\cap\SUx)$. The $p$-dimensional volume of the embedded ball $\ell_2$-ball $B_p^U(R)$ is well-known (e.g., see \citet[Theorem~2.44,~Corollary~2.55]{folland1999real}) to be
    \begin{equation*}
    	\vol_p(B_p^U(R)) = \frac{\pi^{p / 2}}{\Gamma(\frac{p}{2}+1)}R^p.
    \end{equation*}
    On the other hand, since $2R<1-2t$, it holds that $1-2R>2t$. Hence Corollary~\ref{cor:volarbitrary} gives that the $(d-p)$-dimensional volume of $(1-2R)\Cd\cap \SUx$ is lower-bounded as
    \begin{equation*}
    	\vol_{d-p}((1-2R)\Cd\cap\SUx) \ge (1-2R-2t)^{d-p}.
    \end{equation*}
    Therefore,
    \begin{equation*}
    	\vol_{d}(\certt[x]) \ge \frac{\pi^{p / 2}}{\Gamma(\frac{p}{2}+1)}R^p (1-2R-2t)^{d-p},
    \end{equation*}
    which concludes the proof.
\end{proof}

Notice that the lower bound given in Theorem \ref{thm:certified_volume} does not monotonically increase with the certified radius $R$ from the randomized smoothing performed in $\Rp$. Therefore, if the certified radius $R$ is large enough, we may be able to improve our lower bound on the volume $\vol_d (\Cd \cap \cert[x])$ by using a smaller certified radius (which is of course still valid), and in particular, we may choose the optimal such radius to use according to the following closed-form expression.

\begin{propositionrep}
	Let $t$ and $R$ be as in Theorem \ref{thm:certified_volume}. The lower bound \eqref{eqn:projrsvol} is maximized as follows:
	\begin{align}
        r^* \coloneqq \min\left\{R,\frac{p(1-2t)}{2d}\right\} \in \argmax_{r\in[0,R]} \frac{\pi^{p / 2}}{\Gamma\left(\frac{p}{2}+1\right)} r^p \left(1-2r-2t\right)^{d-p}.
	\label{eq:optimal_bound}
	\end{align}
    \vspace*{-0.4cm}
\end{propositionrep}
\begin{proof}
	It suffices to maximize $h(r) \coloneqq r^p\left(1-2r-2t\right)^{d-p}$ over $r\in[0,R]$. The gradient of $h$ vanishes at points satisfying
	\begin{align*}
		\frac{dh}{dr}(r) &= pr^{p-1}\left(1-2r-2t\right)^{d-p} - 2(d-p)r^p\left(1-2r-2t\right)^{d-p-1} \\
				 &= r^{p-1}\left(1-2r-2t\right)^{d-p-1}\big(p\left(1-2r-2t\right) - 2(d-p)r\big) \\
				 &= r^{p-1}\left(1-2r-2t\right)^{d-p-1}\left(p - 2pt - 2dr\right) \\
				 &= 0.
	\end{align*}
	The set of all critical points satisfying this polynomial equation is $\left\{0,\frac{p(1-2t)}{2d}, 1 / 2 - t\right\}$. Notice that $0 < \frac{p(1-2t)}{2d} < \frac{p(1-2t)}{2p} = 1 / 2 - t$, and that $\frac{dh}{dr}(r)\ge 0$ for all $r\in\left[0,\frac{p(1-2t)}{2d}\right]$ whereas $\frac{dh}{dr}(r) \le 0$ for all $r\in \left[\frac{p(1-2t)}{2d}, 1 / 2 - t\right]$. Hence, $h$ is unimodal on $[0,1 / 2 - t]$ with the maximizer $\frac{p(1-2t)}{2d}$. Therefore, if $R < \frac{p(1-2t)}{2d}$, then $h$ is monotone increasing on the feasible interval $[0,R]$, which implies that the right endpoint $r^*=R$ is a maximizer of \eqref{eq:optimal_bound}. On the other hand, if $R \ge \frac{p(1-2t)}{2d}$, then $\frac{p(1-2t)}{2d}$ is contained in the feasible interval $[0,R]$, and thus $r^* = \frac{p(1-2t)}{2d}$ is a maximizer of \eqref{eq:optimal_bound}.
\end{proof}

The overall certification procedure derived in this section is summarized in Algorithm~\ref{alg:cert}. \hl{We note that our method inherits its ABSTAIN behavior from the original randomized smoothing Monte Carlo sampling scheme \citep{cohen2019certified}; namely, we evaluate the certification confidence using many Gaussian-perturbed samples, and if the prediction or certification procedures do not resolve with a user-specified confidence, ABSTAIN is returned.}

\subsection{Asymptotic behavior of the volume bound}
\label{sec:asymptotic}
We briefly compare the volume lower bound \eqref{eqn:projrsvol} of the projected randomized smoothing certified region to that of a standard certified $\ell_2$-ball. The volume of a $d$-dimensional $\ell_2$-ball $B_d(R) \coloneqq \{x\in\Rd : \|x\| \le R\}$ of radius $R\ge 0$ is well-known (e.g., see \citet[Theorem~2.44,~Corollary~2.55]{folland1999real}) to be
\begin{equation} \label{eqn:rsvol}
    \vol_d \left( B_d(R) \right) = \frac{\pi^{d/2}}{\Gamma(\frac{d}{2} + 1)} R^d.
\end{equation}

While the numerator of \eqref{eqn:rsvol} scales exponentially in $d$, the denominator $\Gamma(\frac{d}{2}+1)$ scales factorially, leading to tiny $\ell_2$-ball certified volumes in high-dimensional input spaces. By contrast, the denominator in our bound \eqref{eqn:projrsvol} scales factorially in the \textit{projected dimension} $p$, where $p \ll d$. This suggests dramatic improvements in the volume of our certified regions: \tmlrhl{while the numerator in \eqref{eqn:projrsvol} might be \emph{exponentially} smaller than that of \eqref{eqn:rsvol}, the denominator is smaller by a \emph{factorial} factor. We thus expect the volumes of projected randomized smoothing to dominate at higher dimensions}. We verify our analysis experimentally in Section~\ref{sec:expregion} \hl{and illustrate some simulated certified volume ratios over a range of values for $p$ and $d$ in Appendix~\ref{app:finitesweep}}.

\begin{center}
\begin{minipage}{.7\linewidth}
\input{res/alg}

\end{minipage}
\end{center}

\subsection{Runtime and limitations}
\label{sec:runtime}

Our certification strategy has two additional computational steps outside of the $\textsc{Predict}$ and $\textsc{Certify}$ subroutines from the conventional randomized smoothing method of \citet{cohen2019certified}. The first is a one-time computation of the principal components of the data that occurs at the beginning of training. The second is computing the $\ell_{\infty}$-regression in \eqref{eqn:mindistimpl}, which we solve as a linear program using the standard epigraph formulation. For the CIFAR-10 and SVHN datasets considered in this work, the added runtime is comparable to the certification sampling step from \citet{cohen2019certified}. Namely, we found that the $\ell_{\infty}$-regression averaged around 16 seconds for CIFAR-10 and 19 seconds for SVHN.\footnote{All experiments were run on a Ubuntu 20.04 virtual machine with 6 VCPUs, 56 GiB RAM, and a Tesla K80 GPU. Complete reproduction takes roughly $0.06$ GPU years.}

The number of variables and constraints in the optimization \eqref{eqn:mindistimpl} scales linearly with $d - p$. Since generally $p \ll d$, this makes the volume approximation of the certified region computationally intensive in high-dimensional spaces. We remark that it is still trivial to check whether any particular perturbation lies in the certified region using Proposition~\ref{prop:certset}---it is just that computing a lower bound on the volume of this region for comparison purposes becomes more challenging. For a natural image dataset such as ImageNet, the analysis of Section~\ref{sec:asymptotic} suggests that the certified region volume improvements would in fact be substantially larger than those for CIFAR-10. \tmlrhl{The main challenge to computationally verifying this conjecture lies in holding the optimization problem \eqref{eqn:mindistimpl} in memory, which is infeasible on our hardware for ImageNet-scale inputs. Further research in this vein would likely leverage techniques from the large-scale $\ell_{\infty}$-regression literature, e.g., \citet{shen2014fast}, and is outside the scope of this work}.

\begin{toappendix}
        As an aside, we note that the certified region of our method contains an $\ell_2$-ball of radius usually comparable to that of standard randomized smoothing, although in general the certified region of Propositions~\ref{prop:certset}~and~\ref{prop:certsetgeo} will be much larger as it captures the null space of the projection operator. We nevertheless include the following simple result for completeness.
	\begin{proposition}
		Let $x \in \Rd$ and $R\ge 0$. If $\smoothout$ is certified at $\p(x) = U\T x$ with radius $R$, then $\smoothhard(x + \delta) = \smoothhard(x)$ for all $\delta \in B_d(R) \coloneqq \{x\in\Rd : \|x\| \le R\}$.
	\end{proposition}
	
	\begin{proof}
	Let $\delta\in B_d(R)$. Then, since $U\T$ is a semi-orthogonal matrix, its $\ell_2$-induced operator norm $\|U\T\|$ is less than or equal to $1$. Thus, $\|U\T \delta\| \le \|U\T\| \|\delta\| \le R$. Therefore, $B_d(R) \subseteq \cert$.
	\end{proof}
    \subsection{Finite-dimensional volume analysis} \label{app:finitesweep}
    We complement our asymptotic certified volume comparison in Section~\ref{sec:asymptotic} with a simple finite-dimensional sweep over input dimension $d$ and projected dimension $p$, the results of which are shown in Figure~\ref{fig:finitesweep}.

    Here, we fix $R=0.5$ and $t=0.4$ as typical values for natural image datasets and sweep over a range of choices of $p$ and $d$. We assume that the low-dimensional and high-dimensional certified radii are similar, as indicated by Figure~\ref{fig:expradii}. The plotted values provide the ratio of our certified volume from Theorem~\ref{thm:certified_volume} to the volume of a standard $\ell_2$-ball, as given in Section~\ref{sec:asymptotic}, e.g., a value of $\times 30000$ indicates that the volume of our certified region is $30000$ times greater. This ratio grows rapidly as $d$ increases due to the factorial growth noted in Section~\ref{sec:asymptotic}.

    \begin{figure}[H]
        \centering
        \scalebox{0.8}{\input{figs/finite-sweep.pgf}}
        \caption{Ratio of projected randomized smoothing certified volume versus standard randomized smoothing certified volume for simulated values.}
	\label{fig:finitesweep}
    \end{figure}
\end{toappendix}

%% file: res/alg.tex
% Requires only algorithm and algorithmic
\begin{algorithm}[H]
   \caption{Prediction and certification}
   \label{alg:cert}
\begin{algorithmic}
    \STATE \textbf{def} $\textsc{Predict}$, $\textsc{Certify}$ as in \citet{cohen2019certified}
\vspace*{0.5\baselineskip}
\FUNCTION{$\textsc{ProjectPredict}(\net$, $U$, $\sigma$, $x$, $n$, $\alpha$)}
\STATE \textbf{def} $\p(x) = U\T x$, $\pout(\xtil) = U \xtil$

\STATE \textbf{return} $\textsc{Predict}(\net \circ \pout, \sigma, \p(x), n, \alpha)$
\ENDFUNCTION
\vspace*{0.5\baselineskip}
\FUNCTION{$\textsc{ProjectCertify}(\net$, $U$, $\sigma$, $x$, $n_0$, $n$, $\alpha$)}
\STATE \textbf{def} $\p(x) = U\T x$, $\pout(\xtil) = U \xtil$, $(d,p) \gets \mathrm{shape}(U)$

%\STATE \hspace*{-0.1cm}\begin{multline} \displaystyle
%\textrm{ABSTAIN}, \hat{c}_A, R \gets \\\textsc{Certify}(\net \circ \pout, \sigma, \p(x), n_0, n, \alpha) \nonumber
%\end{multline}
\STATE $\textrm{ABSTAIN}, \hat{c}_A, R \gets \textsc{Certify}(\net \circ \pout, \sigma, \p(x), n_0, n, \alpha)$
\STATE \textbf{if} {ABSTAIN} \textbf{then return} ABSTAIN 
\STATE \textbf{compute} orthonormal basis $v_1, \dots, v_{d-p}$ for $\Null(U\T)$
\STATE \textbf{solve} the optimization
\begin{subequations}
\begin{align} \label{eqn:mindistimpl}
    t \gets \inf_{\alpha \in \R^{d-p}} ~ \left\lVert  x + \sum _{i=1}^{d-p} \alpha_i v_i \right\rVert_{\infty} \tag{Alg1}
\end{align}
\end{subequations}
\STATE \textbf{assign} $R \gets \min \{ R, p (1 - 2t) / (2d) \}$
\STATE \textbf{compute} the certified volume lower bound
\[
    V \gets \frac{\pi^{p/2}}{\Gamma(\frac{p}{2} + 1)} R^p (1 - 2R - 2t)^{d-p}
    \]
\STATE \textbf{return} prediction $\hat{c}_A$ and volume bound $V$
\ENDFUNCTION
\end{algorithmic}
\end{algorithm}

%% file: figs/finite-sweep.pgf
%% Creator: Matplotlib, PGF backend
%%
%% To include the figure in your LaTeX document, write
%%   \input{<filename>.pgf}
%%
%% Make sure the required packages are loaded in your preamble
%%   \usepackage{pgf}
%%
%% Also ensure that all the required font packages are loaded; for instance,
%% the lmodern package is sometimes necessary when using math font.
%%   \usepackage{lmodern}
%%
%% Figures using additional raster images can only be included by \input if
%% they are in the same directory as the main LaTeX file. For loading figures
%% from other directories you can use the `import` package
%%   \usepackage{import}
%%
%% and then include the figures with
%%   \import{<path to file>}{<filename>.pgf}
%%
%% Matplotlib used the following preamble
%%
\begingroup%
\makeatletter%
\begin{pgfpicture}%
\pgfpathrectangle{\pgfpointorigin}{\pgfqpoint{5.401100in}{3.811552in}}%
\pgfusepath{use as bounding box, clip}%
\begin{pgfscope}%
\pgfsetbuttcap%
\pgfsetmiterjoin%
\pgfsetlinewidth{0.000000pt}%
\definecolor{currentstroke}{rgb}{0.000000,0.000000,0.000000}%
\pgfsetstrokecolor{currentstroke}%
\pgfsetstrokeopacity{0.000000}%
\pgfsetdash{}{0pt}%
\pgfpathmoveto{\pgfqpoint{0.000000in}{0.000000in}}%
\pgfpathlineto{\pgfqpoint{5.401100in}{0.000000in}}%
\pgfpathlineto{\pgfqpoint{5.401100in}{3.811552in}}%
\pgfpathlineto{\pgfqpoint{0.000000in}{3.811552in}}%
\pgfpathlineto{\pgfqpoint{0.000000in}{0.000000in}}%
\pgfpathclose%
\pgfusepath{}%
\end{pgfscope}%
\begin{pgfscope}%
\pgfsetbuttcap%
\pgfsetmiterjoin%
\pgfsetlinewidth{0.000000pt}%
\definecolor{currentstroke}{rgb}{0.000000,0.000000,0.000000}%
\pgfsetstrokecolor{currentstroke}%
\pgfsetstrokeopacity{0.000000}%
\pgfsetdash{}{0pt}%
\pgfpathmoveto{\pgfqpoint{0.723458in}{0.499691in}}%
\pgfpathlineto{\pgfqpoint{4.187805in}{0.499691in}}%
\pgfpathlineto{\pgfqpoint{4.187805in}{3.617604in}}%
\pgfpathlineto{\pgfqpoint{0.723458in}{3.617604in}}%
\pgfpathlineto{\pgfqpoint{0.723458in}{0.499691in}}%
\pgfpathclose%
\pgfusepath{}%
\end{pgfscope}%
\begin{pgfscope}%
\pgfpathrectangle{\pgfqpoint{0.723458in}{0.499691in}}{\pgfqpoint{3.464348in}{3.117913in}}%
\pgfusepath{clip}%
\pgfsys@transformshift{0.723458in}{0.499691in}%
\pgftext[left,bottom]{\includegraphics[interpolate=true,width=3.470000in,height=3.120000in]{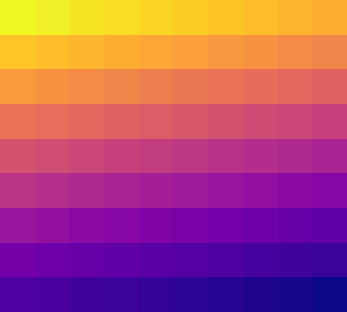}}%
\end{pgfscope}%
\begin{pgfscope}%
\pgfsetbuttcap%
\pgfsetroundjoin%
\definecolor{currentfill}{rgb}{0.000000,0.000000,0.000000}%
\pgfsetfillcolor{currentfill}%
\pgfsetlinewidth{0.803000pt}%
\definecolor{currentstroke}{rgb}{0.000000,0.000000,0.000000}%
\pgfsetstrokecolor{currentstroke}%
\pgfsetdash{}{0pt}%
\pgfsys@defobject{currentmarker}{\pgfqpoint{0.000000in}{-0.048611in}}{\pgfqpoint{0.000000in}{0.000000in}}{%
\pgfpathmoveto{\pgfqpoint{0.000000in}{0.000000in}}%
\pgfpathlineto{\pgfqpoint{0.000000in}{-0.048611in}}%
\pgfusepath{stroke,fill}%
}%
\begin{pgfscope}%
\pgfsys@transformshift{0.896675in}{0.499691in}%
\pgfsys@useobject{currentmarker}{}%
\end{pgfscope}%
\end{pgfscope}%
\begin{pgfscope}%
\definecolor{textcolor}{rgb}{0.000000,0.000000,0.000000}%
\pgfsetstrokecolor{textcolor}%
\pgfsetfillcolor{textcolor}%
\pgftext[x=0.896675in,y=0.402469in,,top]{\color{textcolor}\rmfamily\fontsize{10.000000}{12.000000}\selectfont 100}%
\end{pgfscope}%
\begin{pgfscope}%
\pgfsetbuttcap%
\pgfsetroundjoin%
\definecolor{currentfill}{rgb}{0.000000,0.000000,0.000000}%
\pgfsetfillcolor{currentfill}%
\pgfsetlinewidth{0.803000pt}%
\definecolor{currentstroke}{rgb}{0.000000,0.000000,0.000000}%
\pgfsetstrokecolor{currentstroke}%
\pgfsetdash{}{0pt}%
\pgfsys@defobject{currentmarker}{\pgfqpoint{0.000000in}{-0.048611in}}{\pgfqpoint{0.000000in}{0.000000in}}{%
\pgfpathmoveto{\pgfqpoint{0.000000in}{0.000000in}}%
\pgfpathlineto{\pgfqpoint{0.000000in}{-0.048611in}}%
\pgfusepath{stroke,fill}%
}%
\begin{pgfscope}%
\pgfsys@transformshift{1.243110in}{0.499691in}%
\pgfsys@useobject{currentmarker}{}%
\end{pgfscope}%
\end{pgfscope}%
\begin{pgfscope}%
\definecolor{textcolor}{rgb}{0.000000,0.000000,0.000000}%
\pgfsetstrokecolor{textcolor}%
\pgfsetfillcolor{textcolor}%
\pgftext[x=1.243110in,y=0.402469in,,top]{\color{textcolor}\rmfamily\fontsize{10.000000}{12.000000}\selectfont 200}%
\end{pgfscope}%
\begin{pgfscope}%
\pgfsetbuttcap%
\pgfsetroundjoin%
\definecolor{currentfill}{rgb}{0.000000,0.000000,0.000000}%
\pgfsetfillcolor{currentfill}%
\pgfsetlinewidth{0.803000pt}%
\definecolor{currentstroke}{rgb}{0.000000,0.000000,0.000000}%
\pgfsetstrokecolor{currentstroke}%
\pgfsetdash{}{0pt}%
\pgfsys@defobject{currentmarker}{\pgfqpoint{0.000000in}{-0.048611in}}{\pgfqpoint{0.000000in}{0.000000in}}{%
\pgfpathmoveto{\pgfqpoint{0.000000in}{0.000000in}}%
\pgfpathlineto{\pgfqpoint{0.000000in}{-0.048611in}}%
\pgfusepath{stroke,fill}%
}%
\begin{pgfscope}%
\pgfsys@transformshift{1.589545in}{0.499691in}%
\pgfsys@useobject{currentmarker}{}%
\end{pgfscope}%
\end{pgfscope}%
\begin{pgfscope}%
\definecolor{textcolor}{rgb}{0.000000,0.000000,0.000000}%
\pgfsetstrokecolor{textcolor}%
\pgfsetfillcolor{textcolor}%
\pgftext[x=1.589545in,y=0.402469in,,top]{\color{textcolor}\rmfamily\fontsize{10.000000}{12.000000}\selectfont 300}%
\end{pgfscope}%
\begin{pgfscope}%
\pgfsetbuttcap%
\pgfsetroundjoin%
\definecolor{currentfill}{rgb}{0.000000,0.000000,0.000000}%
\pgfsetfillcolor{currentfill}%
\pgfsetlinewidth{0.803000pt}%
\definecolor{currentstroke}{rgb}{0.000000,0.000000,0.000000}%
\pgfsetstrokecolor{currentstroke}%
\pgfsetdash{}{0pt}%
\pgfsys@defobject{currentmarker}{\pgfqpoint{0.000000in}{-0.048611in}}{\pgfqpoint{0.000000in}{0.000000in}}{%
\pgfpathmoveto{\pgfqpoint{0.000000in}{0.000000in}}%
\pgfpathlineto{\pgfqpoint{0.000000in}{-0.048611in}}%
\pgfusepath{stroke,fill}%
}%
\begin{pgfscope}%
\pgfsys@transformshift{1.935979in}{0.499691in}%
\pgfsys@useobject{currentmarker}{}%
\end{pgfscope}%
\end{pgfscope}%
\begin{pgfscope}%
\definecolor{textcolor}{rgb}{0.000000,0.000000,0.000000}%
\pgfsetstrokecolor{textcolor}%
\pgfsetfillcolor{textcolor}%
\pgftext[x=1.935979in,y=0.402469in,,top]{\color{textcolor}\rmfamily\fontsize{10.000000}{12.000000}\selectfont 400}%
\end{pgfscope}%
\begin{pgfscope}%
\pgfsetbuttcap%
\pgfsetroundjoin%
\definecolor{currentfill}{rgb}{0.000000,0.000000,0.000000}%
\pgfsetfillcolor{currentfill}%
\pgfsetlinewidth{0.803000pt}%
\definecolor{currentstroke}{rgb}{0.000000,0.000000,0.000000}%
\pgfsetstrokecolor{currentstroke}%
\pgfsetdash{}{0pt}%
\pgfsys@defobject{currentmarker}{\pgfqpoint{0.000000in}{-0.048611in}}{\pgfqpoint{0.000000in}{0.000000in}}{%
\pgfpathmoveto{\pgfqpoint{0.000000in}{0.000000in}}%
\pgfpathlineto{\pgfqpoint{0.000000in}{-0.048611in}}%
\pgfusepath{stroke,fill}%
}%
\begin{pgfscope}%
\pgfsys@transformshift{2.282414in}{0.499691in}%
\pgfsys@useobject{currentmarker}{}%
\end{pgfscope}%
\end{pgfscope}%
\begin{pgfscope}%
\definecolor{textcolor}{rgb}{0.000000,0.000000,0.000000}%
\pgfsetstrokecolor{textcolor}%
\pgfsetfillcolor{textcolor}%
\pgftext[x=2.282414in,y=0.402469in,,top]{\color{textcolor}\rmfamily\fontsize{10.000000}{12.000000}\selectfont 500}%
\end{pgfscope}%
\begin{pgfscope}%
\pgfsetbuttcap%
\pgfsetroundjoin%
\definecolor{currentfill}{rgb}{0.000000,0.000000,0.000000}%
\pgfsetfillcolor{currentfill}%
\pgfsetlinewidth{0.803000pt}%
\definecolor{currentstroke}{rgb}{0.000000,0.000000,0.000000}%
\pgfsetstrokecolor{currentstroke}%
\pgfsetdash{}{0pt}%
\pgfsys@defobject{currentmarker}{\pgfqpoint{0.000000in}{-0.048611in}}{\pgfqpoint{0.000000in}{0.000000in}}{%
\pgfpathmoveto{\pgfqpoint{0.000000in}{0.000000in}}%
\pgfpathlineto{\pgfqpoint{0.000000in}{-0.048611in}}%
\pgfusepath{stroke,fill}%
}%
\begin{pgfscope}%
\pgfsys@transformshift{2.628849in}{0.499691in}%
\pgfsys@useobject{currentmarker}{}%
\end{pgfscope}%
\end{pgfscope}%
\begin{pgfscope}%
\definecolor{textcolor}{rgb}{0.000000,0.000000,0.000000}%
\pgfsetstrokecolor{textcolor}%
\pgfsetfillcolor{textcolor}%
\pgftext[x=2.628849in,y=0.402469in,,top]{\color{textcolor}\rmfamily\fontsize{10.000000}{12.000000}\selectfont 600}%
\end{pgfscope}%
\begin{pgfscope}%
\pgfsetbuttcap%
\pgfsetroundjoin%
\definecolor{currentfill}{rgb}{0.000000,0.000000,0.000000}%
\pgfsetfillcolor{currentfill}%
\pgfsetlinewidth{0.803000pt}%
\definecolor{currentstroke}{rgb}{0.000000,0.000000,0.000000}%
\pgfsetstrokecolor{currentstroke}%
\pgfsetdash{}{0pt}%
\pgfsys@defobject{currentmarker}{\pgfqpoint{0.000000in}{-0.048611in}}{\pgfqpoint{0.000000in}{0.000000in}}{%
\pgfpathmoveto{\pgfqpoint{0.000000in}{0.000000in}}%
\pgfpathlineto{\pgfqpoint{0.000000in}{-0.048611in}}%
\pgfusepath{stroke,fill}%
}%
\begin{pgfscope}%
\pgfsys@transformshift{2.975284in}{0.499691in}%
\pgfsys@useobject{currentmarker}{}%
\end{pgfscope}%
\end{pgfscope}%
\begin{pgfscope}%
\definecolor{textcolor}{rgb}{0.000000,0.000000,0.000000}%
\pgfsetstrokecolor{textcolor}%
\pgfsetfillcolor{textcolor}%
\pgftext[x=2.975284in,y=0.402469in,,top]{\color{textcolor}\rmfamily\fontsize{10.000000}{12.000000}\selectfont 700}%
\end{pgfscope}%
\begin{pgfscope}%
\pgfsetbuttcap%
\pgfsetroundjoin%
\definecolor{currentfill}{rgb}{0.000000,0.000000,0.000000}%
\pgfsetfillcolor{currentfill}%
\pgfsetlinewidth{0.803000pt}%
\definecolor{currentstroke}{rgb}{0.000000,0.000000,0.000000}%
\pgfsetstrokecolor{currentstroke}%
\pgfsetdash{}{0pt}%
\pgfsys@defobject{currentmarker}{\pgfqpoint{0.000000in}{-0.048611in}}{\pgfqpoint{0.000000in}{0.000000in}}{%
\pgfpathmoveto{\pgfqpoint{0.000000in}{0.000000in}}%
\pgfpathlineto{\pgfqpoint{0.000000in}{-0.048611in}}%
\pgfusepath{stroke,fill}%
}%
\begin{pgfscope}%
\pgfsys@transformshift{3.321718in}{0.499691in}%
\pgfsys@useobject{currentmarker}{}%
\end{pgfscope}%
\end{pgfscope}%
\begin{pgfscope}%
\definecolor{textcolor}{rgb}{0.000000,0.000000,0.000000}%
\pgfsetstrokecolor{textcolor}%
\pgfsetfillcolor{textcolor}%
\pgftext[x=3.321718in,y=0.402469in,,top]{\color{textcolor}\rmfamily\fontsize{10.000000}{12.000000}\selectfont 800}%
\end{pgfscope}%
\begin{pgfscope}%
\pgfsetbuttcap%
\pgfsetroundjoin%
\definecolor{currentfill}{rgb}{0.000000,0.000000,0.000000}%
\pgfsetfillcolor{currentfill}%
\pgfsetlinewidth{0.803000pt}%
\definecolor{currentstroke}{rgb}{0.000000,0.000000,0.000000}%
\pgfsetstrokecolor{currentstroke}%
\pgfsetdash{}{0pt}%
\pgfsys@defobject{currentmarker}{\pgfqpoint{0.000000in}{-0.048611in}}{\pgfqpoint{0.000000in}{0.000000in}}{%
\pgfpathmoveto{\pgfqpoint{0.000000in}{0.000000in}}%
\pgfpathlineto{\pgfqpoint{0.000000in}{-0.048611in}}%
\pgfusepath{stroke,fill}%
}%
\begin{pgfscope}%
\pgfsys@transformshift{3.668153in}{0.499691in}%
\pgfsys@useobject{currentmarker}{}%
\end{pgfscope}%
\end{pgfscope}%
\begin{pgfscope}%
\definecolor{textcolor}{rgb}{0.000000,0.000000,0.000000}%
\pgfsetstrokecolor{textcolor}%
\pgfsetfillcolor{textcolor}%
\pgftext[x=3.668153in,y=0.402469in,,top]{\color{textcolor}\rmfamily\fontsize{10.000000}{12.000000}\selectfont 900}%
\end{pgfscope}%
\begin{pgfscope}%
\pgfsetbuttcap%
\pgfsetroundjoin%
\definecolor{currentfill}{rgb}{0.000000,0.000000,0.000000}%
\pgfsetfillcolor{currentfill}%
\pgfsetlinewidth{0.803000pt}%
\definecolor{currentstroke}{rgb}{0.000000,0.000000,0.000000}%
\pgfsetstrokecolor{currentstroke}%
\pgfsetdash{}{0pt}%
\pgfsys@defobject{currentmarker}{\pgfqpoint{0.000000in}{-0.048611in}}{\pgfqpoint{0.000000in}{0.000000in}}{%
\pgfpathmoveto{\pgfqpoint{0.000000in}{0.000000in}}%
\pgfpathlineto{\pgfqpoint{0.000000in}{-0.048611in}}%
\pgfusepath{stroke,fill}%
}%
\begin{pgfscope}%
\pgfsys@transformshift{4.014588in}{0.499691in}%
\pgfsys@useobject{currentmarker}{}%
\end{pgfscope}%
\end{pgfscope}%
\begin{pgfscope}%
\definecolor{textcolor}{rgb}{0.000000,0.000000,0.000000}%
\pgfsetstrokecolor{textcolor}%
\pgfsetfillcolor{textcolor}%
\pgftext[x=4.014588in,y=0.402469in,,top]{\color{textcolor}\rmfamily\fontsize{10.000000}{12.000000}\selectfont 1000}%
\end{pgfscope}%
\begin{pgfscope}%
\definecolor{textcolor}{rgb}{0.000000,0.000000,0.000000}%
\pgfsetstrokecolor{textcolor}%
\pgfsetfillcolor{textcolor}%
\pgftext[x=2.455632in,y=0.223457in,,top]{\color{textcolor}\rmfamily\fontsize{10.000000}{12.000000}\selectfont Projected dimension \(\displaystyle p\)}%
\end{pgfscope}%
\begin{pgfscope}%
\pgfsetbuttcap%
\pgfsetroundjoin%
\definecolor{currentfill}{rgb}{0.000000,0.000000,0.000000}%
\pgfsetfillcolor{currentfill}%
\pgfsetlinewidth{0.803000pt}%
\definecolor{currentstroke}{rgb}{0.000000,0.000000,0.000000}%
\pgfsetstrokecolor{currentstroke}%
\pgfsetdash{}{0pt}%
\pgfsys@defobject{currentmarker}{\pgfqpoint{-0.048611in}{0.000000in}}{\pgfqpoint{-0.000000in}{0.000000in}}{%
\pgfpathmoveto{\pgfqpoint{-0.000000in}{0.000000in}}%
\pgfpathlineto{\pgfqpoint{-0.048611in}{0.000000in}}%
\pgfusepath{stroke,fill}%
}%
\begin{pgfscope}%
\pgfsys@transformshift{0.723458in}{3.444387in}%
\pgfsys@useobject{currentmarker}{}%
\end{pgfscope}%
\end{pgfscope}%
\begin{pgfscope}%
\definecolor{textcolor}{rgb}{0.000000,0.000000,0.000000}%
\pgfsetstrokecolor{textcolor}%
\pgfsetfillcolor{textcolor}%
\pgftext[x=0.279012in, y=3.396161in, left, base]{\color{textcolor}\rmfamily\fontsize{10.000000}{12.000000}\selectfont 10000}%
\end{pgfscope}%
\begin{pgfscope}%
\pgfsetbuttcap%
\pgfsetroundjoin%
\definecolor{currentfill}{rgb}{0.000000,0.000000,0.000000}%
\pgfsetfillcolor{currentfill}%
\pgfsetlinewidth{0.803000pt}%
\definecolor{currentstroke}{rgb}{0.000000,0.000000,0.000000}%
\pgfsetstrokecolor{currentstroke}%
\pgfsetdash{}{0pt}%
\pgfsys@defobject{currentmarker}{\pgfqpoint{-0.048611in}{0.000000in}}{\pgfqpoint{-0.000000in}{0.000000in}}{%
\pgfpathmoveto{\pgfqpoint{-0.000000in}{0.000000in}}%
\pgfpathlineto{\pgfqpoint{-0.048611in}{0.000000in}}%
\pgfusepath{stroke,fill}%
}%
\begin{pgfscope}%
\pgfsys@transformshift{0.723458in}{3.097952in}%
\pgfsys@useobject{currentmarker}{}%
\end{pgfscope}%
\end{pgfscope}%
\begin{pgfscope}%
\definecolor{textcolor}{rgb}{0.000000,0.000000,0.000000}%
\pgfsetstrokecolor{textcolor}%
\pgfsetfillcolor{textcolor}%
\pgftext[x=0.348457in, y=3.049726in, left, base]{\color{textcolor}\rmfamily\fontsize{10.000000}{12.000000}\selectfont 9000}%
\end{pgfscope}%
\begin{pgfscope}%
\pgfsetbuttcap%
\pgfsetroundjoin%
\definecolor{currentfill}{rgb}{0.000000,0.000000,0.000000}%
\pgfsetfillcolor{currentfill}%
\pgfsetlinewidth{0.803000pt}%
\definecolor{currentstroke}{rgb}{0.000000,0.000000,0.000000}%
\pgfsetstrokecolor{currentstroke}%
\pgfsetdash{}{0pt}%
\pgfsys@defobject{currentmarker}{\pgfqpoint{-0.048611in}{0.000000in}}{\pgfqpoint{-0.000000in}{0.000000in}}{%
\pgfpathmoveto{\pgfqpoint{-0.000000in}{0.000000in}}%
\pgfpathlineto{\pgfqpoint{-0.048611in}{0.000000in}}%
\pgfusepath{stroke,fill}%
}%
\begin{pgfscope}%
\pgfsys@transformshift{0.723458in}{2.751517in}%
\pgfsys@useobject{currentmarker}{}%
\end{pgfscope}%
\end{pgfscope}%
\begin{pgfscope}%
\definecolor{textcolor}{rgb}{0.000000,0.000000,0.000000}%
\pgfsetstrokecolor{textcolor}%
\pgfsetfillcolor{textcolor}%
\pgftext[x=0.348457in, y=2.703292in, left, base]{\color{textcolor}\rmfamily\fontsize{10.000000}{12.000000}\selectfont 8000}%
\end{pgfscope}%
\begin{pgfscope}%
\pgfsetbuttcap%
\pgfsetroundjoin%
\definecolor{currentfill}{rgb}{0.000000,0.000000,0.000000}%
\pgfsetfillcolor{currentfill}%
\pgfsetlinewidth{0.803000pt}%
\definecolor{currentstroke}{rgb}{0.000000,0.000000,0.000000}%
\pgfsetstrokecolor{currentstroke}%
\pgfsetdash{}{0pt}%
\pgfsys@defobject{currentmarker}{\pgfqpoint{-0.048611in}{0.000000in}}{\pgfqpoint{-0.000000in}{0.000000in}}{%
\pgfpathmoveto{\pgfqpoint{-0.000000in}{0.000000in}}%
\pgfpathlineto{\pgfqpoint{-0.048611in}{0.000000in}}%
\pgfusepath{stroke,fill}%
}%
\begin{pgfscope}%
\pgfsys@transformshift{0.723458in}{2.405082in}%
\pgfsys@useobject{currentmarker}{}%
\end{pgfscope}%
\end{pgfscope}%
\begin{pgfscope}%
\definecolor{textcolor}{rgb}{0.000000,0.000000,0.000000}%
\pgfsetstrokecolor{textcolor}%
\pgfsetfillcolor{textcolor}%
\pgftext[x=0.348457in, y=2.356857in, left, base]{\color{textcolor}\rmfamily\fontsize{10.000000}{12.000000}\selectfont 7000}%
\end{pgfscope}%
\begin{pgfscope}%
\pgfsetbuttcap%
\pgfsetroundjoin%
\definecolor{currentfill}{rgb}{0.000000,0.000000,0.000000}%
\pgfsetfillcolor{currentfill}%
\pgfsetlinewidth{0.803000pt}%
\definecolor{currentstroke}{rgb}{0.000000,0.000000,0.000000}%
\pgfsetstrokecolor{currentstroke}%
\pgfsetdash{}{0pt}%
\pgfsys@defobject{currentmarker}{\pgfqpoint{-0.048611in}{0.000000in}}{\pgfqpoint{-0.000000in}{0.000000in}}{%
\pgfpathmoveto{\pgfqpoint{-0.000000in}{0.000000in}}%
\pgfpathlineto{\pgfqpoint{-0.048611in}{0.000000in}}%
\pgfusepath{stroke,fill}%
}%
\begin{pgfscope}%
\pgfsys@transformshift{0.723458in}{2.058648in}%
\pgfsys@useobject{currentmarker}{}%
\end{pgfscope}%
\end{pgfscope}%
\begin{pgfscope}%
\definecolor{textcolor}{rgb}{0.000000,0.000000,0.000000}%
\pgfsetstrokecolor{textcolor}%
\pgfsetfillcolor{textcolor}%
\pgftext[x=0.348457in, y=2.010422in, left, base]{\color{textcolor}\rmfamily\fontsize{10.000000}{12.000000}\selectfont 6000}%
\end{pgfscope}%
\begin{pgfscope}%
\pgfsetbuttcap%
\pgfsetroundjoin%
\definecolor{currentfill}{rgb}{0.000000,0.000000,0.000000}%
\pgfsetfillcolor{currentfill}%
\pgfsetlinewidth{0.803000pt}%
\definecolor{currentstroke}{rgb}{0.000000,0.000000,0.000000}%
\pgfsetstrokecolor{currentstroke}%
\pgfsetdash{}{0pt}%
\pgfsys@defobject{currentmarker}{\pgfqpoint{-0.048611in}{0.000000in}}{\pgfqpoint{-0.000000in}{0.000000in}}{%
\pgfpathmoveto{\pgfqpoint{-0.000000in}{0.000000in}}%
\pgfpathlineto{\pgfqpoint{-0.048611in}{0.000000in}}%
\pgfusepath{stroke,fill}%
}%
\begin{pgfscope}%
\pgfsys@transformshift{0.723458in}{1.712213in}%
\pgfsys@useobject{currentmarker}{}%
\end{pgfscope}%
\end{pgfscope}%
\begin{pgfscope}%
\definecolor{textcolor}{rgb}{0.000000,0.000000,0.000000}%
\pgfsetstrokecolor{textcolor}%
\pgfsetfillcolor{textcolor}%
\pgftext[x=0.348457in, y=1.663987in, left, base]{\color{textcolor}\rmfamily\fontsize{10.000000}{12.000000}\selectfont 5000}%
\end{pgfscope}%
\begin{pgfscope}%
\pgfsetbuttcap%
\pgfsetroundjoin%
\definecolor{currentfill}{rgb}{0.000000,0.000000,0.000000}%
\pgfsetfillcolor{currentfill}%
\pgfsetlinewidth{0.803000pt}%
\definecolor{currentstroke}{rgb}{0.000000,0.000000,0.000000}%
\pgfsetstrokecolor{currentstroke}%
\pgfsetdash{}{0pt}%
\pgfsys@defobject{currentmarker}{\pgfqpoint{-0.048611in}{0.000000in}}{\pgfqpoint{-0.000000in}{0.000000in}}{%
\pgfpathmoveto{\pgfqpoint{-0.000000in}{0.000000in}}%
\pgfpathlineto{\pgfqpoint{-0.048611in}{0.000000in}}%
\pgfusepath{stroke,fill}%
}%
\begin{pgfscope}%
\pgfsys@transformshift{0.723458in}{1.365778in}%
\pgfsys@useobject{currentmarker}{}%
\end{pgfscope}%
\end{pgfscope}%
\begin{pgfscope}%
\definecolor{textcolor}{rgb}{0.000000,0.000000,0.000000}%
\pgfsetstrokecolor{textcolor}%
\pgfsetfillcolor{textcolor}%
\pgftext[x=0.348457in, y=1.317553in, left, base]{\color{textcolor}\rmfamily\fontsize{10.000000}{12.000000}\selectfont 4000}%
\end{pgfscope}%
\begin{pgfscope}%
\pgfsetbuttcap%
\pgfsetroundjoin%
\definecolor{currentfill}{rgb}{0.000000,0.000000,0.000000}%
\pgfsetfillcolor{currentfill}%
\pgfsetlinewidth{0.803000pt}%
\definecolor{currentstroke}{rgb}{0.000000,0.000000,0.000000}%
\pgfsetstrokecolor{currentstroke}%
\pgfsetdash{}{0pt}%
\pgfsys@defobject{currentmarker}{\pgfqpoint{-0.048611in}{0.000000in}}{\pgfqpoint{-0.000000in}{0.000000in}}{%
\pgfpathmoveto{\pgfqpoint{-0.000000in}{0.000000in}}%
\pgfpathlineto{\pgfqpoint{-0.048611in}{0.000000in}}%
\pgfusepath{stroke,fill}%
}%
\begin{pgfscope}%
\pgfsys@transformshift{0.723458in}{1.019343in}%
\pgfsys@useobject{currentmarker}{}%
\end{pgfscope}%
\end{pgfscope}%
\begin{pgfscope}%
\definecolor{textcolor}{rgb}{0.000000,0.000000,0.000000}%
\pgfsetstrokecolor{textcolor}%
\pgfsetfillcolor{textcolor}%
\pgftext[x=0.348457in, y=0.971118in, left, base]{\color{textcolor}\rmfamily\fontsize{10.000000}{12.000000}\selectfont 3000}%
\end{pgfscope}%
\begin{pgfscope}%
\pgfsetbuttcap%
\pgfsetroundjoin%
\definecolor{currentfill}{rgb}{0.000000,0.000000,0.000000}%
\pgfsetfillcolor{currentfill}%
\pgfsetlinewidth{0.803000pt}%
\definecolor{currentstroke}{rgb}{0.000000,0.000000,0.000000}%
\pgfsetstrokecolor{currentstroke}%
\pgfsetdash{}{0pt}%
\pgfsys@defobject{currentmarker}{\pgfqpoint{-0.048611in}{0.000000in}}{\pgfqpoint{-0.000000in}{0.000000in}}{%
\pgfpathmoveto{\pgfqpoint{-0.000000in}{0.000000in}}%
\pgfpathlineto{\pgfqpoint{-0.048611in}{0.000000in}}%
\pgfusepath{stroke,fill}%
}%
\begin{pgfscope}%
\pgfsys@transformshift{0.723458in}{0.672908in}%
\pgfsys@useobject{currentmarker}{}%
\end{pgfscope}%
\end{pgfscope}%
\begin{pgfscope}%
\definecolor{textcolor}{rgb}{0.000000,0.000000,0.000000}%
\pgfsetstrokecolor{textcolor}%
\pgfsetfillcolor{textcolor}%
\pgftext[x=0.348457in, y=0.624683in, left, base]{\color{textcolor}\rmfamily\fontsize{10.000000}{12.000000}\selectfont 2000}%
\end{pgfscope}%
\begin{pgfscope}%
\definecolor{textcolor}{rgb}{0.000000,0.000000,0.000000}%
\pgfsetstrokecolor{textcolor}%
\pgfsetfillcolor{textcolor}%
\pgftext[x=0.223457in,y=2.058648in,,bottom,rotate=90.000000]{\color{textcolor}\rmfamily\fontsize{10.000000}{12.000000}\selectfont Input dimension \(\displaystyle d\)}%
\end{pgfscope}%
\begin{pgfscope}%
\pgfsetrectcap%
\pgfsetmiterjoin%
\pgfsetlinewidth{0.803000pt}%
\definecolor{currentstroke}{rgb}{0.000000,0.000000,0.000000}%
\pgfsetstrokecolor{currentstroke}%
\pgfsetdash{}{0pt}%
\pgfpathmoveto{\pgfqpoint{0.723458in}{0.499691in}}%
\pgfpathlineto{\pgfqpoint{0.723458in}{3.617604in}}%
\pgfusepath{stroke}%
\end{pgfscope}%
\begin{pgfscope}%
\pgfsetrectcap%
\pgfsetmiterjoin%
\pgfsetlinewidth{0.803000pt}%
\definecolor{currentstroke}{rgb}{0.000000,0.000000,0.000000}%
\pgfsetstrokecolor{currentstroke}%
\pgfsetdash{}{0pt}%
\pgfpathmoveto{\pgfqpoint{4.187805in}{0.499691in}}%
\pgfpathlineto{\pgfqpoint{4.187805in}{3.617604in}}%
\pgfusepath{stroke}%
\end{pgfscope}%
\begin{pgfscope}%
\pgfsetrectcap%
\pgfsetmiterjoin%
\pgfsetlinewidth{0.803000pt}%
\definecolor{currentstroke}{rgb}{0.000000,0.000000,0.000000}%
\pgfsetstrokecolor{currentstroke}%
\pgfsetdash{}{0pt}%
\pgfpathmoveto{\pgfqpoint{0.723458in}{0.499691in}}%
\pgfpathlineto{\pgfqpoint{4.187805in}{0.499691in}}%
\pgfusepath{stroke}%
\end{pgfscope}%
\begin{pgfscope}%
\pgfsetrectcap%
\pgfsetmiterjoin%
\pgfsetlinewidth{0.803000pt}%
\definecolor{currentstroke}{rgb}{0.000000,0.000000,0.000000}%
\pgfsetstrokecolor{currentstroke}%
\pgfsetdash{}{0pt}%
\pgfpathmoveto{\pgfqpoint{0.723458in}{3.617604in}}%
\pgfpathlineto{\pgfqpoint{4.187805in}{3.617604in}}%
\pgfusepath{stroke}%
\end{pgfscope}%
\begin{pgfscope}%
\pgfsetbuttcap%
\pgfsetmiterjoin%
\pgfsetlinewidth{0.000000pt}%
\definecolor{currentstroke}{rgb}{0.000000,0.000000,0.000000}%
\pgfsetstrokecolor{currentstroke}%
\pgfsetstrokeopacity{0.000000}%
\pgfsetdash{}{0pt}%
\pgfpathmoveto{\pgfqpoint{4.404327in}{0.405743in}}%
\pgfpathlineto{\pgfqpoint{4.569617in}{0.405743in}}%
\pgfpathlineto{\pgfqpoint{4.569617in}{3.711552in}}%
\pgfpathlineto{\pgfqpoint{4.404327in}{3.711552in}}%
\pgfpathlineto{\pgfqpoint{4.404327in}{0.405743in}}%
\pgfpathclose%
\pgfusepath{}%
\end{pgfscope}%
\begin{pgfscope}%
\pgfpathrectangle{\pgfqpoint{4.404327in}{0.405743in}}{\pgfqpoint{0.165290in}{3.305809in}}%
\pgfusepath{clip}%
\pgfsetbuttcap%
\pgfsetmiterjoin%
\definecolor{currentfill}{rgb}{1.000000,1.000000,1.000000}%
\pgfsetfillcolor{currentfill}%
\pgfsetlinewidth{0.010037pt}%
\definecolor{currentstroke}{rgb}{1.000000,1.000000,1.000000}%
\pgfsetstrokecolor{currentstroke}%
\pgfsetdash{}{0pt}%
\pgfusepath{stroke,fill}%
\end{pgfscope}%
\begin{pgfscope}%
\pgfsys@transformshift{4.400000in}{0.411552in}%
\pgftext[left,bottom]{\includegraphics[interpolate=true,width=0.170000in,height=3.300000in]{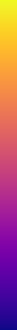}}%
\end{pgfscope}%
\begin{pgfscope}%
\pgfsetbuttcap%
\pgfsetroundjoin%
\definecolor{currentfill}{rgb}{0.000000,0.000000,0.000000}%
\pgfsetfillcolor{currentfill}%
\pgfsetlinewidth{0.803000pt}%
\definecolor{currentstroke}{rgb}{0.000000,0.000000,0.000000}%
\pgfsetstrokecolor{currentstroke}%
\pgfsetdash{}{0pt}%
\pgfsys@defobject{currentmarker}{\pgfqpoint{0.000000in}{0.000000in}}{\pgfqpoint{0.048611in}{0.000000in}}{%
\pgfpathmoveto{\pgfqpoint{0.000000in}{0.000000in}}%
\pgfpathlineto{\pgfqpoint{0.048611in}{0.000000in}}%
\pgfusepath{stroke,fill}%
}%
\begin{pgfscope}%
\pgfsys@transformshift{4.569617in}{0.405743in}%
\pgfsys@useobject{currentmarker}{}%
\end{pgfscope}%
\end{pgfscope}%
\begin{pgfscope}%
\definecolor{textcolor}{rgb}{0.000000,0.000000,0.000000}%
\pgfsetstrokecolor{textcolor}%
\pgfsetfillcolor{textcolor}%
\pgftext[x=4.666840in, y=0.357518in, left, base]{\color{textcolor}\rmfamily\fontsize{10.000000}{12.000000}\selectfont \(\displaystyle \times 0\)}%
\end{pgfscope}%
\begin{pgfscope}%
\pgfsetbuttcap%
\pgfsetroundjoin%
\definecolor{currentfill}{rgb}{0.000000,0.000000,0.000000}%
\pgfsetfillcolor{currentfill}%
\pgfsetlinewidth{0.803000pt}%
\definecolor{currentstroke}{rgb}{0.000000,0.000000,0.000000}%
\pgfsetstrokecolor{currentstroke}%
\pgfsetdash{}{0pt}%
\pgfsys@defobject{currentmarker}{\pgfqpoint{0.000000in}{0.000000in}}{\pgfqpoint{0.048611in}{0.000000in}}{%
\pgfpathmoveto{\pgfqpoint{0.000000in}{0.000000in}}%
\pgfpathlineto{\pgfqpoint{0.048611in}{0.000000in}}%
\pgfusepath{stroke,fill}%
}%
\begin{pgfscope}%
\pgfsys@transformshift{4.569617in}{0.938012in}%
\pgfsys@useobject{currentmarker}{}%
\end{pgfscope}%
\end{pgfscope}%
\begin{pgfscope}%
\definecolor{textcolor}{rgb}{0.000000,0.000000,0.000000}%
\pgfsetstrokecolor{textcolor}%
\pgfsetfillcolor{textcolor}%
\pgftext[x=4.666840in, y=0.889787in, left, base]{\color{textcolor}\rmfamily\fontsize{10.000000}{12.000000}\selectfont \(\displaystyle \times 5000\)}%
\end{pgfscope}%
\begin{pgfscope}%
\pgfsetbuttcap%
\pgfsetroundjoin%
\definecolor{currentfill}{rgb}{0.000000,0.000000,0.000000}%
\pgfsetfillcolor{currentfill}%
\pgfsetlinewidth{0.803000pt}%
\definecolor{currentstroke}{rgb}{0.000000,0.000000,0.000000}%
\pgfsetstrokecolor{currentstroke}%
\pgfsetdash{}{0pt}%
\pgfsys@defobject{currentmarker}{\pgfqpoint{0.000000in}{0.000000in}}{\pgfqpoint{0.048611in}{0.000000in}}{%
\pgfpathmoveto{\pgfqpoint{0.000000in}{0.000000in}}%
\pgfpathlineto{\pgfqpoint{0.048611in}{0.000000in}}%
\pgfusepath{stroke,fill}%
}%
\begin{pgfscope}%
\pgfsys@transformshift{4.569617in}{1.470281in}%
\pgfsys@useobject{currentmarker}{}%
\end{pgfscope}%
\end{pgfscope}%
\begin{pgfscope}%
\definecolor{textcolor}{rgb}{0.000000,0.000000,0.000000}%
\pgfsetstrokecolor{textcolor}%
\pgfsetfillcolor{textcolor}%
\pgftext[x=4.666840in, y=1.422056in, left, base]{\color{textcolor}\rmfamily\fontsize{10.000000}{12.000000}\selectfont \(\displaystyle \times 10000\)}%
\end{pgfscope}%
\begin{pgfscope}%
\pgfsetbuttcap%
\pgfsetroundjoin%
\definecolor{currentfill}{rgb}{0.000000,0.000000,0.000000}%
\pgfsetfillcolor{currentfill}%
\pgfsetlinewidth{0.803000pt}%
\definecolor{currentstroke}{rgb}{0.000000,0.000000,0.000000}%
\pgfsetstrokecolor{currentstroke}%
\pgfsetdash{}{0pt}%
\pgfsys@defobject{currentmarker}{\pgfqpoint{0.000000in}{0.000000in}}{\pgfqpoint{0.048611in}{0.000000in}}{%
\pgfpathmoveto{\pgfqpoint{0.000000in}{0.000000in}}%
\pgfpathlineto{\pgfqpoint{0.048611in}{0.000000in}}%
\pgfusepath{stroke,fill}%
}%
\begin{pgfscope}%
\pgfsys@transformshift{4.569617in}{2.002550in}%
\pgfsys@useobject{currentmarker}{}%
\end{pgfscope}%
\end{pgfscope}%
\begin{pgfscope}%
\definecolor{textcolor}{rgb}{0.000000,0.000000,0.000000}%
\pgfsetstrokecolor{textcolor}%
\pgfsetfillcolor{textcolor}%
\pgftext[x=4.666840in, y=1.954325in, left, base]{\color{textcolor}\rmfamily\fontsize{10.000000}{12.000000}\selectfont \(\displaystyle \times 15000\)}%
\end{pgfscope}%
\begin{pgfscope}%
\pgfsetbuttcap%
\pgfsetroundjoin%
\definecolor{currentfill}{rgb}{0.000000,0.000000,0.000000}%
\pgfsetfillcolor{currentfill}%
\pgfsetlinewidth{0.803000pt}%
\definecolor{currentstroke}{rgb}{0.000000,0.000000,0.000000}%
\pgfsetstrokecolor{currentstroke}%
\pgfsetdash{}{0pt}%
\pgfsys@defobject{currentmarker}{\pgfqpoint{0.000000in}{0.000000in}}{\pgfqpoint{0.048611in}{0.000000in}}{%
\pgfpathmoveto{\pgfqpoint{0.000000in}{0.000000in}}%
\pgfpathlineto{\pgfqpoint{0.048611in}{0.000000in}}%
\pgfusepath{stroke,fill}%
}%
\begin{pgfscope}%
\pgfsys@transformshift{4.569617in}{2.534819in}%
\pgfsys@useobject{currentmarker}{}%
\end{pgfscope}%
\end{pgfscope}%
\begin{pgfscope}%
\definecolor{textcolor}{rgb}{0.000000,0.000000,0.000000}%
\pgfsetstrokecolor{textcolor}%
\pgfsetfillcolor{textcolor}%
\pgftext[x=4.666840in, y=2.486594in, left, base]{\color{textcolor}\rmfamily\fontsize{10.000000}{12.000000}\selectfont \(\displaystyle \times 20000\)}%
\end{pgfscope}%
\begin{pgfscope}%
\pgfsetbuttcap%
\pgfsetroundjoin%
\definecolor{currentfill}{rgb}{0.000000,0.000000,0.000000}%
\pgfsetfillcolor{currentfill}%
\pgfsetlinewidth{0.803000pt}%
\definecolor{currentstroke}{rgb}{0.000000,0.000000,0.000000}%
\pgfsetstrokecolor{currentstroke}%
\pgfsetdash{}{0pt}%
\pgfsys@defobject{currentmarker}{\pgfqpoint{0.000000in}{0.000000in}}{\pgfqpoint{0.048611in}{0.000000in}}{%
\pgfpathmoveto{\pgfqpoint{0.000000in}{0.000000in}}%
\pgfpathlineto{\pgfqpoint{0.048611in}{0.000000in}}%
\pgfusepath{stroke,fill}%
}%
\begin{pgfscope}%
\pgfsys@transformshift{4.569617in}{3.067088in}%
\pgfsys@useobject{currentmarker}{}%
\end{pgfscope}%
\end{pgfscope}%
\begin{pgfscope}%
\definecolor{textcolor}{rgb}{0.000000,0.000000,0.000000}%
\pgfsetstrokecolor{textcolor}%
\pgfsetfillcolor{textcolor}%
\pgftext[x=4.666840in, y=3.018862in, left, base]{\color{textcolor}\rmfamily\fontsize{10.000000}{12.000000}\selectfont \(\displaystyle \times 25000\)}%
\end{pgfscope}%
\begin{pgfscope}%
\pgfsetbuttcap%
\pgfsetroundjoin%
\definecolor{currentfill}{rgb}{0.000000,0.000000,0.000000}%
\pgfsetfillcolor{currentfill}%
\pgfsetlinewidth{0.803000pt}%
\definecolor{currentstroke}{rgb}{0.000000,0.000000,0.000000}%
\pgfsetstrokecolor{currentstroke}%
\pgfsetdash{}{0pt}%
\pgfsys@defobject{currentmarker}{\pgfqpoint{0.000000in}{0.000000in}}{\pgfqpoint{0.048611in}{0.000000in}}{%
\pgfpathmoveto{\pgfqpoint{0.000000in}{0.000000in}}%
\pgfpathlineto{\pgfqpoint{0.048611in}{0.000000in}}%
\pgfusepath{stroke,fill}%
}%
\begin{pgfscope}%
\pgfsys@transformshift{4.569617in}{3.599357in}%
\pgfsys@useobject{currentmarker}{}%
\end{pgfscope}%
\end{pgfscope}%
\begin{pgfscope}%
\definecolor{textcolor}{rgb}{0.000000,0.000000,0.000000}%
\pgfsetstrokecolor{textcolor}%
\pgfsetfillcolor{textcolor}%
\pgftext[x=4.666840in, y=3.551131in, left, base]{\color{textcolor}\rmfamily\fontsize{10.000000}{12.000000}\selectfont \(\displaystyle \times 30000\)}%
\end{pgfscope}%
\begin{pgfscope}%
\definecolor{textcolor}{rgb}{0.000000,0.000000,0.000000}%
\pgfsetstrokecolor{textcolor}%
\pgfsetfillcolor{textcolor}%
\pgftext[x=5.177644in,y=2.058648in,,top,rotate=90.000000]{\color{textcolor}\rmfamily\fontsize{10.000000}{12.000000}\selectfont Certified volume ratio}%
\end{pgfscope}%
\begin{pgfscope}%
\pgfsetrectcap%
\pgfsetmiterjoin%
\pgfsetlinewidth{0.803000pt}%
\definecolor{currentstroke}{rgb}{0.000000,0.000000,0.000000}%
\pgfsetstrokecolor{currentstroke}%
\pgfsetdash{}{0pt}%
\pgfpathmoveto{\pgfqpoint{4.404327in}{0.405743in}}%
\pgfpathlineto{\pgfqpoint{4.486972in}{0.405743in}}%
\pgfpathlineto{\pgfqpoint{4.569617in}{0.405743in}}%
\pgfpathlineto{\pgfqpoint{4.569617in}{3.711552in}}%
\pgfpathlineto{\pgfqpoint{4.486972in}{3.711552in}}%
\pgfpathlineto{\pgfqpoint{4.404327in}{3.711552in}}%
\pgfpathlineto{\pgfqpoint{4.404327in}{0.405743in}}%
\pgfpathclose%
\pgfusepath{stroke}%
\end{pgfscope}%
\end{pgfpicture}%
\makeatother%
\endgroup%

%% file: experiment.tex
\section{Experiments}
\label{sec:experiments}

This section reports our experiments on the CIFAR-10 and SVHN datasets. We first demonstrate in Section~\ref{sec:expattack} that networks are vulnerable to $\ell_{\infty}$-bounded attacks in the subspace of low-variance principal components, to which our architecture is provably robust. Section~\ref{sec:expregion} then presents results comparing the volume of the projected randomized smoothing certified regions to a variety of baseline certified classifiers.

\subsection{Vulnerability to low-variance PCA attacks} \label{sec:expattack}
Consider perturbations $\delta \in \Null(U\T)$ contained in the span of a dataset's low-variance principal components, where we take $U$ to contain sufficient components to account for $99\%$ of the dataset variance for CIFAR-10 and $95\%$ for SVHN, which is more robust to low-variance subspace attacks due to its increased compressibility. Such a perturbation is known to be essentially orthogonal to the true data manifold, and therefore it is reasonable to expect a truly robust classifier to be invariant to small perturbations in $\Null(U\T)$. Our method is directly robust to such perturbations under the simple condition that we use fewer components in our initial projection step, as demonstrated in Proposition~\ref{prop:certsetgeo}.

We now investigate whether this theoretical guarantee adds a degree of robustness over a typical neural network classifier. The answer is affirmative. Namely, we show that our subspace attack can attain a comparable attack success rate to a standard $\ell_{\infty}$-bounded projected gradient descent ($\pgd$) attack, with roughly a four-fold increase in the size of the admissible $\ell_{\infty}$-ball.

Formally, consider a particular hard classifier $g$, to which we assume that our adversaries have white-box access, and take a specific input $x$ that $g$ classifies correctly. We first consider the standard projected gradient descent attack strategy $\pgd(x,\epsilon)$ which seeks to construct a perturbation $\|\delta\|_{\infty} \leq \epsilon$ such that $x + \delta \in \Cd$ and $g(x+\delta) \neq g(x)$. As $x + \delta \in \Cd$ if and only if $\| x + \delta \|_{\infty} \leq 1 / 2$, satisfying both $\ell_\infty$-norm constraints on $\delta$ is easily accomplished using clipping. Our routine $\spgd(x,\epsilon)$ adds the additional constraint $\delta \in \Null(U\T)$. Note that finding a perturbation that satisfies $\delta \in \Null(U\T)$, $\|\delta\|_{\infty} \leq \epsilon$, and $x + \delta \in \Cd$ is nontrivial, as projection onto one set generally removes an input from the other set. The precise details of our attack strategy are detailed in Appendix~\ref{app:subspaceattack}.

For reference, we also consider $\textsc{RandMax}$ and $\textsc{RandUniform}$, which generate perturbations randomly on the boundary of and uniformly in the attack $\ell_{\infty}$-ball, respectively. We instantiate $g$ as the Wide ResNet considered in \citet{yang2020randomized} with the default hyperparameters and $\sigma=0.15$ Gaussian noise augmentation during training. See Appendix~\ref{app:attackhyperparams} for the attack hyperparameters.

Figure~\ref{fig:subspaceattack} demonstrates that unprotected classifiers are indeed vulnerable to adversarial perturbations in the subspace of low-variance principal components. Enlargements of the attack radius do not invalidate that these are true adversarial attacks, as the perturbed images in the third row of Figure~\ref{fig:attackexamples} are still easily classified by a human.
Furthermore, $\spgd$ adversarial examples are substantially less perceptible than $\pgd$ attacks of the same magnitude, which tend to produce stronger visual distortions of the image, \hl{paralleling results from \citet{shamir2021dimpled}};
take as a representative example the area around the frog's head in the second row of the third column in Figure~\ref{fig:attackexamples}, compared with the same image perturbed by $\spgd$ in the third row. The results for the SVHN dataset in Figure~\ref{fig:svhnattackexamples} are even more striking.
%\hl{Similar examples for the SVHN dataset are contained in Appendix~\ref{app:svhnresults}}.
This is likely because $\pgd$ attacks have access to high-variance principal components which convey the dataset information content. Despite visually appearing random, we establish in Figures~\ref{fig:attacksuccess} and \ref{fig:svhnattacksuccess} that the $\spgd$ attack is significantly more successful than random-noise attacks of the same magnitude. These results suggest that undefended classifiers can be attacked in the subspace of low-variance principal components, to which projected randomized smoothing is provably robust by Proposition~\ref{prop:certsetgeo}.

\begin{figure*}[!ht] 
    \centering
    \hfill
    \begin{subfigure}[b]{0.58\linewidth}
        \adjustbox{valign=t}{
            \resizebox{0.9\linewidth}{!}{\input{figs/cifar10_attack_sweep.pgf}}
        }
	    \caption{~}
        \label{fig:attacksuccess}
    \end{subfigure}
    \hfill
    \begin{subfigure}[b]{0.4\linewidth}
        \adjustbox{valign=t}{
            \input{res/cifar10_examples.tex}
        }
	\vspace*{\baselineskip}
	    \caption{~}
        \label{fig:attackexamples}
    \end{subfigure}
    \hfill
    \\
	\hfill
        \begin{subfigure}[b]{0.58\linewidth}
            \adjustbox{valign=t}{
                \resizebox{0.9\linewidth}{!}{\input{figs/svhn_attack_sweep.pgf}}
            }
            \caption{~}
            \label{fig:svhnattacksuccess}
        \end{subfigure}
	\hfill
        \begin{subfigure}[b]{0.4\linewidth}
            \adjustbox{valign=t}{
                \input{res/svhn_examples.tex}
            }
	    \vspace*{\baselineskip}
            \caption{~}
            \label{fig:svhnattackexamples}
        \end{subfigure}
	\hfill

    \caption{
        (\subref{fig:attacksuccess}) CIFAR-10 adversarial attack success rates for the $\pgd$, $\spgd$, and random attack strategies.
        (\subref{fig:attackexamples}) Perturbation examples for CIFAR-10 with an attack radius of $\epsilon=32/255$. The top row represents the original image.
        (\subref{fig:svhnattacksuccess}) SVHN aversarial attack success rates for the $\pgd$, $\spgd$, and random attack strategies.
        (\subref{fig:svhnattackexamples}) Perturbation examples for SVHN with an attack radius of $\epsilon=32/255$.
    }
    \vspace*{-0.3cm}
    \label{fig:subspaceattack}
\end{figure*}

\subsection{Certified region comparison}
\label{sec:expregion}

Having established that the certified region of projected randomized smoothing provides a meaningful robustness improvement against low-variance principal component attacks, we now compare the volume of our certified region with several baselines. Namely, we evaluate the $\ell_2$-balls of \citet{cohen2019certified} (denoted \textsc{RS}), the $\ell_1$- and $\ell_{\infty}$-balls of \citet{yang2020randomized} (denoted $\textsc{RS4A}-\ell_1$ and $\textsc{RS4A}-\ell_\infty$, respectively), and the anisotropic ellipsoids of \citet{eiras2021ancer} (denoted \textsc{ANCER}), without use of the associated memory module.

Some additional remarks on the inclusion of \citet{eiras2021ancer} are warranted. As noted in \citet{sukenik2021intriguing}, without the inclusion of the memory module, the local certificate optimization technique in \citet{eiras2021ancer} yields overly optimistic and mathematically incorrect certificates as the smoothing distribution varies between inputs. The work \citet{eiras2021ancer} corrects this with the use of a memory module that records previous inputs to ensure compatibility of the smoothing certificates. However, this results in a classifier that is dependent on the input order and adds ambiguity about what classifier is actually being certified, as the smoothed classifier is modified at test time after each input. We therefore discard the memory module and report the certified volume at each point as if the locally optimized smoothing distribution were being used globally. This yields an upper bound on the certified volume of any data-dependent anisotropic ellipsoidal smoothing method and is thus a very strong baseline to compare against.

Our results are summarized in Figure~\ref{fig:expvol} and Table~\ref{tbl:data}. We achieve state-of-the-art median certified volumes, easily outperforming standard randomized smoothing and even the optimistic \textsc{ANCER} baseline by $706$ and $2453$ \textit{orders of magnitude} on CIFAR-10 and SVHN, respectively. The larger improvement on SVHN is attributable to the higher compressibility of the dataset. Figure~\ref{fig:expradii} in Appendix~\ref{app:cifarresults} suggests that our performance derives from the added robustness of our method against low-variance features, as the radii of the projected-space certified balls are similar to those of standard randomized smoothing in the input space. This further validates the asymptotic dimension analysis in Section~\ref{sec:asymptotic}. Note that although the \textsc{ANCER} baseline achieves higher accuracy at smaller volumes, its certificates are mathematically invalid \citep{sukenik2021intriguing}, and our method significantly outperforms \textsc{ANCER} at larger volumes.

Figure~\ref{fig:expprojectsweep} examines the CIFAR-10 certified accuracy curves over a range of choices for the dimensionality $p$ of the compressed space. For large $p$, image reconstruction is near-perfect as $p=620$ covers $99\%$ of variance in the CIFAR-10 dataset. Thus, methods with $p \geq 300$ have comparable accuracy at small regions, with the certified volumes increasing as the dimensionality of the projected space decreases, corroborating the discussion in Section~\ref{sec:asymptotic}. We are therefore able to increase the robustness of our classifier to disturbances that are normal to the manifold with only a $2\%$ drop in accuracy (Table~\ref{tbl:cifardata}). Figure~\ref{fig:expprojectsweepsvhn} presents similar results for the SVHN dataset. Note that the due to the compressibility of the dataset, fewer principal components are required to achieve high accuracy.

The hyperparameter $p$ introduces a mild tradeoff between clean accuracy and certified volume; if $p$ is chosen to be very small, the projected images may be too corrupted to classify, while if $p$ is chosen to be very large, certified volume may suffer. However, as Figures~\ref{fig:expprojectsweep} and \ref{fig:expprojectsweepsvhn} suggest, our method's certified volumes comfortably outperform those of standard randomized smoothing for a large range of $p$, indicating that this choice is not particularly sensitive. A practical heuristic for choosing $p$ involves making $p$ just large enough to reconstruct images with high fidelity---roughly corresponding to PCA components that explain $95\%$ to $99\%$ of the dataset variance. If desired, a small, localized sweep of $p$ around this initial choice can be used to further optimize the hyperparameter depending on the experimentalist's target metrics (e.g., clean accuracy, median certified volume, other metrics, or some combination). In any case, we emphasize that the parameter choice is quite robust and any additional tuning is likely to result in minimal gains as compared to the practical heuristic. We select $p=450$ for the CIFAR-10 experiment in Figure~\ref{fig:expcifar} and $p=150$ for SVHN.

\begin{figure*}[ht!] 
    \centering
    \hfil
    \begin{subfigure}[b]{0.48\linewidth}
        \resizebox{0.9\linewidth}{!}{\input{figs/cifar10_main.pgf}}
        \caption{~}
        \label{fig:expcifar}
    \end{subfigure}
    \hfil
    \begin{subfigure}[b]{0.48\linewidth}
        \resizebox{0.9\linewidth}{!}{\input{figs/cifar10_projectsweep.pgf}}
        \caption{~}
        \label{fig:expprojectsweep}
    \end{subfigure}
    \hfil
    \\
    \hfil
    \begin{subfigure}[b]{0.48\linewidth}
        \input{figs/svhn_main.pgf}
        \caption{~}
        \label{fig:expsvhn}
    \end{subfigure}
    \hfil
    \begin{subfigure}[b]{0.48\linewidth}
        \input{figs/svhn_projectsweep.pgf}
        \caption{~}
        \label{fig:expprojectsweepsvhn}
    \end{subfigure}
    \hfil

    \caption{
        (\subref{fig:expcifar}) Certified region volumes for CIFAR-10, with our method highlighted by an asterisk. Here $\alpha \approx 3465$ is a scaling constant corresponding to the $d$-dimensional unit ball volume; i.e. $\vol_d ( B_d(1) ) = 10^{-\alpha}$.
        (\subref{fig:expprojectsweep}) CIFAR-10 certified region volumes while varying the projected space dimension $p$ for our method.
        (\subref{fig:expsvhn}) Certified region volumes for SVHN. 
        (\subref{fig:expprojectsweepsvhn}) SVHN certified region volumes while varying $p$.
    }
    \label{fig:expvol}
\end{figure*}

    \begin{table}[ht]
        \vspace*{0.2cm}

        \begin{center}
            \caption{
                Quantitative representation of the data in Figure~\ref{fig:expvol}. The first column reports the smoothed classifier clean accuracy for each method and the second column reports the median certified volume for correctly classified samples. We use the median instead of the mean due to the $\log$-scaled nature of our data.
                \label{tbl:data}
            }
            \subfloat[CIFAR certification performance.]{
                \begin{tabular}{ c | c | c }
                  & Accuracy & \thead{\normalsize Median cert.\\ \normalsize vol. ($\log_{10}$)} \\ 
                 \hline
                 $\textsc{ProjectedRS}$ & $85.8\%$ & $\mathbf{-3175}$ \\  
                 RS & $\mathbf{87.8\%}$ & $-4377$ \\
                 ANCER & $87.4\%$ & $-3881$ \\
                 $\textsc{RS4A}-\ell_1$ & $83.8\%$ & $-9573$ \\
                 $\textsc{RS4A}-\ell_{\infty}$ & $85.4\%$ & $-6102$
                \end{tabular}
                \label{tbl:cifardata}
            }
            \hspace*{0.5cm}
            \subfloat[SVHN certification performance.]{
            \begin{tabular}{ c | c | c }
                  & Accuracy & \thead{\normalsize Median cert.\\ \normalsize vol. ($\log_{10}$)} \\ 
             \hline
             $\textsc{ProjectedRS}$ & $91.4\%$ & $\mathbf{-1578}$ \\  
             RS & $92.6\%$ & $-4280$ \\
             ANCER & $91.2\%$ & $-4031$ \\
             $\textsc{RS4A}-\ell_1$ & $\mathbf{93.0\%}$ & $-9573$ \\
             $\textsc{RS4A}-\ell_{\infty}$ & $92.6\%$ & $-6171$
            \end{tabular}
	        \label{tbl:svhndata}
            }
        \end{center}
    \end{table}

\begin{toappendix} 
    \tmlrhl{
        \subsection{Random projections ablation} \label{app: ablation}
    We compare the PCA projections used in our experiments with projection onto a random subspace for CIFAR-10. As reconstruction fidelity is much poorer for random subspaces, we expect to need many more components to achieve a comparable clean accuracy to PCA projections. We correspondingly adjust the number of components to $p=1500$ in Figure~\ref{fig:ablation} for random projections, and retain $p=450$ PCA components. As with the hyperparameter sweeps, we use $n=10^4$ smoothing samples. Figure~\ref{fig:ablation} shows that projecting onto the PCA basis generally provides superior certificates. We also plot standard randomized smoothing for reference.

    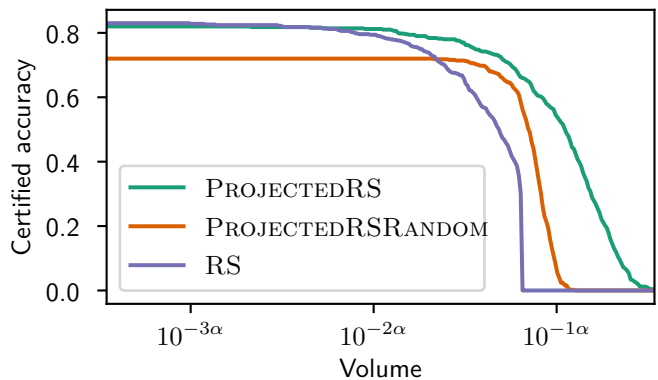
\begin{figure}[H]
        \centering
        \input{figs/cifar10_ablation.pgf}
        \caption{\tmlrhl{Ablation test comparing PCA basis projections with projection onto a random orthonormal set of vectors.}}
	    \label{fig:ablation}
    \end{figure}
    }

    \subsection{Subspace attack procedure}
    \label{app:subspaceattack}
    A typical $\pgd$ attack constructs adversarial examples by iteratively perturbing the image along the gradient of the loss and projecting onto the unit cube of feasible inputs:
    \[
        \xip = \p_{\Cd} \left( x + \p_{\Ce} \left(\alpha \sign\left( \nabla_{\delta} \mathcal{L}(\xi + \delta, y) - \xi \right)\right) \right),
    \]
    where $\xi$ is the $i$th iterate of the PGD attack, $\mathcal{L}$ is the loss function, $\sign(\cdot)$ is the element-wise sign operator, $\alpha$ is the step size hyperparameter, $\p_{\Cd}$ projects a point in $\Rd$ onto $\Cd$ by simple clipping, and $\p_{\Ce}$ is defined similarly for the zero-centered cube of sidelength $2 \epsilon$. We initialize $\xo = x$, where $(x,y)$ are the original input and label from the dataset.

    We desire a final perturbation $\delta$ such that $x + \delta \in \Cd$, $\delta \in \Ce$, and $\delta \in \Null(U\T)$. We first parameterize our perturbation in terms of the vectors $v_1, \dots, v_{d-p}$ spanning $\Null(U\T)$. Stacking these vectors columnwise to yield $V \in \R^{d \times d-p}$, we can express our pertubation as $\delta = V \delta_V$ with $\delta_V \in \R^{d-p}$. We then iterate over our parameterized perturbations $\di$, first solving for our ``target'' perturbation
    \[
        \left(\dip\right)^* = \di + \alpha \sign \left(\nabla _{\delta'} \mathcal{L} ( x + V (\di + \delta'), y) \right).
    \] 
    We then project the perturbation to satisfy the $\ell_\infty$-constraints, which takes the form of a quadratic program:
    \begin{equation*}
	    \begin{array}{ll}
	\underset{\delta_V \in \R^{d-p}}{\text{minimize}} & \Big|\Big| V \left(\dip\right)^* + V \delta_V \Big|\Big|_2^2 \\
	\text{subject to} & \| V \delta_V \|_{\infty} \leq \epsilon, \\
	& \| x + V \delta_V \|_{\infty} \leq 1/2.
	\end{array}
    \end{equation*}
    This program is always feasible with $\delta_V = 0$, and its solution satisfies our requirements for each iteration of the attack procedure.

    \subsection{CIFAR-10 additional results} \label{app:cifarresults}

    %We provide a quantitative interpretation of the data in Figure~\ref{fig:expcifar} in Table~\ref{tbl:cifardata}. The first column reports the smoothed classifier accuracy for each method, disregarding certified volume, while the second column reports the median certified volume for correctly classified samples. We use the median instead of the mean due to the $\log$-scaled nature of our data.

    %\begin{table}[ht]
        %\begin{center}
        %\caption{CIFAR certification performance.}
            %\begin{tabular}{ c | c | c }
              %& Accuracy & Median cert. vol. ($\log_{10}$) \\ 
             %\hline
             %$\textsc{ProjectedRS}$ & $85.8\%$ & $\mathbf{-3175}$ \\  
             %RS & $\mathbf{87.8\%}$ & $-4377$ \\
             %ANCER & $87.4\%$ & $-3881$ \\
             %$\textsc{RS4A}-\ell_1$ & $83.8\%$ & $-9573$ \\
             %$\textsc{RS4A}-\ell_{\infty}$ & $85.4\%$ & $-6102$
            %\end{tabular}
		%\label{tbl:cifardata}
        %\end{center}
    %\end{table}

    Here we present an additional plot comparing the radii of the \textit{low-dimensional} projected randomized smoothing balls to the radii of standard randomized smoothing balls in the high-dimensional space. These are very similar, suggesting that the increase in the volume of the certified region comes from the ``extrustion'' of the ball, which amounts to added robustness against unnecessary features that are removed in the initial projection step. For the anisotropic ANCER method, we report the geometric mean of the radii along each coordinate axis, which \citet{eiras2021ancer} defines to be the ``proxy radius.'' We only compare methods with $\ell_2$-based certified regions in this plot.
    \begin{figure}[H]
        \centering
        \input{figs/cifar10_main_radii.pgf}
        \caption{Certified radii on the CIFAR-10 dataset.}
	    \label{fig:expradii}
    \end{figure}
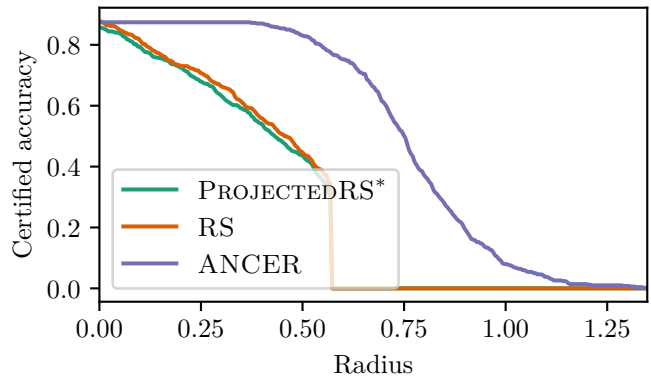

    \subsection{Empirical manifold-perpendicular robustness}
    While the focus of our work is certified robustness, for expository purposes we briefly compare the empirical performance of our model against that of standard randomized smoothing on an approximately manifold-perpendicular threat model. Specifically, following a projected gradient descent iteration, we clamp the high-variance components of the perturbation to have norm at most $R$, where $R$ is a parameter that we sweep over in the table below. An $R$ of zero corresponds to a manifold-perpendicular attack, and as $R$ increases perturbations are allowed a larger on-manifold component. We also range over the perturbation magnitude $\epsilon$. The reported numbers are CIFAR-10 empirical accuracies in percentages, with our method highlighted in bold.

\begin{table}[ht]
\begin{center}
\caption{CIFAR-10 empirical robust accuracy percentages subject to an approximately manifold-perpendicular threat model, with randomized smoothing in non-bold font and our method in bold.}
\begin{tabular}{l|l|l|l|l}
\backslashbox{$\epsilon$}{$R$} & $0.0$ & $0.25$ & $0.5$ & $1.0$ \\ \hline
$8/255$ & $76-\mathbf{79}$ & $65-\mathbf{69}$ & $54-\mathbf{58}$ & $34-\mathbf{35}$ \\ \hline
$16/255$ & $71-\mathbf{79}$ & $60-\mathbf{69}$ & $50-\mathbf{58}$ & $28-\mathbf{35}$ \\ \hline
$32/255$ & $59-\mathbf{79}$ & $50-\mathbf{69}$ & $39-\mathbf{58}$ & $20-\mathbf{35}$
\end{tabular}
\label{tbl:empirical}
\end{center}
\end{table}
We evaluate over $500$ test images and execute smoothing for $100$ samples at each input; while this is fewer than what is typically used for certification, it is a standard number of samples for the prediction problem \citep{cohen2019certified}.

As seen in Table~\ref{tbl:empirical}, our method consistently attains higher empirical robust accuracy than conventional randomized smoothing across all $\epsilon$-$R$ pairs tested. The advantages of our projection-based method are most apparent for small $R$, where the underlying geometry of the data distribution---which conventional randomized smoothing is naive to---is most influential, as well as for large $\epsilon$, where an attacker is able to create large off-manifold perturbations of the input data. As our theory suggests, we find that the empirical robust accuracy of our model remains constant over increasing $\epsilon$ for a fixed value of $R$, indicating that our model is not sensitive to increases in the off-manifold components of an attack. On the contrary, we see that, for fixed $R$, as $\epsilon$ increases, the empirical robust accuracy of conventional randomized smoothing decreases, meaning that conventional randomized smoothing is empirically sensitive to moving further away from the natural data manifold, an attack strategy that we have now shown our method to be both theoretically and empirically robust to.

    \subsection{Hyperparameter selection}
    \label{app:hyperparams}
    To maintain consistency, all networks we consider are Wide ResNets pretrained with various noise distributions using the code provided by \citet{yang2020randomized}. For networks composed with an initial projection, we finetune the network with a learning rate or $0.001$, momentum of $0.9$, and weight decay of $0.0005$ for $20$ epochs, decaying the learning rate by a multiplicative factor of $0.95$ per epoch.

    \subsubsection{Attack hyperparameters} \label{app:attackhyperparams}
    We kept the attack hyperparameters fixed across both CIFAR-10 and SVHN. For the $\pgd$ attack, we use the torchattacks library with $40$ steps and step size $\alpha = 2 / 255$ \citet{kim2020torchattacks}. We lowered this to $5$ steps with $\alpha = \e / 4$ for $\spgd$ due to the solve time of the projection step.

    \subsubsection{CIFAR-10 certification hyperparameters}
    We include the results of our hyperparameter sweeps for CIFAR-10 in Figure~\ref{fig:cifar10_sweeps}. For the $\textsc{RS4A}-\ell_1$ method, we used uniform noise and stability training to reproduce the state-of-the-art result from \citet{yang2020randomized}. The $\textsc{RS4A}-\ell_{\infty}$ sweep used Guassian noise, which we found to perform better in practice. Our sweep over the \textsc{ANCER} learning rate held the number of steps and regularization weight fixed at their defaults of $900$ and $2$, respectively. All sweeps were performed over $500$ random test samples besides ANCER which was run over $100$ samples due to the method's high computational burden.

    The results from these sweeps informed the choice of hyperparameters in Figure~\ref{fig:expcifar}. Namely, we choose $\sigma=0.25$ for our $\textsc{RS4A}-\ell_1$ baseline and $\sigma=0.15$ for our $\textsc{RS4A}-\ell_{\infty}$ baseline, as the clean accuracy drops substantially for higher variances without approaching comparable certified volume to the other methods considered. We choose a learning rate of $0.01$ for \textsc{ANCER} and $p=450$ components for projected randomized smoothing. All experiments in the hyperparameter sweeps were performed with the smoothing hyperparameters of $n_0=100$ samples to guess the smoothed class, $n=10^4$ samples to lower-bound the smoothed class probability, and a confidence of $\alpha=0.001$. For reproducing the final results in Figure~\ref{fig:expcifar} we increased $n$ to $10^6$ as is standard \citep{cohen2019certified} and used $500$ test samples to generate the plots. The attack experiment illustrated in Figure~\ref{fig:attacksuccess} was conducted over $100$ test samples.
    \begin{figure}[ht]
        \centering
        \begin{subfigure}[b]{0.475\textwidth}
            \centering
            \input{figs/cifar10_rs4al1sweep.pgf}
            \phantomcaption{}
            \label{fig:cifar10_rs4al1sweep}
        \end{subfigure}
        \hfill
        \begin{subfigure}[b]{0.475\textwidth}  
            \centering 
            \input{figs/cifar10_rs4alinfsweep.pgf}
            \phantomcaption{}
            \label{fig:cifar10_rs4alinfsweep}
        \end{subfigure}
        \vskip\baselineskip
        \begin{subfigure}[b]{0.475\textwidth}   
            \centering 
            \input{figs/cifar10_ancersweep.pgf}
            \phantomcaption{}
            \label{fig:cifar10_ancersweep}
        \end{subfigure}
        \hfill
        \begin{subfigure}[b]{0.475\textwidth}   
            \centering 
            \input{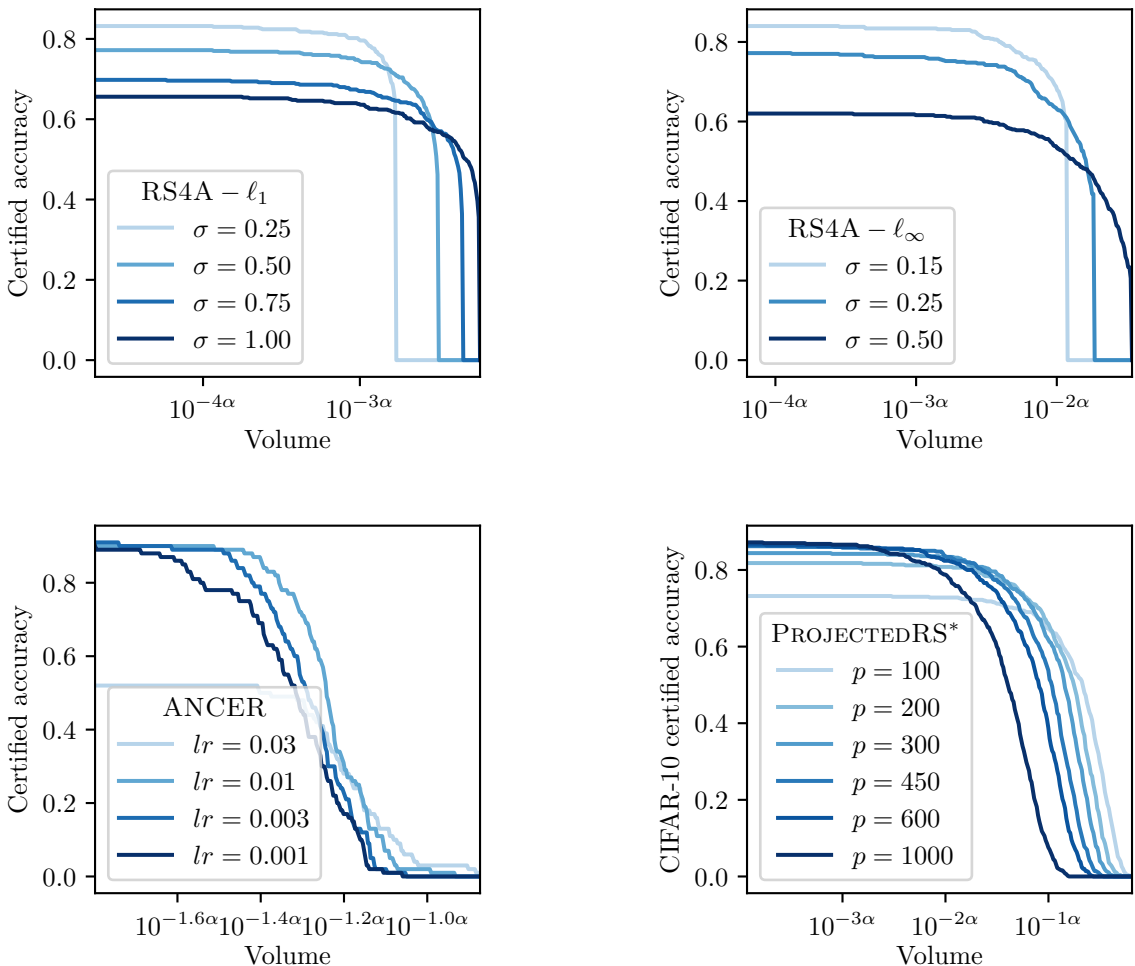}
            \phantomcaption{}
            \label{fig:cifar10_projectsweep}
        \end{subfigure}
        \caption{
            Hyperparameter sweeps for the CIFAR-10 dataset. Here $\alpha \approx 3465$ is a scaling constant corresponding to the $d$-dimensional unit ball volume; i.e. $\vol_d ( B_d(1) ) = 10^{-\alpha}$. 
        } 
        \label{fig:cifar10_sweeps}
    \end{figure} 

    \subsubsection{SVHN certification hyperparameters}
    The SVHN hyperparameter sweep, shown in Figure \ref{fig:svhn_sweeps}, is similar to that for CIFAR-10, besides the use of fewer principal components in the projected randomized smoothing sweeps due to the higher compressibility of the data. Our final plots in Figure~\ref{fig:expsvhn} use $\sigma=0.25$ for the $\ell_{1}$-baseline, $\sigma=0.15$ for the $\ell_{\infty}$-baseline, an ANCER learning rate of $0.01$, and $p=150$ for projected randomized smoothing.

    \begin{figure}[ht]
        \centering
        \begin{subfigure}[b]{0.475\textwidth}
            \centering
            \input{figs/svhn_rs4al1sweep.pgf}
            \phantomcaption{}
            \label{fig:svhn_rs4al1sweep}
        \end{subfigure}
        \hfill
        \begin{subfigure}[b]{0.475\textwidth}  
            \centering 
            \input{figs/svhn_rs4alinfsweep.pgf}
            \phantomcaption{}
            \label{fig:svhn_rs4alinfsweep}
        \end{subfigure}
        \vskip\baselineskip
        \begin{subfigure}[b]{0.475\textwidth}   
            \centering 
            \input{figs/svhn_ancersweep.pgf}
            \phantomcaption{}
            \label{fig:svhn_ancersweep}
        \end{subfigure}
        \hfill
        \begin{subfigure}[b]{0.475\textwidth}   
            \centering 
            \input{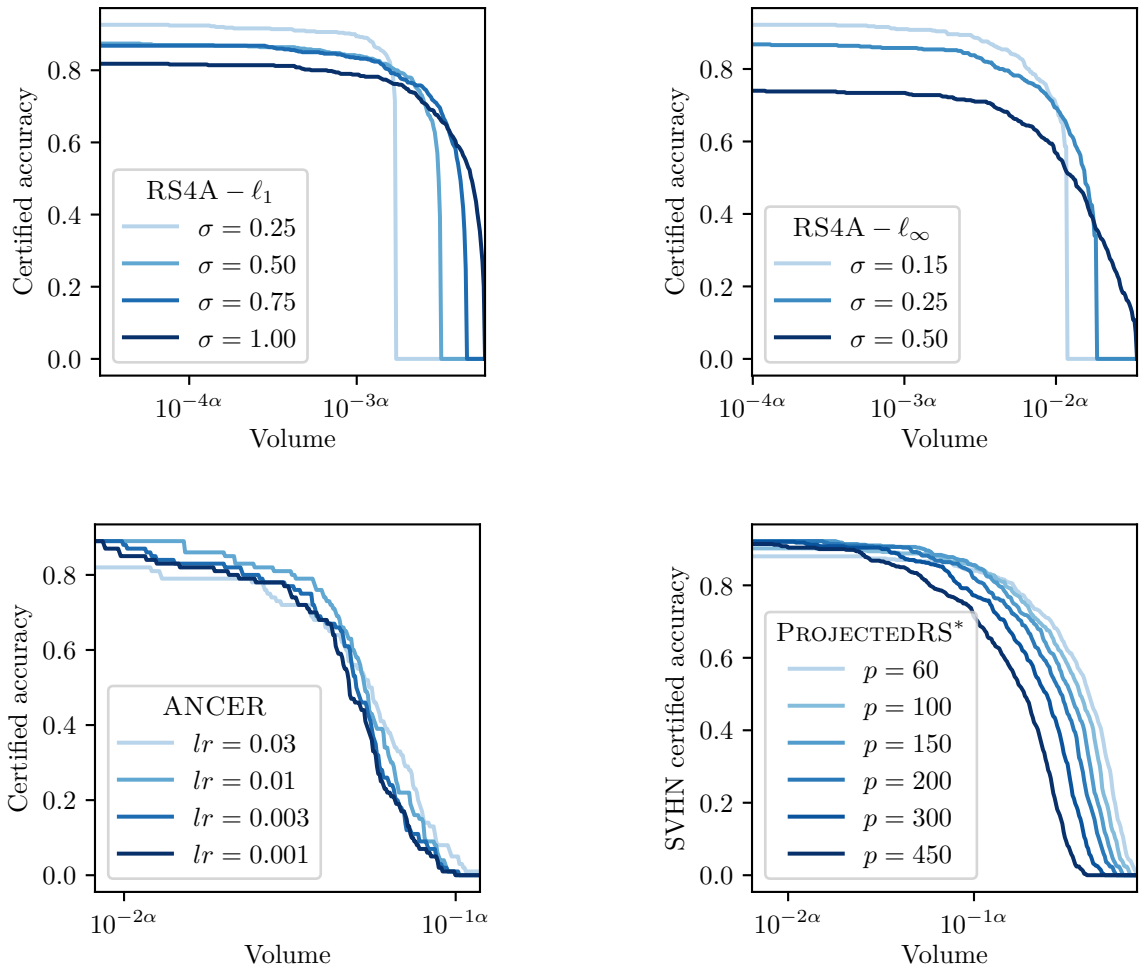}
            \phantomcaption{}
            \label{fig:svhn_projectsweep}
        \end{subfigure}
        \caption{
            Hyperparameter sweeps for the SVHN dataset. Here $\alpha \approx 3465$ is a scaling constant corresponding to the $d$-dimensional unit ball volume; i.e. $\vol_d ( B_d(1) ) = 10^{-\alpha}$. 
        } 
        \label{fig:svhn_sweeps}
    \end{figure} 

    \subsection{Licenses} \label{app:licenses}
    The CIFAR-10 dataset is covered by the MIT license, and the SVHN dataset is covered by the GPL 3 license.
\end{toappendix}

%% file: figs/cifar10_attack_sweep.pgf
%% Creator: Matplotlib, PGF backend
%%
%% To include the figure in your LaTeX document, write
%%   \input{<filename>.pgf}
%%
%% Make sure the required packages are loaded in your preamble
%%   \usepackage{pgf}
%%
%% Also ensure that all the required font packages are loaded; for instance,
%% the lmodern package is sometimes necessary when using math font.
%%   \usepackage{lmodern}
%%
%% Figures using additional raster images can only be included by \input if
%% they are in the same directory as the main LaTeX file. For loading figures
%% from other directories you can use the `import` package
%%   \usepackage{import}
%%
%% and then include the figures with
%%   \import{<path to file>}{<filename>.pgf}
%%
%% Matplotlib used the following preamble
%%
\begingroup%
\makeatletter%
\begin{pgfpicture}%
\pgfpathrectangle{\pgfpointorigin}{\pgfqpoint{3.076824in}{2.232123in}}%
\pgfusepath{use as bounding box, clip}%
\begin{pgfscope}%
\pgfsetbuttcap%
\pgfsetmiterjoin%
\pgfsetlinewidth{0.000000pt}%
\definecolor{currentstroke}{rgb}{0.000000,0.000000,0.000000}%
\pgfsetstrokecolor{currentstroke}%
\pgfsetstrokeopacity{0.000000}%
\pgfsetdash{}{0pt}%
\pgfpathmoveto{\pgfqpoint{0.000000in}{0.000000in}}%
\pgfpathlineto{\pgfqpoint{3.076824in}{0.000000in}}%
\pgfpathlineto{\pgfqpoint{3.076824in}{2.232123in}}%
\pgfpathlineto{\pgfqpoint{0.000000in}{2.232123in}}%
\pgfpathlineto{\pgfqpoint{0.000000in}{0.000000in}}%
\pgfpathclose%
\pgfusepath{}%
\end{pgfscope}%
\begin{pgfscope}%
\pgfsetbuttcap%
\pgfsetmiterjoin%
\pgfsetlinewidth{0.000000pt}%
\definecolor{currentstroke}{rgb}{0.000000,0.000000,0.000000}%
\pgfsetstrokecolor{currentstroke}%
\pgfsetstrokeopacity{0.000000}%
\pgfsetdash{}{0pt}%
\pgfpathmoveto{\pgfqpoint{0.623149in}{0.515123in}}%
\pgfpathlineto{\pgfqpoint{2.870649in}{0.515123in}}%
\pgfpathlineto{\pgfqpoint{2.870649in}{2.132123in}}%
\pgfpathlineto{\pgfqpoint{0.623149in}{2.132123in}}%
\pgfpathlineto{\pgfqpoint{0.623149in}{0.515123in}}%
\pgfpathclose%
\pgfusepath{}%
\end{pgfscope}%
\begin{pgfscope}%
\pgfsetbuttcap%
\pgfsetroundjoin%
\definecolor{currentfill}{rgb}{0.000000,0.000000,0.000000}%
\pgfsetfillcolor{currentfill}%
\pgfsetlinewidth{0.803000pt}%
\definecolor{currentstroke}{rgb}{0.000000,0.000000,0.000000}%
\pgfsetstrokecolor{currentstroke}%
\pgfsetdash{}{0pt}%
\pgfsys@defobject{currentmarker}{\pgfqpoint{0.000000in}{-0.048611in}}{\pgfqpoint{0.000000in}{0.000000in}}{%
\pgfpathmoveto{\pgfqpoint{0.000000in}{0.000000in}}%
\pgfpathlineto{\pgfqpoint{0.000000in}{-0.048611in}}%
\pgfusepath{stroke,fill}%
}%
\begin{pgfscope}%
\pgfsys@transformshift{0.725308in}{0.515123in}%
\pgfsys@useobject{currentmarker}{}%
\end{pgfscope}%
\end{pgfscope}%
\begin{pgfscope}%
\definecolor{textcolor}{rgb}{0.000000,0.000000,0.000000}%
\pgfsetstrokecolor{textcolor}%
\pgfsetfillcolor{textcolor}%
\pgftext[x=0.725308in,y=0.417901in,,top]{\color{textcolor}\rmfamily\fontsize{10.000000}{12.000000}\selectfont 4/255}%
\end{pgfscope}%
\begin{pgfscope}%
\pgfsetbuttcap%
\pgfsetroundjoin%
\definecolor{currentfill}{rgb}{0.000000,0.000000,0.000000}%
\pgfsetfillcolor{currentfill}%
\pgfsetlinewidth{0.803000pt}%
\definecolor{currentstroke}{rgb}{0.000000,0.000000,0.000000}%
\pgfsetstrokecolor{currentstroke}%
\pgfsetdash{}{0pt}%
\pgfsys@defobject{currentmarker}{\pgfqpoint{0.000000in}{-0.048611in}}{\pgfqpoint{0.000000in}{0.000000in}}{%
\pgfpathmoveto{\pgfqpoint{0.000000in}{0.000000in}}%
\pgfpathlineto{\pgfqpoint{0.000000in}{-0.048611in}}%
\pgfusepath{stroke,fill}%
}%
\begin{pgfscope}%
\pgfsys@transformshift{1.236103in}{0.515123in}%
\pgfsys@useobject{currentmarker}{}%
\end{pgfscope}%
\end{pgfscope}%
\begin{pgfscope}%
\definecolor{textcolor}{rgb}{0.000000,0.000000,0.000000}%
\pgfsetstrokecolor{textcolor}%
\pgfsetfillcolor{textcolor}%
\pgftext[x=1.236103in,y=0.417901in,,top]{\color{textcolor}\rmfamily\fontsize{10.000000}{12.000000}\selectfont 8/255}%
\end{pgfscope}%
\begin{pgfscope}%
\pgfsetbuttcap%
\pgfsetroundjoin%
\definecolor{currentfill}{rgb}{0.000000,0.000000,0.000000}%
\pgfsetfillcolor{currentfill}%
\pgfsetlinewidth{0.803000pt}%
\definecolor{currentstroke}{rgb}{0.000000,0.000000,0.000000}%
\pgfsetstrokecolor{currentstroke}%
\pgfsetdash{}{0pt}%
\pgfsys@defobject{currentmarker}{\pgfqpoint{0.000000in}{-0.048611in}}{\pgfqpoint{0.000000in}{0.000000in}}{%
\pgfpathmoveto{\pgfqpoint{0.000000in}{0.000000in}}%
\pgfpathlineto{\pgfqpoint{0.000000in}{-0.048611in}}%
\pgfusepath{stroke,fill}%
}%
\begin{pgfscope}%
\pgfsys@transformshift{1.746899in}{0.515123in}%
\pgfsys@useobject{currentmarker}{}%
\end{pgfscope}%
\end{pgfscope}%
\begin{pgfscope}%
\definecolor{textcolor}{rgb}{0.000000,0.000000,0.000000}%
\pgfsetstrokecolor{textcolor}%
\pgfsetfillcolor{textcolor}%
\pgftext[x=1.746899in,y=0.417901in,,top]{\color{textcolor}\rmfamily\fontsize{10.000000}{12.000000}\selectfont 16/255}%
\end{pgfscope}%
\begin{pgfscope}%
\pgfsetbuttcap%
\pgfsetroundjoin%
\definecolor{currentfill}{rgb}{0.000000,0.000000,0.000000}%
\pgfsetfillcolor{currentfill}%
\pgfsetlinewidth{0.803000pt}%
\definecolor{currentstroke}{rgb}{0.000000,0.000000,0.000000}%
\pgfsetstrokecolor{currentstroke}%
\pgfsetdash{}{0pt}%
\pgfsys@defobject{currentmarker}{\pgfqpoint{0.000000in}{-0.048611in}}{\pgfqpoint{0.000000in}{0.000000in}}{%
\pgfpathmoveto{\pgfqpoint{0.000000in}{0.000000in}}%
\pgfpathlineto{\pgfqpoint{0.000000in}{-0.048611in}}%
\pgfusepath{stroke,fill}%
}%
\begin{pgfscope}%
\pgfsys@transformshift{2.257694in}{0.515123in}%
\pgfsys@useobject{currentmarker}{}%
\end{pgfscope}%
\end{pgfscope}%
\begin{pgfscope}%
\definecolor{textcolor}{rgb}{0.000000,0.000000,0.000000}%
\pgfsetstrokecolor{textcolor}%
\pgfsetfillcolor{textcolor}%
\pgftext[x=2.257694in,y=0.417901in,,top]{\color{textcolor}\rmfamily\fontsize{10.000000}{12.000000}\selectfont 32/255}%
\end{pgfscope}%
\begin{pgfscope}%
\pgfsetbuttcap%
\pgfsetroundjoin%
\definecolor{currentfill}{rgb}{0.000000,0.000000,0.000000}%
\pgfsetfillcolor{currentfill}%
\pgfsetlinewidth{0.803000pt}%
\definecolor{currentstroke}{rgb}{0.000000,0.000000,0.000000}%
\pgfsetstrokecolor{currentstroke}%
\pgfsetdash{}{0pt}%
\pgfsys@defobject{currentmarker}{\pgfqpoint{0.000000in}{-0.048611in}}{\pgfqpoint{0.000000in}{0.000000in}}{%
\pgfpathmoveto{\pgfqpoint{0.000000in}{0.000000in}}%
\pgfpathlineto{\pgfqpoint{0.000000in}{-0.048611in}}%
\pgfusepath{stroke,fill}%
}%
\begin{pgfscope}%
\pgfsys@transformshift{2.768490in}{0.515123in}%
\pgfsys@useobject{currentmarker}{}%
\end{pgfscope}%
\end{pgfscope}%
\begin{pgfscope}%
\definecolor{textcolor}{rgb}{0.000000,0.000000,0.000000}%
\pgfsetstrokecolor{textcolor}%
\pgfsetfillcolor{textcolor}%
\pgftext[x=2.768490in,y=0.417901in,,top]{\color{textcolor}\rmfamily\fontsize{10.000000}{12.000000}\selectfont 64/255}%
\end{pgfscope}%
\begin{pgfscope}%
\definecolor{textcolor}{rgb}{0.000000,0.000000,0.000000}%
\pgfsetstrokecolor{textcolor}%
\pgfsetfillcolor{textcolor}%
\pgftext[x=1.746899in,y=0.223457in,,top]{\color{textcolor}\rmfamily\fontsize{10.000000}{12.000000}\selectfont \(\displaystyle \ell_{\infty}\) attack radius \(\displaystyle \epsilon\)}%
\end{pgfscope}%
\begin{pgfscope}%
\pgfsetbuttcap%
\pgfsetroundjoin%
\definecolor{currentfill}{rgb}{0.000000,0.000000,0.000000}%
\pgfsetfillcolor{currentfill}%
\pgfsetlinewidth{0.803000pt}%
\definecolor{currentstroke}{rgb}{0.000000,0.000000,0.000000}%
\pgfsetstrokecolor{currentstroke}%
\pgfsetdash{}{0pt}%
\pgfsys@defobject{currentmarker}{\pgfqpoint{-0.048611in}{0.000000in}}{\pgfqpoint{-0.000000in}{0.000000in}}{%
\pgfpathmoveto{\pgfqpoint{-0.000000in}{0.000000in}}%
\pgfpathlineto{\pgfqpoint{-0.048611in}{0.000000in}}%
\pgfusepath{stroke,fill}%
}%
\begin{pgfscope}%
\pgfsys@transformshift{0.623149in}{0.588623in}%
\pgfsys@useobject{currentmarker}{}%
\end{pgfscope}%
\end{pgfscope}%
\begin{pgfscope}%
\definecolor{textcolor}{rgb}{0.000000,0.000000,0.000000}%
\pgfsetstrokecolor{textcolor}%
\pgfsetfillcolor{textcolor}%
\pgftext[x=0.279012in, y=0.540398in, left, base]{\color{textcolor}\rmfamily\fontsize{10.000000}{12.000000}\selectfont \(\displaystyle {0.00}\)}%
\end{pgfscope}%
\begin{pgfscope}%
\pgfsetbuttcap%
\pgfsetroundjoin%
\definecolor{currentfill}{rgb}{0.000000,0.000000,0.000000}%
\pgfsetfillcolor{currentfill}%
\pgfsetlinewidth{0.803000pt}%
\definecolor{currentstroke}{rgb}{0.000000,0.000000,0.000000}%
\pgfsetstrokecolor{currentstroke}%
\pgfsetdash{}{0pt}%
\pgfsys@defobject{currentmarker}{\pgfqpoint{-0.048611in}{0.000000in}}{\pgfqpoint{-0.000000in}{0.000000in}}{%
\pgfpathmoveto{\pgfqpoint{-0.000000in}{0.000000in}}%
\pgfpathlineto{\pgfqpoint{-0.048611in}{0.000000in}}%
\pgfusepath{stroke,fill}%
}%
\begin{pgfscope}%
\pgfsys@transformshift{0.623149in}{0.956123in}%
\pgfsys@useobject{currentmarker}{}%
\end{pgfscope}%
\end{pgfscope}%
\begin{pgfscope}%
\definecolor{textcolor}{rgb}{0.000000,0.000000,0.000000}%
\pgfsetstrokecolor{textcolor}%
\pgfsetfillcolor{textcolor}%
\pgftext[x=0.279012in, y=0.907898in, left, base]{\color{textcolor}\rmfamily\fontsize{10.000000}{12.000000}\selectfont \(\displaystyle {0.25}\)}%
\end{pgfscope}%
\begin{pgfscope}%
\pgfsetbuttcap%
\pgfsetroundjoin%
\definecolor{currentfill}{rgb}{0.000000,0.000000,0.000000}%
\pgfsetfillcolor{currentfill}%
\pgfsetlinewidth{0.803000pt}%
\definecolor{currentstroke}{rgb}{0.000000,0.000000,0.000000}%
\pgfsetstrokecolor{currentstroke}%
\pgfsetdash{}{0pt}%
\pgfsys@defobject{currentmarker}{\pgfqpoint{-0.048611in}{0.000000in}}{\pgfqpoint{-0.000000in}{0.000000in}}{%
\pgfpathmoveto{\pgfqpoint{-0.000000in}{0.000000in}}%
\pgfpathlineto{\pgfqpoint{-0.048611in}{0.000000in}}%
\pgfusepath{stroke,fill}%
}%
\begin{pgfscope}%
\pgfsys@transformshift{0.623149in}{1.323623in}%
\pgfsys@useobject{currentmarker}{}%
\end{pgfscope}%
\end{pgfscope}%
\begin{pgfscope}%
\definecolor{textcolor}{rgb}{0.000000,0.000000,0.000000}%
\pgfsetstrokecolor{textcolor}%
\pgfsetfillcolor{textcolor}%
\pgftext[x=0.279012in, y=1.275398in, left, base]{\color{textcolor}\rmfamily\fontsize{10.000000}{12.000000}\selectfont \(\displaystyle {0.50}\)}%
\end{pgfscope}%
\begin{pgfscope}%
\pgfsetbuttcap%
\pgfsetroundjoin%
\definecolor{currentfill}{rgb}{0.000000,0.000000,0.000000}%
\pgfsetfillcolor{currentfill}%
\pgfsetlinewidth{0.803000pt}%
\definecolor{currentstroke}{rgb}{0.000000,0.000000,0.000000}%
\pgfsetstrokecolor{currentstroke}%
\pgfsetdash{}{0pt}%
\pgfsys@defobject{currentmarker}{\pgfqpoint{-0.048611in}{0.000000in}}{\pgfqpoint{-0.000000in}{0.000000in}}{%
\pgfpathmoveto{\pgfqpoint{-0.000000in}{0.000000in}}%
\pgfpathlineto{\pgfqpoint{-0.048611in}{0.000000in}}%
\pgfusepath{stroke,fill}%
}%
\begin{pgfscope}%
\pgfsys@transformshift{0.623149in}{1.691123in}%
\pgfsys@useobject{currentmarker}{}%
\end{pgfscope}%
\end{pgfscope}%
\begin{pgfscope}%
\definecolor{textcolor}{rgb}{0.000000,0.000000,0.000000}%
\pgfsetstrokecolor{textcolor}%
\pgfsetfillcolor{textcolor}%
\pgftext[x=0.279012in, y=1.642898in, left, base]{\color{textcolor}\rmfamily\fontsize{10.000000}{12.000000}\selectfont \(\displaystyle {0.75}\)}%
\end{pgfscope}%
\begin{pgfscope}%
\pgfsetbuttcap%
\pgfsetroundjoin%
\definecolor{currentfill}{rgb}{0.000000,0.000000,0.000000}%
\pgfsetfillcolor{currentfill}%
\pgfsetlinewidth{0.803000pt}%
\definecolor{currentstroke}{rgb}{0.000000,0.000000,0.000000}%
\pgfsetstrokecolor{currentstroke}%
\pgfsetdash{}{0pt}%
\pgfsys@defobject{currentmarker}{\pgfqpoint{-0.048611in}{0.000000in}}{\pgfqpoint{-0.000000in}{0.000000in}}{%
\pgfpathmoveto{\pgfqpoint{-0.000000in}{0.000000in}}%
\pgfpathlineto{\pgfqpoint{-0.048611in}{0.000000in}}%
\pgfusepath{stroke,fill}%
}%
\begin{pgfscope}%
\pgfsys@transformshift{0.623149in}{2.058623in}%
\pgfsys@useobject{currentmarker}{}%
\end{pgfscope}%
\end{pgfscope}%
\begin{pgfscope}%
\definecolor{textcolor}{rgb}{0.000000,0.000000,0.000000}%
\pgfsetstrokecolor{textcolor}%
\pgfsetfillcolor{textcolor}%
\pgftext[x=0.279012in, y=2.010398in, left, base]{\color{textcolor}\rmfamily\fontsize{10.000000}{12.000000}\selectfont \(\displaystyle {1.00}\)}%
\end{pgfscope}%
\begin{pgfscope}%
\definecolor{textcolor}{rgb}{0.000000,0.000000,0.000000}%
\pgfsetstrokecolor{textcolor}%
\pgfsetfillcolor{textcolor}%
\pgftext[x=0.223457in,y=1.323623in,,bottom,rotate=90.000000]{\color{textcolor}\rmfamily\fontsize{10.000000}{12.000000}\selectfont CIFAR-10 attack success rate}%
\end{pgfscope}%
\begin{pgfscope}%
\pgfpathrectangle{\pgfqpoint{0.623149in}{0.515123in}}{\pgfqpoint{2.247500in}{1.617000in}}%
\pgfusepath{clip}%
\pgfsetbuttcap%
\pgfsetroundjoin%
\pgfsetlinewidth{1.505625pt}%
\definecolor{currentstroke}{rgb}{0.105882,0.619608,0.466667}%
\pgfsetstrokecolor{currentstroke}%
\pgfsetdash{{5.550000pt}{2.400000pt}}{0.000000pt}%
\pgfpathmoveto{\pgfqpoint{0.725308in}{0.796153in}}%
\pgfpathlineto{\pgfqpoint{1.236103in}{1.142035in}}%
\pgfpathlineto{\pgfqpoint{1.746899in}{1.816506in}}%
\pgfpathlineto{\pgfqpoint{2.257694in}{2.041329in}}%
\pgfpathlineto{\pgfqpoint{2.768490in}{2.058623in}}%
\pgfusepath{stroke}%
\end{pgfscope}%
\begin{pgfscope}%
\pgfpathrectangle{\pgfqpoint{0.623149in}{0.515123in}}{\pgfqpoint{2.247500in}{1.617000in}}%
\pgfusepath{clip}%
\pgfsetbuttcap%
\pgfsetroundjoin%
\definecolor{currentfill}{rgb}{0.105882,0.619608,0.466667}%
\pgfsetfillcolor{currentfill}%
\pgfsetlinewidth{1.003750pt}%
\definecolor{currentstroke}{rgb}{0.105882,0.619608,0.466667}%
\pgfsetstrokecolor{currentstroke}%
\pgfsetdash{}{0pt}%
\pgfsys@defobject{currentmarker}{\pgfqpoint{-0.027778in}{-0.027778in}}{\pgfqpoint{0.027778in}{0.027778in}}{%
\pgfpathmoveto{\pgfqpoint{0.000000in}{-0.027778in}}%
\pgfpathcurveto{\pgfqpoint{0.007367in}{-0.027778in}}{\pgfqpoint{0.014433in}{-0.024851in}}{\pgfqpoint{0.019642in}{-0.019642in}}%
\pgfpathcurveto{\pgfqpoint{0.024851in}{-0.014433in}}{\pgfqpoint{0.027778in}{-0.007367in}}{\pgfqpoint{0.027778in}{0.000000in}}%
\pgfpathcurveto{\pgfqpoint{0.027778in}{0.007367in}}{\pgfqpoint{0.024851in}{0.014433in}}{\pgfqpoint{0.019642in}{0.019642in}}%
\pgfpathcurveto{\pgfqpoint{0.014433in}{0.024851in}}{\pgfqpoint{0.007367in}{0.027778in}}{\pgfqpoint{0.000000in}{0.027778in}}%
\pgfpathcurveto{\pgfqpoint{-0.007367in}{0.027778in}}{\pgfqpoint{-0.014433in}{0.024851in}}{\pgfqpoint{-0.019642in}{0.019642in}}%
\pgfpathcurveto{\pgfqpoint{-0.024851in}{0.014433in}}{\pgfqpoint{-0.027778in}{0.007367in}}{\pgfqpoint{-0.027778in}{0.000000in}}%
\pgfpathcurveto{\pgfqpoint{-0.027778in}{-0.007367in}}{\pgfqpoint{-0.024851in}{-0.014433in}}{\pgfqpoint{-0.019642in}{-0.019642in}}%
\pgfpathcurveto{\pgfqpoint{-0.014433in}{-0.024851in}}{\pgfqpoint{-0.007367in}{-0.027778in}}{\pgfqpoint{0.000000in}{-0.027778in}}%
\pgfpathlineto{\pgfqpoint{0.000000in}{-0.027778in}}%
\pgfpathclose%
\pgfusepath{stroke,fill}%
}%
\begin{pgfscope}%
\pgfsys@transformshift{0.725308in}{0.796153in}%
\pgfsys@useobject{currentmarker}{}%
\end{pgfscope}%
\begin{pgfscope}%
\pgfsys@transformshift{1.236103in}{1.142035in}%
\pgfsys@useobject{currentmarker}{}%
\end{pgfscope}%
\begin{pgfscope}%
\pgfsys@transformshift{1.746899in}{1.816506in}%
\pgfsys@useobject{currentmarker}{}%
\end{pgfscope}%
\begin{pgfscope}%
\pgfsys@transformshift{2.257694in}{2.041329in}%
\pgfsys@useobject{currentmarker}{}%
\end{pgfscope}%
\begin{pgfscope}%
\pgfsys@transformshift{2.768490in}{2.058623in}%
\pgfsys@useobject{currentmarker}{}%
\end{pgfscope}%
\end{pgfscope}%
\begin{pgfscope}%
\pgfpathrectangle{\pgfqpoint{0.623149in}{0.515123in}}{\pgfqpoint{2.247500in}{1.617000in}}%
\pgfusepath{clip}%
\pgfsetbuttcap%
\pgfsetroundjoin%
\pgfsetlinewidth{1.505625pt}%
\definecolor{currentstroke}{rgb}{0.850980,0.372549,0.007843}%
\pgfsetstrokecolor{currentstroke}%
\pgfsetdash{{5.550000pt}{2.400000pt}}{0.000000pt}%
\pgfpathmoveto{\pgfqpoint{0.725308in}{0.657800in}}%
\pgfpathlineto{\pgfqpoint{1.236103in}{0.744270in}}%
\pgfpathlineto{\pgfqpoint{1.746899in}{0.917212in}}%
\pgfpathlineto{\pgfqpoint{2.257694in}{1.263094in}}%
\pgfpathlineto{\pgfqpoint{2.768490in}{1.816506in}}%
\pgfusepath{stroke}%
\end{pgfscope}%
\begin{pgfscope}%
\pgfpathrectangle{\pgfqpoint{0.623149in}{0.515123in}}{\pgfqpoint{2.247500in}{1.617000in}}%
\pgfusepath{clip}%
\pgfsetbuttcap%
\pgfsetroundjoin%
\definecolor{currentfill}{rgb}{0.850980,0.372549,0.007843}%
\pgfsetfillcolor{currentfill}%
\pgfsetlinewidth{1.003750pt}%
\definecolor{currentstroke}{rgb}{0.850980,0.372549,0.007843}%
\pgfsetstrokecolor{currentstroke}%
\pgfsetdash{}{0pt}%
\pgfsys@defobject{currentmarker}{\pgfqpoint{-0.027778in}{-0.027778in}}{\pgfqpoint{0.027778in}{0.027778in}}{%
\pgfpathmoveto{\pgfqpoint{0.000000in}{-0.027778in}}%
\pgfpathcurveto{\pgfqpoint{0.007367in}{-0.027778in}}{\pgfqpoint{0.014433in}{-0.024851in}}{\pgfqpoint{0.019642in}{-0.019642in}}%
\pgfpathcurveto{\pgfqpoint{0.024851in}{-0.014433in}}{\pgfqpoint{0.027778in}{-0.007367in}}{\pgfqpoint{0.027778in}{0.000000in}}%
\pgfpathcurveto{\pgfqpoint{0.027778in}{0.007367in}}{\pgfqpoint{0.024851in}{0.014433in}}{\pgfqpoint{0.019642in}{0.019642in}}%
\pgfpathcurveto{\pgfqpoint{0.014433in}{0.024851in}}{\pgfqpoint{0.007367in}{0.027778in}}{\pgfqpoint{0.000000in}{0.027778in}}%
\pgfpathcurveto{\pgfqpoint{-0.007367in}{0.027778in}}{\pgfqpoint{-0.014433in}{0.024851in}}{\pgfqpoint{-0.019642in}{0.019642in}}%
\pgfpathcurveto{\pgfqpoint{-0.024851in}{0.014433in}}{\pgfqpoint{-0.027778in}{0.007367in}}{\pgfqpoint{-0.027778in}{0.000000in}}%
\pgfpathcurveto{\pgfqpoint{-0.027778in}{-0.007367in}}{\pgfqpoint{-0.024851in}{-0.014433in}}{\pgfqpoint{-0.019642in}{-0.019642in}}%
\pgfpathcurveto{\pgfqpoint{-0.014433in}{-0.024851in}}{\pgfqpoint{-0.007367in}{-0.027778in}}{\pgfqpoint{0.000000in}{-0.027778in}}%
\pgfpathlineto{\pgfqpoint{0.000000in}{-0.027778in}}%
\pgfpathclose%
\pgfusepath{stroke,fill}%
}%
\begin{pgfscope}%
\pgfsys@transformshift{0.725308in}{0.657800in}%
\pgfsys@useobject{currentmarker}{}%
\end{pgfscope}%
\begin{pgfscope}%
\pgfsys@transformshift{1.236103in}{0.744270in}%
\pgfsys@useobject{currentmarker}{}%
\end{pgfscope}%
\begin{pgfscope}%
\pgfsys@transformshift{1.746899in}{0.917212in}%
\pgfsys@useobject{currentmarker}{}%
\end{pgfscope}%
\begin{pgfscope}%
\pgfsys@transformshift{2.257694in}{1.263094in}%
\pgfsys@useobject{currentmarker}{}%
\end{pgfscope}%
\begin{pgfscope}%
\pgfsys@transformshift{2.768490in}{1.816506in}%
\pgfsys@useobject{currentmarker}{}%
\end{pgfscope}%
\end{pgfscope}%
\begin{pgfscope}%
\pgfpathrectangle{\pgfqpoint{0.623149in}{0.515123in}}{\pgfqpoint{2.247500in}{1.617000in}}%
\pgfusepath{clip}%
\pgfsetbuttcap%
\pgfsetroundjoin%
\pgfsetlinewidth{1.505625pt}%
\definecolor{currentstroke}{rgb}{0.458824,0.439216,0.701961}%
\pgfsetstrokecolor{currentstroke}%
\pgfsetdash{{5.550000pt}{2.400000pt}}{0.000000pt}%
\pgfpathmoveto{\pgfqpoint{0.725308in}{0.588623in}}%
\pgfpathlineto{\pgfqpoint{1.236103in}{0.588623in}}%
\pgfpathlineto{\pgfqpoint{1.746899in}{0.605917in}}%
\pgfpathlineto{\pgfqpoint{2.257694in}{0.640506in}}%
\pgfpathlineto{\pgfqpoint{2.768490in}{1.107447in}}%
\pgfusepath{stroke}%
\end{pgfscope}%
\begin{pgfscope}%
\pgfpathrectangle{\pgfqpoint{0.623149in}{0.515123in}}{\pgfqpoint{2.247500in}{1.617000in}}%
\pgfusepath{clip}%
\pgfsetbuttcap%
\pgfsetroundjoin%
\definecolor{currentfill}{rgb}{0.458824,0.439216,0.701961}%
\pgfsetfillcolor{currentfill}%
\pgfsetlinewidth{1.003750pt}%
\definecolor{currentstroke}{rgb}{0.458824,0.439216,0.701961}%
\pgfsetstrokecolor{currentstroke}%
\pgfsetdash{}{0pt}%
\pgfsys@defobject{currentmarker}{\pgfqpoint{-0.027778in}{-0.027778in}}{\pgfqpoint{0.027778in}{0.027778in}}{%
\pgfpathmoveto{\pgfqpoint{0.000000in}{-0.027778in}}%
\pgfpathcurveto{\pgfqpoint{0.007367in}{-0.027778in}}{\pgfqpoint{0.014433in}{-0.024851in}}{\pgfqpoint{0.019642in}{-0.019642in}}%
\pgfpathcurveto{\pgfqpoint{0.024851in}{-0.014433in}}{\pgfqpoint{0.027778in}{-0.007367in}}{\pgfqpoint{0.027778in}{0.000000in}}%
\pgfpathcurveto{\pgfqpoint{0.027778in}{0.007367in}}{\pgfqpoint{0.024851in}{0.014433in}}{\pgfqpoint{0.019642in}{0.019642in}}%
\pgfpathcurveto{\pgfqpoint{0.014433in}{0.024851in}}{\pgfqpoint{0.007367in}{0.027778in}}{\pgfqpoint{0.000000in}{0.027778in}}%
\pgfpathcurveto{\pgfqpoint{-0.007367in}{0.027778in}}{\pgfqpoint{-0.014433in}{0.024851in}}{\pgfqpoint{-0.019642in}{0.019642in}}%
\pgfpathcurveto{\pgfqpoint{-0.024851in}{0.014433in}}{\pgfqpoint{-0.027778in}{0.007367in}}{\pgfqpoint{-0.027778in}{0.000000in}}%
\pgfpathcurveto{\pgfqpoint{-0.027778in}{-0.007367in}}{\pgfqpoint{-0.024851in}{-0.014433in}}{\pgfqpoint{-0.019642in}{-0.019642in}}%
\pgfpathcurveto{\pgfqpoint{-0.014433in}{-0.024851in}}{\pgfqpoint{-0.007367in}{-0.027778in}}{\pgfqpoint{0.000000in}{-0.027778in}}%
\pgfpathlineto{\pgfqpoint{0.000000in}{-0.027778in}}%
\pgfpathclose%
\pgfusepath{stroke,fill}%
}%
\begin{pgfscope}%
\pgfsys@transformshift{0.725308in}{0.588623in}%
\pgfsys@useobject{currentmarker}{}%
\end{pgfscope}%
\begin{pgfscope}%
\pgfsys@transformshift{1.236103in}{0.588623in}%
\pgfsys@useobject{currentmarker}{}%
\end{pgfscope}%
\begin{pgfscope}%
\pgfsys@transformshift{1.746899in}{0.605917in}%
\pgfsys@useobject{currentmarker}{}%
\end{pgfscope}%
\begin{pgfscope}%
\pgfsys@transformshift{2.257694in}{0.640506in}%
\pgfsys@useobject{currentmarker}{}%
\end{pgfscope}%
\begin{pgfscope}%
\pgfsys@transformshift{2.768490in}{1.107447in}%
\pgfsys@useobject{currentmarker}{}%
\end{pgfscope}%
\end{pgfscope}%
\begin{pgfscope}%
\pgfpathrectangle{\pgfqpoint{0.623149in}{0.515123in}}{\pgfqpoint{2.247500in}{1.617000in}}%
\pgfusepath{clip}%
\pgfsetbuttcap%
\pgfsetroundjoin%
\pgfsetlinewidth{1.505625pt}%
\definecolor{currentstroke}{rgb}{0.905882,0.160784,0.541176}%
\pgfsetstrokecolor{currentstroke}%
\pgfsetdash{{5.550000pt}{2.400000pt}}{0.000000pt}%
\pgfpathmoveto{\pgfqpoint{0.725308in}{0.588623in}}%
\pgfpathlineto{\pgfqpoint{1.236103in}{0.588623in}}%
\pgfpathlineto{\pgfqpoint{1.746899in}{0.588623in}}%
\pgfpathlineto{\pgfqpoint{2.257694in}{0.640506in}}%
\pgfpathlineto{\pgfqpoint{2.768490in}{0.692388in}}%
\pgfusepath{stroke}%
\end{pgfscope}%
\begin{pgfscope}%
\pgfpathrectangle{\pgfqpoint{0.623149in}{0.515123in}}{\pgfqpoint{2.247500in}{1.617000in}}%
\pgfusepath{clip}%
\pgfsetbuttcap%
\pgfsetroundjoin%
\definecolor{currentfill}{rgb}{0.905882,0.160784,0.541176}%
\pgfsetfillcolor{currentfill}%
\pgfsetlinewidth{1.003750pt}%
\definecolor{currentstroke}{rgb}{0.905882,0.160784,0.541176}%
\pgfsetstrokecolor{currentstroke}%
\pgfsetdash{}{0pt}%
\pgfsys@defobject{currentmarker}{\pgfqpoint{-0.027778in}{-0.027778in}}{\pgfqpoint{0.027778in}{0.027778in}}{%
\pgfpathmoveto{\pgfqpoint{0.000000in}{-0.027778in}}%
\pgfpathcurveto{\pgfqpoint{0.007367in}{-0.027778in}}{\pgfqpoint{0.014433in}{-0.024851in}}{\pgfqpoint{0.019642in}{-0.019642in}}%
\pgfpathcurveto{\pgfqpoint{0.024851in}{-0.014433in}}{\pgfqpoint{0.027778in}{-0.007367in}}{\pgfqpoint{0.027778in}{0.000000in}}%
\pgfpathcurveto{\pgfqpoint{0.027778in}{0.007367in}}{\pgfqpoint{0.024851in}{0.014433in}}{\pgfqpoint{0.019642in}{0.019642in}}%
\pgfpathcurveto{\pgfqpoint{0.014433in}{0.024851in}}{\pgfqpoint{0.007367in}{0.027778in}}{\pgfqpoint{0.000000in}{0.027778in}}%
\pgfpathcurveto{\pgfqpoint{-0.007367in}{0.027778in}}{\pgfqpoint{-0.014433in}{0.024851in}}{\pgfqpoint{-0.019642in}{0.019642in}}%
\pgfpathcurveto{\pgfqpoint{-0.024851in}{0.014433in}}{\pgfqpoint{-0.027778in}{0.007367in}}{\pgfqpoint{-0.027778in}{0.000000in}}%
\pgfpathcurveto{\pgfqpoint{-0.027778in}{-0.007367in}}{\pgfqpoint{-0.024851in}{-0.014433in}}{\pgfqpoint{-0.019642in}{-0.019642in}}%
\pgfpathcurveto{\pgfqpoint{-0.014433in}{-0.024851in}}{\pgfqpoint{-0.007367in}{-0.027778in}}{\pgfqpoint{0.000000in}{-0.027778in}}%
\pgfpathlineto{\pgfqpoint{0.000000in}{-0.027778in}}%
\pgfpathclose%
\pgfusepath{stroke,fill}%
}%
\begin{pgfscope}%
\pgfsys@transformshift{0.725308in}{0.588623in}%
\pgfsys@useobject{currentmarker}{}%
\end{pgfscope}%
\begin{pgfscope}%
\pgfsys@transformshift{1.236103in}{0.588623in}%
\pgfsys@useobject{currentmarker}{}%
\end{pgfscope}%
\begin{pgfscope}%
\pgfsys@transformshift{1.746899in}{0.588623in}%
\pgfsys@useobject{currentmarker}{}%
\end{pgfscope}%
\begin{pgfscope}%
\pgfsys@transformshift{2.257694in}{0.640506in}%
\pgfsys@useobject{currentmarker}{}%
\end{pgfscope}%
\begin{pgfscope}%
\pgfsys@transformshift{2.768490in}{0.692388in}%
\pgfsys@useobject{currentmarker}{}%
\end{pgfscope}%
\end{pgfscope}%
\begin{pgfscope}%
\pgfsetrectcap%
\pgfsetmiterjoin%
\pgfsetlinewidth{0.803000pt}%
\definecolor{currentstroke}{rgb}{0.000000,0.000000,0.000000}%
\pgfsetstrokecolor{currentstroke}%
\pgfsetdash{}{0pt}%
\pgfpathmoveto{\pgfqpoint{0.623149in}{0.515123in}}%
\pgfpathlineto{\pgfqpoint{0.623149in}{2.132123in}}%
\pgfusepath{stroke}%
\end{pgfscope}%
\begin{pgfscope}%
\pgfsetrectcap%
\pgfsetmiterjoin%
\pgfsetlinewidth{0.803000pt}%
\definecolor{currentstroke}{rgb}{0.000000,0.000000,0.000000}%
\pgfsetstrokecolor{currentstroke}%
\pgfsetdash{}{0pt}%
\pgfpathmoveto{\pgfqpoint{2.870649in}{0.515123in}}%
\pgfpathlineto{\pgfqpoint{2.870649in}{2.132123in}}%
\pgfusepath{stroke}%
\end{pgfscope}%
\begin{pgfscope}%
\pgfsetrectcap%
\pgfsetmiterjoin%
\pgfsetlinewidth{0.803000pt}%
\definecolor{currentstroke}{rgb}{0.000000,0.000000,0.000000}%
\pgfsetstrokecolor{currentstroke}%
\pgfsetdash{}{0pt}%
\pgfpathmoveto{\pgfqpoint{0.623149in}{0.515123in}}%
\pgfpathlineto{\pgfqpoint{2.870649in}{0.515123in}}%
\pgfusepath{stroke}%
\end{pgfscope}%
\begin{pgfscope}%
\pgfsetrectcap%
\pgfsetmiterjoin%
\pgfsetlinewidth{0.803000pt}%
\definecolor{currentstroke}{rgb}{0.000000,0.000000,0.000000}%
\pgfsetstrokecolor{currentstroke}%
\pgfsetdash{}{0pt}%
\pgfpathmoveto{\pgfqpoint{0.623149in}{2.132123in}}%
\pgfpathlineto{\pgfqpoint{2.870649in}{2.132123in}}%
\pgfusepath{stroke}%
\end{pgfscope}%
\begin{pgfscope}%
\pgfsetbuttcap%
\pgfsetmiterjoin%
\definecolor{currentfill}{rgb}{1.000000,1.000000,1.000000}%
\pgfsetfillcolor{currentfill}%
\pgfsetfillopacity{0.800000}%
\pgfsetlinewidth{1.003750pt}%
\definecolor{currentstroke}{rgb}{0.800000,0.800000,0.800000}%
\pgfsetstrokecolor{currentstroke}%
\pgfsetstrokeopacity{0.800000}%
\pgfsetdash{}{0pt}%
\pgfpathmoveto{\pgfqpoint{0.720371in}{1.246321in}}%
\pgfpathlineto{\pgfqpoint{2.125302in}{1.246321in}}%
\pgfpathquadraticcurveto{\pgfqpoint{2.153080in}{1.246321in}}{\pgfqpoint{2.153080in}{1.274099in}}%
\pgfpathlineto{\pgfqpoint{2.153080in}{2.034901in}}%
\pgfpathquadraticcurveto{\pgfqpoint{2.153080in}{2.062679in}}{\pgfqpoint{2.125302in}{2.062679in}}%
\pgfpathlineto{\pgfqpoint{0.720371in}{2.062679in}}%
\pgfpathquadraticcurveto{\pgfqpoint{0.692593in}{2.062679in}}{\pgfqpoint{0.692593in}{2.034901in}}%
\pgfpathlineto{\pgfqpoint{0.692593in}{1.274099in}}%
\pgfpathquadraticcurveto{\pgfqpoint{0.692593in}{1.246321in}}{\pgfqpoint{0.720371in}{1.246321in}}%
\pgfpathlineto{\pgfqpoint{0.720371in}{1.246321in}}%
\pgfpathclose%
\pgfusepath{stroke,fill}%
\end{pgfscope}%
\begin{pgfscope}%
\pgfsetbuttcap%
\pgfsetroundjoin%
\pgfsetlinewidth{1.505625pt}%
\definecolor{currentstroke}{rgb}{0.105882,0.619608,0.466667}%
\pgfsetstrokecolor{currentstroke}%
\pgfsetdash{{5.550000pt}{2.400000pt}}{0.000000pt}%
\pgfpathmoveto{\pgfqpoint{0.748149in}{1.958512in}}%
\pgfpathlineto{\pgfqpoint{0.887038in}{1.958512in}}%
\pgfpathlineto{\pgfqpoint{1.025927in}{1.958512in}}%
\pgfusepath{stroke}%
\end{pgfscope}%
\begin{pgfscope}%
\pgfsetbuttcap%
\pgfsetroundjoin%
\definecolor{currentfill}{rgb}{0.105882,0.619608,0.466667}%
\pgfsetfillcolor{currentfill}%
\pgfsetlinewidth{1.003750pt}%
\definecolor{currentstroke}{rgb}{0.105882,0.619608,0.466667}%
\pgfsetstrokecolor{currentstroke}%
\pgfsetdash{}{0pt}%
\pgfsys@defobject{currentmarker}{\pgfqpoint{-0.027778in}{-0.027778in}}{\pgfqpoint{0.027778in}{0.027778in}}{%
\pgfpathmoveto{\pgfqpoint{0.000000in}{-0.027778in}}%
\pgfpathcurveto{\pgfqpoint{0.007367in}{-0.027778in}}{\pgfqpoint{0.014433in}{-0.024851in}}{\pgfqpoint{0.019642in}{-0.019642in}}%
\pgfpathcurveto{\pgfqpoint{0.024851in}{-0.014433in}}{\pgfqpoint{0.027778in}{-0.007367in}}{\pgfqpoint{0.027778in}{0.000000in}}%
\pgfpathcurveto{\pgfqpoint{0.027778in}{0.007367in}}{\pgfqpoint{0.024851in}{0.014433in}}{\pgfqpoint{0.019642in}{0.019642in}}%
\pgfpathcurveto{\pgfqpoint{0.014433in}{0.024851in}}{\pgfqpoint{0.007367in}{0.027778in}}{\pgfqpoint{0.000000in}{0.027778in}}%
\pgfpathcurveto{\pgfqpoint{-0.007367in}{0.027778in}}{\pgfqpoint{-0.014433in}{0.024851in}}{\pgfqpoint{-0.019642in}{0.019642in}}%
\pgfpathcurveto{\pgfqpoint{-0.024851in}{0.014433in}}{\pgfqpoint{-0.027778in}{0.007367in}}{\pgfqpoint{-0.027778in}{0.000000in}}%
\pgfpathcurveto{\pgfqpoint{-0.027778in}{-0.007367in}}{\pgfqpoint{-0.024851in}{-0.014433in}}{\pgfqpoint{-0.019642in}{-0.019642in}}%
\pgfpathcurveto{\pgfqpoint{-0.014433in}{-0.024851in}}{\pgfqpoint{-0.007367in}{-0.027778in}}{\pgfqpoint{0.000000in}{-0.027778in}}%
\pgfpathlineto{\pgfqpoint{0.000000in}{-0.027778in}}%
\pgfpathclose%
\pgfusepath{stroke,fill}%
}%
\begin{pgfscope}%
\pgfsys@transformshift{0.887038in}{1.958512in}%
\pgfsys@useobject{currentmarker}{}%
\end{pgfscope}%
\end{pgfscope}%
\begin{pgfscope}%
\definecolor{textcolor}{rgb}{0.000000,0.000000,0.000000}%
\pgfsetstrokecolor{textcolor}%
\pgfsetfillcolor{textcolor}%
\pgftext[x=1.137038in,y=1.909901in,left,base]{\color{textcolor}\rmfamily\fontsize{10.000000}{12.000000}\selectfont \(\displaystyle \textsc{PGD}\)}%
\end{pgfscope}%
\begin{pgfscope}%
\pgfsetbuttcap%
\pgfsetroundjoin%
\pgfsetlinewidth{1.505625pt}%
\definecolor{currentstroke}{rgb}{0.850980,0.372549,0.007843}%
\pgfsetstrokecolor{currentstroke}%
\pgfsetdash{{5.550000pt}{2.400000pt}}{0.000000pt}%
\pgfpathmoveto{\pgfqpoint{0.748149in}{1.764839in}}%
\pgfpathlineto{\pgfqpoint{0.887038in}{1.764839in}}%
\pgfpathlineto{\pgfqpoint{1.025927in}{1.764839in}}%
\pgfusepath{stroke}%
\end{pgfscope}%
\begin{pgfscope}%
\pgfsetbuttcap%
\pgfsetroundjoin%
\definecolor{currentfill}{rgb}{0.850980,0.372549,0.007843}%
\pgfsetfillcolor{currentfill}%
\pgfsetlinewidth{1.003750pt}%
\definecolor{currentstroke}{rgb}{0.850980,0.372549,0.007843}%
\pgfsetstrokecolor{currentstroke}%
\pgfsetdash{}{0pt}%
\pgfsys@defobject{currentmarker}{\pgfqpoint{-0.027778in}{-0.027778in}}{\pgfqpoint{0.027778in}{0.027778in}}{%
\pgfpathmoveto{\pgfqpoint{0.000000in}{-0.027778in}}%
\pgfpathcurveto{\pgfqpoint{0.007367in}{-0.027778in}}{\pgfqpoint{0.014433in}{-0.024851in}}{\pgfqpoint{0.019642in}{-0.019642in}}%
\pgfpathcurveto{\pgfqpoint{0.024851in}{-0.014433in}}{\pgfqpoint{0.027778in}{-0.007367in}}{\pgfqpoint{0.027778in}{0.000000in}}%
\pgfpathcurveto{\pgfqpoint{0.027778in}{0.007367in}}{\pgfqpoint{0.024851in}{0.014433in}}{\pgfqpoint{0.019642in}{0.019642in}}%
\pgfpathcurveto{\pgfqpoint{0.014433in}{0.024851in}}{\pgfqpoint{0.007367in}{0.027778in}}{\pgfqpoint{0.000000in}{0.027778in}}%
\pgfpathcurveto{\pgfqpoint{-0.007367in}{0.027778in}}{\pgfqpoint{-0.014433in}{0.024851in}}{\pgfqpoint{-0.019642in}{0.019642in}}%
\pgfpathcurveto{\pgfqpoint{-0.024851in}{0.014433in}}{\pgfqpoint{-0.027778in}{0.007367in}}{\pgfqpoint{-0.027778in}{0.000000in}}%
\pgfpathcurveto{\pgfqpoint{-0.027778in}{-0.007367in}}{\pgfqpoint{-0.024851in}{-0.014433in}}{\pgfqpoint{-0.019642in}{-0.019642in}}%
\pgfpathcurveto{\pgfqpoint{-0.014433in}{-0.024851in}}{\pgfqpoint{-0.007367in}{-0.027778in}}{\pgfqpoint{0.000000in}{-0.027778in}}%
\pgfpathlineto{\pgfqpoint{0.000000in}{-0.027778in}}%
\pgfpathclose%
\pgfusepath{stroke,fill}%
}%
\begin{pgfscope}%
\pgfsys@transformshift{0.887038in}{1.764839in}%
\pgfsys@useobject{currentmarker}{}%
\end{pgfscope}%
\end{pgfscope}%
\begin{pgfscope}%
\definecolor{textcolor}{rgb}{0.000000,0.000000,0.000000}%
\pgfsetstrokecolor{textcolor}%
\pgfsetfillcolor{textcolor}%
\pgftext[x=1.137038in,y=1.716228in,left,base]{\color{textcolor}\rmfamily\fontsize{10.000000}{12.000000}\selectfont \(\displaystyle \textsc{SubspacePGD}\)}%
\end{pgfscope}%
\begin{pgfscope}%
\pgfsetbuttcap%
\pgfsetroundjoin%
\pgfsetlinewidth{1.505625pt}%
\definecolor{currentstroke}{rgb}{0.458824,0.439216,0.701961}%
\pgfsetstrokecolor{currentstroke}%
\pgfsetdash{{5.550000pt}{2.400000pt}}{0.000000pt}%
\pgfpathmoveto{\pgfqpoint{0.748149in}{1.571167in}}%
\pgfpathlineto{\pgfqpoint{0.887038in}{1.571167in}}%
\pgfpathlineto{\pgfqpoint{1.025927in}{1.571167in}}%
\pgfusepath{stroke}%
\end{pgfscope}%
\begin{pgfscope}%
\pgfsetbuttcap%
\pgfsetroundjoin%
\definecolor{currentfill}{rgb}{0.458824,0.439216,0.701961}%
\pgfsetfillcolor{currentfill}%
\pgfsetlinewidth{1.003750pt}%
\definecolor{currentstroke}{rgb}{0.458824,0.439216,0.701961}%
\pgfsetstrokecolor{currentstroke}%
\pgfsetdash{}{0pt}%
\pgfsys@defobject{currentmarker}{\pgfqpoint{-0.027778in}{-0.027778in}}{\pgfqpoint{0.027778in}{0.027778in}}{%
\pgfpathmoveto{\pgfqpoint{0.000000in}{-0.027778in}}%
\pgfpathcurveto{\pgfqpoint{0.007367in}{-0.027778in}}{\pgfqpoint{0.014433in}{-0.024851in}}{\pgfqpoint{0.019642in}{-0.019642in}}%
\pgfpathcurveto{\pgfqpoint{0.024851in}{-0.014433in}}{\pgfqpoint{0.027778in}{-0.007367in}}{\pgfqpoint{0.027778in}{0.000000in}}%
\pgfpathcurveto{\pgfqpoint{0.027778in}{0.007367in}}{\pgfqpoint{0.024851in}{0.014433in}}{\pgfqpoint{0.019642in}{0.019642in}}%
\pgfpathcurveto{\pgfqpoint{0.014433in}{0.024851in}}{\pgfqpoint{0.007367in}{0.027778in}}{\pgfqpoint{0.000000in}{0.027778in}}%
\pgfpathcurveto{\pgfqpoint{-0.007367in}{0.027778in}}{\pgfqpoint{-0.014433in}{0.024851in}}{\pgfqpoint{-0.019642in}{0.019642in}}%
\pgfpathcurveto{\pgfqpoint{-0.024851in}{0.014433in}}{\pgfqpoint{-0.027778in}{0.007367in}}{\pgfqpoint{-0.027778in}{0.000000in}}%
\pgfpathcurveto{\pgfqpoint{-0.027778in}{-0.007367in}}{\pgfqpoint{-0.024851in}{-0.014433in}}{\pgfqpoint{-0.019642in}{-0.019642in}}%
\pgfpathcurveto{\pgfqpoint{-0.014433in}{-0.024851in}}{\pgfqpoint{-0.007367in}{-0.027778in}}{\pgfqpoint{0.000000in}{-0.027778in}}%
\pgfpathlineto{\pgfqpoint{0.000000in}{-0.027778in}}%
\pgfpathclose%
\pgfusepath{stroke,fill}%
}%
\begin{pgfscope}%
\pgfsys@transformshift{0.887038in}{1.571167in}%
\pgfsys@useobject{currentmarker}{}%
\end{pgfscope}%
\end{pgfscope}%
\begin{pgfscope}%
\definecolor{textcolor}{rgb}{0.000000,0.000000,0.000000}%
\pgfsetstrokecolor{textcolor}%
\pgfsetfillcolor{textcolor}%
\pgftext[x=1.137038in,y=1.522556in,left,base]{\color{textcolor}\rmfamily\fontsize{10.000000}{12.000000}\selectfont \(\displaystyle \textsc{RandMax}\)}%
\end{pgfscope}%
\begin{pgfscope}%
\pgfsetbuttcap%
\pgfsetroundjoin%
\pgfsetlinewidth{1.505625pt}%
\definecolor{currentstroke}{rgb}{0.905882,0.160784,0.541176}%
\pgfsetstrokecolor{currentstroke}%
\pgfsetdash{{5.550000pt}{2.400000pt}}{0.000000pt}%
\pgfpathmoveto{\pgfqpoint{0.748149in}{1.377494in}}%
\pgfpathlineto{\pgfqpoint{0.887038in}{1.377494in}}%
\pgfpathlineto{\pgfqpoint{1.025927in}{1.377494in}}%
\pgfusepath{stroke}%
\end{pgfscope}%
\begin{pgfscope}%
\pgfsetbuttcap%
\pgfsetroundjoin%
\definecolor{currentfill}{rgb}{0.905882,0.160784,0.541176}%
\pgfsetfillcolor{currentfill}%
\pgfsetlinewidth{1.003750pt}%
\definecolor{currentstroke}{rgb}{0.905882,0.160784,0.541176}%
\pgfsetstrokecolor{currentstroke}%
\pgfsetdash{}{0pt}%
\pgfsys@defobject{currentmarker}{\pgfqpoint{-0.027778in}{-0.027778in}}{\pgfqpoint{0.027778in}{0.027778in}}{%
\pgfpathmoveto{\pgfqpoint{0.000000in}{-0.027778in}}%
\pgfpathcurveto{\pgfqpoint{0.007367in}{-0.027778in}}{\pgfqpoint{0.014433in}{-0.024851in}}{\pgfqpoint{0.019642in}{-0.019642in}}%
\pgfpathcurveto{\pgfqpoint{0.024851in}{-0.014433in}}{\pgfqpoint{0.027778in}{-0.007367in}}{\pgfqpoint{0.027778in}{0.000000in}}%
\pgfpathcurveto{\pgfqpoint{0.027778in}{0.007367in}}{\pgfqpoint{0.024851in}{0.014433in}}{\pgfqpoint{0.019642in}{0.019642in}}%
\pgfpathcurveto{\pgfqpoint{0.014433in}{0.024851in}}{\pgfqpoint{0.007367in}{0.027778in}}{\pgfqpoint{0.000000in}{0.027778in}}%
\pgfpathcurveto{\pgfqpoint{-0.007367in}{0.027778in}}{\pgfqpoint{-0.014433in}{0.024851in}}{\pgfqpoint{-0.019642in}{0.019642in}}%
\pgfpathcurveto{\pgfqpoint{-0.024851in}{0.014433in}}{\pgfqpoint{-0.027778in}{0.007367in}}{\pgfqpoint{-0.027778in}{0.000000in}}%
\pgfpathcurveto{\pgfqpoint{-0.027778in}{-0.007367in}}{\pgfqpoint{-0.024851in}{-0.014433in}}{\pgfqpoint{-0.019642in}{-0.019642in}}%
\pgfpathcurveto{\pgfqpoint{-0.014433in}{-0.024851in}}{\pgfqpoint{-0.007367in}{-0.027778in}}{\pgfqpoint{0.000000in}{-0.027778in}}%
\pgfpathlineto{\pgfqpoint{0.000000in}{-0.027778in}}%
\pgfpathclose%
\pgfusepath{stroke,fill}%
}%
\begin{pgfscope}%
\pgfsys@transformshift{0.887038in}{1.377494in}%
\pgfsys@useobject{currentmarker}{}%
\end{pgfscope}%
\end{pgfscope}%
\begin{pgfscope}%
\definecolor{textcolor}{rgb}{0.000000,0.000000,0.000000}%
\pgfsetstrokecolor{textcolor}%
\pgfsetfillcolor{textcolor}%
\pgftext[x=1.137038in,y=1.328883in,left,base]{\color{textcolor}\rmfamily\fontsize{10.000000}{12.000000}\selectfont \(\displaystyle \textsc{RandUniform}\)}%
\end{pgfscope}%
\end{pgfpicture}%
\makeatother%
\endgroup%

%% file: res/cifar10_examples.tex
\begin{tikzpicture}
\node [anchor=west, rotate=90, align=left, anchor=south west] (note5) at (0,4.5) {\small Orig};
\node [anchor=west, rotate=90, align=left, anchor=south west] (note1) at (0,3.4) {\small PGD};
\node [anchor=west, rotate=90, align=left, anchor=south west] (note2) at (0,2.3) {\small Sub.\\\small PGD};
\node [anchor=west, rotate=90, align=left, anchor=south west] (note3) at (0,1.2) {\small Rnd.\\\small Max};
\node [anchor=west, rotate=90, align=left, anchor=south west] (note4) at (0,0) {\small Rnd.\\\small Uni.};

%\begin{scope}[xshift=0cm]
\node[anchor=south west,inner sep=0] (image) at (0,0) {\includegraphics[width=0.8\textwidth]{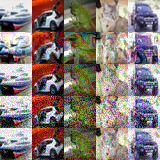}};
%\end{scope}

\node [anchor=west, rotate=0, align=center, anchor=south] (class1) at (0.5,5.5) {\small boat};
\node [anchor=west, rotate=0, align=center, anchor=south] (class2) at (1.7,5.5) {\small car};
\node [anchor=west, rotate=0, align=center, anchor=south] (class3) at (2.9,5.42) {\small frog};
\node [anchor=west, rotate=0, align=center, anchor=south] (class4) at (4.0,5.5) {\small cat};
\node [anchor=west, rotate=0, align=center, anchor=south] (class5) at (5.1,5.5) {\small car};

%\node [anchor=west, rotate=0, align=center, anchor=south] (class1) at (0.5,4.8) {\small boat};
%\node [anchor=west, rotate=0, align=center, anchor=south] (class2) at (1.5,4.8) {\small car};
%\node [anchor=west, rotate=0, align=center, anchor=south] (class3) at (2.5,4.72) {\small frog};
%\node [anchor=west, rotate=0, align=center, anchor=south] (class4) at (3.5,4.8) {\small cat};
%\node [anchor=west, rotate=0, align=center, anchor=south] (class5) at (4.5,4.8) {\small car};
\end{tikzpicture}%

%% file: figs/svhn_attack_sweep.pgf
%% Creator: Matplotlib, PGF backend
%%
%% To include the figure in your LaTeX document, write
%%   \input{<filename>.pgf}
%%
%% Make sure the required packages are loaded in your preamble
%%   \usepackage{pgf}
%%
%% Also ensure that all the required font packages are loaded; for instance,
%% the lmodern package is sometimes necessary when using math font.
%%   \usepackage{lmodern}
%%
%% Figures using additional raster images can only be included by \input if
%% they are in the same directory as the main LaTeX file. For loading figures
%% from other directories you can use the `import` package
%%   \usepackage{import}
%%
%% and then include the figures with
%%   \import{<path to file>}{<filename>.pgf}
%%
%% Matplotlib used the following preamble
%%   
%%   \makeatletter\@ifpackageloaded{underscore}{}{\usepackage[strings]{underscore}}\makeatother
%%
\begingroup%
\makeatletter%
\begin{pgfpicture}%
\pgfpathrectangle{\pgfpointorigin}{\pgfqpoint{3.076824in}{2.232123in}}%
\pgfusepath{use as bounding box, clip}%
\begin{pgfscope}%
\pgfsetbuttcap%
\pgfsetmiterjoin%
\pgfsetlinewidth{0.000000pt}%
\definecolor{currentstroke}{rgb}{0.000000,0.000000,0.000000}%
\pgfsetstrokecolor{currentstroke}%
\pgfsetstrokeopacity{0.000000}%
\pgfsetdash{}{0pt}%
\pgfpathmoveto{\pgfqpoint{0.000000in}{0.000000in}}%
\pgfpathlineto{\pgfqpoint{3.076824in}{0.000000in}}%
\pgfpathlineto{\pgfqpoint{3.076824in}{2.232123in}}%
\pgfpathlineto{\pgfqpoint{0.000000in}{2.232123in}}%
\pgfpathlineto{\pgfqpoint{0.000000in}{0.000000in}}%
\pgfpathclose%
\pgfusepath{}%
\end{pgfscope}%
\begin{pgfscope}%
\pgfsetbuttcap%
\pgfsetmiterjoin%
\pgfsetlinewidth{0.000000pt}%
\definecolor{currentstroke}{rgb}{0.000000,0.000000,0.000000}%
\pgfsetstrokecolor{currentstroke}%
\pgfsetstrokeopacity{0.000000}%
\pgfsetdash{}{0pt}%
\pgfpathmoveto{\pgfqpoint{0.623149in}{0.515123in}}%
\pgfpathlineto{\pgfqpoint{2.870649in}{0.515123in}}%
\pgfpathlineto{\pgfqpoint{2.870649in}{2.132123in}}%
\pgfpathlineto{\pgfqpoint{0.623149in}{2.132123in}}%
\pgfpathlineto{\pgfqpoint{0.623149in}{0.515123in}}%
\pgfpathclose%
\pgfusepath{}%
\end{pgfscope}%
\begin{pgfscope}%
\pgfsetbuttcap%
\pgfsetroundjoin%
\definecolor{currentfill}{rgb}{0.000000,0.000000,0.000000}%
\pgfsetfillcolor{currentfill}%
\pgfsetlinewidth{0.803000pt}%
\definecolor{currentstroke}{rgb}{0.000000,0.000000,0.000000}%
\pgfsetstrokecolor{currentstroke}%
\pgfsetdash{}{0pt}%
\pgfsys@defobject{currentmarker}{\pgfqpoint{0.000000in}{-0.048611in}}{\pgfqpoint{0.000000in}{0.000000in}}{%
\pgfpathmoveto{\pgfqpoint{0.000000in}{0.000000in}}%
\pgfpathlineto{\pgfqpoint{0.000000in}{-0.048611in}}%
\pgfusepath{stroke,fill}%
}%
\begin{pgfscope}%
\pgfsys@transformshift{0.725308in}{0.515123in}%
\pgfsys@useobject{currentmarker}{}%
\end{pgfscope}%
\end{pgfscope}%
\begin{pgfscope}%
\definecolor{textcolor}{rgb}{0.000000,0.000000,0.000000}%
\pgfsetstrokecolor{textcolor}%
\pgfsetfillcolor{textcolor}%
\pgftext[x=0.725308in,y=0.417901in,,top]{\color{textcolor}\rmfamily\fontsize{10.000000}{12.000000}\selectfont 4/255}%
\end{pgfscope}%
\begin{pgfscope}%
\pgfsetbuttcap%
\pgfsetroundjoin%
\definecolor{currentfill}{rgb}{0.000000,0.000000,0.000000}%
\pgfsetfillcolor{currentfill}%
\pgfsetlinewidth{0.803000pt}%
\definecolor{currentstroke}{rgb}{0.000000,0.000000,0.000000}%
\pgfsetstrokecolor{currentstroke}%
\pgfsetdash{}{0pt}%
\pgfsys@defobject{currentmarker}{\pgfqpoint{0.000000in}{-0.048611in}}{\pgfqpoint{0.000000in}{0.000000in}}{%
\pgfpathmoveto{\pgfqpoint{0.000000in}{0.000000in}}%
\pgfpathlineto{\pgfqpoint{0.000000in}{-0.048611in}}%
\pgfusepath{stroke,fill}%
}%
\begin{pgfscope}%
\pgfsys@transformshift{1.236103in}{0.515123in}%
\pgfsys@useobject{currentmarker}{}%
\end{pgfscope}%
\end{pgfscope}%
\begin{pgfscope}%
\definecolor{textcolor}{rgb}{0.000000,0.000000,0.000000}%
\pgfsetstrokecolor{textcolor}%
\pgfsetfillcolor{textcolor}%
\pgftext[x=1.236103in,y=0.417901in,,top]{\color{textcolor}\rmfamily\fontsize{10.000000}{12.000000}\selectfont 8/255}%
\end{pgfscope}%
\begin{pgfscope}%
\pgfsetbuttcap%
\pgfsetroundjoin%
\definecolor{currentfill}{rgb}{0.000000,0.000000,0.000000}%
\pgfsetfillcolor{currentfill}%
\pgfsetlinewidth{0.803000pt}%
\definecolor{currentstroke}{rgb}{0.000000,0.000000,0.000000}%
\pgfsetstrokecolor{currentstroke}%
\pgfsetdash{}{0pt}%
\pgfsys@defobject{currentmarker}{\pgfqpoint{0.000000in}{-0.048611in}}{\pgfqpoint{0.000000in}{0.000000in}}{%
\pgfpathmoveto{\pgfqpoint{0.000000in}{0.000000in}}%
\pgfpathlineto{\pgfqpoint{0.000000in}{-0.048611in}}%
\pgfusepath{stroke,fill}%
}%
\begin{pgfscope}%
\pgfsys@transformshift{1.746899in}{0.515123in}%
\pgfsys@useobject{currentmarker}{}%
\end{pgfscope}%
\end{pgfscope}%
\begin{pgfscope}%
\definecolor{textcolor}{rgb}{0.000000,0.000000,0.000000}%
\pgfsetstrokecolor{textcolor}%
\pgfsetfillcolor{textcolor}%
\pgftext[x=1.746899in,y=0.417901in,,top]{\color{textcolor}\rmfamily\fontsize{10.000000}{12.000000}\selectfont 16/255}%
\end{pgfscope}%
\begin{pgfscope}%
\pgfsetbuttcap%
\pgfsetroundjoin%
\definecolor{currentfill}{rgb}{0.000000,0.000000,0.000000}%
\pgfsetfillcolor{currentfill}%
\pgfsetlinewidth{0.803000pt}%
\definecolor{currentstroke}{rgb}{0.000000,0.000000,0.000000}%
\pgfsetstrokecolor{currentstroke}%
\pgfsetdash{}{0pt}%
\pgfsys@defobject{currentmarker}{\pgfqpoint{0.000000in}{-0.048611in}}{\pgfqpoint{0.000000in}{0.000000in}}{%
\pgfpathmoveto{\pgfqpoint{0.000000in}{0.000000in}}%
\pgfpathlineto{\pgfqpoint{0.000000in}{-0.048611in}}%
\pgfusepath{stroke,fill}%
}%
\begin{pgfscope}%
\pgfsys@transformshift{2.257694in}{0.515123in}%
\pgfsys@useobject{currentmarker}{}%
\end{pgfscope}%
\end{pgfscope}%
\begin{pgfscope}%
\definecolor{textcolor}{rgb}{0.000000,0.000000,0.000000}%
\pgfsetstrokecolor{textcolor}%
\pgfsetfillcolor{textcolor}%
\pgftext[x=2.257694in,y=0.417901in,,top]{\color{textcolor}\rmfamily\fontsize{10.000000}{12.000000}\selectfont 32/255}%
\end{pgfscope}%
\begin{pgfscope}%
\pgfsetbuttcap%
\pgfsetroundjoin%
\definecolor{currentfill}{rgb}{0.000000,0.000000,0.000000}%
\pgfsetfillcolor{currentfill}%
\pgfsetlinewidth{0.803000pt}%
\definecolor{currentstroke}{rgb}{0.000000,0.000000,0.000000}%
\pgfsetstrokecolor{currentstroke}%
\pgfsetdash{}{0pt}%
\pgfsys@defobject{currentmarker}{\pgfqpoint{0.000000in}{-0.048611in}}{\pgfqpoint{0.000000in}{0.000000in}}{%
\pgfpathmoveto{\pgfqpoint{0.000000in}{0.000000in}}%
\pgfpathlineto{\pgfqpoint{0.000000in}{-0.048611in}}%
\pgfusepath{stroke,fill}%
}%
\begin{pgfscope}%
\pgfsys@transformshift{2.768490in}{0.515123in}%
\pgfsys@useobject{currentmarker}{}%
\end{pgfscope}%
\end{pgfscope}%
\begin{pgfscope}%
\definecolor{textcolor}{rgb}{0.000000,0.000000,0.000000}%
\pgfsetstrokecolor{textcolor}%
\pgfsetfillcolor{textcolor}%
\pgftext[x=2.768490in,y=0.417901in,,top]{\color{textcolor}\rmfamily\fontsize{10.000000}{12.000000}\selectfont 64/255}%
\end{pgfscope}%
\begin{pgfscope}%
\definecolor{textcolor}{rgb}{0.000000,0.000000,0.000000}%
\pgfsetstrokecolor{textcolor}%
\pgfsetfillcolor{textcolor}%
\pgftext[x=1.746899in,y=0.223457in,,top]{\color{textcolor}\rmfamily\fontsize{10.000000}{12.000000}\selectfont \(\displaystyle \ell_{\infty}\) attack radius \(\displaystyle \epsilon\)}%
\end{pgfscope}%
\begin{pgfscope}%
\pgfsetbuttcap%
\pgfsetroundjoin%
\definecolor{currentfill}{rgb}{0.000000,0.000000,0.000000}%
\pgfsetfillcolor{currentfill}%
\pgfsetlinewidth{0.803000pt}%
\definecolor{currentstroke}{rgb}{0.000000,0.000000,0.000000}%
\pgfsetstrokecolor{currentstroke}%
\pgfsetdash{}{0pt}%
\pgfsys@defobject{currentmarker}{\pgfqpoint{-0.048611in}{0.000000in}}{\pgfqpoint{-0.000000in}{0.000000in}}{%
\pgfpathmoveto{\pgfqpoint{-0.000000in}{0.000000in}}%
\pgfpathlineto{\pgfqpoint{-0.048611in}{0.000000in}}%
\pgfusepath{stroke,fill}%
}%
\begin{pgfscope}%
\pgfsys@transformshift{0.623149in}{0.588623in}%
\pgfsys@useobject{currentmarker}{}%
\end{pgfscope}%
\end{pgfscope}%
\begin{pgfscope}%
\definecolor{textcolor}{rgb}{0.000000,0.000000,0.000000}%
\pgfsetstrokecolor{textcolor}%
\pgfsetfillcolor{textcolor}%
\pgftext[x=0.279012in, y=0.540398in, left, base]{\color{textcolor}\rmfamily\fontsize{10.000000}{12.000000}\selectfont \(\displaystyle {0.00}\)}%
\end{pgfscope}%
\begin{pgfscope}%
\pgfsetbuttcap%
\pgfsetroundjoin%
\definecolor{currentfill}{rgb}{0.000000,0.000000,0.000000}%
\pgfsetfillcolor{currentfill}%
\pgfsetlinewidth{0.803000pt}%
\definecolor{currentstroke}{rgb}{0.000000,0.000000,0.000000}%
\pgfsetstrokecolor{currentstroke}%
\pgfsetdash{}{0pt}%
\pgfsys@defobject{currentmarker}{\pgfqpoint{-0.048611in}{0.000000in}}{\pgfqpoint{-0.000000in}{0.000000in}}{%
\pgfpathmoveto{\pgfqpoint{-0.000000in}{0.000000in}}%
\pgfpathlineto{\pgfqpoint{-0.048611in}{0.000000in}}%
\pgfusepath{stroke,fill}%
}%
\begin{pgfscope}%
\pgfsys@transformshift{0.623149in}{0.956123in}%
\pgfsys@useobject{currentmarker}{}%
\end{pgfscope}%
\end{pgfscope}%
\begin{pgfscope}%
\definecolor{textcolor}{rgb}{0.000000,0.000000,0.000000}%
\pgfsetstrokecolor{textcolor}%
\pgfsetfillcolor{textcolor}%
\pgftext[x=0.279012in, y=0.907898in, left, base]{\color{textcolor}\rmfamily\fontsize{10.000000}{12.000000}\selectfont \(\displaystyle {0.25}\)}%
\end{pgfscope}%
\begin{pgfscope}%
\pgfsetbuttcap%
\pgfsetroundjoin%
\definecolor{currentfill}{rgb}{0.000000,0.000000,0.000000}%
\pgfsetfillcolor{currentfill}%
\pgfsetlinewidth{0.803000pt}%
\definecolor{currentstroke}{rgb}{0.000000,0.000000,0.000000}%
\pgfsetstrokecolor{currentstroke}%
\pgfsetdash{}{0pt}%
\pgfsys@defobject{currentmarker}{\pgfqpoint{-0.048611in}{0.000000in}}{\pgfqpoint{-0.000000in}{0.000000in}}{%
\pgfpathmoveto{\pgfqpoint{-0.000000in}{0.000000in}}%
\pgfpathlineto{\pgfqpoint{-0.048611in}{0.000000in}}%
\pgfusepath{stroke,fill}%
}%
\begin{pgfscope}%
\pgfsys@transformshift{0.623149in}{1.323623in}%
\pgfsys@useobject{currentmarker}{}%
\end{pgfscope}%
\end{pgfscope}%
\begin{pgfscope}%
\definecolor{textcolor}{rgb}{0.000000,0.000000,0.000000}%
\pgfsetstrokecolor{textcolor}%
\pgfsetfillcolor{textcolor}%
\pgftext[x=0.279012in, y=1.275398in, left, base]{\color{textcolor}\rmfamily\fontsize{10.000000}{12.000000}\selectfont \(\displaystyle {0.50}\)}%
\end{pgfscope}%
\begin{pgfscope}%
\pgfsetbuttcap%
\pgfsetroundjoin%
\definecolor{currentfill}{rgb}{0.000000,0.000000,0.000000}%
\pgfsetfillcolor{currentfill}%
\pgfsetlinewidth{0.803000pt}%
\definecolor{currentstroke}{rgb}{0.000000,0.000000,0.000000}%
\pgfsetstrokecolor{currentstroke}%
\pgfsetdash{}{0pt}%
\pgfsys@defobject{currentmarker}{\pgfqpoint{-0.048611in}{0.000000in}}{\pgfqpoint{-0.000000in}{0.000000in}}{%
\pgfpathmoveto{\pgfqpoint{-0.000000in}{0.000000in}}%
\pgfpathlineto{\pgfqpoint{-0.048611in}{0.000000in}}%
\pgfusepath{stroke,fill}%
}%
\begin{pgfscope}%
\pgfsys@transformshift{0.623149in}{1.691123in}%
\pgfsys@useobject{currentmarker}{}%
\end{pgfscope}%
\end{pgfscope}%
\begin{pgfscope}%
\definecolor{textcolor}{rgb}{0.000000,0.000000,0.000000}%
\pgfsetstrokecolor{textcolor}%
\pgfsetfillcolor{textcolor}%
\pgftext[x=0.279012in, y=1.642898in, left, base]{\color{textcolor}\rmfamily\fontsize{10.000000}{12.000000}\selectfont \(\displaystyle {0.75}\)}%
\end{pgfscope}%
\begin{pgfscope}%
\pgfsetbuttcap%
\pgfsetroundjoin%
\definecolor{currentfill}{rgb}{0.000000,0.000000,0.000000}%
\pgfsetfillcolor{currentfill}%
\pgfsetlinewidth{0.803000pt}%
\definecolor{currentstroke}{rgb}{0.000000,0.000000,0.000000}%
\pgfsetstrokecolor{currentstroke}%
\pgfsetdash{}{0pt}%
\pgfsys@defobject{currentmarker}{\pgfqpoint{-0.048611in}{0.000000in}}{\pgfqpoint{-0.000000in}{0.000000in}}{%
\pgfpathmoveto{\pgfqpoint{-0.000000in}{0.000000in}}%
\pgfpathlineto{\pgfqpoint{-0.048611in}{0.000000in}}%
\pgfusepath{stroke,fill}%
}%
\begin{pgfscope}%
\pgfsys@transformshift{0.623149in}{2.058623in}%
\pgfsys@useobject{currentmarker}{}%
\end{pgfscope}%
\end{pgfscope}%
\begin{pgfscope}%
\definecolor{textcolor}{rgb}{0.000000,0.000000,0.000000}%
\pgfsetstrokecolor{textcolor}%
\pgfsetfillcolor{textcolor}%
\pgftext[x=0.279012in, y=2.010398in, left, base]{\color{textcolor}\rmfamily\fontsize{10.000000}{12.000000}\selectfont \(\displaystyle {1.00}\)}%
\end{pgfscope}%
\begin{pgfscope}%
\definecolor{textcolor}{rgb}{0.000000,0.000000,0.000000}%
\pgfsetstrokecolor{textcolor}%
\pgfsetfillcolor{textcolor}%
\pgftext[x=0.223457in,y=1.323623in,,bottom,rotate=90.000000]{\color{textcolor}\rmfamily\fontsize{10.000000}{12.000000}\selectfont SVHN attack success rate}%
\end{pgfscope}%
\begin{pgfscope}%
\pgfpathrectangle{\pgfqpoint{0.623149in}{0.515123in}}{\pgfqpoint{2.247500in}{1.617000in}}%
\pgfusepath{clip}%
\pgfsetbuttcap%
\pgfsetroundjoin%
\pgfsetlinewidth{1.505625pt}%
\definecolor{currentstroke}{rgb}{0.105882,0.619608,0.466667}%
\pgfsetstrokecolor{currentstroke}%
\pgfsetdash{{5.550000pt}{2.400000pt}}{0.000000pt}%
\pgfpathmoveto{\pgfqpoint{0.725308in}{0.869410in}}%
\pgfpathlineto{\pgfqpoint{1.236103in}{1.381432in}}%
\pgfpathlineto{\pgfqpoint{1.746899in}{1.976039in}}%
\pgfpathlineto{\pgfqpoint{2.257694in}{2.058623in}}%
\pgfpathlineto{\pgfqpoint{2.768490in}{2.058623in}}%
\pgfusepath{stroke}%
\end{pgfscope}%
\begin{pgfscope}%
\pgfpathrectangle{\pgfqpoint{0.623149in}{0.515123in}}{\pgfqpoint{2.247500in}{1.617000in}}%
\pgfusepath{clip}%
\pgfsetbuttcap%
\pgfsetroundjoin%
\definecolor{currentfill}{rgb}{0.105882,0.619608,0.466667}%
\pgfsetfillcolor{currentfill}%
\pgfsetlinewidth{1.003750pt}%
\definecolor{currentstroke}{rgb}{0.105882,0.619608,0.466667}%
\pgfsetstrokecolor{currentstroke}%
\pgfsetdash{}{0pt}%
\pgfsys@defobject{currentmarker}{\pgfqpoint{-0.027778in}{-0.027778in}}{\pgfqpoint{0.027778in}{0.027778in}}{%
\pgfpathmoveto{\pgfqpoint{0.000000in}{-0.027778in}}%
\pgfpathcurveto{\pgfqpoint{0.007367in}{-0.027778in}}{\pgfqpoint{0.014433in}{-0.024851in}}{\pgfqpoint{0.019642in}{-0.019642in}}%
\pgfpathcurveto{\pgfqpoint{0.024851in}{-0.014433in}}{\pgfqpoint{0.027778in}{-0.007367in}}{\pgfqpoint{0.027778in}{0.000000in}}%
\pgfpathcurveto{\pgfqpoint{0.027778in}{0.007367in}}{\pgfqpoint{0.024851in}{0.014433in}}{\pgfqpoint{0.019642in}{0.019642in}}%
\pgfpathcurveto{\pgfqpoint{0.014433in}{0.024851in}}{\pgfqpoint{0.007367in}{0.027778in}}{\pgfqpoint{0.000000in}{0.027778in}}%
\pgfpathcurveto{\pgfqpoint{-0.007367in}{0.027778in}}{\pgfqpoint{-0.014433in}{0.024851in}}{\pgfqpoint{-0.019642in}{0.019642in}}%
\pgfpathcurveto{\pgfqpoint{-0.024851in}{0.014433in}}{\pgfqpoint{-0.027778in}{0.007367in}}{\pgfqpoint{-0.027778in}{0.000000in}}%
\pgfpathcurveto{\pgfqpoint{-0.027778in}{-0.007367in}}{\pgfqpoint{-0.024851in}{-0.014433in}}{\pgfqpoint{-0.019642in}{-0.019642in}}%
\pgfpathcurveto{\pgfqpoint{-0.014433in}{-0.024851in}}{\pgfqpoint{-0.007367in}{-0.027778in}}{\pgfqpoint{0.000000in}{-0.027778in}}%
\pgfpathlineto{\pgfqpoint{0.000000in}{-0.027778in}}%
\pgfpathclose%
\pgfusepath{stroke,fill}%
}%
\begin{pgfscope}%
\pgfsys@transformshift{0.725308in}{0.869410in}%
\pgfsys@useobject{currentmarker}{}%
\end{pgfscope}%
\begin{pgfscope}%
\pgfsys@transformshift{1.236103in}{1.381432in}%
\pgfsys@useobject{currentmarker}{}%
\end{pgfscope}%
\begin{pgfscope}%
\pgfsys@transformshift{1.746899in}{1.976039in}%
\pgfsys@useobject{currentmarker}{}%
\end{pgfscope}%
\begin{pgfscope}%
\pgfsys@transformshift{2.257694in}{2.058623in}%
\pgfsys@useobject{currentmarker}{}%
\end{pgfscope}%
\begin{pgfscope}%
\pgfsys@transformshift{2.768490in}{2.058623in}%
\pgfsys@useobject{currentmarker}{}%
\end{pgfscope}%
\end{pgfscope}%
\begin{pgfscope}%
\pgfpathrectangle{\pgfqpoint{0.623149in}{0.515123in}}{\pgfqpoint{2.247500in}{1.617000in}}%
\pgfusepath{clip}%
\pgfsetbuttcap%
\pgfsetroundjoin%
\pgfsetlinewidth{1.505625pt}%
\definecolor{currentstroke}{rgb}{0.850980,0.372549,0.007843}%
\pgfsetstrokecolor{currentstroke}%
\pgfsetdash{{5.550000pt}{2.400000pt}}{0.000000pt}%
\pgfpathmoveto{\pgfqpoint{0.725308in}{0.605140in}}%
\pgfpathlineto{\pgfqpoint{1.236103in}{0.638174in}}%
\pgfpathlineto{\pgfqpoint{1.746899in}{0.770309in}}%
\pgfpathlineto{\pgfqpoint{2.257694in}{0.935477in}}%
\pgfpathlineto{\pgfqpoint{2.768490in}{1.777837in}}%
\pgfusepath{stroke}%
\end{pgfscope}%
\begin{pgfscope}%
\pgfpathrectangle{\pgfqpoint{0.623149in}{0.515123in}}{\pgfqpoint{2.247500in}{1.617000in}}%
\pgfusepath{clip}%
\pgfsetbuttcap%
\pgfsetroundjoin%
\definecolor{currentfill}{rgb}{0.850980,0.372549,0.007843}%
\pgfsetfillcolor{currentfill}%
\pgfsetlinewidth{1.003750pt}%
\definecolor{currentstroke}{rgb}{0.850980,0.372549,0.007843}%
\pgfsetstrokecolor{currentstroke}%
\pgfsetdash{}{0pt}%
\pgfsys@defobject{currentmarker}{\pgfqpoint{-0.027778in}{-0.027778in}}{\pgfqpoint{0.027778in}{0.027778in}}{%
\pgfpathmoveto{\pgfqpoint{0.000000in}{-0.027778in}}%
\pgfpathcurveto{\pgfqpoint{0.007367in}{-0.027778in}}{\pgfqpoint{0.014433in}{-0.024851in}}{\pgfqpoint{0.019642in}{-0.019642in}}%
\pgfpathcurveto{\pgfqpoint{0.024851in}{-0.014433in}}{\pgfqpoint{0.027778in}{-0.007367in}}{\pgfqpoint{0.027778in}{0.000000in}}%
\pgfpathcurveto{\pgfqpoint{0.027778in}{0.007367in}}{\pgfqpoint{0.024851in}{0.014433in}}{\pgfqpoint{0.019642in}{0.019642in}}%
\pgfpathcurveto{\pgfqpoint{0.014433in}{0.024851in}}{\pgfqpoint{0.007367in}{0.027778in}}{\pgfqpoint{0.000000in}{0.027778in}}%
\pgfpathcurveto{\pgfqpoint{-0.007367in}{0.027778in}}{\pgfqpoint{-0.014433in}{0.024851in}}{\pgfqpoint{-0.019642in}{0.019642in}}%
\pgfpathcurveto{\pgfqpoint{-0.024851in}{0.014433in}}{\pgfqpoint{-0.027778in}{0.007367in}}{\pgfqpoint{-0.027778in}{0.000000in}}%
\pgfpathcurveto{\pgfqpoint{-0.027778in}{-0.007367in}}{\pgfqpoint{-0.024851in}{-0.014433in}}{\pgfqpoint{-0.019642in}{-0.019642in}}%
\pgfpathcurveto{\pgfqpoint{-0.014433in}{-0.024851in}}{\pgfqpoint{-0.007367in}{-0.027778in}}{\pgfqpoint{0.000000in}{-0.027778in}}%
\pgfpathlineto{\pgfqpoint{0.000000in}{-0.027778in}}%
\pgfpathclose%
\pgfusepath{stroke,fill}%
}%
\begin{pgfscope}%
\pgfsys@transformshift{0.725308in}{0.605140in}%
\pgfsys@useobject{currentmarker}{}%
\end{pgfscope}%
\begin{pgfscope}%
\pgfsys@transformshift{1.236103in}{0.638174in}%
\pgfsys@useobject{currentmarker}{}%
\end{pgfscope}%
\begin{pgfscope}%
\pgfsys@transformshift{1.746899in}{0.770309in}%
\pgfsys@useobject{currentmarker}{}%
\end{pgfscope}%
\begin{pgfscope}%
\pgfsys@transformshift{2.257694in}{0.935477in}%
\pgfsys@useobject{currentmarker}{}%
\end{pgfscope}%
\begin{pgfscope}%
\pgfsys@transformshift{2.768490in}{1.777837in}%
\pgfsys@useobject{currentmarker}{}%
\end{pgfscope}%
\end{pgfscope}%
\begin{pgfscope}%
\pgfpathrectangle{\pgfqpoint{0.623149in}{0.515123in}}{\pgfqpoint{2.247500in}{1.617000in}}%
\pgfusepath{clip}%
\pgfsetbuttcap%
\pgfsetroundjoin%
\pgfsetlinewidth{1.505625pt}%
\definecolor{currentstroke}{rgb}{0.458824,0.439216,0.701961}%
\pgfsetstrokecolor{currentstroke}%
\pgfsetdash{{5.550000pt}{2.400000pt}}{0.000000pt}%
\pgfpathmoveto{\pgfqpoint{0.725308in}{0.588623in}}%
\pgfpathlineto{\pgfqpoint{1.236103in}{0.588623in}}%
\pgfpathlineto{\pgfqpoint{1.746899in}{0.588623in}}%
\pgfpathlineto{\pgfqpoint{2.257694in}{0.638174in}}%
\pgfpathlineto{\pgfqpoint{2.768490in}{0.770309in}}%
\pgfusepath{stroke}%
\end{pgfscope}%
\begin{pgfscope}%
\pgfpathrectangle{\pgfqpoint{0.623149in}{0.515123in}}{\pgfqpoint{2.247500in}{1.617000in}}%
\pgfusepath{clip}%
\pgfsetbuttcap%
\pgfsetroundjoin%
\definecolor{currentfill}{rgb}{0.458824,0.439216,0.701961}%
\pgfsetfillcolor{currentfill}%
\pgfsetlinewidth{1.003750pt}%
\definecolor{currentstroke}{rgb}{0.458824,0.439216,0.701961}%
\pgfsetstrokecolor{currentstroke}%
\pgfsetdash{}{0pt}%
\pgfsys@defobject{currentmarker}{\pgfqpoint{-0.027778in}{-0.027778in}}{\pgfqpoint{0.027778in}{0.027778in}}{%
\pgfpathmoveto{\pgfqpoint{0.000000in}{-0.027778in}}%
\pgfpathcurveto{\pgfqpoint{0.007367in}{-0.027778in}}{\pgfqpoint{0.014433in}{-0.024851in}}{\pgfqpoint{0.019642in}{-0.019642in}}%
\pgfpathcurveto{\pgfqpoint{0.024851in}{-0.014433in}}{\pgfqpoint{0.027778in}{-0.007367in}}{\pgfqpoint{0.027778in}{0.000000in}}%
\pgfpathcurveto{\pgfqpoint{0.027778in}{0.007367in}}{\pgfqpoint{0.024851in}{0.014433in}}{\pgfqpoint{0.019642in}{0.019642in}}%
\pgfpathcurveto{\pgfqpoint{0.014433in}{0.024851in}}{\pgfqpoint{0.007367in}{0.027778in}}{\pgfqpoint{0.000000in}{0.027778in}}%
\pgfpathcurveto{\pgfqpoint{-0.007367in}{0.027778in}}{\pgfqpoint{-0.014433in}{0.024851in}}{\pgfqpoint{-0.019642in}{0.019642in}}%
\pgfpathcurveto{\pgfqpoint{-0.024851in}{0.014433in}}{\pgfqpoint{-0.027778in}{0.007367in}}{\pgfqpoint{-0.027778in}{0.000000in}}%
\pgfpathcurveto{\pgfqpoint{-0.027778in}{-0.007367in}}{\pgfqpoint{-0.024851in}{-0.014433in}}{\pgfqpoint{-0.019642in}{-0.019642in}}%
\pgfpathcurveto{\pgfqpoint{-0.014433in}{-0.024851in}}{\pgfqpoint{-0.007367in}{-0.027778in}}{\pgfqpoint{0.000000in}{-0.027778in}}%
\pgfpathlineto{\pgfqpoint{0.000000in}{-0.027778in}}%
\pgfpathclose%
\pgfusepath{stroke,fill}%
}%
\begin{pgfscope}%
\pgfsys@transformshift{0.725308in}{0.588623in}%
\pgfsys@useobject{currentmarker}{}%
\end{pgfscope}%
\begin{pgfscope}%
\pgfsys@transformshift{1.236103in}{0.588623in}%
\pgfsys@useobject{currentmarker}{}%
\end{pgfscope}%
\begin{pgfscope}%
\pgfsys@transformshift{1.746899in}{0.588623in}%
\pgfsys@useobject{currentmarker}{}%
\end{pgfscope}%
\begin{pgfscope}%
\pgfsys@transformshift{2.257694in}{0.638174in}%
\pgfsys@useobject{currentmarker}{}%
\end{pgfscope}%
\begin{pgfscope}%
\pgfsys@transformshift{2.768490in}{0.770309in}%
\pgfsys@useobject{currentmarker}{}%
\end{pgfscope}%
\end{pgfscope}%
\begin{pgfscope}%
\pgfpathrectangle{\pgfqpoint{0.623149in}{0.515123in}}{\pgfqpoint{2.247500in}{1.617000in}}%
\pgfusepath{clip}%
\pgfsetbuttcap%
\pgfsetroundjoin%
\pgfsetlinewidth{1.505625pt}%
\definecolor{currentstroke}{rgb}{0.905882,0.160784,0.541176}%
\pgfsetstrokecolor{currentstroke}%
\pgfsetdash{{5.550000pt}{2.400000pt}}{0.000000pt}%
\pgfpathmoveto{\pgfqpoint{0.725308in}{0.588623in}}%
\pgfpathlineto{\pgfqpoint{1.236103in}{0.588623in}}%
\pgfpathlineto{\pgfqpoint{1.746899in}{0.588623in}}%
\pgfpathlineto{\pgfqpoint{2.257694in}{0.621657in}}%
\pgfpathlineto{\pgfqpoint{2.768490in}{0.621657in}}%
\pgfusepath{stroke}%
\end{pgfscope}%
\begin{pgfscope}%
\pgfpathrectangle{\pgfqpoint{0.623149in}{0.515123in}}{\pgfqpoint{2.247500in}{1.617000in}}%
\pgfusepath{clip}%
\pgfsetbuttcap%
\pgfsetroundjoin%
\definecolor{currentfill}{rgb}{0.905882,0.160784,0.541176}%
\pgfsetfillcolor{currentfill}%
\pgfsetlinewidth{1.003750pt}%
\definecolor{currentstroke}{rgb}{0.905882,0.160784,0.541176}%
\pgfsetstrokecolor{currentstroke}%
\pgfsetdash{}{0pt}%
\pgfsys@defobject{currentmarker}{\pgfqpoint{-0.027778in}{-0.027778in}}{\pgfqpoint{0.027778in}{0.027778in}}{%
\pgfpathmoveto{\pgfqpoint{0.000000in}{-0.027778in}}%
\pgfpathcurveto{\pgfqpoint{0.007367in}{-0.027778in}}{\pgfqpoint{0.014433in}{-0.024851in}}{\pgfqpoint{0.019642in}{-0.019642in}}%
\pgfpathcurveto{\pgfqpoint{0.024851in}{-0.014433in}}{\pgfqpoint{0.027778in}{-0.007367in}}{\pgfqpoint{0.027778in}{0.000000in}}%
\pgfpathcurveto{\pgfqpoint{0.027778in}{0.007367in}}{\pgfqpoint{0.024851in}{0.014433in}}{\pgfqpoint{0.019642in}{0.019642in}}%
\pgfpathcurveto{\pgfqpoint{0.014433in}{0.024851in}}{\pgfqpoint{0.007367in}{0.027778in}}{\pgfqpoint{0.000000in}{0.027778in}}%
\pgfpathcurveto{\pgfqpoint{-0.007367in}{0.027778in}}{\pgfqpoint{-0.014433in}{0.024851in}}{\pgfqpoint{-0.019642in}{0.019642in}}%
\pgfpathcurveto{\pgfqpoint{-0.024851in}{0.014433in}}{\pgfqpoint{-0.027778in}{0.007367in}}{\pgfqpoint{-0.027778in}{0.000000in}}%
\pgfpathcurveto{\pgfqpoint{-0.027778in}{-0.007367in}}{\pgfqpoint{-0.024851in}{-0.014433in}}{\pgfqpoint{-0.019642in}{-0.019642in}}%
\pgfpathcurveto{\pgfqpoint{-0.014433in}{-0.024851in}}{\pgfqpoint{-0.007367in}{-0.027778in}}{\pgfqpoint{0.000000in}{-0.027778in}}%
\pgfpathlineto{\pgfqpoint{0.000000in}{-0.027778in}}%
\pgfpathclose%
\pgfusepath{stroke,fill}%
}%
\begin{pgfscope}%
\pgfsys@transformshift{0.725308in}{0.588623in}%
\pgfsys@useobject{currentmarker}{}%
\end{pgfscope}%
\begin{pgfscope}%
\pgfsys@transformshift{1.236103in}{0.588623in}%
\pgfsys@useobject{currentmarker}{}%
\end{pgfscope}%
\begin{pgfscope}%
\pgfsys@transformshift{1.746899in}{0.588623in}%
\pgfsys@useobject{currentmarker}{}%
\end{pgfscope}%
\begin{pgfscope}%
\pgfsys@transformshift{2.257694in}{0.621657in}%
\pgfsys@useobject{currentmarker}{}%
\end{pgfscope}%
\begin{pgfscope}%
\pgfsys@transformshift{2.768490in}{0.621657in}%
\pgfsys@useobject{currentmarker}{}%
\end{pgfscope}%
\end{pgfscope}%
\begin{pgfscope}%
\pgfsetrectcap%
\pgfsetmiterjoin%
\pgfsetlinewidth{0.803000pt}%
\definecolor{currentstroke}{rgb}{0.000000,0.000000,0.000000}%
\pgfsetstrokecolor{currentstroke}%
\pgfsetdash{}{0pt}%
\pgfpathmoveto{\pgfqpoint{0.623149in}{0.515123in}}%
\pgfpathlineto{\pgfqpoint{0.623149in}{2.132123in}}%
\pgfusepath{stroke}%
\end{pgfscope}%
\begin{pgfscope}%
\pgfsetrectcap%
\pgfsetmiterjoin%
\pgfsetlinewidth{0.803000pt}%
\definecolor{currentstroke}{rgb}{0.000000,0.000000,0.000000}%
\pgfsetstrokecolor{currentstroke}%
\pgfsetdash{}{0pt}%
\pgfpathmoveto{\pgfqpoint{2.870649in}{0.515123in}}%
\pgfpathlineto{\pgfqpoint{2.870649in}{2.132123in}}%
\pgfusepath{stroke}%
\end{pgfscope}%
\begin{pgfscope}%
\pgfsetrectcap%
\pgfsetmiterjoin%
\pgfsetlinewidth{0.803000pt}%
\definecolor{currentstroke}{rgb}{0.000000,0.000000,0.000000}%
\pgfsetstrokecolor{currentstroke}%
\pgfsetdash{}{0pt}%
\pgfpathmoveto{\pgfqpoint{0.623149in}{0.515123in}}%
\pgfpathlineto{\pgfqpoint{2.870649in}{0.515123in}}%
\pgfusepath{stroke}%
\end{pgfscope}%
\begin{pgfscope}%
\pgfsetrectcap%
\pgfsetmiterjoin%
\pgfsetlinewidth{0.803000pt}%
\definecolor{currentstroke}{rgb}{0.000000,0.000000,0.000000}%
\pgfsetstrokecolor{currentstroke}%
\pgfsetdash{}{0pt}%
\pgfpathmoveto{\pgfqpoint{0.623149in}{2.132123in}}%
\pgfpathlineto{\pgfqpoint{2.870649in}{2.132123in}}%
\pgfusepath{stroke}%
\end{pgfscope}%
\begin{pgfscope}%
\pgfsetbuttcap%
\pgfsetmiterjoin%
\definecolor{currentfill}{rgb}{1.000000,1.000000,1.000000}%
\pgfsetfillcolor{currentfill}%
\pgfsetfillopacity{0.800000}%
\pgfsetlinewidth{1.003750pt}%
\definecolor{currentstroke}{rgb}{0.800000,0.800000,0.800000}%
\pgfsetstrokecolor{currentstroke}%
\pgfsetstrokeopacity{0.800000}%
\pgfsetdash{}{0pt}%
\pgfpathmoveto{\pgfqpoint{0.720371in}{1.246321in}}%
\pgfpathlineto{\pgfqpoint{2.125302in}{1.246321in}}%
\pgfpathquadraticcurveto{\pgfqpoint{2.153080in}{1.246321in}}{\pgfqpoint{2.153080in}{1.274099in}}%
\pgfpathlineto{\pgfqpoint{2.153080in}{2.034901in}}%
\pgfpathquadraticcurveto{\pgfqpoint{2.153080in}{2.062679in}}{\pgfqpoint{2.125302in}{2.062679in}}%
\pgfpathlineto{\pgfqpoint{0.720371in}{2.062679in}}%
\pgfpathquadraticcurveto{\pgfqpoint{0.692593in}{2.062679in}}{\pgfqpoint{0.692593in}{2.034901in}}%
\pgfpathlineto{\pgfqpoint{0.692593in}{1.274099in}}%
\pgfpathquadraticcurveto{\pgfqpoint{0.692593in}{1.246321in}}{\pgfqpoint{0.720371in}{1.246321in}}%
\pgfpathlineto{\pgfqpoint{0.720371in}{1.246321in}}%
\pgfpathclose%
\pgfusepath{stroke,fill}%
\end{pgfscope}%
\begin{pgfscope}%
\pgfsetbuttcap%
\pgfsetroundjoin%
\pgfsetlinewidth{1.505625pt}%
\definecolor{currentstroke}{rgb}{0.105882,0.619608,0.466667}%
\pgfsetstrokecolor{currentstroke}%
\pgfsetdash{{5.550000pt}{2.400000pt}}{0.000000pt}%
\pgfpathmoveto{\pgfqpoint{0.748149in}{1.958512in}}%
\pgfpathlineto{\pgfqpoint{0.887038in}{1.958512in}}%
\pgfpathlineto{\pgfqpoint{1.025927in}{1.958512in}}%
\pgfusepath{stroke}%
\end{pgfscope}%
\begin{pgfscope}%
\pgfsetbuttcap%
\pgfsetroundjoin%
\definecolor{currentfill}{rgb}{0.105882,0.619608,0.466667}%
\pgfsetfillcolor{currentfill}%
\pgfsetlinewidth{1.003750pt}%
\definecolor{currentstroke}{rgb}{0.105882,0.619608,0.466667}%
\pgfsetstrokecolor{currentstroke}%
\pgfsetdash{}{0pt}%
\pgfsys@defobject{currentmarker}{\pgfqpoint{-0.027778in}{-0.027778in}}{\pgfqpoint{0.027778in}{0.027778in}}{%
\pgfpathmoveto{\pgfqpoint{0.000000in}{-0.027778in}}%
\pgfpathcurveto{\pgfqpoint{0.007367in}{-0.027778in}}{\pgfqpoint{0.014433in}{-0.024851in}}{\pgfqpoint{0.019642in}{-0.019642in}}%
\pgfpathcurveto{\pgfqpoint{0.024851in}{-0.014433in}}{\pgfqpoint{0.027778in}{-0.007367in}}{\pgfqpoint{0.027778in}{0.000000in}}%
\pgfpathcurveto{\pgfqpoint{0.027778in}{0.007367in}}{\pgfqpoint{0.024851in}{0.014433in}}{\pgfqpoint{0.019642in}{0.019642in}}%
\pgfpathcurveto{\pgfqpoint{0.014433in}{0.024851in}}{\pgfqpoint{0.007367in}{0.027778in}}{\pgfqpoint{0.000000in}{0.027778in}}%
\pgfpathcurveto{\pgfqpoint{-0.007367in}{0.027778in}}{\pgfqpoint{-0.014433in}{0.024851in}}{\pgfqpoint{-0.019642in}{0.019642in}}%
\pgfpathcurveto{\pgfqpoint{-0.024851in}{0.014433in}}{\pgfqpoint{-0.027778in}{0.007367in}}{\pgfqpoint{-0.027778in}{0.000000in}}%
\pgfpathcurveto{\pgfqpoint{-0.027778in}{-0.007367in}}{\pgfqpoint{-0.024851in}{-0.014433in}}{\pgfqpoint{-0.019642in}{-0.019642in}}%
\pgfpathcurveto{\pgfqpoint{-0.014433in}{-0.024851in}}{\pgfqpoint{-0.007367in}{-0.027778in}}{\pgfqpoint{0.000000in}{-0.027778in}}%
\pgfpathlineto{\pgfqpoint{0.000000in}{-0.027778in}}%
\pgfpathclose%
\pgfusepath{stroke,fill}%
}%
\begin{pgfscope}%
\pgfsys@transformshift{0.887038in}{1.958512in}%
\pgfsys@useobject{currentmarker}{}%
\end{pgfscope}%
\end{pgfscope}%
\begin{pgfscope}%
\definecolor{textcolor}{rgb}{0.000000,0.000000,0.000000}%
\pgfsetstrokecolor{textcolor}%
\pgfsetfillcolor{textcolor}%
\pgftext[x=1.137038in,y=1.909901in,left,base]{\color{textcolor}\rmfamily\fontsize{10.000000}{12.000000}\selectfont \(\displaystyle \textsc{PGD}\)}%
\end{pgfscope}%
\begin{pgfscope}%
\pgfsetbuttcap%
\pgfsetroundjoin%
\pgfsetlinewidth{1.505625pt}%
\definecolor{currentstroke}{rgb}{0.850980,0.372549,0.007843}%
\pgfsetstrokecolor{currentstroke}%
\pgfsetdash{{5.550000pt}{2.400000pt}}{0.000000pt}%
\pgfpathmoveto{\pgfqpoint{0.748149in}{1.764839in}}%
\pgfpathlineto{\pgfqpoint{0.887038in}{1.764839in}}%
\pgfpathlineto{\pgfqpoint{1.025927in}{1.764839in}}%
\pgfusepath{stroke}%
\end{pgfscope}%
\begin{pgfscope}%
\pgfsetbuttcap%
\pgfsetroundjoin%
\definecolor{currentfill}{rgb}{0.850980,0.372549,0.007843}%
\pgfsetfillcolor{currentfill}%
\pgfsetlinewidth{1.003750pt}%
\definecolor{currentstroke}{rgb}{0.850980,0.372549,0.007843}%
\pgfsetstrokecolor{currentstroke}%
\pgfsetdash{}{0pt}%
\pgfsys@defobject{currentmarker}{\pgfqpoint{-0.027778in}{-0.027778in}}{\pgfqpoint{0.027778in}{0.027778in}}{%
\pgfpathmoveto{\pgfqpoint{0.000000in}{-0.027778in}}%
\pgfpathcurveto{\pgfqpoint{0.007367in}{-0.027778in}}{\pgfqpoint{0.014433in}{-0.024851in}}{\pgfqpoint{0.019642in}{-0.019642in}}%
\pgfpathcurveto{\pgfqpoint{0.024851in}{-0.014433in}}{\pgfqpoint{0.027778in}{-0.007367in}}{\pgfqpoint{0.027778in}{0.000000in}}%
\pgfpathcurveto{\pgfqpoint{0.027778in}{0.007367in}}{\pgfqpoint{0.024851in}{0.014433in}}{\pgfqpoint{0.019642in}{0.019642in}}%
\pgfpathcurveto{\pgfqpoint{0.014433in}{0.024851in}}{\pgfqpoint{0.007367in}{0.027778in}}{\pgfqpoint{0.000000in}{0.027778in}}%
\pgfpathcurveto{\pgfqpoint{-0.007367in}{0.027778in}}{\pgfqpoint{-0.014433in}{0.024851in}}{\pgfqpoint{-0.019642in}{0.019642in}}%
\pgfpathcurveto{\pgfqpoint{-0.024851in}{0.014433in}}{\pgfqpoint{-0.027778in}{0.007367in}}{\pgfqpoint{-0.027778in}{0.000000in}}%
\pgfpathcurveto{\pgfqpoint{-0.027778in}{-0.007367in}}{\pgfqpoint{-0.024851in}{-0.014433in}}{\pgfqpoint{-0.019642in}{-0.019642in}}%
\pgfpathcurveto{\pgfqpoint{-0.014433in}{-0.024851in}}{\pgfqpoint{-0.007367in}{-0.027778in}}{\pgfqpoint{0.000000in}{-0.027778in}}%
\pgfpathlineto{\pgfqpoint{0.000000in}{-0.027778in}}%
\pgfpathclose%
\pgfusepath{stroke,fill}%
}%
\begin{pgfscope}%
\pgfsys@transformshift{0.887038in}{1.764839in}%
\pgfsys@useobject{currentmarker}{}%
\end{pgfscope}%
\end{pgfscope}%
\begin{pgfscope}%
\definecolor{textcolor}{rgb}{0.000000,0.000000,0.000000}%
\pgfsetstrokecolor{textcolor}%
\pgfsetfillcolor{textcolor}%
\pgftext[x=1.137038in,y=1.716228in,left,base]{\color{textcolor}\rmfamily\fontsize{10.000000}{12.000000}\selectfont \(\displaystyle \textsc{SubspacePGD}\)}%
\end{pgfscope}%
\begin{pgfscope}%
\pgfsetbuttcap%
\pgfsetroundjoin%
\pgfsetlinewidth{1.505625pt}%
\definecolor{currentstroke}{rgb}{0.458824,0.439216,0.701961}%
\pgfsetstrokecolor{currentstroke}%
\pgfsetdash{{5.550000pt}{2.400000pt}}{0.000000pt}%
\pgfpathmoveto{\pgfqpoint{0.748149in}{1.571167in}}%
\pgfpathlineto{\pgfqpoint{0.887038in}{1.571167in}}%
\pgfpathlineto{\pgfqpoint{1.025927in}{1.571167in}}%
\pgfusepath{stroke}%
\end{pgfscope}%
\begin{pgfscope}%
\pgfsetbuttcap%
\pgfsetroundjoin%
\definecolor{currentfill}{rgb}{0.458824,0.439216,0.701961}%
\pgfsetfillcolor{currentfill}%
\pgfsetlinewidth{1.003750pt}%
\definecolor{currentstroke}{rgb}{0.458824,0.439216,0.701961}%
\pgfsetstrokecolor{currentstroke}%
\pgfsetdash{}{0pt}%
\pgfsys@defobject{currentmarker}{\pgfqpoint{-0.027778in}{-0.027778in}}{\pgfqpoint{0.027778in}{0.027778in}}{%
\pgfpathmoveto{\pgfqpoint{0.000000in}{-0.027778in}}%
\pgfpathcurveto{\pgfqpoint{0.007367in}{-0.027778in}}{\pgfqpoint{0.014433in}{-0.024851in}}{\pgfqpoint{0.019642in}{-0.019642in}}%
\pgfpathcurveto{\pgfqpoint{0.024851in}{-0.014433in}}{\pgfqpoint{0.027778in}{-0.007367in}}{\pgfqpoint{0.027778in}{0.000000in}}%
\pgfpathcurveto{\pgfqpoint{0.027778in}{0.007367in}}{\pgfqpoint{0.024851in}{0.014433in}}{\pgfqpoint{0.019642in}{0.019642in}}%
\pgfpathcurveto{\pgfqpoint{0.014433in}{0.024851in}}{\pgfqpoint{0.007367in}{0.027778in}}{\pgfqpoint{0.000000in}{0.027778in}}%
\pgfpathcurveto{\pgfqpoint{-0.007367in}{0.027778in}}{\pgfqpoint{-0.014433in}{0.024851in}}{\pgfqpoint{-0.019642in}{0.019642in}}%
\pgfpathcurveto{\pgfqpoint{-0.024851in}{0.014433in}}{\pgfqpoint{-0.027778in}{0.007367in}}{\pgfqpoint{-0.027778in}{0.000000in}}%
\pgfpathcurveto{\pgfqpoint{-0.027778in}{-0.007367in}}{\pgfqpoint{-0.024851in}{-0.014433in}}{\pgfqpoint{-0.019642in}{-0.019642in}}%
\pgfpathcurveto{\pgfqpoint{-0.014433in}{-0.024851in}}{\pgfqpoint{-0.007367in}{-0.027778in}}{\pgfqpoint{0.000000in}{-0.027778in}}%
\pgfpathlineto{\pgfqpoint{0.000000in}{-0.027778in}}%
\pgfpathclose%
\pgfusepath{stroke,fill}%
}%
\begin{pgfscope}%
\pgfsys@transformshift{0.887038in}{1.571167in}%
\pgfsys@useobject{currentmarker}{}%
\end{pgfscope}%
\end{pgfscope}%
\begin{pgfscope}%
\definecolor{textcolor}{rgb}{0.000000,0.000000,0.000000}%
\pgfsetstrokecolor{textcolor}%
\pgfsetfillcolor{textcolor}%
\pgftext[x=1.137038in,y=1.522556in,left,base]{\color{textcolor}\rmfamily\fontsize{10.000000}{12.000000}\selectfont \(\displaystyle \textsc{RandMax}\)}%
\end{pgfscope}%
\begin{pgfscope}%
\pgfsetbuttcap%
\pgfsetroundjoin%
\pgfsetlinewidth{1.505625pt}%
\definecolor{currentstroke}{rgb}{0.905882,0.160784,0.541176}%
\pgfsetstrokecolor{currentstroke}%
\pgfsetdash{{5.550000pt}{2.400000pt}}{0.000000pt}%
\pgfpathmoveto{\pgfqpoint{0.748149in}{1.377494in}}%
\pgfpathlineto{\pgfqpoint{0.887038in}{1.377494in}}%
\pgfpathlineto{\pgfqpoint{1.025927in}{1.377494in}}%
\pgfusepath{stroke}%
\end{pgfscope}%
\begin{pgfscope}%
\pgfsetbuttcap%
\pgfsetroundjoin%
\definecolor{currentfill}{rgb}{0.905882,0.160784,0.541176}%
\pgfsetfillcolor{currentfill}%
\pgfsetlinewidth{1.003750pt}%
\definecolor{currentstroke}{rgb}{0.905882,0.160784,0.541176}%
\pgfsetstrokecolor{currentstroke}%
\pgfsetdash{}{0pt}%
\pgfsys@defobject{currentmarker}{\pgfqpoint{-0.027778in}{-0.027778in}}{\pgfqpoint{0.027778in}{0.027778in}}{%
\pgfpathmoveto{\pgfqpoint{0.000000in}{-0.027778in}}%
\pgfpathcurveto{\pgfqpoint{0.007367in}{-0.027778in}}{\pgfqpoint{0.014433in}{-0.024851in}}{\pgfqpoint{0.019642in}{-0.019642in}}%
\pgfpathcurveto{\pgfqpoint{0.024851in}{-0.014433in}}{\pgfqpoint{0.027778in}{-0.007367in}}{\pgfqpoint{0.027778in}{0.000000in}}%
\pgfpathcurveto{\pgfqpoint{0.027778in}{0.007367in}}{\pgfqpoint{0.024851in}{0.014433in}}{\pgfqpoint{0.019642in}{0.019642in}}%
\pgfpathcurveto{\pgfqpoint{0.014433in}{0.024851in}}{\pgfqpoint{0.007367in}{0.027778in}}{\pgfqpoint{0.000000in}{0.027778in}}%
\pgfpathcurveto{\pgfqpoint{-0.007367in}{0.027778in}}{\pgfqpoint{-0.014433in}{0.024851in}}{\pgfqpoint{-0.019642in}{0.019642in}}%
\pgfpathcurveto{\pgfqpoint{-0.024851in}{0.014433in}}{\pgfqpoint{-0.027778in}{0.007367in}}{\pgfqpoint{-0.027778in}{0.000000in}}%
\pgfpathcurveto{\pgfqpoint{-0.027778in}{-0.007367in}}{\pgfqpoint{-0.024851in}{-0.014433in}}{\pgfqpoint{-0.019642in}{-0.019642in}}%
\pgfpathcurveto{\pgfqpoint{-0.014433in}{-0.024851in}}{\pgfqpoint{-0.007367in}{-0.027778in}}{\pgfqpoint{0.000000in}{-0.027778in}}%
\pgfpathlineto{\pgfqpoint{0.000000in}{-0.027778in}}%
\pgfpathclose%
\pgfusepath{stroke,fill}%
}%
\begin{pgfscope}%
\pgfsys@transformshift{0.887038in}{1.377494in}%
\pgfsys@useobject{currentmarker}{}%
\end{pgfscope}%
\end{pgfscope}%
\begin{pgfscope}%
\definecolor{textcolor}{rgb}{0.000000,0.000000,0.000000}%
\pgfsetstrokecolor{textcolor}%
\pgfsetfillcolor{textcolor}%
\pgftext[x=1.137038in,y=1.328883in,left,base]{\color{textcolor}\rmfamily\fontsize{10.000000}{12.000000}\selectfont \(\displaystyle \textsc{RandUniform}\)}%
\end{pgfscope}%
\end{pgfpicture}%
\makeatother%
\endgroup%

%% file: res/svhn_examples.tex
\begin{tikzpicture}
\node [anchor=west, rotate=90, align=left, anchor=south west] (note1) at (0,3.4) {\small PGD};
\node [anchor=west, rotate=90, align=left, anchor=south west] (note2) at (0,2.32) {\small Sub.\\\small PGD};
\node [anchor=west, rotate=90, align=left, anchor=south west] (note3) at (0,1.2) {\small Rnd.\\\small Max};
\node [anchor=west, rotate=90, align=left, anchor=south west] (note4) at (0,0.2) {\small Rnd.\\\small Uni.};
\node [anchor=west, rotate=90, align=left, anchor=south west] (note5) at (0,4.5) {\small Orig};

%\begin{scope}[xshift=0cm]
\node[anchor=south west,inner sep=0] (image) at (0,0) {\includegraphics[width=0.8\textwidth]{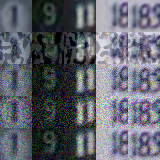}};
%\end{scope}

%\node [anchor=west, rotate=0, align=center, anchor=south] (class1) at (0.5,5.8) {\small boat};
%\node [anchor=west, rotate=0, align=center, anchor=south] (class2) at (1.7,5.8) {\small car};
%\node [anchor=west, rotate=0, align=center, anchor=south] (class3) at (2.9,5.72) {\small frog};
%\node [anchor=west, rotate=0, align=center, anchor=south] (class4) at (4.0,5.8) {\small cat};
%\node [anchor=west, rotate=0, align=center, anchor=south] (class5) at (5.1,5.8) {\small car};

%\node [anchor=west, rotate=0, align=center, anchor=south] (class1) at (0.5,4.8) {\small boat};
%\node [anchor=west, rotate=0, align=center, anchor=south] (class2) at (1.5,4.8) {\small car};
%\node [anchor=west, rotate=0, align=center, anchor=south] (class3) at (2.5,4.72) {\small frog};
%\node [anchor=west, rotate=0, align=center, anchor=south] (class4) at (3.5,4.8) {\small cat};
%\node [anchor=west, rotate=0, align=center, anchor=south] (class5) at (4.5,4.8) {\small car};
\end{tikzpicture}%

%% file: figs/cifar10_main.pgf
%% Creator: Matplotlib, PGF backend
%%
%% To include the figure in your LaTeX document, write
%%   \input{<filename>.pgf}
%%
%% Make sure the required packages are loaded in your preamble
%%   \usepackage{pgf}
%%
%% Also ensure that all the required font packages are loaded; for instance,
%% the lmodern package is sometimes necessary when using math font.
%%   \usepackage{lmodern}
%%
%% Figures using additional raster images can only be included by \input if
%% they are in the same directory as the main LaTeX file. For loading figures
%% from other directories you can use the `import` package
%%   \usepackage{import}
%%
%% and then include the figures with
%%   \import{<path to file>}{<filename>.pgf}
%%
%% Matplotlib used the following preamble
%%
\begingroup%
\makeatletter%
\begin{pgfpicture}%
\pgfpathrectangle{\pgfpointorigin}{\pgfqpoint{2.668704in}{2.524691in}}%
\pgfusepath{use as bounding box, clip}%
\begin{pgfscope}%
\pgfsetbuttcap%
\pgfsetmiterjoin%
\pgfsetlinewidth{0.000000pt}%
\definecolor{currentstroke}{rgb}{0.000000,0.000000,0.000000}%
\pgfsetstrokecolor{currentstroke}%
\pgfsetstrokeopacity{0.000000}%
\pgfsetdash{}{0pt}%
\pgfpathmoveto{\pgfqpoint{0.000000in}{0.000000in}}%
\pgfpathlineto{\pgfqpoint{2.668704in}{0.000000in}}%
\pgfpathlineto{\pgfqpoint{2.668704in}{2.524691in}}%
\pgfpathlineto{\pgfqpoint{0.000000in}{2.524691in}}%
\pgfpathlineto{\pgfqpoint{0.000000in}{0.000000in}}%
\pgfpathclose%
\pgfusepath{}%
\end{pgfscope}%
\begin{pgfscope}%
\pgfsetbuttcap%
\pgfsetmiterjoin%
\pgfsetlinewidth{0.000000pt}%
\definecolor{currentstroke}{rgb}{0.000000,0.000000,0.000000}%
\pgfsetstrokecolor{currentstroke}%
\pgfsetstrokeopacity{0.000000}%
\pgfsetdash{}{0pt}%
\pgfpathmoveto{\pgfqpoint{0.553704in}{0.499691in}}%
\pgfpathlineto{\pgfqpoint{2.568704in}{0.499691in}}%
\pgfpathlineto{\pgfqpoint{2.568704in}{2.424691in}}%
\pgfpathlineto{\pgfqpoint{0.553704in}{2.424691in}}%
\pgfpathlineto{\pgfqpoint{0.553704in}{0.499691in}}%
\pgfpathclose%
\pgfusepath{}%
\end{pgfscope}%
\begin{pgfscope}%
\pgfsetbuttcap%
\pgfsetroundjoin%
\definecolor{currentfill}{rgb}{0.000000,0.000000,0.000000}%
\pgfsetfillcolor{currentfill}%
\pgfsetlinewidth{0.803000pt}%
\definecolor{currentstroke}{rgb}{0.000000,0.000000,0.000000}%
\pgfsetstrokecolor{currentstroke}%
\pgfsetdash{}{0pt}%
\pgfsys@defobject{currentmarker}{\pgfqpoint{0.000000in}{-0.048611in}}{\pgfqpoint{0.000000in}{0.000000in}}{%
\pgfpathmoveto{\pgfqpoint{0.000000in}{0.000000in}}%
\pgfpathlineto{\pgfqpoint{0.000000in}{-0.048611in}}%
\pgfusepath{stroke,fill}%
}%
\begin{pgfscope}%
\pgfsys@transformshift{0.556707in}{0.499691in}%
\pgfsys@useobject{currentmarker}{}%
\end{pgfscope}%
\end{pgfscope}%
\begin{pgfscope}%
\definecolor{textcolor}{rgb}{0.000000,0.000000,0.000000}%
\pgfsetstrokecolor{textcolor}%
\pgfsetfillcolor{textcolor}%
\pgftext[x=0.556707in,y=0.402469in,,top]{\color{textcolor}\rmfamily\fontsize{10.000000}{12.000000}\selectfont \(\displaystyle 10^{-5 \alpha}\)}%
\end{pgfscope}%
\begin{pgfscope}%
\pgfsetbuttcap%
\pgfsetroundjoin%
\definecolor{currentfill}{rgb}{0.000000,0.000000,0.000000}%
\pgfsetfillcolor{currentfill}%
\pgfsetlinewidth{0.803000pt}%
\definecolor{currentstroke}{rgb}{0.000000,0.000000,0.000000}%
\pgfsetstrokecolor{currentstroke}%
\pgfsetdash{}{0pt}%
\pgfsys@defobject{currentmarker}{\pgfqpoint{0.000000in}{-0.048611in}}{\pgfqpoint{0.000000in}{0.000000in}}{%
\pgfpathmoveto{\pgfqpoint{0.000000in}{0.000000in}}%
\pgfpathlineto{\pgfqpoint{0.000000in}{-0.048611in}}%
\pgfusepath{stroke,fill}%
}%
\begin{pgfscope}%
\pgfsys@transformshift{1.000775in}{0.499691in}%
\pgfsys@useobject{currentmarker}{}%
\end{pgfscope}%
\end{pgfscope}%
\begin{pgfscope}%
\definecolor{textcolor}{rgb}{0.000000,0.000000,0.000000}%
\pgfsetstrokecolor{textcolor}%
\pgfsetfillcolor{textcolor}%
\pgftext[x=1.000775in,y=0.402469in,,top]{\color{textcolor}\rmfamily\fontsize{10.000000}{12.000000}\selectfont \(\displaystyle 10^{-4 \alpha}\)}%
\end{pgfscope}%
\begin{pgfscope}%
\pgfsetbuttcap%
\pgfsetroundjoin%
\definecolor{currentfill}{rgb}{0.000000,0.000000,0.000000}%
\pgfsetfillcolor{currentfill}%
\pgfsetlinewidth{0.803000pt}%
\definecolor{currentstroke}{rgb}{0.000000,0.000000,0.000000}%
\pgfsetstrokecolor{currentstroke}%
\pgfsetdash{}{0pt}%
\pgfsys@defobject{currentmarker}{\pgfqpoint{0.000000in}{-0.048611in}}{\pgfqpoint{0.000000in}{0.000000in}}{%
\pgfpathmoveto{\pgfqpoint{0.000000in}{0.000000in}}%
\pgfpathlineto{\pgfqpoint{0.000000in}{-0.048611in}}%
\pgfusepath{stroke,fill}%
}%
\begin{pgfscope}%
\pgfsys@transformshift{1.444844in}{0.499691in}%
\pgfsys@useobject{currentmarker}{}%
\end{pgfscope}%
\end{pgfscope}%
\begin{pgfscope}%
\definecolor{textcolor}{rgb}{0.000000,0.000000,0.000000}%
\pgfsetstrokecolor{textcolor}%
\pgfsetfillcolor{textcolor}%
\pgftext[x=1.444844in,y=0.402469in,,top]{\color{textcolor}\rmfamily\fontsize{10.000000}{12.000000}\selectfont \(\displaystyle 10^{-3 \alpha}\)}%
\end{pgfscope}%
\begin{pgfscope}%
\pgfsetbuttcap%
\pgfsetroundjoin%
\definecolor{currentfill}{rgb}{0.000000,0.000000,0.000000}%
\pgfsetfillcolor{currentfill}%
\pgfsetlinewidth{0.803000pt}%
\definecolor{currentstroke}{rgb}{0.000000,0.000000,0.000000}%
\pgfsetstrokecolor{currentstroke}%
\pgfsetdash{}{0pt}%
\pgfsys@defobject{currentmarker}{\pgfqpoint{0.000000in}{-0.048611in}}{\pgfqpoint{0.000000in}{0.000000in}}{%
\pgfpathmoveto{\pgfqpoint{0.000000in}{0.000000in}}%
\pgfpathlineto{\pgfqpoint{0.000000in}{-0.048611in}}%
\pgfusepath{stroke,fill}%
}%
\begin{pgfscope}%
\pgfsys@transformshift{1.888912in}{0.499691in}%
\pgfsys@useobject{currentmarker}{}%
\end{pgfscope}%
\end{pgfscope}%
\begin{pgfscope}%
\definecolor{textcolor}{rgb}{0.000000,0.000000,0.000000}%
\pgfsetstrokecolor{textcolor}%
\pgfsetfillcolor{textcolor}%
\pgftext[x=1.888912in,y=0.402469in,,top]{\color{textcolor}\rmfamily\fontsize{10.000000}{12.000000}\selectfont \(\displaystyle 10^{-2 \alpha}\)}%
\end{pgfscope}%
\begin{pgfscope}%
\pgfsetbuttcap%
\pgfsetroundjoin%
\definecolor{currentfill}{rgb}{0.000000,0.000000,0.000000}%
\pgfsetfillcolor{currentfill}%
\pgfsetlinewidth{0.803000pt}%
\definecolor{currentstroke}{rgb}{0.000000,0.000000,0.000000}%
\pgfsetstrokecolor{currentstroke}%
\pgfsetdash{}{0pt}%
\pgfsys@defobject{currentmarker}{\pgfqpoint{0.000000in}{-0.048611in}}{\pgfqpoint{0.000000in}{0.000000in}}{%
\pgfpathmoveto{\pgfqpoint{0.000000in}{0.000000in}}%
\pgfpathlineto{\pgfqpoint{0.000000in}{-0.048611in}}%
\pgfusepath{stroke,fill}%
}%
\begin{pgfscope}%
\pgfsys@transformshift{2.332981in}{0.499691in}%
\pgfsys@useobject{currentmarker}{}%
\end{pgfscope}%
\end{pgfscope}%
\begin{pgfscope}%
\definecolor{textcolor}{rgb}{0.000000,0.000000,0.000000}%
\pgfsetstrokecolor{textcolor}%
\pgfsetfillcolor{textcolor}%
\pgftext[x=2.332981in,y=0.402469in,,top]{\color{textcolor}\rmfamily\fontsize{10.000000}{12.000000}\selectfont \(\displaystyle 10^{-1 \alpha}\)}%
\end{pgfscope}%
\begin{pgfscope}%
\definecolor{textcolor}{rgb}{0.000000,0.000000,0.000000}%
\pgfsetstrokecolor{textcolor}%
\pgfsetfillcolor{textcolor}%
\pgftext[x=1.561204in,y=0.223457in,,top]{\color{textcolor}\rmfamily\fontsize{10.000000}{12.000000}\selectfont Volume}%
\end{pgfscope}%
\begin{pgfscope}%
\pgfsetbuttcap%
\pgfsetroundjoin%
\definecolor{currentfill}{rgb}{0.000000,0.000000,0.000000}%
\pgfsetfillcolor{currentfill}%
\pgfsetlinewidth{0.803000pt}%
\definecolor{currentstroke}{rgb}{0.000000,0.000000,0.000000}%
\pgfsetstrokecolor{currentstroke}%
\pgfsetdash{}{0pt}%
\pgfsys@defobject{currentmarker}{\pgfqpoint{-0.048611in}{0.000000in}}{\pgfqpoint{-0.000000in}{0.000000in}}{%
\pgfpathmoveto{\pgfqpoint{-0.000000in}{0.000000in}}%
\pgfpathlineto{\pgfqpoint{-0.048611in}{0.000000in}}%
\pgfusepath{stroke,fill}%
}%
\begin{pgfscope}%
\pgfsys@transformshift{0.553704in}{0.587191in}%
\pgfsys@useobject{currentmarker}{}%
\end{pgfscope}%
\end{pgfscope}%
\begin{pgfscope}%
\definecolor{textcolor}{rgb}{0.000000,0.000000,0.000000}%
\pgfsetstrokecolor{textcolor}%
\pgfsetfillcolor{textcolor}%
\pgftext[x=0.279012in, y=0.538966in, left, base]{\color{textcolor}\rmfamily\fontsize{10.000000}{12.000000}\selectfont \(\displaystyle {0.0}\)}%
\end{pgfscope}%
\begin{pgfscope}%
\pgfsetbuttcap%
\pgfsetroundjoin%
\definecolor{currentfill}{rgb}{0.000000,0.000000,0.000000}%
\pgfsetfillcolor{currentfill}%
\pgfsetlinewidth{0.803000pt}%
\definecolor{currentstroke}{rgb}{0.000000,0.000000,0.000000}%
\pgfsetstrokecolor{currentstroke}%
\pgfsetdash{}{0pt}%
\pgfsys@defobject{currentmarker}{\pgfqpoint{-0.048611in}{0.000000in}}{\pgfqpoint{-0.000000in}{0.000000in}}{%
\pgfpathmoveto{\pgfqpoint{-0.000000in}{0.000000in}}%
\pgfpathlineto{\pgfqpoint{-0.048611in}{0.000000in}}%
\pgfusepath{stroke,fill}%
}%
\begin{pgfscope}%
\pgfsys@transformshift{0.553704in}{0.985824in}%
\pgfsys@useobject{currentmarker}{}%
\end{pgfscope}%
\end{pgfscope}%
\begin{pgfscope}%
\definecolor{textcolor}{rgb}{0.000000,0.000000,0.000000}%
\pgfsetstrokecolor{textcolor}%
\pgfsetfillcolor{textcolor}%
\pgftext[x=0.279012in, y=0.937599in, left, base]{\color{textcolor}\rmfamily\fontsize{10.000000}{12.000000}\selectfont \(\displaystyle {0.2}\)}%
\end{pgfscope}%
\begin{pgfscope}%
\pgfsetbuttcap%
\pgfsetroundjoin%
\definecolor{currentfill}{rgb}{0.000000,0.000000,0.000000}%
\pgfsetfillcolor{currentfill}%
\pgfsetlinewidth{0.803000pt}%
\definecolor{currentstroke}{rgb}{0.000000,0.000000,0.000000}%
\pgfsetstrokecolor{currentstroke}%
\pgfsetdash{}{0pt}%
\pgfsys@defobject{currentmarker}{\pgfqpoint{-0.048611in}{0.000000in}}{\pgfqpoint{-0.000000in}{0.000000in}}{%
\pgfpathmoveto{\pgfqpoint{-0.000000in}{0.000000in}}%
\pgfpathlineto{\pgfqpoint{-0.048611in}{0.000000in}}%
\pgfusepath{stroke,fill}%
}%
\begin{pgfscope}%
\pgfsys@transformshift{0.553704in}{1.384458in}%
\pgfsys@useobject{currentmarker}{}%
\end{pgfscope}%
\end{pgfscope}%
\begin{pgfscope}%
\definecolor{textcolor}{rgb}{0.000000,0.000000,0.000000}%
\pgfsetstrokecolor{textcolor}%
\pgfsetfillcolor{textcolor}%
\pgftext[x=0.279012in, y=1.336232in, left, base]{\color{textcolor}\rmfamily\fontsize{10.000000}{12.000000}\selectfont \(\displaystyle {0.4}\)}%
\end{pgfscope}%
\begin{pgfscope}%
\pgfsetbuttcap%
\pgfsetroundjoin%
\definecolor{currentfill}{rgb}{0.000000,0.000000,0.000000}%
\pgfsetfillcolor{currentfill}%
\pgfsetlinewidth{0.803000pt}%
\definecolor{currentstroke}{rgb}{0.000000,0.000000,0.000000}%
\pgfsetstrokecolor{currentstroke}%
\pgfsetdash{}{0pt}%
\pgfsys@defobject{currentmarker}{\pgfqpoint{-0.048611in}{0.000000in}}{\pgfqpoint{-0.000000in}{0.000000in}}{%
\pgfpathmoveto{\pgfqpoint{-0.000000in}{0.000000in}}%
\pgfpathlineto{\pgfqpoint{-0.048611in}{0.000000in}}%
\pgfusepath{stroke,fill}%
}%
\begin{pgfscope}%
\pgfsys@transformshift{0.553704in}{1.783091in}%
\pgfsys@useobject{currentmarker}{}%
\end{pgfscope}%
\end{pgfscope}%
\begin{pgfscope}%
\definecolor{textcolor}{rgb}{0.000000,0.000000,0.000000}%
\pgfsetstrokecolor{textcolor}%
\pgfsetfillcolor{textcolor}%
\pgftext[x=0.279012in, y=1.734866in, left, base]{\color{textcolor}\rmfamily\fontsize{10.000000}{12.000000}\selectfont \(\displaystyle {0.6}\)}%
\end{pgfscope}%
\begin{pgfscope}%
\pgfsetbuttcap%
\pgfsetroundjoin%
\definecolor{currentfill}{rgb}{0.000000,0.000000,0.000000}%
\pgfsetfillcolor{currentfill}%
\pgfsetlinewidth{0.803000pt}%
\definecolor{currentstroke}{rgb}{0.000000,0.000000,0.000000}%
\pgfsetstrokecolor{currentstroke}%
\pgfsetdash{}{0pt}%
\pgfsys@defobject{currentmarker}{\pgfqpoint{-0.048611in}{0.000000in}}{\pgfqpoint{-0.000000in}{0.000000in}}{%
\pgfpathmoveto{\pgfqpoint{-0.000000in}{0.000000in}}%
\pgfpathlineto{\pgfqpoint{-0.048611in}{0.000000in}}%
\pgfusepath{stroke,fill}%
}%
\begin{pgfscope}%
\pgfsys@transformshift{0.553704in}{2.181724in}%
\pgfsys@useobject{currentmarker}{}%
\end{pgfscope}%
\end{pgfscope}%
\begin{pgfscope}%
\definecolor{textcolor}{rgb}{0.000000,0.000000,0.000000}%
\pgfsetstrokecolor{textcolor}%
\pgfsetfillcolor{textcolor}%
\pgftext[x=0.279012in, y=2.133499in, left, base]{\color{textcolor}\rmfamily\fontsize{10.000000}{12.000000}\selectfont \(\displaystyle {0.8}\)}%
\end{pgfscope}%
\begin{pgfscope}%
\definecolor{textcolor}{rgb}{0.000000,0.000000,0.000000}%
\pgfsetstrokecolor{textcolor}%
\pgfsetfillcolor{textcolor}%
\pgftext[x=0.223457in,y=1.462191in,,bottom,rotate=90.000000]{\color{textcolor}\rmfamily\fontsize{10.000000}{12.000000}\selectfont CIFAR-10 certified accuracy}%
\end{pgfscope}%
\begin{pgfscope}%
\pgfpathrectangle{\pgfqpoint{0.553704in}{0.499691in}}{\pgfqpoint{2.015000in}{1.925000in}}%
\pgfusepath{clip}%
\pgfsetrectcap%
\pgfsetroundjoin%
\pgfsetlinewidth{1.505625pt}%
\definecolor{currentstroke}{rgb}{0.105882,0.619608,0.466667}%
\pgfsetstrokecolor{currentstroke}%
\pgfsetdash{}{0pt}%
\pgfpathmoveto{\pgfqpoint{0.553704in}{2.297328in}}%
\pgfpathlineto{\pgfqpoint{1.247835in}{2.297328in}}%
\pgfpathlineto{\pgfqpoint{1.254574in}{2.293341in}}%
\pgfpathlineto{\pgfqpoint{1.510661in}{2.293341in}}%
\pgfpathlineto{\pgfqpoint{1.517400in}{2.289355in}}%
\pgfpathlineto{\pgfqpoint{1.631965in}{2.289355in}}%
\pgfpathlineto{\pgfqpoint{1.638704in}{2.285369in}}%
\pgfpathlineto{\pgfqpoint{1.692617in}{2.285369in}}%
\pgfpathlineto{\pgfqpoint{1.699356in}{2.281382in}}%
\pgfpathlineto{\pgfqpoint{1.780226in}{2.281382in}}%
\pgfpathlineto{\pgfqpoint{1.786965in}{2.277396in}}%
\pgfpathlineto{\pgfqpoint{1.840878in}{2.277396in}}%
\pgfpathlineto{\pgfqpoint{1.847617in}{2.273410in}}%
\pgfpathlineto{\pgfqpoint{1.867835in}{2.273410in}}%
\pgfpathlineto{\pgfqpoint{1.874574in}{2.269423in}}%
\pgfpathlineto{\pgfqpoint{1.881313in}{2.253478in}}%
\pgfpathlineto{\pgfqpoint{1.894791in}{2.245505in}}%
\pgfpathlineto{\pgfqpoint{1.908269in}{2.245505in}}%
\pgfpathlineto{\pgfqpoint{1.915009in}{2.241519in}}%
\pgfpathlineto{\pgfqpoint{1.921748in}{2.233546in}}%
\pgfpathlineto{\pgfqpoint{1.928487in}{2.233546in}}%
\pgfpathlineto{\pgfqpoint{1.935226in}{2.225574in}}%
\pgfpathlineto{\pgfqpoint{1.962182in}{2.225574in}}%
\pgfpathlineto{\pgfqpoint{1.968922in}{2.221587in}}%
\pgfpathlineto{\pgfqpoint{1.975661in}{2.221587in}}%
\pgfpathlineto{\pgfqpoint{1.982400in}{2.217601in}}%
\pgfpathlineto{\pgfqpoint{1.995878in}{2.217601in}}%
\pgfpathlineto{\pgfqpoint{2.002617in}{2.213615in}}%
\pgfpathlineto{\pgfqpoint{2.009356in}{2.213615in}}%
\pgfpathlineto{\pgfqpoint{2.016095in}{2.205642in}}%
\pgfpathlineto{\pgfqpoint{2.022835in}{2.201656in}}%
\pgfpathlineto{\pgfqpoint{2.029574in}{2.189697in}}%
\pgfpathlineto{\pgfqpoint{2.036313in}{2.185710in}}%
\pgfpathlineto{\pgfqpoint{2.043052in}{2.185710in}}%
\pgfpathlineto{\pgfqpoint{2.049791in}{2.173751in}}%
\pgfpathlineto{\pgfqpoint{2.063269in}{2.165779in}}%
\pgfpathlineto{\pgfqpoint{2.070009in}{2.165779in}}%
\pgfpathlineto{\pgfqpoint{2.090226in}{2.141861in}}%
\pgfpathlineto{\pgfqpoint{2.096965in}{2.125915in}}%
\pgfpathlineto{\pgfqpoint{2.110443in}{2.125915in}}%
\pgfpathlineto{\pgfqpoint{2.123922in}{2.101997in}}%
\pgfpathlineto{\pgfqpoint{2.130661in}{2.101997in}}%
\pgfpathlineto{\pgfqpoint{2.137400in}{2.098011in}}%
\pgfpathlineto{\pgfqpoint{2.144139in}{2.086052in}}%
\pgfpathlineto{\pgfqpoint{2.150878in}{2.082066in}}%
\pgfpathlineto{\pgfqpoint{2.157617in}{2.074093in}}%
\pgfpathlineto{\pgfqpoint{2.164356in}{2.062134in}}%
\pgfpathlineto{\pgfqpoint{2.171095in}{2.054161in}}%
\pgfpathlineto{\pgfqpoint{2.177835in}{2.038216in}}%
\pgfpathlineto{\pgfqpoint{2.191313in}{2.022271in}}%
\pgfpathlineto{\pgfqpoint{2.204791in}{1.998353in}}%
\pgfpathlineto{\pgfqpoint{2.211530in}{1.970449in}}%
\pgfpathlineto{\pgfqpoint{2.218269in}{1.958490in}}%
\pgfpathlineto{\pgfqpoint{2.225009in}{1.950517in}}%
\pgfpathlineto{\pgfqpoint{2.231748in}{1.934572in}}%
\pgfpathlineto{\pgfqpoint{2.238487in}{1.934572in}}%
\pgfpathlineto{\pgfqpoint{2.245226in}{1.918626in}}%
\pgfpathlineto{\pgfqpoint{2.251965in}{1.890722in}}%
\pgfpathlineto{\pgfqpoint{2.258704in}{1.870790in}}%
\pgfpathlineto{\pgfqpoint{2.265443in}{1.846872in}}%
\pgfpathlineto{\pgfqpoint{2.285661in}{1.799036in}}%
\pgfpathlineto{\pgfqpoint{2.292400in}{1.787077in}}%
\pgfpathlineto{\pgfqpoint{2.299139in}{1.767146in}}%
\pgfpathlineto{\pgfqpoint{2.305878in}{1.755187in}}%
\pgfpathlineto{\pgfqpoint{2.312617in}{1.719310in}}%
\pgfpathlineto{\pgfqpoint{2.319356in}{1.711337in}}%
\pgfpathlineto{\pgfqpoint{2.326095in}{1.675460in}}%
\pgfpathlineto{\pgfqpoint{2.332835in}{1.655528in}}%
\pgfpathlineto{\pgfqpoint{2.339574in}{1.603706in}}%
\pgfpathlineto{\pgfqpoint{2.346313in}{1.591747in}}%
\pgfpathlineto{\pgfqpoint{2.353052in}{1.531952in}}%
\pgfpathlineto{\pgfqpoint{2.366530in}{1.476143in}}%
\pgfpathlineto{\pgfqpoint{2.373269in}{1.428307in}}%
\pgfpathlineto{\pgfqpoint{2.380009in}{1.400403in}}%
\pgfpathlineto{\pgfqpoint{2.393487in}{1.284799in}}%
\pgfpathlineto{\pgfqpoint{2.400226in}{1.248922in}}%
\pgfpathlineto{\pgfqpoint{2.406965in}{1.201086in}}%
\pgfpathlineto{\pgfqpoint{2.413704in}{1.173182in}}%
\pgfpathlineto{\pgfqpoint{2.420443in}{1.137305in}}%
\pgfpathlineto{\pgfqpoint{2.427182in}{1.089469in}}%
\pgfpathlineto{\pgfqpoint{2.433922in}{1.049606in}}%
\pgfpathlineto{\pgfqpoint{2.440661in}{1.021701in}}%
\pgfpathlineto{\pgfqpoint{2.460878in}{0.878193in}}%
\pgfpathlineto{\pgfqpoint{2.474356in}{0.802453in}}%
\pgfpathlineto{\pgfqpoint{2.481095in}{0.766576in}}%
\pgfpathlineto{\pgfqpoint{2.487835in}{0.742658in}}%
\pgfpathlineto{\pgfqpoint{2.501313in}{0.710767in}}%
\pgfpathlineto{\pgfqpoint{2.508052in}{0.670904in}}%
\pgfpathlineto{\pgfqpoint{2.514791in}{0.662931in}}%
\pgfpathlineto{\pgfqpoint{2.535009in}{0.615095in}}%
\pgfpathlineto{\pgfqpoint{2.548487in}{0.607123in}}%
\pgfpathlineto{\pgfqpoint{2.561965in}{0.591177in}}%
\pgfpathlineto{\pgfqpoint{2.568704in}{0.587191in}}%
\pgfpathlineto{\pgfqpoint{2.568704in}{0.587191in}}%
\pgfusepath{stroke}%
\end{pgfscope}%
\begin{pgfscope}%
\pgfpathrectangle{\pgfqpoint{0.553704in}{0.499691in}}{\pgfqpoint{2.015000in}{1.925000in}}%
\pgfusepath{clip}%
\pgfsetrectcap%
\pgfsetroundjoin%
\pgfsetlinewidth{1.505625pt}%
\definecolor{currentstroke}{rgb}{0.850980,0.372549,0.007843}%
\pgfsetstrokecolor{currentstroke}%
\pgfsetdash{}{0pt}%
\pgfpathmoveto{\pgfqpoint{0.553704in}{2.337191in}}%
\pgfpathlineto{\pgfqpoint{0.870443in}{2.337191in}}%
\pgfpathlineto{\pgfqpoint{0.877182in}{2.333205in}}%
\pgfpathlineto{\pgfqpoint{1.429791in}{2.333205in}}%
\pgfpathlineto{\pgfqpoint{1.436530in}{2.329218in}}%
\pgfpathlineto{\pgfqpoint{1.524139in}{2.329218in}}%
\pgfpathlineto{\pgfqpoint{1.530878in}{2.325232in}}%
\pgfpathlineto{\pgfqpoint{1.571313in}{2.325232in}}%
\pgfpathlineto{\pgfqpoint{1.584791in}{2.317259in}}%
\pgfpathlineto{\pgfqpoint{1.598269in}{2.317259in}}%
\pgfpathlineto{\pgfqpoint{1.605009in}{2.313273in}}%
\pgfpathlineto{\pgfqpoint{1.611748in}{2.313273in}}%
\pgfpathlineto{\pgfqpoint{1.618487in}{2.309287in}}%
\pgfpathlineto{\pgfqpoint{1.692617in}{2.309287in}}%
\pgfpathlineto{\pgfqpoint{1.699356in}{2.305300in}}%
\pgfpathlineto{\pgfqpoint{1.773487in}{2.305300in}}%
\pgfpathlineto{\pgfqpoint{1.793704in}{2.293341in}}%
\pgfpathlineto{\pgfqpoint{1.800443in}{2.293341in}}%
\pgfpathlineto{\pgfqpoint{1.807182in}{2.281382in}}%
\pgfpathlineto{\pgfqpoint{1.820661in}{2.281382in}}%
\pgfpathlineto{\pgfqpoint{1.834139in}{2.273410in}}%
\pgfpathlineto{\pgfqpoint{1.840878in}{2.273410in}}%
\pgfpathlineto{\pgfqpoint{1.847617in}{2.269423in}}%
\pgfpathlineto{\pgfqpoint{1.854356in}{2.269423in}}%
\pgfpathlineto{\pgfqpoint{1.861095in}{2.257464in}}%
\pgfpathlineto{\pgfqpoint{1.874574in}{2.257464in}}%
\pgfpathlineto{\pgfqpoint{1.881313in}{2.249492in}}%
\pgfpathlineto{\pgfqpoint{1.894791in}{2.249492in}}%
\pgfpathlineto{\pgfqpoint{1.908269in}{2.241519in}}%
\pgfpathlineto{\pgfqpoint{1.915009in}{2.229560in}}%
\pgfpathlineto{\pgfqpoint{1.921748in}{2.225574in}}%
\pgfpathlineto{\pgfqpoint{1.928487in}{2.225574in}}%
\pgfpathlineto{\pgfqpoint{1.935226in}{2.217601in}}%
\pgfpathlineto{\pgfqpoint{1.941965in}{2.205642in}}%
\pgfpathlineto{\pgfqpoint{1.948704in}{2.197669in}}%
\pgfpathlineto{\pgfqpoint{1.955443in}{2.185710in}}%
\pgfpathlineto{\pgfqpoint{1.968922in}{2.169765in}}%
\pgfpathlineto{\pgfqpoint{1.975661in}{2.165779in}}%
\pgfpathlineto{\pgfqpoint{1.995878in}{2.129902in}}%
\pgfpathlineto{\pgfqpoint{2.002617in}{2.129902in}}%
\pgfpathlineto{\pgfqpoint{2.009356in}{2.117943in}}%
\pgfpathlineto{\pgfqpoint{2.016095in}{2.109970in}}%
\pgfpathlineto{\pgfqpoint{2.022835in}{2.094025in}}%
\pgfpathlineto{\pgfqpoint{2.036313in}{2.078079in}}%
\pgfpathlineto{\pgfqpoint{2.043052in}{2.058148in}}%
\pgfpathlineto{\pgfqpoint{2.049791in}{2.054161in}}%
\pgfpathlineto{\pgfqpoint{2.056530in}{2.046189in}}%
\pgfpathlineto{\pgfqpoint{2.063269in}{2.042203in}}%
\pgfpathlineto{\pgfqpoint{2.070009in}{2.042203in}}%
\pgfpathlineto{\pgfqpoint{2.076748in}{2.030244in}}%
\pgfpathlineto{\pgfqpoint{2.083487in}{2.026257in}}%
\pgfpathlineto{\pgfqpoint{2.090226in}{2.014298in}}%
\pgfpathlineto{\pgfqpoint{2.096965in}{1.994367in}}%
\pgfpathlineto{\pgfqpoint{2.103704in}{1.978421in}}%
\pgfpathlineto{\pgfqpoint{2.117182in}{1.934572in}}%
\pgfpathlineto{\pgfqpoint{2.130661in}{1.902681in}}%
\pgfpathlineto{\pgfqpoint{2.137400in}{1.894708in}}%
\pgfpathlineto{\pgfqpoint{2.144139in}{1.850859in}}%
\pgfpathlineto{\pgfqpoint{2.157617in}{1.787077in}}%
\pgfpathlineto{\pgfqpoint{2.164356in}{1.767146in}}%
\pgfpathlineto{\pgfqpoint{2.177835in}{1.687419in}}%
\pgfpathlineto{\pgfqpoint{2.184574in}{1.663501in}}%
\pgfpathlineto{\pgfqpoint{2.204791in}{1.555870in}}%
\pgfpathlineto{\pgfqpoint{2.211530in}{1.492089in}}%
\pgfpathlineto{\pgfqpoint{2.218269in}{1.456212in}}%
\pgfpathlineto{\pgfqpoint{2.225009in}{1.388444in}}%
\pgfpathlineto{\pgfqpoint{2.231748in}{1.348581in}}%
\pgfpathlineto{\pgfqpoint{2.238487in}{0.587191in}}%
\pgfpathlineto{\pgfqpoint{2.568704in}{0.587191in}}%
\pgfpathlineto{\pgfqpoint{2.568704in}{0.587191in}}%
\pgfusepath{stroke}%
\end{pgfscope}%
\begin{pgfscope}%
\pgfpathrectangle{\pgfqpoint{0.553704in}{0.499691in}}{\pgfqpoint{2.015000in}{1.925000in}}%
\pgfusepath{clip}%
\pgfsetrectcap%
\pgfsetroundjoin%
\pgfsetlinewidth{1.505625pt}%
\definecolor{currentstroke}{rgb}{0.458824,0.439216,0.701961}%
\pgfsetstrokecolor{currentstroke}%
\pgfsetdash{}{0pt}%
\pgfpathmoveto{\pgfqpoint{0.553704in}{2.329218in}}%
\pgfpathlineto{\pgfqpoint{2.150878in}{2.329218in}}%
\pgfpathlineto{\pgfqpoint{2.157617in}{2.321246in}}%
\pgfpathlineto{\pgfqpoint{2.164356in}{2.321246in}}%
\pgfpathlineto{\pgfqpoint{2.171095in}{2.313273in}}%
\pgfpathlineto{\pgfqpoint{2.177835in}{2.293341in}}%
\pgfpathlineto{\pgfqpoint{2.184574in}{2.293341in}}%
\pgfpathlineto{\pgfqpoint{2.191313in}{2.277396in}}%
\pgfpathlineto{\pgfqpoint{2.198052in}{2.257464in}}%
\pgfpathlineto{\pgfqpoint{2.204791in}{2.229560in}}%
\pgfpathlineto{\pgfqpoint{2.211530in}{2.213615in}}%
\pgfpathlineto{\pgfqpoint{2.231748in}{2.105984in}}%
\pgfpathlineto{\pgfqpoint{2.245226in}{2.014298in}}%
\pgfpathlineto{\pgfqpoint{2.251965in}{1.934572in}}%
\pgfpathlineto{\pgfqpoint{2.265443in}{1.735255in}}%
\pgfpathlineto{\pgfqpoint{2.272182in}{1.615665in}}%
\pgfpathlineto{\pgfqpoint{2.285661in}{1.348581in}}%
\pgfpathlineto{\pgfqpoint{2.299139in}{1.121360in}}%
\pgfpathlineto{\pgfqpoint{2.305878in}{1.009742in}}%
\pgfpathlineto{\pgfqpoint{2.312617in}{0.926029in}}%
\pgfpathlineto{\pgfqpoint{2.319356in}{0.858262in}}%
\pgfpathlineto{\pgfqpoint{2.326095in}{0.762590in}}%
\pgfpathlineto{\pgfqpoint{2.332835in}{0.718740in}}%
\pgfpathlineto{\pgfqpoint{2.339574in}{0.690836in}}%
\pgfpathlineto{\pgfqpoint{2.346313in}{0.639013in}}%
\pgfpathlineto{\pgfqpoint{2.353052in}{0.623068in}}%
\pgfpathlineto{\pgfqpoint{2.359791in}{0.611109in}}%
\pgfpathlineto{\pgfqpoint{2.366530in}{0.603136in}}%
\pgfpathlineto{\pgfqpoint{2.373269in}{0.603136in}}%
\pgfpathlineto{\pgfqpoint{2.386748in}{0.587191in}}%
\pgfpathlineto{\pgfqpoint{2.568704in}{0.587191in}}%
\pgfpathlineto{\pgfqpoint{2.568704in}{0.587191in}}%
\pgfusepath{stroke}%
\end{pgfscope}%
\begin{pgfscope}%
\pgfpathrectangle{\pgfqpoint{0.553704in}{0.499691in}}{\pgfqpoint{2.015000in}{1.925000in}}%
\pgfusepath{clip}%
\pgfsetrectcap%
\pgfsetroundjoin%
\pgfsetlinewidth{1.505625pt}%
\definecolor{currentstroke}{rgb}{0.905882,0.160784,0.541176}%
\pgfsetstrokecolor{currentstroke}%
\pgfsetdash{}{0pt}%
\pgfpathmoveto{\pgfqpoint{0.553704in}{2.253478in}}%
\pgfpathlineto{\pgfqpoint{1.052400in}{2.253478in}}%
\pgfpathlineto{\pgfqpoint{1.059139in}{2.249492in}}%
\pgfpathlineto{\pgfqpoint{1.126530in}{2.249492in}}%
\pgfpathlineto{\pgfqpoint{1.133269in}{2.245505in}}%
\pgfpathlineto{\pgfqpoint{1.180443in}{2.245505in}}%
\pgfpathlineto{\pgfqpoint{1.187182in}{2.241519in}}%
\pgfpathlineto{\pgfqpoint{1.241095in}{2.241519in}}%
\pgfpathlineto{\pgfqpoint{1.247835in}{2.237533in}}%
\pgfpathlineto{\pgfqpoint{1.254574in}{2.237533in}}%
\pgfpathlineto{\pgfqpoint{1.261313in}{2.233546in}}%
\pgfpathlineto{\pgfqpoint{1.268052in}{2.233546in}}%
\pgfpathlineto{\pgfqpoint{1.274791in}{2.229560in}}%
\pgfpathlineto{\pgfqpoint{1.321965in}{2.229560in}}%
\pgfpathlineto{\pgfqpoint{1.342182in}{2.217601in}}%
\pgfpathlineto{\pgfqpoint{1.348922in}{2.217601in}}%
\pgfpathlineto{\pgfqpoint{1.355661in}{2.213615in}}%
\pgfpathlineto{\pgfqpoint{1.362400in}{2.213615in}}%
\pgfpathlineto{\pgfqpoint{1.369139in}{2.209628in}}%
\pgfpathlineto{\pgfqpoint{1.402835in}{2.209628in}}%
\pgfpathlineto{\pgfqpoint{1.409574in}{2.205642in}}%
\pgfpathlineto{\pgfqpoint{1.416313in}{2.205642in}}%
\pgfpathlineto{\pgfqpoint{1.456748in}{2.181724in}}%
\pgfpathlineto{\pgfqpoint{1.463487in}{2.181724in}}%
\pgfpathlineto{\pgfqpoint{1.470226in}{2.173751in}}%
\pgfpathlineto{\pgfqpoint{1.476965in}{2.173751in}}%
\pgfpathlineto{\pgfqpoint{1.483704in}{2.157806in}}%
\pgfpathlineto{\pgfqpoint{1.503922in}{2.133888in}}%
\pgfpathlineto{\pgfqpoint{1.510661in}{2.117943in}}%
\pgfpathlineto{\pgfqpoint{1.517400in}{2.098011in}}%
\pgfpathlineto{\pgfqpoint{1.530878in}{2.034230in}}%
\pgfpathlineto{\pgfqpoint{1.537617in}{1.990380in}}%
\pgfpathlineto{\pgfqpoint{1.544356in}{1.934572in}}%
\pgfpathlineto{\pgfqpoint{1.551095in}{0.587191in}}%
\pgfpathlineto{\pgfqpoint{2.568704in}{0.587191in}}%
\pgfpathlineto{\pgfqpoint{2.568704in}{0.587191in}}%
\pgfusepath{stroke}%
\end{pgfscope}%
\begin{pgfscope}%
\pgfpathrectangle{\pgfqpoint{0.553704in}{0.499691in}}{\pgfqpoint{2.015000in}{1.925000in}}%
\pgfusepath{clip}%
\pgfsetrectcap%
\pgfsetroundjoin%
\pgfsetlinewidth{1.505625pt}%
\definecolor{currentstroke}{rgb}{0.400000,0.650980,0.117647}%
\pgfsetstrokecolor{currentstroke}%
\pgfsetdash{}{0pt}%
\pgfpathmoveto{\pgfqpoint{0.553704in}{2.289355in}}%
\pgfpathlineto{\pgfqpoint{0.809791in}{2.289355in}}%
\pgfpathlineto{\pgfqpoint{0.816530in}{2.285369in}}%
\pgfpathlineto{\pgfqpoint{0.890661in}{2.285369in}}%
\pgfpathlineto{\pgfqpoint{0.897400in}{2.281382in}}%
\pgfpathlineto{\pgfqpoint{1.173704in}{2.281382in}}%
\pgfpathlineto{\pgfqpoint{1.180443in}{2.277396in}}%
\pgfpathlineto{\pgfqpoint{1.193922in}{2.277396in}}%
\pgfpathlineto{\pgfqpoint{1.200661in}{2.273410in}}%
\pgfpathlineto{\pgfqpoint{1.281530in}{2.273410in}}%
\pgfpathlineto{\pgfqpoint{1.288269in}{2.269423in}}%
\pgfpathlineto{\pgfqpoint{1.342182in}{2.269423in}}%
\pgfpathlineto{\pgfqpoint{1.348922in}{2.265437in}}%
\pgfpathlineto{\pgfqpoint{1.375878in}{2.265437in}}%
\pgfpathlineto{\pgfqpoint{1.389356in}{2.257464in}}%
\pgfpathlineto{\pgfqpoint{1.429791in}{2.257464in}}%
\pgfpathlineto{\pgfqpoint{1.436530in}{2.253478in}}%
\pgfpathlineto{\pgfqpoint{1.497182in}{2.253478in}}%
\pgfpathlineto{\pgfqpoint{1.503922in}{2.249492in}}%
\pgfpathlineto{\pgfqpoint{1.578052in}{2.249492in}}%
\pgfpathlineto{\pgfqpoint{1.591530in}{2.241519in}}%
\pgfpathlineto{\pgfqpoint{1.652182in}{2.241519in}}%
\pgfpathlineto{\pgfqpoint{1.658922in}{2.233546in}}%
\pgfpathlineto{\pgfqpoint{1.672400in}{2.225574in}}%
\pgfpathlineto{\pgfqpoint{1.679139in}{2.213615in}}%
\pgfpathlineto{\pgfqpoint{1.685878in}{2.205642in}}%
\pgfpathlineto{\pgfqpoint{1.692617in}{2.205642in}}%
\pgfpathlineto{\pgfqpoint{1.699356in}{2.201656in}}%
\pgfpathlineto{\pgfqpoint{1.712835in}{2.201656in}}%
\pgfpathlineto{\pgfqpoint{1.726313in}{2.193683in}}%
\pgfpathlineto{\pgfqpoint{1.733052in}{2.193683in}}%
\pgfpathlineto{\pgfqpoint{1.739791in}{2.177738in}}%
\pgfpathlineto{\pgfqpoint{1.760009in}{2.165779in}}%
\pgfpathlineto{\pgfqpoint{1.786965in}{2.133888in}}%
\pgfpathlineto{\pgfqpoint{1.793704in}{2.129902in}}%
\pgfpathlineto{\pgfqpoint{1.800443in}{2.117943in}}%
\pgfpathlineto{\pgfqpoint{1.807182in}{2.113956in}}%
\pgfpathlineto{\pgfqpoint{1.813922in}{2.101997in}}%
\pgfpathlineto{\pgfqpoint{1.820661in}{2.098011in}}%
\pgfpathlineto{\pgfqpoint{1.827400in}{2.098011in}}%
\pgfpathlineto{\pgfqpoint{1.834139in}{2.090038in}}%
\pgfpathlineto{\pgfqpoint{1.847617in}{2.090038in}}%
\pgfpathlineto{\pgfqpoint{1.854356in}{2.050175in}}%
\pgfpathlineto{\pgfqpoint{1.861095in}{2.046189in}}%
\pgfpathlineto{\pgfqpoint{1.867835in}{2.046189in}}%
\pgfpathlineto{\pgfqpoint{1.881313in}{2.022271in}}%
\pgfpathlineto{\pgfqpoint{1.888052in}{1.994367in}}%
\pgfpathlineto{\pgfqpoint{1.894791in}{1.982408in}}%
\pgfpathlineto{\pgfqpoint{1.908269in}{1.966462in}}%
\pgfpathlineto{\pgfqpoint{1.915009in}{1.938558in}}%
\pgfpathlineto{\pgfqpoint{1.921748in}{1.926599in}}%
\pgfpathlineto{\pgfqpoint{1.928487in}{1.910654in}}%
\pgfpathlineto{\pgfqpoint{1.935226in}{1.882749in}}%
\pgfpathlineto{\pgfqpoint{1.941965in}{1.870790in}}%
\pgfpathlineto{\pgfqpoint{1.948704in}{1.838900in}}%
\pgfpathlineto{\pgfqpoint{1.955443in}{1.818968in}}%
\pgfpathlineto{\pgfqpoint{1.962182in}{1.791064in}}%
\pgfpathlineto{\pgfqpoint{1.968922in}{1.739241in}}%
\pgfpathlineto{\pgfqpoint{1.975661in}{1.719310in}}%
\pgfpathlineto{\pgfqpoint{1.982400in}{1.647556in}}%
\pgfpathlineto{\pgfqpoint{1.989139in}{1.591747in}}%
\pgfpathlineto{\pgfqpoint{1.995878in}{0.587191in}}%
\pgfpathlineto{\pgfqpoint{2.568704in}{0.587191in}}%
\pgfpathlineto{\pgfqpoint{2.568704in}{0.587191in}}%
\pgfusepath{stroke}%
\end{pgfscope}%
\begin{pgfscope}%
\pgfsetrectcap%
\pgfsetmiterjoin%
\pgfsetlinewidth{0.803000pt}%
\definecolor{currentstroke}{rgb}{0.000000,0.000000,0.000000}%
\pgfsetstrokecolor{currentstroke}%
\pgfsetdash{}{0pt}%
\pgfpathmoveto{\pgfqpoint{0.553704in}{0.499691in}}%
\pgfpathlineto{\pgfqpoint{0.553704in}{2.424691in}}%
\pgfusepath{stroke}%
\end{pgfscope}%
\begin{pgfscope}%
\pgfsetrectcap%
\pgfsetmiterjoin%
\pgfsetlinewidth{0.803000pt}%
\definecolor{currentstroke}{rgb}{0.000000,0.000000,0.000000}%
\pgfsetstrokecolor{currentstroke}%
\pgfsetdash{}{0pt}%
\pgfpathmoveto{\pgfqpoint{2.568704in}{0.499691in}}%
\pgfpathlineto{\pgfqpoint{2.568704in}{2.424691in}}%
\pgfusepath{stroke}%
\end{pgfscope}%
\begin{pgfscope}%
\pgfsetrectcap%
\pgfsetmiterjoin%
\pgfsetlinewidth{0.803000pt}%
\definecolor{currentstroke}{rgb}{0.000000,0.000000,0.000000}%
\pgfsetstrokecolor{currentstroke}%
\pgfsetdash{}{0pt}%
\pgfpathmoveto{\pgfqpoint{0.553704in}{0.499691in}}%
\pgfpathlineto{\pgfqpoint{2.568704in}{0.499691in}}%
\pgfusepath{stroke}%
\end{pgfscope}%
\begin{pgfscope}%
\pgfsetrectcap%
\pgfsetmiterjoin%
\pgfsetlinewidth{0.803000pt}%
\definecolor{currentstroke}{rgb}{0.000000,0.000000,0.000000}%
\pgfsetstrokecolor{currentstroke}%
\pgfsetdash{}{0pt}%
\pgfpathmoveto{\pgfqpoint{0.553704in}{2.424691in}}%
\pgfpathlineto{\pgfqpoint{2.568704in}{2.424691in}}%
\pgfusepath{stroke}%
\end{pgfscope}%
\begin{pgfscope}%
\pgfsetbuttcap%
\pgfsetmiterjoin%
\definecolor{currentfill}{rgb}{1.000000,1.000000,1.000000}%
\pgfsetfillcolor{currentfill}%
\pgfsetfillopacity{0.800000}%
\pgfsetlinewidth{1.003750pt}%
\definecolor{currentstroke}{rgb}{0.800000,0.800000,0.800000}%
\pgfsetstrokecolor{currentstroke}%
\pgfsetstrokeopacity{0.800000}%
\pgfsetdash{}{0pt}%
\pgfpathmoveto{\pgfqpoint{0.650926in}{0.569136in}}%
\pgfpathlineto{\pgfqpoint{2.091854in}{0.569136in}}%
\pgfpathquadraticcurveto{\pgfqpoint{2.119632in}{0.569136in}}{\pgfqpoint{2.119632in}{0.596913in}}%
\pgfpathlineto{\pgfqpoint{2.119632in}{1.551388in}}%
\pgfpathquadraticcurveto{\pgfqpoint{2.119632in}{1.579166in}}{\pgfqpoint{2.091854in}{1.579166in}}%
\pgfpathlineto{\pgfqpoint{0.650926in}{1.579166in}}%
\pgfpathquadraticcurveto{\pgfqpoint{0.623149in}{1.579166in}}{\pgfqpoint{0.623149in}{1.551388in}}%
\pgfpathlineto{\pgfqpoint{0.623149in}{0.596913in}}%
\pgfpathquadraticcurveto{\pgfqpoint{0.623149in}{0.569136in}}{\pgfqpoint{0.650926in}{0.569136in}}%
\pgfpathlineto{\pgfqpoint{0.650926in}{0.569136in}}%
\pgfpathclose%
\pgfusepath{stroke,fill}%
\end{pgfscope}%
\begin{pgfscope}%
\pgfsetrectcap%
\pgfsetroundjoin%
\pgfsetlinewidth{1.505625pt}%
\definecolor{currentstroke}{rgb}{0.105882,0.619608,0.466667}%
\pgfsetstrokecolor{currentstroke}%
\pgfsetdash{}{0pt}%
\pgfpathmoveto{\pgfqpoint{0.678704in}{1.474999in}}%
\pgfpathlineto{\pgfqpoint{0.817593in}{1.474999in}}%
\pgfpathlineto{\pgfqpoint{0.956482in}{1.474999in}}%
\pgfusepath{stroke}%
\end{pgfscope}%
\begin{pgfscope}%
\definecolor{textcolor}{rgb}{0.000000,0.000000,0.000000}%
\pgfsetstrokecolor{textcolor}%
\pgfsetfillcolor{textcolor}%
\pgftext[x=1.067593in,y=1.426388in,left,base]{\color{textcolor}\rmfamily\fontsize{10.000000}{12.000000}\selectfont \(\displaystyle \textsc{ProjectedRS}^*\)}%
\end{pgfscope}%
\begin{pgfscope}%
\pgfsetrectcap%
\pgfsetroundjoin%
\pgfsetlinewidth{1.505625pt}%
\definecolor{currentstroke}{rgb}{0.850980,0.372549,0.007843}%
\pgfsetstrokecolor{currentstroke}%
\pgfsetdash{}{0pt}%
\pgfpathmoveto{\pgfqpoint{0.678704in}{1.281327in}}%
\pgfpathlineto{\pgfqpoint{0.817593in}{1.281327in}}%
\pgfpathlineto{\pgfqpoint{0.956482in}{1.281327in}}%
\pgfusepath{stroke}%
\end{pgfscope}%
\begin{pgfscope}%
\definecolor{textcolor}{rgb}{0.000000,0.000000,0.000000}%
\pgfsetstrokecolor{textcolor}%
\pgfsetfillcolor{textcolor}%
\pgftext[x=1.067593in,y=1.232716in,left,base]{\color{textcolor}\rmfamily\fontsize{10.000000}{12.000000}\selectfont \(\displaystyle \textsc{RS}\)}%
\end{pgfscope}%
\begin{pgfscope}%
\pgfsetrectcap%
\pgfsetroundjoin%
\pgfsetlinewidth{1.505625pt}%
\definecolor{currentstroke}{rgb}{0.458824,0.439216,0.701961}%
\pgfsetstrokecolor{currentstroke}%
\pgfsetdash{}{0pt}%
\pgfpathmoveto{\pgfqpoint{0.678704in}{1.087654in}}%
\pgfpathlineto{\pgfqpoint{0.817593in}{1.087654in}}%
\pgfpathlineto{\pgfqpoint{0.956482in}{1.087654in}}%
\pgfusepath{stroke}%
\end{pgfscope}%
\begin{pgfscope}%
\definecolor{textcolor}{rgb}{0.000000,0.000000,0.000000}%
\pgfsetstrokecolor{textcolor}%
\pgfsetfillcolor{textcolor}%
\pgftext[x=1.067593in,y=1.039043in,left,base]{\color{textcolor}\rmfamily\fontsize{10.000000}{12.000000}\selectfont \(\displaystyle \textsc{ANCER}\)}%
\end{pgfscope}%
\begin{pgfscope}%
\pgfsetrectcap%
\pgfsetroundjoin%
\pgfsetlinewidth{1.505625pt}%
\definecolor{currentstroke}{rgb}{0.905882,0.160784,0.541176}%
\pgfsetstrokecolor{currentstroke}%
\pgfsetdash{}{0pt}%
\pgfpathmoveto{\pgfqpoint{0.678704in}{0.893981in}}%
\pgfpathlineto{\pgfqpoint{0.817593in}{0.893981in}}%
\pgfpathlineto{\pgfqpoint{0.956482in}{0.893981in}}%
\pgfusepath{stroke}%
\end{pgfscope}%
\begin{pgfscope}%
\definecolor{textcolor}{rgb}{0.000000,0.000000,0.000000}%
\pgfsetstrokecolor{textcolor}%
\pgfsetfillcolor{textcolor}%
\pgftext[x=1.067593in,y=0.845370in,left,base]{\color{textcolor}\rmfamily\fontsize{10.000000}{12.000000}\selectfont \(\displaystyle \textsc{RS4A}-\ell_{1}\)}%
\end{pgfscope}%
\begin{pgfscope}%
\pgfsetrectcap%
\pgfsetroundjoin%
\pgfsetlinewidth{1.505625pt}%
\definecolor{currentstroke}{rgb}{0.400000,0.650980,0.117647}%
\pgfsetstrokecolor{currentstroke}%
\pgfsetdash{}{0pt}%
\pgfpathmoveto{\pgfqpoint{0.678704in}{0.700308in}}%
\pgfpathlineto{\pgfqpoint{0.817593in}{0.700308in}}%
\pgfpathlineto{\pgfqpoint{0.956482in}{0.700308in}}%
\pgfusepath{stroke}%
\end{pgfscope}%
\begin{pgfscope}%
\definecolor{textcolor}{rgb}{0.000000,0.000000,0.000000}%
\pgfsetstrokecolor{textcolor}%
\pgfsetfillcolor{textcolor}%
\pgftext[x=1.067593in,y=0.651697in,left,base]{\color{textcolor}\rmfamily\fontsize{10.000000}{12.000000}\selectfont \(\displaystyle \textsc{RS4A}-\ell_{\infty}\)}%
\end{pgfscope}%
\end{pgfpicture}%
\makeatother%
\endgroup%

%% file: figs/svhn_main.pgf
%% Creator: Matplotlib, PGF backend
%%
%% To include the figure in your LaTeX document, write
%%   \input{<filename>.pgf}
%%
%% Make sure the required packages are loaded in your preamble
%%   \usepackage{pgf}
%%
%% Also ensure that all the required font packages are loaded; for instance,
%% the lmodern package is sometimes necessary when using math font.
%%   \usepackage{lmodern}
%%
%% Figures using additional raster images can only be included by \input if
%% they are in the same directory as the main LaTeX file. For loading figures
%% from other directories you can use the `import` package
%%   \usepackage{import}
%%
%% and then include the figures with
%%   \import{<path to file>}{<filename>.pgf}
%%
%% Matplotlib used the following preamble
%%
\begingroup%
\makeatletter%
\begin{pgfpicture}%
\pgfpathrectangle{\pgfpointorigin}{\pgfqpoint{2.668704in}{2.524691in}}%
\pgfusepath{use as bounding box, clip}%
\begin{pgfscope}%
\pgfsetbuttcap%
\pgfsetmiterjoin%
\pgfsetlinewidth{0.000000pt}%
\definecolor{currentstroke}{rgb}{0.000000,0.000000,0.000000}%
\pgfsetstrokecolor{currentstroke}%
\pgfsetstrokeopacity{0.000000}%
\pgfsetdash{}{0pt}%
\pgfpathmoveto{\pgfqpoint{0.000000in}{0.000000in}}%
\pgfpathlineto{\pgfqpoint{2.668704in}{0.000000in}}%
\pgfpathlineto{\pgfqpoint{2.668704in}{2.524691in}}%
\pgfpathlineto{\pgfqpoint{0.000000in}{2.524691in}}%
\pgfpathlineto{\pgfqpoint{0.000000in}{0.000000in}}%
\pgfpathclose%
\pgfusepath{}%
\end{pgfscope}%
\begin{pgfscope}%
\pgfsetbuttcap%
\pgfsetmiterjoin%
\pgfsetlinewidth{0.000000pt}%
\definecolor{currentstroke}{rgb}{0.000000,0.000000,0.000000}%
\pgfsetstrokecolor{currentstroke}%
\pgfsetstrokeopacity{0.000000}%
\pgfsetdash{}{0pt}%
\pgfpathmoveto{\pgfqpoint{0.553704in}{0.499691in}}%
\pgfpathlineto{\pgfqpoint{2.568704in}{0.499691in}}%
\pgfpathlineto{\pgfqpoint{2.568704in}{2.424691in}}%
\pgfpathlineto{\pgfqpoint{0.553704in}{2.424691in}}%
\pgfpathlineto{\pgfqpoint{0.553704in}{0.499691in}}%
\pgfpathclose%
\pgfusepath{}%
\end{pgfscope}%
\begin{pgfscope}%
\pgfsetbuttcap%
\pgfsetroundjoin%
\definecolor{currentfill}{rgb}{0.000000,0.000000,0.000000}%
\pgfsetfillcolor{currentfill}%
\pgfsetlinewidth{0.803000pt}%
\definecolor{currentstroke}{rgb}{0.000000,0.000000,0.000000}%
\pgfsetstrokecolor{currentstroke}%
\pgfsetdash{}{0pt}%
\pgfsys@defobject{currentmarker}{\pgfqpoint{0.000000in}{-0.048611in}}{\pgfqpoint{0.000000in}{0.000000in}}{%
\pgfpathmoveto{\pgfqpoint{0.000000in}{0.000000in}}%
\pgfpathlineto{\pgfqpoint{0.000000in}{-0.048611in}}%
\pgfusepath{stroke,fill}%
}%
\begin{pgfscope}%
\pgfsys@transformshift{0.804506in}{0.499691in}%
\pgfsys@useobject{currentmarker}{}%
\end{pgfscope}%
\end{pgfscope}%
\begin{pgfscope}%
\definecolor{textcolor}{rgb}{0.000000,0.000000,0.000000}%
\pgfsetstrokecolor{textcolor}%
\pgfsetfillcolor{textcolor}%
\pgftext[x=0.804506in,y=0.402469in,,top]{\color{textcolor}\rmfamily\fontsize{10.000000}{12.000000}\selectfont \(\displaystyle 10^{-4 \alpha}\)}%
\end{pgfscope}%
\begin{pgfscope}%
\pgfsetbuttcap%
\pgfsetroundjoin%
\definecolor{currentfill}{rgb}{0.000000,0.000000,0.000000}%
\pgfsetfillcolor{currentfill}%
\pgfsetlinewidth{0.803000pt}%
\definecolor{currentstroke}{rgb}{0.000000,0.000000,0.000000}%
\pgfsetstrokecolor{currentstroke}%
\pgfsetdash{}{0pt}%
\pgfsys@defobject{currentmarker}{\pgfqpoint{0.000000in}{-0.048611in}}{\pgfqpoint{0.000000in}{0.000000in}}{%
\pgfpathmoveto{\pgfqpoint{0.000000in}{0.000000in}}%
\pgfpathlineto{\pgfqpoint{0.000000in}{-0.048611in}}%
\pgfusepath{stroke,fill}%
}%
\begin{pgfscope}%
\pgfsys@transformshift{1.268341in}{0.499691in}%
\pgfsys@useobject{currentmarker}{}%
\end{pgfscope}%
\end{pgfscope}%
\begin{pgfscope}%
\definecolor{textcolor}{rgb}{0.000000,0.000000,0.000000}%
\pgfsetstrokecolor{textcolor}%
\pgfsetfillcolor{textcolor}%
\pgftext[x=1.268341in,y=0.402469in,,top]{\color{textcolor}\rmfamily\fontsize{10.000000}{12.000000}\selectfont \(\displaystyle 10^{-3 \alpha}\)}%
\end{pgfscope}%
\begin{pgfscope}%
\pgfsetbuttcap%
\pgfsetroundjoin%
\definecolor{currentfill}{rgb}{0.000000,0.000000,0.000000}%
\pgfsetfillcolor{currentfill}%
\pgfsetlinewidth{0.803000pt}%
\definecolor{currentstroke}{rgb}{0.000000,0.000000,0.000000}%
\pgfsetstrokecolor{currentstroke}%
\pgfsetdash{}{0pt}%
\pgfsys@defobject{currentmarker}{\pgfqpoint{0.000000in}{-0.048611in}}{\pgfqpoint{0.000000in}{0.000000in}}{%
\pgfpathmoveto{\pgfqpoint{0.000000in}{0.000000in}}%
\pgfpathlineto{\pgfqpoint{0.000000in}{-0.048611in}}%
\pgfusepath{stroke,fill}%
}%
\begin{pgfscope}%
\pgfsys@transformshift{1.732176in}{0.499691in}%
\pgfsys@useobject{currentmarker}{}%
\end{pgfscope}%
\end{pgfscope}%
\begin{pgfscope}%
\definecolor{textcolor}{rgb}{0.000000,0.000000,0.000000}%
\pgfsetstrokecolor{textcolor}%
\pgfsetfillcolor{textcolor}%
\pgftext[x=1.732176in,y=0.402469in,,top]{\color{textcolor}\rmfamily\fontsize{10.000000}{12.000000}\selectfont \(\displaystyle 10^{-2 \alpha}\)}%
\end{pgfscope}%
\begin{pgfscope}%
\pgfsetbuttcap%
\pgfsetroundjoin%
\definecolor{currentfill}{rgb}{0.000000,0.000000,0.000000}%
\pgfsetfillcolor{currentfill}%
\pgfsetlinewidth{0.803000pt}%
\definecolor{currentstroke}{rgb}{0.000000,0.000000,0.000000}%
\pgfsetstrokecolor{currentstroke}%
\pgfsetdash{}{0pt}%
\pgfsys@defobject{currentmarker}{\pgfqpoint{0.000000in}{-0.048611in}}{\pgfqpoint{0.000000in}{0.000000in}}{%
\pgfpathmoveto{\pgfqpoint{0.000000in}{0.000000in}}%
\pgfpathlineto{\pgfqpoint{0.000000in}{-0.048611in}}%
\pgfusepath{stroke,fill}%
}%
\begin{pgfscope}%
\pgfsys@transformshift{2.196011in}{0.499691in}%
\pgfsys@useobject{currentmarker}{}%
\end{pgfscope}%
\end{pgfscope}%
\begin{pgfscope}%
\definecolor{textcolor}{rgb}{0.000000,0.000000,0.000000}%
\pgfsetstrokecolor{textcolor}%
\pgfsetfillcolor{textcolor}%
\pgftext[x=2.196011in,y=0.402469in,,top]{\color{textcolor}\rmfamily\fontsize{10.000000}{12.000000}\selectfont \(\displaystyle 10^{-1 \alpha}\)}%
\end{pgfscope}%
\begin{pgfscope}%
\definecolor{textcolor}{rgb}{0.000000,0.000000,0.000000}%
\pgfsetstrokecolor{textcolor}%
\pgfsetfillcolor{textcolor}%
\pgftext[x=1.561204in,y=0.223457in,,top]{\color{textcolor}\rmfamily\fontsize{10.000000}{12.000000}\selectfont Volume}%
\end{pgfscope}%
\begin{pgfscope}%
\pgfsetbuttcap%
\pgfsetroundjoin%
\definecolor{currentfill}{rgb}{0.000000,0.000000,0.000000}%
\pgfsetfillcolor{currentfill}%
\pgfsetlinewidth{0.803000pt}%
\definecolor{currentstroke}{rgb}{0.000000,0.000000,0.000000}%
\pgfsetstrokecolor{currentstroke}%
\pgfsetdash{}{0pt}%
\pgfsys@defobject{currentmarker}{\pgfqpoint{-0.048611in}{0.000000in}}{\pgfqpoint{-0.000000in}{0.000000in}}{%
\pgfpathmoveto{\pgfqpoint{-0.000000in}{0.000000in}}%
\pgfpathlineto{\pgfqpoint{-0.048611in}{0.000000in}}%
\pgfusepath{stroke,fill}%
}%
\begin{pgfscope}%
\pgfsys@transformshift{0.553704in}{0.587191in}%
\pgfsys@useobject{currentmarker}{}%
\end{pgfscope}%
\end{pgfscope}%
\begin{pgfscope}%
\definecolor{textcolor}{rgb}{0.000000,0.000000,0.000000}%
\pgfsetstrokecolor{textcolor}%
\pgfsetfillcolor{textcolor}%
\pgftext[x=0.279012in, y=0.538966in, left, base]{\color{textcolor}\rmfamily\fontsize{10.000000}{12.000000}\selectfont \(\displaystyle {0.0}\)}%
\end{pgfscope}%
\begin{pgfscope}%
\pgfsetbuttcap%
\pgfsetroundjoin%
\definecolor{currentfill}{rgb}{0.000000,0.000000,0.000000}%
\pgfsetfillcolor{currentfill}%
\pgfsetlinewidth{0.803000pt}%
\definecolor{currentstroke}{rgb}{0.000000,0.000000,0.000000}%
\pgfsetstrokecolor{currentstroke}%
\pgfsetdash{}{0pt}%
\pgfsys@defobject{currentmarker}{\pgfqpoint{-0.048611in}{0.000000in}}{\pgfqpoint{-0.000000in}{0.000000in}}{%
\pgfpathmoveto{\pgfqpoint{-0.000000in}{0.000000in}}%
\pgfpathlineto{\pgfqpoint{-0.048611in}{0.000000in}}%
\pgfusepath{stroke,fill}%
}%
\begin{pgfscope}%
\pgfsys@transformshift{0.553704in}{0.964346in}%
\pgfsys@useobject{currentmarker}{}%
\end{pgfscope}%
\end{pgfscope}%
\begin{pgfscope}%
\definecolor{textcolor}{rgb}{0.000000,0.000000,0.000000}%
\pgfsetstrokecolor{textcolor}%
\pgfsetfillcolor{textcolor}%
\pgftext[x=0.279012in, y=0.916121in, left, base]{\color{textcolor}\rmfamily\fontsize{10.000000}{12.000000}\selectfont \(\displaystyle {0.2}\)}%
\end{pgfscope}%
\begin{pgfscope}%
\pgfsetbuttcap%
\pgfsetroundjoin%
\definecolor{currentfill}{rgb}{0.000000,0.000000,0.000000}%
\pgfsetfillcolor{currentfill}%
\pgfsetlinewidth{0.803000pt}%
\definecolor{currentstroke}{rgb}{0.000000,0.000000,0.000000}%
\pgfsetstrokecolor{currentstroke}%
\pgfsetdash{}{0pt}%
\pgfsys@defobject{currentmarker}{\pgfqpoint{-0.048611in}{0.000000in}}{\pgfqpoint{-0.000000in}{0.000000in}}{%
\pgfpathmoveto{\pgfqpoint{-0.000000in}{0.000000in}}%
\pgfpathlineto{\pgfqpoint{-0.048611in}{0.000000in}}%
\pgfusepath{stroke,fill}%
}%
\begin{pgfscope}%
\pgfsys@transformshift{0.553704in}{1.341501in}%
\pgfsys@useobject{currentmarker}{}%
\end{pgfscope}%
\end{pgfscope}%
\begin{pgfscope}%
\definecolor{textcolor}{rgb}{0.000000,0.000000,0.000000}%
\pgfsetstrokecolor{textcolor}%
\pgfsetfillcolor{textcolor}%
\pgftext[x=0.279012in, y=1.293276in, left, base]{\color{textcolor}\rmfamily\fontsize{10.000000}{12.000000}\selectfont \(\displaystyle {0.4}\)}%
\end{pgfscope}%
\begin{pgfscope}%
\pgfsetbuttcap%
\pgfsetroundjoin%
\definecolor{currentfill}{rgb}{0.000000,0.000000,0.000000}%
\pgfsetfillcolor{currentfill}%
\pgfsetlinewidth{0.803000pt}%
\definecolor{currentstroke}{rgb}{0.000000,0.000000,0.000000}%
\pgfsetstrokecolor{currentstroke}%
\pgfsetdash{}{0pt}%
\pgfsys@defobject{currentmarker}{\pgfqpoint{-0.048611in}{0.000000in}}{\pgfqpoint{-0.000000in}{0.000000in}}{%
\pgfpathmoveto{\pgfqpoint{-0.000000in}{0.000000in}}%
\pgfpathlineto{\pgfqpoint{-0.048611in}{0.000000in}}%
\pgfusepath{stroke,fill}%
}%
\begin{pgfscope}%
\pgfsys@transformshift{0.553704in}{1.718657in}%
\pgfsys@useobject{currentmarker}{}%
\end{pgfscope}%
\end{pgfscope}%
\begin{pgfscope}%
\definecolor{textcolor}{rgb}{0.000000,0.000000,0.000000}%
\pgfsetstrokecolor{textcolor}%
\pgfsetfillcolor{textcolor}%
\pgftext[x=0.279012in, y=1.670431in, left, base]{\color{textcolor}\rmfamily\fontsize{10.000000}{12.000000}\selectfont \(\displaystyle {0.6}\)}%
\end{pgfscope}%
\begin{pgfscope}%
\pgfsetbuttcap%
\pgfsetroundjoin%
\definecolor{currentfill}{rgb}{0.000000,0.000000,0.000000}%
\pgfsetfillcolor{currentfill}%
\pgfsetlinewidth{0.803000pt}%
\definecolor{currentstroke}{rgb}{0.000000,0.000000,0.000000}%
\pgfsetstrokecolor{currentstroke}%
\pgfsetdash{}{0pt}%
\pgfsys@defobject{currentmarker}{\pgfqpoint{-0.048611in}{0.000000in}}{\pgfqpoint{-0.000000in}{0.000000in}}{%
\pgfpathmoveto{\pgfqpoint{-0.000000in}{0.000000in}}%
\pgfpathlineto{\pgfqpoint{-0.048611in}{0.000000in}}%
\pgfusepath{stroke,fill}%
}%
\begin{pgfscope}%
\pgfsys@transformshift{0.553704in}{2.095812in}%
\pgfsys@useobject{currentmarker}{}%
\end{pgfscope}%
\end{pgfscope}%
\begin{pgfscope}%
\definecolor{textcolor}{rgb}{0.000000,0.000000,0.000000}%
\pgfsetstrokecolor{textcolor}%
\pgfsetfillcolor{textcolor}%
\pgftext[x=0.279012in, y=2.047587in, left, base]{\color{textcolor}\rmfamily\fontsize{10.000000}{12.000000}\selectfont \(\displaystyle {0.8}\)}%
\end{pgfscope}%
\begin{pgfscope}%
\definecolor{textcolor}{rgb}{0.000000,0.000000,0.000000}%
\pgfsetstrokecolor{textcolor}%
\pgfsetfillcolor{textcolor}%
\pgftext[x=0.223457in,y=1.462191in,,bottom,rotate=90.000000]{\color{textcolor}\rmfamily\fontsize{10.000000}{12.000000}\selectfont SVHN certified accuracy}%
\end{pgfscope}%
\begin{pgfscope}%
\pgfpathrectangle{\pgfqpoint{0.553704in}{0.499691in}}{\pgfqpoint{2.015000in}{1.925000in}}%
\pgfusepath{clip}%
\pgfsetrectcap%
\pgfsetroundjoin%
\pgfsetlinewidth{1.505625pt}%
\definecolor{currentstroke}{rgb}{0.105882,0.619608,0.466667}%
\pgfsetstrokecolor{currentstroke}%
\pgfsetdash{}{0pt}%
\pgfpathmoveto{\pgfqpoint{0.553704in}{2.310790in}}%
\pgfpathlineto{\pgfqpoint{1.847617in}{2.310790in}}%
\pgfpathlineto{\pgfqpoint{1.861095in}{2.303247in}}%
\pgfpathlineto{\pgfqpoint{1.867835in}{2.303247in}}%
\pgfpathlineto{\pgfqpoint{1.874574in}{2.295704in}}%
\pgfpathlineto{\pgfqpoint{1.901530in}{2.295704in}}%
\pgfpathlineto{\pgfqpoint{1.908269in}{2.291932in}}%
\pgfpathlineto{\pgfqpoint{1.921748in}{2.291932in}}%
\pgfpathlineto{\pgfqpoint{1.928487in}{2.288161in}}%
\pgfpathlineto{\pgfqpoint{1.962182in}{2.288161in}}%
\pgfpathlineto{\pgfqpoint{1.968922in}{2.284389in}}%
\pgfpathlineto{\pgfqpoint{1.995878in}{2.284389in}}%
\pgfpathlineto{\pgfqpoint{2.002617in}{2.280618in}}%
\pgfpathlineto{\pgfqpoint{2.049791in}{2.280618in}}%
\pgfpathlineto{\pgfqpoint{2.070009in}{2.269303in}}%
\pgfpathlineto{\pgfqpoint{2.076748in}{2.261760in}}%
\pgfpathlineto{\pgfqpoint{2.083487in}{2.261760in}}%
\pgfpathlineto{\pgfqpoint{2.096965in}{2.239131in}}%
\pgfpathlineto{\pgfqpoint{2.110443in}{2.239131in}}%
\pgfpathlineto{\pgfqpoint{2.130661in}{2.227816in}}%
\pgfpathlineto{\pgfqpoint{2.144139in}{2.227816in}}%
\pgfpathlineto{\pgfqpoint{2.150878in}{2.216501in}}%
\pgfpathlineto{\pgfqpoint{2.177835in}{2.201415in}}%
\pgfpathlineto{\pgfqpoint{2.184574in}{2.186329in}}%
\pgfpathlineto{\pgfqpoint{2.191313in}{2.182557in}}%
\pgfpathlineto{\pgfqpoint{2.198052in}{2.182557in}}%
\pgfpathlineto{\pgfqpoint{2.204791in}{2.167471in}}%
\pgfpathlineto{\pgfqpoint{2.211530in}{2.163700in}}%
\pgfpathlineto{\pgfqpoint{2.218269in}{2.156157in}}%
\pgfpathlineto{\pgfqpoint{2.225009in}{2.129756in}}%
\pgfpathlineto{\pgfqpoint{2.231748in}{2.122213in}}%
\pgfpathlineto{\pgfqpoint{2.238487in}{2.118441in}}%
\pgfpathlineto{\pgfqpoint{2.245226in}{2.095812in}}%
\pgfpathlineto{\pgfqpoint{2.251965in}{2.092040in}}%
\pgfpathlineto{\pgfqpoint{2.258704in}{2.073182in}}%
\pgfpathlineto{\pgfqpoint{2.265443in}{2.065639in}}%
\pgfpathlineto{\pgfqpoint{2.272182in}{2.046782in}}%
\pgfpathlineto{\pgfqpoint{2.278922in}{2.024152in}}%
\pgfpathlineto{\pgfqpoint{2.285661in}{2.016609in}}%
\pgfpathlineto{\pgfqpoint{2.292400in}{1.990208in}}%
\pgfpathlineto{\pgfqpoint{2.299139in}{1.978894in}}%
\pgfpathlineto{\pgfqpoint{2.305878in}{1.963807in}}%
\pgfpathlineto{\pgfqpoint{2.312617in}{1.933635in}}%
\pgfpathlineto{\pgfqpoint{2.319356in}{1.926092in}}%
\pgfpathlineto{\pgfqpoint{2.326095in}{1.926092in}}%
\pgfpathlineto{\pgfqpoint{2.332835in}{1.914777in}}%
\pgfpathlineto{\pgfqpoint{2.339574in}{1.895920in}}%
\pgfpathlineto{\pgfqpoint{2.346313in}{1.880833in}}%
\pgfpathlineto{\pgfqpoint{2.353052in}{1.858204in}}%
\pgfpathlineto{\pgfqpoint{2.359791in}{1.850661in}}%
\pgfpathlineto{\pgfqpoint{2.386748in}{1.729971in}}%
\pgfpathlineto{\pgfqpoint{2.393487in}{1.722428in}}%
\pgfpathlineto{\pgfqpoint{2.400226in}{1.692256in}}%
\pgfpathlineto{\pgfqpoint{2.406965in}{1.669626in}}%
\pgfpathlineto{\pgfqpoint{2.420443in}{1.590424in}}%
\pgfpathlineto{\pgfqpoint{2.427182in}{1.567795in}}%
\pgfpathlineto{\pgfqpoint{2.433922in}{1.537622in}}%
\pgfpathlineto{\pgfqpoint{2.440661in}{1.499907in}}%
\pgfpathlineto{\pgfqpoint{2.447400in}{1.454648in}}%
\pgfpathlineto{\pgfqpoint{2.454139in}{1.432019in}}%
\pgfpathlineto{\pgfqpoint{2.460878in}{1.386760in}}%
\pgfpathlineto{\pgfqpoint{2.467617in}{1.356588in}}%
\pgfpathlineto{\pgfqpoint{2.487835in}{1.130295in}}%
\pgfpathlineto{\pgfqpoint{2.494574in}{1.036006in}}%
\pgfpathlineto{\pgfqpoint{2.514791in}{0.885144in}}%
\pgfpathlineto{\pgfqpoint{2.521530in}{0.809713in}}%
\pgfpathlineto{\pgfqpoint{2.535009in}{0.722967in}}%
\pgfpathlineto{\pgfqpoint{2.541748in}{0.673937in}}%
\pgfpathlineto{\pgfqpoint{2.548487in}{0.643764in}}%
\pgfpathlineto{\pgfqpoint{2.555226in}{0.632450in}}%
\pgfpathlineto{\pgfqpoint{2.561965in}{0.617364in}}%
\pgfpathlineto{\pgfqpoint{2.568704in}{0.587191in}}%
\pgfpathlineto{\pgfqpoint{2.568704in}{0.587191in}}%
\pgfusepath{stroke}%
\end{pgfscope}%
\begin{pgfscope}%
\pgfpathrectangle{\pgfqpoint{0.553704in}{0.499691in}}{\pgfqpoint{2.015000in}{1.925000in}}%
\pgfusepath{clip}%
\pgfsetrectcap%
\pgfsetroundjoin%
\pgfsetlinewidth{1.505625pt}%
\definecolor{currentstroke}{rgb}{0.850980,0.372549,0.007843}%
\pgfsetstrokecolor{currentstroke}%
\pgfsetdash{}{0pt}%
\pgfpathmoveto{\pgfqpoint{0.553704in}{2.333420in}}%
\pgfpathlineto{\pgfqpoint{1.214139in}{2.333420in}}%
\pgfpathlineto{\pgfqpoint{1.220878in}{2.329648in}}%
\pgfpathlineto{\pgfqpoint{1.301748in}{2.329648in}}%
\pgfpathlineto{\pgfqpoint{1.308487in}{2.325876in}}%
\pgfpathlineto{\pgfqpoint{1.497182in}{2.325876in}}%
\pgfpathlineto{\pgfqpoint{1.503922in}{2.322105in}}%
\pgfpathlineto{\pgfqpoint{1.524139in}{2.322105in}}%
\pgfpathlineto{\pgfqpoint{1.530878in}{2.318333in}}%
\pgfpathlineto{\pgfqpoint{1.557835in}{2.318333in}}%
\pgfpathlineto{\pgfqpoint{1.564574in}{2.310790in}}%
\pgfpathlineto{\pgfqpoint{1.591530in}{2.310790in}}%
\pgfpathlineto{\pgfqpoint{1.598269in}{2.307019in}}%
\pgfpathlineto{\pgfqpoint{1.611748in}{2.307019in}}%
\pgfpathlineto{\pgfqpoint{1.618487in}{2.299476in}}%
\pgfpathlineto{\pgfqpoint{1.652182in}{2.299476in}}%
\pgfpathlineto{\pgfqpoint{1.658922in}{2.295704in}}%
\pgfpathlineto{\pgfqpoint{1.672400in}{2.295704in}}%
\pgfpathlineto{\pgfqpoint{1.679139in}{2.291932in}}%
\pgfpathlineto{\pgfqpoint{1.685878in}{2.291932in}}%
\pgfpathlineto{\pgfqpoint{1.692617in}{2.288161in}}%
\pgfpathlineto{\pgfqpoint{1.706095in}{2.288161in}}%
\pgfpathlineto{\pgfqpoint{1.712835in}{2.284389in}}%
\pgfpathlineto{\pgfqpoint{1.719574in}{2.284389in}}%
\pgfpathlineto{\pgfqpoint{1.726313in}{2.276846in}}%
\pgfpathlineto{\pgfqpoint{1.733052in}{2.276846in}}%
\pgfpathlineto{\pgfqpoint{1.739791in}{2.273075in}}%
\pgfpathlineto{\pgfqpoint{1.753269in}{2.273075in}}%
\pgfpathlineto{\pgfqpoint{1.760009in}{2.265532in}}%
\pgfpathlineto{\pgfqpoint{1.800443in}{2.242902in}}%
\pgfpathlineto{\pgfqpoint{1.807182in}{2.242902in}}%
\pgfpathlineto{\pgfqpoint{1.813922in}{2.235359in}}%
\pgfpathlineto{\pgfqpoint{1.827400in}{2.212730in}}%
\pgfpathlineto{\pgfqpoint{1.834139in}{2.212730in}}%
\pgfpathlineto{\pgfqpoint{1.840878in}{2.201415in}}%
\pgfpathlineto{\pgfqpoint{1.847617in}{2.197644in}}%
\pgfpathlineto{\pgfqpoint{1.854356in}{2.197644in}}%
\pgfpathlineto{\pgfqpoint{1.861095in}{2.193872in}}%
\pgfpathlineto{\pgfqpoint{1.867835in}{2.186329in}}%
\pgfpathlineto{\pgfqpoint{1.874574in}{2.175014in}}%
\pgfpathlineto{\pgfqpoint{1.881313in}{2.171243in}}%
\pgfpathlineto{\pgfqpoint{1.888052in}{2.163700in}}%
\pgfpathlineto{\pgfqpoint{1.901530in}{2.133527in}}%
\pgfpathlineto{\pgfqpoint{1.908269in}{2.114670in}}%
\pgfpathlineto{\pgfqpoint{1.915009in}{2.110898in}}%
\pgfpathlineto{\pgfqpoint{1.921748in}{2.099583in}}%
\pgfpathlineto{\pgfqpoint{1.928487in}{2.080726in}}%
\pgfpathlineto{\pgfqpoint{1.935226in}{2.069411in}}%
\pgfpathlineto{\pgfqpoint{1.941965in}{2.065639in}}%
\pgfpathlineto{\pgfqpoint{1.948704in}{2.046782in}}%
\pgfpathlineto{\pgfqpoint{1.955443in}{2.035467in}}%
\pgfpathlineto{\pgfqpoint{1.962182in}{2.035467in}}%
\pgfpathlineto{\pgfqpoint{1.968922in}{2.024152in}}%
\pgfpathlineto{\pgfqpoint{1.975661in}{2.009066in}}%
\pgfpathlineto{\pgfqpoint{1.989139in}{1.971351in}}%
\pgfpathlineto{\pgfqpoint{1.995878in}{1.956264in}}%
\pgfpathlineto{\pgfqpoint{2.036313in}{1.820489in}}%
\pgfpathlineto{\pgfqpoint{2.043052in}{1.782773in}}%
\pgfpathlineto{\pgfqpoint{2.056530in}{1.722428in}}%
\pgfpathlineto{\pgfqpoint{2.063269in}{1.684713in}}%
\pgfpathlineto{\pgfqpoint{2.070009in}{1.597967in}}%
\pgfpathlineto{\pgfqpoint{2.076748in}{1.575338in}}%
\pgfpathlineto{\pgfqpoint{2.083487in}{1.507450in}}%
\pgfpathlineto{\pgfqpoint{2.090226in}{1.420704in}}%
\pgfpathlineto{\pgfqpoint{2.096965in}{0.587191in}}%
\pgfpathlineto{\pgfqpoint{2.568704in}{0.587191in}}%
\pgfpathlineto{\pgfqpoint{2.568704in}{0.587191in}}%
\pgfusepath{stroke}%
\end{pgfscope}%
\begin{pgfscope}%
\pgfpathrectangle{\pgfqpoint{0.553704in}{0.499691in}}{\pgfqpoint{2.015000in}{1.925000in}}%
\pgfusepath{clip}%
\pgfsetrectcap%
\pgfsetroundjoin%
\pgfsetlinewidth{1.505625pt}%
\definecolor{currentstroke}{rgb}{0.458824,0.439216,0.701961}%
\pgfsetstrokecolor{currentstroke}%
\pgfsetdash{}{0pt}%
\pgfpathmoveto{\pgfqpoint{0.553704in}{2.307019in}}%
\pgfpathlineto{\pgfqpoint{1.564574in}{2.307019in}}%
\pgfpathlineto{\pgfqpoint{1.571313in}{2.303247in}}%
\pgfpathlineto{\pgfqpoint{1.598269in}{2.303247in}}%
\pgfpathlineto{\pgfqpoint{1.605009in}{2.299476in}}%
\pgfpathlineto{\pgfqpoint{1.800443in}{2.299476in}}%
\pgfpathlineto{\pgfqpoint{1.807182in}{2.295704in}}%
\pgfpathlineto{\pgfqpoint{1.813922in}{2.295704in}}%
\pgfpathlineto{\pgfqpoint{1.840878in}{2.280618in}}%
\pgfpathlineto{\pgfqpoint{1.861095in}{2.280618in}}%
\pgfpathlineto{\pgfqpoint{1.867835in}{2.273075in}}%
\pgfpathlineto{\pgfqpoint{1.874574in}{2.269303in}}%
\pgfpathlineto{\pgfqpoint{1.881313in}{2.269303in}}%
\pgfpathlineto{\pgfqpoint{1.888052in}{2.265532in}}%
\pgfpathlineto{\pgfqpoint{1.894791in}{2.257989in}}%
\pgfpathlineto{\pgfqpoint{1.901530in}{2.257989in}}%
\pgfpathlineto{\pgfqpoint{1.908269in}{2.254217in}}%
\pgfpathlineto{\pgfqpoint{1.915009in}{2.254217in}}%
\pgfpathlineto{\pgfqpoint{1.921748in}{2.242902in}}%
\pgfpathlineto{\pgfqpoint{1.928487in}{2.242902in}}%
\pgfpathlineto{\pgfqpoint{1.941965in}{2.227816in}}%
\pgfpathlineto{\pgfqpoint{1.948704in}{2.208958in}}%
\pgfpathlineto{\pgfqpoint{1.968922in}{2.186329in}}%
\pgfpathlineto{\pgfqpoint{1.975661in}{2.167471in}}%
\pgfpathlineto{\pgfqpoint{1.982400in}{2.159928in}}%
\pgfpathlineto{\pgfqpoint{1.989139in}{2.148614in}}%
\pgfpathlineto{\pgfqpoint{1.995878in}{2.141070in}}%
\pgfpathlineto{\pgfqpoint{2.002617in}{2.129756in}}%
\pgfpathlineto{\pgfqpoint{2.009356in}{2.122213in}}%
\pgfpathlineto{\pgfqpoint{2.016095in}{2.118441in}}%
\pgfpathlineto{\pgfqpoint{2.022835in}{2.099583in}}%
\pgfpathlineto{\pgfqpoint{2.029574in}{2.084497in}}%
\pgfpathlineto{\pgfqpoint{2.036313in}{2.054325in}}%
\pgfpathlineto{\pgfqpoint{2.043052in}{2.005295in}}%
\pgfpathlineto{\pgfqpoint{2.049791in}{1.982665in}}%
\pgfpathlineto{\pgfqpoint{2.056530in}{1.952493in}}%
\pgfpathlineto{\pgfqpoint{2.076748in}{1.801631in}}%
\pgfpathlineto{\pgfqpoint{2.083487in}{1.763915in}}%
\pgfpathlineto{\pgfqpoint{2.090226in}{1.699799in}}%
\pgfpathlineto{\pgfqpoint{2.103704in}{1.605510in}}%
\pgfpathlineto{\pgfqpoint{2.110443in}{1.552708in}}%
\pgfpathlineto{\pgfqpoint{2.117182in}{1.484820in}}%
\pgfpathlineto{\pgfqpoint{2.130661in}{1.330187in}}%
\pgfpathlineto{\pgfqpoint{2.137400in}{1.228355in}}%
\pgfpathlineto{\pgfqpoint{2.144139in}{1.160467in}}%
\pgfpathlineto{\pgfqpoint{2.150878in}{1.103894in}}%
\pgfpathlineto{\pgfqpoint{2.157617in}{1.032234in}}%
\pgfpathlineto{\pgfqpoint{2.164356in}{0.990747in}}%
\pgfpathlineto{\pgfqpoint{2.171095in}{0.937945in}}%
\pgfpathlineto{\pgfqpoint{2.177835in}{0.862514in}}%
\pgfpathlineto{\pgfqpoint{2.184574in}{0.824799in}}%
\pgfpathlineto{\pgfqpoint{2.198052in}{0.771997in}}%
\pgfpathlineto{\pgfqpoint{2.211530in}{0.700338in}}%
\pgfpathlineto{\pgfqpoint{2.218269in}{0.681480in}}%
\pgfpathlineto{\pgfqpoint{2.225009in}{0.651307in}}%
\pgfpathlineto{\pgfqpoint{2.231748in}{0.643764in}}%
\pgfpathlineto{\pgfqpoint{2.238487in}{0.628678in}}%
\pgfpathlineto{\pgfqpoint{2.245226in}{0.621135in}}%
\pgfpathlineto{\pgfqpoint{2.251965in}{0.621135in}}%
\pgfpathlineto{\pgfqpoint{2.258704in}{0.606049in}}%
\pgfpathlineto{\pgfqpoint{2.265443in}{0.606049in}}%
\pgfpathlineto{\pgfqpoint{2.272182in}{0.598506in}}%
\pgfpathlineto{\pgfqpoint{2.278922in}{0.598506in}}%
\pgfpathlineto{\pgfqpoint{2.292400in}{0.590963in}}%
\pgfpathlineto{\pgfqpoint{2.299139in}{0.590963in}}%
\pgfpathlineto{\pgfqpoint{2.305878in}{0.587191in}}%
\pgfpathlineto{\pgfqpoint{2.568704in}{0.587191in}}%
\pgfpathlineto{\pgfqpoint{2.568704in}{0.587191in}}%
\pgfusepath{stroke}%
\end{pgfscope}%
\begin{pgfscope}%
\pgfpathrectangle{\pgfqpoint{0.553704in}{0.499691in}}{\pgfqpoint{2.015000in}{1.925000in}}%
\pgfusepath{clip}%
\pgfsetrectcap%
\pgfsetroundjoin%
\pgfsetlinewidth{1.505625pt}%
\definecolor{currentstroke}{rgb}{0.905882,0.160784,0.541176}%
\pgfsetstrokecolor{currentstroke}%
\pgfsetdash{}{0pt}%
\pgfpathmoveto{\pgfqpoint{0.553704in}{2.337191in}}%
\pgfpathlineto{\pgfqpoint{0.910878in}{2.337191in}}%
\pgfpathlineto{\pgfqpoint{0.917617in}{2.333420in}}%
\pgfpathlineto{\pgfqpoint{0.971530in}{2.333420in}}%
\pgfpathlineto{\pgfqpoint{0.978269in}{2.329648in}}%
\pgfpathlineto{\pgfqpoint{0.991748in}{2.329648in}}%
\pgfpathlineto{\pgfqpoint{0.998487in}{2.325876in}}%
\pgfpathlineto{\pgfqpoint{1.025443in}{2.325876in}}%
\pgfpathlineto{\pgfqpoint{1.032182in}{2.322105in}}%
\pgfpathlineto{\pgfqpoint{1.052400in}{2.322105in}}%
\pgfpathlineto{\pgfqpoint{1.059139in}{2.318333in}}%
\pgfpathlineto{\pgfqpoint{1.079356in}{2.318333in}}%
\pgfpathlineto{\pgfqpoint{1.086095in}{2.314562in}}%
\pgfpathlineto{\pgfqpoint{1.133269in}{2.314562in}}%
\pgfpathlineto{\pgfqpoint{1.146748in}{2.307019in}}%
\pgfpathlineto{\pgfqpoint{1.166965in}{2.307019in}}%
\pgfpathlineto{\pgfqpoint{1.173704in}{2.303247in}}%
\pgfpathlineto{\pgfqpoint{1.193922in}{2.303247in}}%
\pgfpathlineto{\pgfqpoint{1.200661in}{2.299476in}}%
\pgfpathlineto{\pgfqpoint{1.214139in}{2.299476in}}%
\pgfpathlineto{\pgfqpoint{1.220878in}{2.295704in}}%
\pgfpathlineto{\pgfqpoint{1.241095in}{2.295704in}}%
\pgfpathlineto{\pgfqpoint{1.247835in}{2.291932in}}%
\pgfpathlineto{\pgfqpoint{1.268052in}{2.291932in}}%
\pgfpathlineto{\pgfqpoint{1.288269in}{2.280618in}}%
\pgfpathlineto{\pgfqpoint{1.295009in}{2.280618in}}%
\pgfpathlineto{\pgfqpoint{1.301748in}{2.269303in}}%
\pgfpathlineto{\pgfqpoint{1.315226in}{2.261760in}}%
\pgfpathlineto{\pgfqpoint{1.321965in}{2.254217in}}%
\pgfpathlineto{\pgfqpoint{1.328704in}{2.239131in}}%
\pgfpathlineto{\pgfqpoint{1.335443in}{2.220273in}}%
\pgfpathlineto{\pgfqpoint{1.342182in}{2.205187in}}%
\pgfpathlineto{\pgfqpoint{1.348922in}{2.201415in}}%
\pgfpathlineto{\pgfqpoint{1.355661in}{2.186329in}}%
\pgfpathlineto{\pgfqpoint{1.362400in}{2.159928in}}%
\pgfpathlineto{\pgfqpoint{1.369139in}{2.114670in}}%
\pgfpathlineto{\pgfqpoint{1.375878in}{2.012838in}}%
\pgfpathlineto{\pgfqpoint{1.382617in}{0.587191in}}%
\pgfpathlineto{\pgfqpoint{2.568704in}{0.587191in}}%
\pgfpathlineto{\pgfqpoint{2.568704in}{0.587191in}}%
\pgfusepath{stroke}%
\end{pgfscope}%
\begin{pgfscope}%
\pgfpathrectangle{\pgfqpoint{0.553704in}{0.499691in}}{\pgfqpoint{2.015000in}{1.925000in}}%
\pgfusepath{clip}%
\pgfsetrectcap%
\pgfsetroundjoin%
\pgfsetlinewidth{1.505625pt}%
\definecolor{currentstroke}{rgb}{0.400000,0.650980,0.117647}%
\pgfsetstrokecolor{currentstroke}%
\pgfsetdash{}{0pt}%
\pgfpathmoveto{\pgfqpoint{0.553704in}{2.333420in}}%
\pgfpathlineto{\pgfqpoint{0.944574in}{2.333420in}}%
\pgfpathlineto{\pgfqpoint{0.951313in}{2.329648in}}%
\pgfpathlineto{\pgfqpoint{1.052400in}{2.329648in}}%
\pgfpathlineto{\pgfqpoint{1.059139in}{2.325876in}}%
\pgfpathlineto{\pgfqpoint{1.241095in}{2.325876in}}%
\pgfpathlineto{\pgfqpoint{1.247835in}{2.322105in}}%
\pgfpathlineto{\pgfqpoint{1.281530in}{2.322105in}}%
\pgfpathlineto{\pgfqpoint{1.288269in}{2.318333in}}%
\pgfpathlineto{\pgfqpoint{1.301748in}{2.318333in}}%
\pgfpathlineto{\pgfqpoint{1.308487in}{2.314562in}}%
\pgfpathlineto{\pgfqpoint{1.315226in}{2.314562in}}%
\pgfpathlineto{\pgfqpoint{1.321965in}{2.310790in}}%
\pgfpathlineto{\pgfqpoint{1.342182in}{2.310790in}}%
\pgfpathlineto{\pgfqpoint{1.348922in}{2.307019in}}%
\pgfpathlineto{\pgfqpoint{1.355661in}{2.307019in}}%
\pgfpathlineto{\pgfqpoint{1.369139in}{2.299476in}}%
\pgfpathlineto{\pgfqpoint{1.402835in}{2.299476in}}%
\pgfpathlineto{\pgfqpoint{1.409574in}{2.295704in}}%
\pgfpathlineto{\pgfqpoint{1.416313in}{2.295704in}}%
\pgfpathlineto{\pgfqpoint{1.423052in}{2.291932in}}%
\pgfpathlineto{\pgfqpoint{1.429791in}{2.291932in}}%
\pgfpathlineto{\pgfqpoint{1.436530in}{2.288161in}}%
\pgfpathlineto{\pgfqpoint{1.450009in}{2.288161in}}%
\pgfpathlineto{\pgfqpoint{1.456748in}{2.284389in}}%
\pgfpathlineto{\pgfqpoint{1.463487in}{2.284389in}}%
\pgfpathlineto{\pgfqpoint{1.483704in}{2.273075in}}%
\pgfpathlineto{\pgfqpoint{1.497182in}{2.273075in}}%
\pgfpathlineto{\pgfqpoint{1.503922in}{2.269303in}}%
\pgfpathlineto{\pgfqpoint{1.510661in}{2.261760in}}%
\pgfpathlineto{\pgfqpoint{1.517400in}{2.261760in}}%
\pgfpathlineto{\pgfqpoint{1.524139in}{2.254217in}}%
\pgfpathlineto{\pgfqpoint{1.530878in}{2.254217in}}%
\pgfpathlineto{\pgfqpoint{1.537617in}{2.246674in}}%
\pgfpathlineto{\pgfqpoint{1.544356in}{2.242902in}}%
\pgfpathlineto{\pgfqpoint{1.551095in}{2.242902in}}%
\pgfpathlineto{\pgfqpoint{1.564574in}{2.235359in}}%
\pgfpathlineto{\pgfqpoint{1.571313in}{2.220273in}}%
\pgfpathlineto{\pgfqpoint{1.578052in}{2.212730in}}%
\pgfpathlineto{\pgfqpoint{1.584791in}{2.212730in}}%
\pgfpathlineto{\pgfqpoint{1.591530in}{2.197644in}}%
\pgfpathlineto{\pgfqpoint{1.605009in}{2.197644in}}%
\pgfpathlineto{\pgfqpoint{1.611748in}{2.190101in}}%
\pgfpathlineto{\pgfqpoint{1.618487in}{2.186329in}}%
\pgfpathlineto{\pgfqpoint{1.625226in}{2.171243in}}%
\pgfpathlineto{\pgfqpoint{1.631965in}{2.167471in}}%
\pgfpathlineto{\pgfqpoint{1.638704in}{2.156157in}}%
\pgfpathlineto{\pgfqpoint{1.658922in}{2.110898in}}%
\pgfpathlineto{\pgfqpoint{1.665661in}{2.107126in}}%
\pgfpathlineto{\pgfqpoint{1.672400in}{2.084497in}}%
\pgfpathlineto{\pgfqpoint{1.679139in}{2.076954in}}%
\pgfpathlineto{\pgfqpoint{1.685878in}{2.065639in}}%
\pgfpathlineto{\pgfqpoint{1.692617in}{2.058096in}}%
\pgfpathlineto{\pgfqpoint{1.706095in}{2.035467in}}%
\pgfpathlineto{\pgfqpoint{1.712835in}{2.027924in}}%
\pgfpathlineto{\pgfqpoint{1.719574in}{2.016609in}}%
\pgfpathlineto{\pgfqpoint{1.726313in}{1.997751in}}%
\pgfpathlineto{\pgfqpoint{1.739791in}{1.967579in}}%
\pgfpathlineto{\pgfqpoint{1.753269in}{1.914777in}}%
\pgfpathlineto{\pgfqpoint{1.760009in}{1.895920in}}%
\pgfpathlineto{\pgfqpoint{1.780226in}{1.828032in}}%
\pgfpathlineto{\pgfqpoint{1.786965in}{1.786545in}}%
\pgfpathlineto{\pgfqpoint{1.793704in}{1.767687in}}%
\pgfpathlineto{\pgfqpoint{1.800443in}{1.729971in}}%
\pgfpathlineto{\pgfqpoint{1.807182in}{1.699799in}}%
\pgfpathlineto{\pgfqpoint{1.813922in}{1.628139in}}%
\pgfpathlineto{\pgfqpoint{1.827400in}{1.522536in}}%
\pgfpathlineto{\pgfqpoint{1.834139in}{1.443333in}}%
\pgfpathlineto{\pgfqpoint{1.840878in}{1.311329in}}%
\pgfpathlineto{\pgfqpoint{1.847617in}{0.587191in}}%
\pgfpathlineto{\pgfqpoint{2.568704in}{0.587191in}}%
\pgfpathlineto{\pgfqpoint{2.568704in}{0.587191in}}%
\pgfusepath{stroke}%
\end{pgfscope}%
\begin{pgfscope}%
\pgfsetrectcap%
\pgfsetmiterjoin%
\pgfsetlinewidth{0.803000pt}%
\definecolor{currentstroke}{rgb}{0.000000,0.000000,0.000000}%
\pgfsetstrokecolor{currentstroke}%
\pgfsetdash{}{0pt}%
\pgfpathmoveto{\pgfqpoint{0.553704in}{0.499691in}}%
\pgfpathlineto{\pgfqpoint{0.553704in}{2.424691in}}%
\pgfusepath{stroke}%
\end{pgfscope}%
\begin{pgfscope}%
\pgfsetrectcap%
\pgfsetmiterjoin%
\pgfsetlinewidth{0.803000pt}%
\definecolor{currentstroke}{rgb}{0.000000,0.000000,0.000000}%
\pgfsetstrokecolor{currentstroke}%
\pgfsetdash{}{0pt}%
\pgfpathmoveto{\pgfqpoint{2.568704in}{0.499691in}}%
\pgfpathlineto{\pgfqpoint{2.568704in}{2.424691in}}%
\pgfusepath{stroke}%
\end{pgfscope}%
\begin{pgfscope}%
\pgfsetrectcap%
\pgfsetmiterjoin%
\pgfsetlinewidth{0.803000pt}%
\definecolor{currentstroke}{rgb}{0.000000,0.000000,0.000000}%
\pgfsetstrokecolor{currentstroke}%
\pgfsetdash{}{0pt}%
\pgfpathmoveto{\pgfqpoint{0.553704in}{0.499691in}}%
\pgfpathlineto{\pgfqpoint{2.568704in}{0.499691in}}%
\pgfusepath{stroke}%
\end{pgfscope}%
\begin{pgfscope}%
\pgfsetrectcap%
\pgfsetmiterjoin%
\pgfsetlinewidth{0.803000pt}%
\definecolor{currentstroke}{rgb}{0.000000,0.000000,0.000000}%
\pgfsetstrokecolor{currentstroke}%
\pgfsetdash{}{0pt}%
\pgfpathmoveto{\pgfqpoint{0.553704in}{2.424691in}}%
\pgfpathlineto{\pgfqpoint{2.568704in}{2.424691in}}%
\pgfusepath{stroke}%
\end{pgfscope}%
\begin{pgfscope}%
\pgfsetbuttcap%
\pgfsetmiterjoin%
\definecolor{currentfill}{rgb}{1.000000,1.000000,1.000000}%
\pgfsetfillcolor{currentfill}%
\pgfsetfillopacity{0.800000}%
\pgfsetlinewidth{1.003750pt}%
\definecolor{currentstroke}{rgb}{0.800000,0.800000,0.800000}%
\pgfsetstrokecolor{currentstroke}%
\pgfsetstrokeopacity{0.800000}%
\pgfsetdash{}{0pt}%
\pgfpathmoveto{\pgfqpoint{0.650926in}{0.569136in}}%
\pgfpathlineto{\pgfqpoint{2.091854in}{0.569136in}}%
\pgfpathquadraticcurveto{\pgfqpoint{2.119632in}{0.569136in}}{\pgfqpoint{2.119632in}{0.596913in}}%
\pgfpathlineto{\pgfqpoint{2.119632in}{1.551388in}}%
\pgfpathquadraticcurveto{\pgfqpoint{2.119632in}{1.579166in}}{\pgfqpoint{2.091854in}{1.579166in}}%
\pgfpathlineto{\pgfqpoint{0.650926in}{1.579166in}}%
\pgfpathquadraticcurveto{\pgfqpoint{0.623149in}{1.579166in}}{\pgfqpoint{0.623149in}{1.551388in}}%
\pgfpathlineto{\pgfqpoint{0.623149in}{0.596913in}}%
\pgfpathquadraticcurveto{\pgfqpoint{0.623149in}{0.569136in}}{\pgfqpoint{0.650926in}{0.569136in}}%
\pgfpathlineto{\pgfqpoint{0.650926in}{0.569136in}}%
\pgfpathclose%
\pgfusepath{stroke,fill}%
\end{pgfscope}%
\begin{pgfscope}%
\pgfsetrectcap%
\pgfsetroundjoin%
\pgfsetlinewidth{1.505625pt}%
\definecolor{currentstroke}{rgb}{0.105882,0.619608,0.466667}%
\pgfsetstrokecolor{currentstroke}%
\pgfsetdash{}{0pt}%
\pgfpathmoveto{\pgfqpoint{0.678704in}{1.474999in}}%
\pgfpathlineto{\pgfqpoint{0.817593in}{1.474999in}}%
\pgfpathlineto{\pgfqpoint{0.956482in}{1.474999in}}%
\pgfusepath{stroke}%
\end{pgfscope}%
\begin{pgfscope}%
\definecolor{textcolor}{rgb}{0.000000,0.000000,0.000000}%
\pgfsetstrokecolor{textcolor}%
\pgfsetfillcolor{textcolor}%
\pgftext[x=1.067593in,y=1.426388in,left,base]{\color{textcolor}\rmfamily\fontsize{10.000000}{12.000000}\selectfont \(\displaystyle \textsc{ProjectedRS}^*\)}%
\end{pgfscope}%
\begin{pgfscope}%
\pgfsetrectcap%
\pgfsetroundjoin%
\pgfsetlinewidth{1.505625pt}%
\definecolor{currentstroke}{rgb}{0.850980,0.372549,0.007843}%
\pgfsetstrokecolor{currentstroke}%
\pgfsetdash{}{0pt}%
\pgfpathmoveto{\pgfqpoint{0.678704in}{1.281327in}}%
\pgfpathlineto{\pgfqpoint{0.817593in}{1.281327in}}%
\pgfpathlineto{\pgfqpoint{0.956482in}{1.281327in}}%
\pgfusepath{stroke}%
\end{pgfscope}%
\begin{pgfscope}%
\definecolor{textcolor}{rgb}{0.000000,0.000000,0.000000}%
\pgfsetstrokecolor{textcolor}%
\pgfsetfillcolor{textcolor}%
\pgftext[x=1.067593in,y=1.232716in,left,base]{\color{textcolor}\rmfamily\fontsize{10.000000}{12.000000}\selectfont \(\displaystyle \textsc{RS}\)}%
\end{pgfscope}%
\begin{pgfscope}%
\pgfsetrectcap%
\pgfsetroundjoin%
\pgfsetlinewidth{1.505625pt}%
\definecolor{currentstroke}{rgb}{0.458824,0.439216,0.701961}%
\pgfsetstrokecolor{currentstroke}%
\pgfsetdash{}{0pt}%
\pgfpathmoveto{\pgfqpoint{0.678704in}{1.087654in}}%
\pgfpathlineto{\pgfqpoint{0.817593in}{1.087654in}}%
\pgfpathlineto{\pgfqpoint{0.956482in}{1.087654in}}%
\pgfusepath{stroke}%
\end{pgfscope}%
\begin{pgfscope}%
\definecolor{textcolor}{rgb}{0.000000,0.000000,0.000000}%
\pgfsetstrokecolor{textcolor}%
\pgfsetfillcolor{textcolor}%
\pgftext[x=1.067593in,y=1.039043in,left,base]{\color{textcolor}\rmfamily\fontsize{10.000000}{12.000000}\selectfont \(\displaystyle \textsc{ANCER}\)}%
\end{pgfscope}%
\begin{pgfscope}%
\pgfsetrectcap%
\pgfsetroundjoin%
\pgfsetlinewidth{1.505625pt}%
\definecolor{currentstroke}{rgb}{0.905882,0.160784,0.541176}%
\pgfsetstrokecolor{currentstroke}%
\pgfsetdash{}{0pt}%
\pgfpathmoveto{\pgfqpoint{0.678704in}{0.893981in}}%
\pgfpathlineto{\pgfqpoint{0.817593in}{0.893981in}}%
\pgfpathlineto{\pgfqpoint{0.956482in}{0.893981in}}%
\pgfusepath{stroke}%
\end{pgfscope}%
\begin{pgfscope}%
\definecolor{textcolor}{rgb}{0.000000,0.000000,0.000000}%
\pgfsetstrokecolor{textcolor}%
\pgfsetfillcolor{textcolor}%
\pgftext[x=1.067593in,y=0.845370in,left,base]{\color{textcolor}\rmfamily\fontsize{10.000000}{12.000000}\selectfont \(\displaystyle \textsc{RS4A}-\ell_{1}\)}%
\end{pgfscope}%
\begin{pgfscope}%
\pgfsetrectcap%
\pgfsetroundjoin%
\pgfsetlinewidth{1.505625pt}%
\definecolor{currentstroke}{rgb}{0.400000,0.650980,0.117647}%
\pgfsetstrokecolor{currentstroke}%
\pgfsetdash{}{0pt}%
\pgfpathmoveto{\pgfqpoint{0.678704in}{0.700308in}}%
\pgfpathlineto{\pgfqpoint{0.817593in}{0.700308in}}%
\pgfpathlineto{\pgfqpoint{0.956482in}{0.700308in}}%
\pgfusepath{stroke}%
\end{pgfscope}%
\begin{pgfscope}%
\definecolor{textcolor}{rgb}{0.000000,0.000000,0.000000}%
\pgfsetstrokecolor{textcolor}%
\pgfsetfillcolor{textcolor}%
\pgftext[x=1.067593in,y=0.651697in,left,base]{\color{textcolor}\rmfamily\fontsize{10.000000}{12.000000}\selectfont \(\displaystyle \textsc{RS4A}-\ell_{\infty}\)}%
\end{pgfscope}%
\end{pgfpicture}%
\makeatother%
\endgroup%

%% file: figs/cifar10_ablation.pgf
%% Creator: Matplotlib, PGF backend
%%
%% To include the figure in your LaTeX document, write
%%   \input{<filename>.pgf}
%%
%% Make sure the required packages are loaded in your preamble
%%   \usepackage{pgf}
%%
%% Also ensure that all the required font packages are loaded; for instance,
%% the lmodern package is sometimes necessary when using math font.
%%   \usepackage{lmodern}
%%
%% Figures using additional raster images can only be included by \input if
%% they are in the same directory as the main LaTeX file. For loading figures
%% from other directories you can use the `import` package
%%   \usepackage{import}
%%
%% and then include the figures with
%%   \import{<path to file>}{<filename>.pgf}
%%
%% Matplotlib used the following preamble
%%   
%%   \usepackage{fontspec}
%%   \setmainfont{DejaVuSerif.ttf}[Path=\detokenize{/home/paperspace/.virtualenvs/randsmooth/lib/python3.10/site-packages/matplotlib/mpl-data/fonts/ttf/}]
%%   \setsansfont{DejaVuSans.ttf}[Path=\detokenize{/home/paperspace/.virtualenvs/randsmooth/lib/python3.10/site-packages/matplotlib/mpl-data/fonts/ttf/}]
%%   \setmonofont{DejaVuSansMono.ttf}[Path=\detokenize{/home/paperspace/.virtualenvs/randsmooth/lib/python3.10/site-packages/matplotlib/mpl-data/fonts/ttf/}]
%%   \makeatletter\@ifpackageloaded{underscore}{}{\usepackage[strings]{underscore}}\makeatother
%%
\begingroup%
\makeatletter%
\begin{pgfpicture}%
\pgfpathrectangle{\pgfpointorigin}{\pgfqpoint{3.575570in}{2.161603in}}%
\pgfusepath{use as bounding box, clip}%
\begin{pgfscope}%
\pgfsetbuttcap%
\pgfsetmiterjoin%
\pgfsetlinewidth{0.000000pt}%
\definecolor{currentstroke}{rgb}{0.000000,0.000000,0.000000}%
\pgfsetstrokecolor{currentstroke}%
\pgfsetstrokeopacity{0.000000}%
\pgfsetdash{}{0pt}%
\pgfpathmoveto{\pgfqpoint{0.000000in}{0.000000in}}%
\pgfpathlineto{\pgfqpoint{3.575570in}{0.000000in}}%
\pgfpathlineto{\pgfqpoint{3.575570in}{2.161603in}}%
\pgfpathlineto{\pgfqpoint{0.000000in}{2.161603in}}%
\pgfpathlineto{\pgfqpoint{0.000000in}{0.000000in}}%
\pgfpathclose%
\pgfusepath{}%
\end{pgfscope}%
\begin{pgfscope}%
\pgfsetbuttcap%
\pgfsetmiterjoin%
\pgfsetlinewidth{0.000000pt}%
\definecolor{currentstroke}{rgb}{0.000000,0.000000,0.000000}%
\pgfsetstrokecolor{currentstroke}%
\pgfsetstrokeopacity{0.000000}%
\pgfsetdash{}{0pt}%
\pgfpathmoveto{\pgfqpoint{0.608070in}{0.521603in}}%
\pgfpathlineto{\pgfqpoint{3.475570in}{0.521603in}}%
\pgfpathlineto{\pgfqpoint{3.475570in}{2.061603in}}%
\pgfpathlineto{\pgfqpoint{0.608070in}{2.061603in}}%
\pgfpathlineto{\pgfqpoint{0.608070in}{0.521603in}}%
\pgfpathclose%
\pgfusepath{}%
\end{pgfscope}%
\begin{pgfscope}%
\pgfsetbuttcap%
\pgfsetroundjoin%
\definecolor{currentfill}{rgb}{0.000000,0.000000,0.000000}%
\pgfsetfillcolor{currentfill}%
\pgfsetlinewidth{0.803000pt}%
\definecolor{currentstroke}{rgb}{0.000000,0.000000,0.000000}%
\pgfsetstrokecolor{currentstroke}%
\pgfsetdash{}{0pt}%
\pgfsys@defobject{currentmarker}{\pgfqpoint{0.000000in}{-0.048611in}}{\pgfqpoint{0.000000in}{0.000000in}}{%
\pgfpathmoveto{\pgfqpoint{0.000000in}{0.000000in}}%
\pgfpathlineto{\pgfqpoint{0.000000in}{-0.048611in}}%
\pgfusepath{stroke,fill}%
}%
\begin{pgfscope}%
\pgfsys@transformshift{1.048028in}{0.521603in}%
\pgfsys@useobject{currentmarker}{}%
\end{pgfscope}%
\end{pgfscope}%
\begin{pgfscope}%
\definecolor{textcolor}{rgb}{0.000000,0.000000,0.000000}%
\pgfsetstrokecolor{textcolor}%
\pgfsetfillcolor{textcolor}%
\pgftext[x=1.048028in,y=0.424381in,,top]{\color{textcolor}\sffamily\fontsize{10.000000}{12.000000}\selectfont \(\displaystyle 10^{-3 \alpha}\)}%
\end{pgfscope}%
\begin{pgfscope}%
\pgfsetbuttcap%
\pgfsetroundjoin%
\definecolor{currentfill}{rgb}{0.000000,0.000000,0.000000}%
\pgfsetfillcolor{currentfill}%
\pgfsetlinewidth{0.803000pt}%
\definecolor{currentstroke}{rgb}{0.000000,0.000000,0.000000}%
\pgfsetstrokecolor{currentstroke}%
\pgfsetdash{}{0pt}%
\pgfsys@defobject{currentmarker}{\pgfqpoint{0.000000in}{-0.048611in}}{\pgfqpoint{0.000000in}{0.000000in}}{%
\pgfpathmoveto{\pgfqpoint{0.000000in}{0.000000in}}%
\pgfpathlineto{\pgfqpoint{0.000000in}{-0.048611in}}%
\pgfusepath{stroke,fill}%
}%
\begin{pgfscope}%
\pgfsys@transformshift{2.005797in}{0.521603in}%
\pgfsys@useobject{currentmarker}{}%
\end{pgfscope}%
\end{pgfscope}%
\begin{pgfscope}%
\definecolor{textcolor}{rgb}{0.000000,0.000000,0.000000}%
\pgfsetstrokecolor{textcolor}%
\pgfsetfillcolor{textcolor}%
\pgftext[x=2.005797in,y=0.424381in,,top]{\color{textcolor}\sffamily\fontsize{10.000000}{12.000000}\selectfont \(\displaystyle 10^{-2 \alpha}\)}%
\end{pgfscope}%
\begin{pgfscope}%
\pgfsetbuttcap%
\pgfsetroundjoin%
\definecolor{currentfill}{rgb}{0.000000,0.000000,0.000000}%
\pgfsetfillcolor{currentfill}%
\pgfsetlinewidth{0.803000pt}%
\definecolor{currentstroke}{rgb}{0.000000,0.000000,0.000000}%
\pgfsetstrokecolor{currentstroke}%
\pgfsetdash{}{0pt}%
\pgfsys@defobject{currentmarker}{\pgfqpoint{0.000000in}{-0.048611in}}{\pgfqpoint{0.000000in}{0.000000in}}{%
\pgfpathmoveto{\pgfqpoint{0.000000in}{0.000000in}}%
\pgfpathlineto{\pgfqpoint{0.000000in}{-0.048611in}}%
\pgfusepath{stroke,fill}%
}%
\begin{pgfscope}%
\pgfsys@transformshift{2.963566in}{0.521603in}%
\pgfsys@useobject{currentmarker}{}%
\end{pgfscope}%
\end{pgfscope}%
\begin{pgfscope}%
\definecolor{textcolor}{rgb}{0.000000,0.000000,0.000000}%
\pgfsetstrokecolor{textcolor}%
\pgfsetfillcolor{textcolor}%
\pgftext[x=2.963566in,y=0.424381in,,top]{\color{textcolor}\sffamily\fontsize{10.000000}{12.000000}\selectfont \(\displaystyle 10^{-1 \alpha}\)}%
\end{pgfscope}%
\begin{pgfscope}%
\definecolor{textcolor}{rgb}{0.000000,0.000000,0.000000}%
\pgfsetstrokecolor{textcolor}%
\pgfsetfillcolor{textcolor}%
\pgftext[x=2.041820in,y=0.234413in,,top]{\color{textcolor}\sffamily\fontsize{10.000000}{12.000000}\selectfont Volume}%
\end{pgfscope}%
\begin{pgfscope}%
\pgfsetbuttcap%
\pgfsetroundjoin%
\definecolor{currentfill}{rgb}{0.000000,0.000000,0.000000}%
\pgfsetfillcolor{currentfill}%
\pgfsetlinewidth{0.803000pt}%
\definecolor{currentstroke}{rgb}{0.000000,0.000000,0.000000}%
\pgfsetstrokecolor{currentstroke}%
\pgfsetdash{}{0pt}%
\pgfsys@defobject{currentmarker}{\pgfqpoint{-0.048611in}{0.000000in}}{\pgfqpoint{-0.000000in}{0.000000in}}{%
\pgfpathmoveto{\pgfqpoint{-0.000000in}{0.000000in}}%
\pgfpathlineto{\pgfqpoint{-0.048611in}{0.000000in}}%
\pgfusepath{stroke,fill}%
}%
\begin{pgfscope}%
\pgfsys@transformshift{0.608070in}{0.591603in}%
\pgfsys@useobject{currentmarker}{}%
\end{pgfscope}%
\end{pgfscope}%
\begin{pgfscope}%
\definecolor{textcolor}{rgb}{0.000000,0.000000,0.000000}%
\pgfsetstrokecolor{textcolor}%
\pgfsetfillcolor{textcolor}%
\pgftext[x=0.289968in, y=0.538842in, left, base]{\color{textcolor}\sffamily\fontsize{10.000000}{12.000000}\selectfont 0.0}%
\end{pgfscope}%
\begin{pgfscope}%
\pgfsetbuttcap%
\pgfsetroundjoin%
\definecolor{currentfill}{rgb}{0.000000,0.000000,0.000000}%
\pgfsetfillcolor{currentfill}%
\pgfsetlinewidth{0.803000pt}%
\definecolor{currentstroke}{rgb}{0.000000,0.000000,0.000000}%
\pgfsetstrokecolor{currentstroke}%
\pgfsetdash{}{0pt}%
\pgfsys@defobject{currentmarker}{\pgfqpoint{-0.048611in}{0.000000in}}{\pgfqpoint{-0.000000in}{0.000000in}}{%
\pgfpathmoveto{\pgfqpoint{-0.000000in}{0.000000in}}%
\pgfpathlineto{\pgfqpoint{-0.048611in}{0.000000in}}%
\pgfusepath{stroke,fill}%
}%
\begin{pgfscope}%
\pgfsys@transformshift{0.608070in}{0.928953in}%
\pgfsys@useobject{currentmarker}{}%
\end{pgfscope}%
\end{pgfscope}%
\begin{pgfscope}%
\definecolor{textcolor}{rgb}{0.000000,0.000000,0.000000}%
\pgfsetstrokecolor{textcolor}%
\pgfsetfillcolor{textcolor}%
\pgftext[x=0.289968in, y=0.876191in, left, base]{\color{textcolor}\sffamily\fontsize{10.000000}{12.000000}\selectfont 0.2}%
\end{pgfscope}%
\begin{pgfscope}%
\pgfsetbuttcap%
\pgfsetroundjoin%
\definecolor{currentfill}{rgb}{0.000000,0.000000,0.000000}%
\pgfsetfillcolor{currentfill}%
\pgfsetlinewidth{0.803000pt}%
\definecolor{currentstroke}{rgb}{0.000000,0.000000,0.000000}%
\pgfsetstrokecolor{currentstroke}%
\pgfsetdash{}{0pt}%
\pgfsys@defobject{currentmarker}{\pgfqpoint{-0.048611in}{0.000000in}}{\pgfqpoint{-0.000000in}{0.000000in}}{%
\pgfpathmoveto{\pgfqpoint{-0.000000in}{0.000000in}}%
\pgfpathlineto{\pgfqpoint{-0.048611in}{0.000000in}}%
\pgfusepath{stroke,fill}%
}%
\begin{pgfscope}%
\pgfsys@transformshift{0.608070in}{1.266302in}%
\pgfsys@useobject{currentmarker}{}%
\end{pgfscope}%
\end{pgfscope}%
\begin{pgfscope}%
\definecolor{textcolor}{rgb}{0.000000,0.000000,0.000000}%
\pgfsetstrokecolor{textcolor}%
\pgfsetfillcolor{textcolor}%
\pgftext[x=0.289968in, y=1.213541in, left, base]{\color{textcolor}\sffamily\fontsize{10.000000}{12.000000}\selectfont 0.4}%
\end{pgfscope}%
\begin{pgfscope}%
\pgfsetbuttcap%
\pgfsetroundjoin%
\definecolor{currentfill}{rgb}{0.000000,0.000000,0.000000}%
\pgfsetfillcolor{currentfill}%
\pgfsetlinewidth{0.803000pt}%
\definecolor{currentstroke}{rgb}{0.000000,0.000000,0.000000}%
\pgfsetstrokecolor{currentstroke}%
\pgfsetdash{}{0pt}%
\pgfsys@defobject{currentmarker}{\pgfqpoint{-0.048611in}{0.000000in}}{\pgfqpoint{-0.000000in}{0.000000in}}{%
\pgfpathmoveto{\pgfqpoint{-0.000000in}{0.000000in}}%
\pgfpathlineto{\pgfqpoint{-0.048611in}{0.000000in}}%
\pgfusepath{stroke,fill}%
}%
\begin{pgfscope}%
\pgfsys@transformshift{0.608070in}{1.603652in}%
\pgfsys@useobject{currentmarker}{}%
\end{pgfscope}%
\end{pgfscope}%
\begin{pgfscope}%
\definecolor{textcolor}{rgb}{0.000000,0.000000,0.000000}%
\pgfsetstrokecolor{textcolor}%
\pgfsetfillcolor{textcolor}%
\pgftext[x=0.289968in, y=1.550890in, left, base]{\color{textcolor}\sffamily\fontsize{10.000000}{12.000000}\selectfont 0.6}%
\end{pgfscope}%
\begin{pgfscope}%
\pgfsetbuttcap%
\pgfsetroundjoin%
\definecolor{currentfill}{rgb}{0.000000,0.000000,0.000000}%
\pgfsetfillcolor{currentfill}%
\pgfsetlinewidth{0.803000pt}%
\definecolor{currentstroke}{rgb}{0.000000,0.000000,0.000000}%
\pgfsetstrokecolor{currentstroke}%
\pgfsetdash{}{0pt}%
\pgfsys@defobject{currentmarker}{\pgfqpoint{-0.048611in}{0.000000in}}{\pgfqpoint{-0.000000in}{0.000000in}}{%
\pgfpathmoveto{\pgfqpoint{-0.000000in}{0.000000in}}%
\pgfpathlineto{\pgfqpoint{-0.048611in}{0.000000in}}%
\pgfusepath{stroke,fill}%
}%
\begin{pgfscope}%
\pgfsys@transformshift{0.608070in}{1.941001in}%
\pgfsys@useobject{currentmarker}{}%
\end{pgfscope}%
\end{pgfscope}%
\begin{pgfscope}%
\definecolor{textcolor}{rgb}{0.000000,0.000000,0.000000}%
\pgfsetstrokecolor{textcolor}%
\pgfsetfillcolor{textcolor}%
\pgftext[x=0.289968in, y=1.888239in, left, base]{\color{textcolor}\sffamily\fontsize{10.000000}{12.000000}\selectfont 0.8}%
\end{pgfscope}%
\begin{pgfscope}%
\definecolor{textcolor}{rgb}{0.000000,0.000000,0.000000}%
\pgfsetstrokecolor{textcolor}%
\pgfsetfillcolor{textcolor}%
\pgftext[x=0.234413in,y=1.291603in,,bottom,rotate=90.000000]{\color{textcolor}\sffamily\fontsize{10.000000}{12.000000}\selectfont Certified accuracy}%
\end{pgfscope}%
\begin{pgfscope}%
\pgfpathrectangle{\pgfqpoint{0.608070in}{0.521603in}}{\pgfqpoint{2.867500in}{1.540000in}}%
\pgfusepath{clip}%
\pgfsetrectcap%
\pgfsetroundjoin%
\pgfsetlinewidth{1.505625pt}%
\definecolor{currentstroke}{rgb}{0.105882,0.619608,0.466667}%
\pgfsetstrokecolor{currentstroke}%
\pgfsetdash{}{0pt}%
\pgfpathmoveto{\pgfqpoint{0.608070in}{1.974736in}}%
\pgfpathlineto{\pgfqpoint{1.356113in}{1.974736in}}%
\pgfpathlineto{\pgfqpoint{1.365704in}{1.971362in}}%
\pgfpathlineto{\pgfqpoint{1.663003in}{1.971362in}}%
\pgfpathlineto{\pgfqpoint{1.672593in}{1.967989in}}%
\pgfpathlineto{\pgfqpoint{1.730135in}{1.967989in}}%
\pgfpathlineto{\pgfqpoint{1.739726in}{1.964615in}}%
\pgfpathlineto{\pgfqpoint{1.950712in}{1.964615in}}%
\pgfpathlineto{\pgfqpoint{1.960302in}{1.961242in}}%
\pgfpathlineto{\pgfqpoint{2.046615in}{1.961242in}}%
\pgfpathlineto{\pgfqpoint{2.056205in}{1.957868in}}%
\pgfpathlineto{\pgfqpoint{2.094567in}{1.957868in}}%
\pgfpathlineto{\pgfqpoint{2.113747in}{1.951121in}}%
\pgfpathlineto{\pgfqpoint{2.123338in}{1.951121in}}%
\pgfpathlineto{\pgfqpoint{2.132928in}{1.941001in}}%
\pgfpathlineto{\pgfqpoint{2.142518in}{1.937627in}}%
\pgfpathlineto{\pgfqpoint{2.161699in}{1.937627in}}%
\pgfpathlineto{\pgfqpoint{2.171289in}{1.934254in}}%
\pgfpathlineto{\pgfqpoint{2.180879in}{1.934254in}}%
\pgfpathlineto{\pgfqpoint{2.190470in}{1.930880in}}%
\pgfpathlineto{\pgfqpoint{2.219241in}{1.930880in}}%
\pgfpathlineto{\pgfqpoint{2.228831in}{1.924133in}}%
\pgfpathlineto{\pgfqpoint{2.276782in}{1.924133in}}%
\pgfpathlineto{\pgfqpoint{2.286373in}{1.917386in}}%
\pgfpathlineto{\pgfqpoint{2.295963in}{1.914013in}}%
\pgfpathlineto{\pgfqpoint{2.334324in}{1.914013in}}%
\pgfpathlineto{\pgfqpoint{2.343914in}{1.910639in}}%
\pgfpathlineto{\pgfqpoint{2.363095in}{1.910639in}}%
\pgfpathlineto{\pgfqpoint{2.372685in}{1.907266in}}%
\pgfpathlineto{\pgfqpoint{2.430227in}{1.907266in}}%
\pgfpathlineto{\pgfqpoint{2.439817in}{1.900519in}}%
\pgfpathlineto{\pgfqpoint{2.449408in}{1.900519in}}%
\pgfpathlineto{\pgfqpoint{2.458998in}{1.897146in}}%
\pgfpathlineto{\pgfqpoint{2.468588in}{1.890399in}}%
\pgfpathlineto{\pgfqpoint{2.478179in}{1.890399in}}%
\pgfpathlineto{\pgfqpoint{2.487769in}{1.876905in}}%
\pgfpathlineto{\pgfqpoint{2.497359in}{1.876905in}}%
\pgfpathlineto{\pgfqpoint{2.506950in}{1.873531in}}%
\pgfpathlineto{\pgfqpoint{2.516540in}{1.866784in}}%
\pgfpathlineto{\pgfqpoint{2.526130in}{1.863411in}}%
\pgfpathlineto{\pgfqpoint{2.535721in}{1.856664in}}%
\pgfpathlineto{\pgfqpoint{2.545311in}{1.853290in}}%
\pgfpathlineto{\pgfqpoint{2.554901in}{1.853290in}}%
\pgfpathlineto{\pgfqpoint{2.564491in}{1.846543in}}%
\pgfpathlineto{\pgfqpoint{2.574082in}{1.843170in}}%
\pgfpathlineto{\pgfqpoint{2.593262in}{1.843170in}}%
\pgfpathlineto{\pgfqpoint{2.641214in}{1.826302in}}%
\pgfpathlineto{\pgfqpoint{2.650804in}{1.826302in}}%
\pgfpathlineto{\pgfqpoint{2.660394in}{1.819555in}}%
\pgfpathlineto{\pgfqpoint{2.669985in}{1.816182in}}%
\pgfpathlineto{\pgfqpoint{2.679575in}{1.802688in}}%
\pgfpathlineto{\pgfqpoint{2.698756in}{1.789194in}}%
\pgfpathlineto{\pgfqpoint{2.708346in}{1.785820in}}%
\pgfpathlineto{\pgfqpoint{2.717936in}{1.779073in}}%
\pgfpathlineto{\pgfqpoint{2.737117in}{1.758832in}}%
\pgfpathlineto{\pgfqpoint{2.746707in}{1.752085in}}%
\pgfpathlineto{\pgfqpoint{2.756297in}{1.738591in}}%
\pgfpathlineto{\pgfqpoint{2.765888in}{1.721724in}}%
\pgfpathlineto{\pgfqpoint{2.775478in}{1.708230in}}%
\pgfpathlineto{\pgfqpoint{2.794659in}{1.701483in}}%
\pgfpathlineto{\pgfqpoint{2.823430in}{1.681242in}}%
\pgfpathlineto{\pgfqpoint{2.833020in}{1.661001in}}%
\pgfpathlineto{\pgfqpoint{2.842610in}{1.637386in}}%
\pgfpathlineto{\pgfqpoint{2.852200in}{1.627266in}}%
\pgfpathlineto{\pgfqpoint{2.861791in}{1.613772in}}%
\pgfpathlineto{\pgfqpoint{2.880971in}{1.607025in}}%
\pgfpathlineto{\pgfqpoint{2.890562in}{1.600278in}}%
\pgfpathlineto{\pgfqpoint{2.900152in}{1.586784in}}%
\pgfpathlineto{\pgfqpoint{2.909742in}{1.569917in}}%
\pgfpathlineto{\pgfqpoint{2.919333in}{1.563170in}}%
\pgfpathlineto{\pgfqpoint{2.928923in}{1.553049in}}%
\pgfpathlineto{\pgfqpoint{2.948103in}{1.539555in}}%
\pgfpathlineto{\pgfqpoint{2.957694in}{1.512567in}}%
\pgfpathlineto{\pgfqpoint{2.986465in}{1.472085in}}%
\pgfpathlineto{\pgfqpoint{2.996055in}{1.461965in}}%
\pgfpathlineto{\pgfqpoint{3.015236in}{1.414736in}}%
\pgfpathlineto{\pgfqpoint{3.024826in}{1.401242in}}%
\pgfpathlineto{\pgfqpoint{3.034416in}{1.374254in}}%
\pgfpathlineto{\pgfqpoint{3.044006in}{1.340519in}}%
\pgfpathlineto{\pgfqpoint{3.053597in}{1.333772in}}%
\pgfpathlineto{\pgfqpoint{3.072777in}{1.279796in}}%
\pgfpathlineto{\pgfqpoint{3.082368in}{1.262929in}}%
\pgfpathlineto{\pgfqpoint{3.091958in}{1.242688in}}%
\pgfpathlineto{\pgfqpoint{3.101548in}{1.208953in}}%
\pgfpathlineto{\pgfqpoint{3.111139in}{1.185338in}}%
\pgfpathlineto{\pgfqpoint{3.120729in}{1.165097in}}%
\pgfpathlineto{\pgfqpoint{3.130319in}{1.127989in}}%
\pgfpathlineto{\pgfqpoint{3.139909in}{1.111121in}}%
\pgfpathlineto{\pgfqpoint{3.149500in}{1.077386in}}%
\pgfpathlineto{\pgfqpoint{3.168680in}{1.043652in}}%
\pgfpathlineto{\pgfqpoint{3.178271in}{1.020037in}}%
\pgfpathlineto{\pgfqpoint{3.187861in}{0.999796in}}%
\pgfpathlineto{\pgfqpoint{3.197451in}{0.966061in}}%
\pgfpathlineto{\pgfqpoint{3.216632in}{0.932326in}}%
\pgfpathlineto{\pgfqpoint{3.226222in}{0.918832in}}%
\pgfpathlineto{\pgfqpoint{3.235812in}{0.891844in}}%
\pgfpathlineto{\pgfqpoint{3.245403in}{0.858109in}}%
\pgfpathlineto{\pgfqpoint{3.264583in}{0.810880in}}%
\pgfpathlineto{\pgfqpoint{3.274174in}{0.794013in}}%
\pgfpathlineto{\pgfqpoint{3.283764in}{0.770399in}}%
\pgfpathlineto{\pgfqpoint{3.293354in}{0.753531in}}%
\pgfpathlineto{\pgfqpoint{3.312535in}{0.713049in}}%
\pgfpathlineto{\pgfqpoint{3.322125in}{0.706302in}}%
\pgfpathlineto{\pgfqpoint{3.331715in}{0.696182in}}%
\pgfpathlineto{\pgfqpoint{3.341306in}{0.682688in}}%
\pgfpathlineto{\pgfqpoint{3.350896in}{0.648953in}}%
\pgfpathlineto{\pgfqpoint{3.360486in}{0.645579in}}%
\pgfpathlineto{\pgfqpoint{3.389257in}{0.625338in}}%
\pgfpathlineto{\pgfqpoint{3.398848in}{0.611844in}}%
\pgfpathlineto{\pgfqpoint{3.427618in}{0.611844in}}%
\pgfpathlineto{\pgfqpoint{3.437209in}{0.608471in}}%
\pgfpathlineto{\pgfqpoint{3.446799in}{0.601724in}}%
\pgfpathlineto{\pgfqpoint{3.456389in}{0.598350in}}%
\pgfpathlineto{\pgfqpoint{3.465980in}{0.598350in}}%
\pgfpathlineto{\pgfqpoint{3.475570in}{0.591603in}}%
\pgfpathlineto{\pgfqpoint{3.475570in}{0.591603in}}%
\pgfusepath{stroke}%
\end{pgfscope}%
\begin{pgfscope}%
\pgfpathrectangle{\pgfqpoint{0.608070in}{0.521603in}}{\pgfqpoint{2.867500in}{1.540000in}}%
\pgfusepath{clip}%
\pgfsetrectcap%
\pgfsetroundjoin%
\pgfsetlinewidth{1.505625pt}%
\definecolor{currentstroke}{rgb}{0.850980,0.372549,0.007843}%
\pgfsetstrokecolor{currentstroke}%
\pgfsetdash{}{0pt}%
\pgfpathmoveto{\pgfqpoint{0.608070in}{1.806061in}}%
\pgfpathlineto{\pgfqpoint{2.315144in}{1.806061in}}%
\pgfpathlineto{\pgfqpoint{2.324734in}{1.802688in}}%
\pgfpathlineto{\pgfqpoint{2.391866in}{1.802688in}}%
\pgfpathlineto{\pgfqpoint{2.401456in}{1.799314in}}%
\pgfpathlineto{\pgfqpoint{2.430227in}{1.799314in}}%
\pgfpathlineto{\pgfqpoint{2.439817in}{1.795941in}}%
\pgfpathlineto{\pgfqpoint{2.478179in}{1.795941in}}%
\pgfpathlineto{\pgfqpoint{2.487769in}{1.792567in}}%
\pgfpathlineto{\pgfqpoint{2.497359in}{1.792567in}}%
\pgfpathlineto{\pgfqpoint{2.506950in}{1.785820in}}%
\pgfpathlineto{\pgfqpoint{2.516540in}{1.782447in}}%
\pgfpathlineto{\pgfqpoint{2.526130in}{1.782447in}}%
\pgfpathlineto{\pgfqpoint{2.564491in}{1.768953in}}%
\pgfpathlineto{\pgfqpoint{2.583672in}{1.768953in}}%
\pgfpathlineto{\pgfqpoint{2.593262in}{1.762206in}}%
\pgfpathlineto{\pgfqpoint{2.602853in}{1.752085in}}%
\pgfpathlineto{\pgfqpoint{2.612443in}{1.738591in}}%
\pgfpathlineto{\pgfqpoint{2.622033in}{1.738591in}}%
\pgfpathlineto{\pgfqpoint{2.660394in}{1.725097in}}%
\pgfpathlineto{\pgfqpoint{2.669985in}{1.711603in}}%
\pgfpathlineto{\pgfqpoint{2.679575in}{1.704856in}}%
\pgfpathlineto{\pgfqpoint{2.689165in}{1.694736in}}%
\pgfpathlineto{\pgfqpoint{2.708346in}{1.681242in}}%
\pgfpathlineto{\pgfqpoint{2.717936in}{1.681242in}}%
\pgfpathlineto{\pgfqpoint{2.727527in}{1.671121in}}%
\pgfpathlineto{\pgfqpoint{2.737117in}{1.654254in}}%
\pgfpathlineto{\pgfqpoint{2.756297in}{1.634013in}}%
\pgfpathlineto{\pgfqpoint{2.765888in}{1.617146in}}%
\pgfpathlineto{\pgfqpoint{2.785068in}{1.539555in}}%
\pgfpathlineto{\pgfqpoint{2.804249in}{1.482206in}}%
\pgfpathlineto{\pgfqpoint{2.813839in}{1.441724in}}%
\pgfpathlineto{\pgfqpoint{2.823430in}{1.418109in}}%
\pgfpathlineto{\pgfqpoint{2.833020in}{1.367507in}}%
\pgfpathlineto{\pgfqpoint{2.842610in}{1.333772in}}%
\pgfpathlineto{\pgfqpoint{2.871381in}{1.158350in}}%
\pgfpathlineto{\pgfqpoint{2.880971in}{1.104374in}}%
\pgfpathlineto{\pgfqpoint{2.890562in}{1.030158in}}%
\pgfpathlineto{\pgfqpoint{2.900152in}{0.979555in}}%
\pgfpathlineto{\pgfqpoint{2.909742in}{0.952567in}}%
\pgfpathlineto{\pgfqpoint{2.919333in}{0.918832in}}%
\pgfpathlineto{\pgfqpoint{2.928923in}{0.861483in}}%
\pgfpathlineto{\pgfqpoint{2.938513in}{0.814254in}}%
\pgfpathlineto{\pgfqpoint{2.957694in}{0.729917in}}%
\pgfpathlineto{\pgfqpoint{2.967284in}{0.682688in}}%
\pgfpathlineto{\pgfqpoint{2.976874in}{0.659073in}}%
\pgfpathlineto{\pgfqpoint{2.986465in}{0.632085in}}%
\pgfpathlineto{\pgfqpoint{2.996055in}{0.632085in}}%
\pgfpathlineto{\pgfqpoint{3.005645in}{0.618591in}}%
\pgfpathlineto{\pgfqpoint{3.015236in}{0.618591in}}%
\pgfpathlineto{\pgfqpoint{3.024826in}{0.601724in}}%
\pgfpathlineto{\pgfqpoint{3.034416in}{0.594977in}}%
\pgfpathlineto{\pgfqpoint{3.053597in}{0.594977in}}%
\pgfpathlineto{\pgfqpoint{3.063187in}{0.591603in}}%
\pgfpathlineto{\pgfqpoint{3.475570in}{0.591603in}}%
\pgfpathlineto{\pgfqpoint{3.475570in}{0.591603in}}%
\pgfusepath{stroke}%
\end{pgfscope}%
\begin{pgfscope}%
\pgfpathrectangle{\pgfqpoint{0.608070in}{0.521603in}}{\pgfqpoint{2.867500in}{1.540000in}}%
\pgfusepath{clip}%
\pgfsetrectcap%
\pgfsetroundjoin%
\pgfsetlinewidth{1.505625pt}%
\definecolor{currentstroke}{rgb}{0.458824,0.439216,0.701961}%
\pgfsetstrokecolor{currentstroke}%
\pgfsetdash{}{0pt}%
\pgfpathmoveto{\pgfqpoint{0.608070in}{1.991603in}}%
\pgfpathlineto{\pgfqpoint{1.039634in}{1.991603in}}%
\pgfpathlineto{\pgfqpoint{1.049224in}{1.988230in}}%
\pgfpathlineto{\pgfqpoint{1.116356in}{1.988230in}}%
\pgfpathlineto{\pgfqpoint{1.125946in}{1.984856in}}%
\pgfpathlineto{\pgfqpoint{1.173898in}{1.984856in}}%
\pgfpathlineto{\pgfqpoint{1.183488in}{1.981483in}}%
\pgfpathlineto{\pgfqpoint{1.499968in}{1.981483in}}%
\pgfpathlineto{\pgfqpoint{1.509558in}{1.978109in}}%
\pgfpathlineto{\pgfqpoint{1.519149in}{1.978109in}}%
\pgfpathlineto{\pgfqpoint{1.538329in}{1.971362in}}%
\pgfpathlineto{\pgfqpoint{1.672593in}{1.971362in}}%
\pgfpathlineto{\pgfqpoint{1.682184in}{1.967989in}}%
\pgfpathlineto{\pgfqpoint{1.730135in}{1.967989in}}%
\pgfpathlineto{\pgfqpoint{1.739726in}{1.964615in}}%
\pgfpathlineto{\pgfqpoint{1.758906in}{1.964615in}}%
\pgfpathlineto{\pgfqpoint{1.768496in}{1.961242in}}%
\pgfpathlineto{\pgfqpoint{1.806858in}{1.961242in}}%
\pgfpathlineto{\pgfqpoint{1.835629in}{1.951121in}}%
\pgfpathlineto{\pgfqpoint{1.873990in}{1.951121in}}%
\pgfpathlineto{\pgfqpoint{1.893170in}{1.944374in}}%
\pgfpathlineto{\pgfqpoint{1.902761in}{1.944374in}}%
\pgfpathlineto{\pgfqpoint{1.912351in}{1.937627in}}%
\pgfpathlineto{\pgfqpoint{1.921941in}{1.937627in}}%
\pgfpathlineto{\pgfqpoint{1.931532in}{1.934254in}}%
\pgfpathlineto{\pgfqpoint{1.979483in}{1.934254in}}%
\pgfpathlineto{\pgfqpoint{1.989073in}{1.930880in}}%
\pgfpathlineto{\pgfqpoint{2.017844in}{1.930880in}}%
\pgfpathlineto{\pgfqpoint{2.027435in}{1.924133in}}%
\pgfpathlineto{\pgfqpoint{2.037025in}{1.924133in}}%
\pgfpathlineto{\pgfqpoint{2.075386in}{1.910639in}}%
\pgfpathlineto{\pgfqpoint{2.084976in}{1.910639in}}%
\pgfpathlineto{\pgfqpoint{2.094567in}{1.907266in}}%
\pgfpathlineto{\pgfqpoint{2.104157in}{1.907266in}}%
\pgfpathlineto{\pgfqpoint{2.113747in}{1.903892in}}%
\pgfpathlineto{\pgfqpoint{2.123338in}{1.903892in}}%
\pgfpathlineto{\pgfqpoint{2.132928in}{1.900519in}}%
\pgfpathlineto{\pgfqpoint{2.142518in}{1.900519in}}%
\pgfpathlineto{\pgfqpoint{2.190470in}{1.883652in}}%
\pgfpathlineto{\pgfqpoint{2.200060in}{1.873531in}}%
\pgfpathlineto{\pgfqpoint{2.219241in}{1.866784in}}%
\pgfpathlineto{\pgfqpoint{2.228831in}{1.866784in}}%
\pgfpathlineto{\pgfqpoint{2.238421in}{1.863411in}}%
\pgfpathlineto{\pgfqpoint{2.248011in}{1.856664in}}%
\pgfpathlineto{\pgfqpoint{2.257602in}{1.856664in}}%
\pgfpathlineto{\pgfqpoint{2.267192in}{1.846543in}}%
\pgfpathlineto{\pgfqpoint{2.295963in}{1.836423in}}%
\pgfpathlineto{\pgfqpoint{2.305553in}{1.826302in}}%
\pgfpathlineto{\pgfqpoint{2.315144in}{1.819555in}}%
\pgfpathlineto{\pgfqpoint{2.324734in}{1.816182in}}%
\pgfpathlineto{\pgfqpoint{2.334324in}{1.802688in}}%
\pgfpathlineto{\pgfqpoint{2.343914in}{1.792567in}}%
\pgfpathlineto{\pgfqpoint{2.363095in}{1.779073in}}%
\pgfpathlineto{\pgfqpoint{2.372685in}{1.775700in}}%
\pgfpathlineto{\pgfqpoint{2.382276in}{1.765579in}}%
\pgfpathlineto{\pgfqpoint{2.391866in}{1.752085in}}%
\pgfpathlineto{\pgfqpoint{2.401456in}{1.741965in}}%
\pgfpathlineto{\pgfqpoint{2.411047in}{1.735218in}}%
\pgfpathlineto{\pgfqpoint{2.420637in}{1.731844in}}%
\pgfpathlineto{\pgfqpoint{2.430227in}{1.731844in}}%
\pgfpathlineto{\pgfqpoint{2.449408in}{1.725097in}}%
\pgfpathlineto{\pgfqpoint{2.458998in}{1.718350in}}%
\pgfpathlineto{\pgfqpoint{2.468588in}{1.718350in}}%
\pgfpathlineto{\pgfqpoint{2.478179in}{1.701483in}}%
\pgfpathlineto{\pgfqpoint{2.487769in}{1.674495in}}%
\pgfpathlineto{\pgfqpoint{2.497359in}{1.664374in}}%
\pgfpathlineto{\pgfqpoint{2.506950in}{1.640760in}}%
\pgfpathlineto{\pgfqpoint{2.516540in}{1.637386in}}%
\pgfpathlineto{\pgfqpoint{2.535721in}{1.617146in}}%
\pgfpathlineto{\pgfqpoint{2.564491in}{1.566543in}}%
\pgfpathlineto{\pgfqpoint{2.574082in}{1.553049in}}%
\pgfpathlineto{\pgfqpoint{2.593262in}{1.546302in}}%
\pgfpathlineto{\pgfqpoint{2.602853in}{1.526061in}}%
\pgfpathlineto{\pgfqpoint{2.612443in}{1.512567in}}%
\pgfpathlineto{\pgfqpoint{2.631624in}{1.492326in}}%
\pgfpathlineto{\pgfqpoint{2.641214in}{1.475459in}}%
\pgfpathlineto{\pgfqpoint{2.650804in}{1.451844in}}%
\pgfpathlineto{\pgfqpoint{2.660394in}{1.438350in}}%
\pgfpathlineto{\pgfqpoint{2.669985in}{1.411362in}}%
\pgfpathlineto{\pgfqpoint{2.679575in}{1.391121in}}%
\pgfpathlineto{\pgfqpoint{2.689165in}{1.381001in}}%
\pgfpathlineto{\pgfqpoint{2.698756in}{1.364133in}}%
\pgfpathlineto{\pgfqpoint{2.708346in}{1.350639in}}%
\pgfpathlineto{\pgfqpoint{2.727527in}{1.296664in}}%
\pgfpathlineto{\pgfqpoint{2.746707in}{1.262929in}}%
\pgfpathlineto{\pgfqpoint{2.756297in}{1.242688in}}%
\pgfpathlineto{\pgfqpoint{2.765888in}{1.192085in}}%
\pgfpathlineto{\pgfqpoint{2.775478in}{1.101001in}}%
\pgfpathlineto{\pgfqpoint{2.785068in}{0.591603in}}%
\pgfpathlineto{\pgfqpoint{3.475570in}{0.591603in}}%
\pgfpathlineto{\pgfqpoint{3.475570in}{0.591603in}}%
\pgfusepath{stroke}%
\end{pgfscope}%
\begin{pgfscope}%
\pgfsetrectcap%
\pgfsetmiterjoin%
\pgfsetlinewidth{0.803000pt}%
\definecolor{currentstroke}{rgb}{0.000000,0.000000,0.000000}%
\pgfsetstrokecolor{currentstroke}%
\pgfsetdash{}{0pt}%
\pgfpathmoveto{\pgfqpoint{0.608070in}{0.521603in}}%
\pgfpathlineto{\pgfqpoint{0.608070in}{2.061603in}}%
\pgfusepath{stroke}%
\end{pgfscope}%
\begin{pgfscope}%
\pgfsetrectcap%
\pgfsetmiterjoin%
\pgfsetlinewidth{0.803000pt}%
\definecolor{currentstroke}{rgb}{0.000000,0.000000,0.000000}%
\pgfsetstrokecolor{currentstroke}%
\pgfsetdash{}{0pt}%
\pgfpathmoveto{\pgfqpoint{3.475570in}{0.521603in}}%
\pgfpathlineto{\pgfqpoint{3.475570in}{2.061603in}}%
\pgfusepath{stroke}%
\end{pgfscope}%
\begin{pgfscope}%
\pgfsetrectcap%
\pgfsetmiterjoin%
\pgfsetlinewidth{0.803000pt}%
\definecolor{currentstroke}{rgb}{0.000000,0.000000,0.000000}%
\pgfsetstrokecolor{currentstroke}%
\pgfsetdash{}{0pt}%
\pgfpathmoveto{\pgfqpoint{0.608070in}{0.521603in}}%
\pgfpathlineto{\pgfqpoint{3.475570in}{0.521603in}}%
\pgfusepath{stroke}%
\end{pgfscope}%
\begin{pgfscope}%
\pgfsetrectcap%
\pgfsetmiterjoin%
\pgfsetlinewidth{0.803000pt}%
\definecolor{currentstroke}{rgb}{0.000000,0.000000,0.000000}%
\pgfsetstrokecolor{currentstroke}%
\pgfsetdash{}{0pt}%
\pgfpathmoveto{\pgfqpoint{0.608070in}{2.061603in}}%
\pgfpathlineto{\pgfqpoint{3.475570in}{2.061603in}}%
\pgfusepath{stroke}%
\end{pgfscope}%
\begin{pgfscope}%
\pgfsetbuttcap%
\pgfsetmiterjoin%
\definecolor{currentfill}{rgb}{1.000000,1.000000,1.000000}%
\pgfsetfillcolor{currentfill}%
\pgfsetfillopacity{0.800000}%
\pgfsetlinewidth{1.003750pt}%
\definecolor{currentstroke}{rgb}{0.800000,0.800000,0.800000}%
\pgfsetstrokecolor{currentstroke}%
\pgfsetstrokeopacity{0.800000}%
\pgfsetdash{}{0pt}%
\pgfpathmoveto{\pgfqpoint{0.705292in}{0.591048in}}%
\pgfpathlineto{\pgfqpoint{2.558225in}{0.591048in}}%
\pgfpathquadraticcurveto{\pgfqpoint{2.586002in}{0.591048in}}{\pgfqpoint{2.586002in}{0.618826in}}%
\pgfpathlineto{\pgfqpoint{2.586002in}{1.216508in}}%
\pgfpathquadraticcurveto{\pgfqpoint{2.586002in}{1.244286in}}{\pgfqpoint{2.558225in}{1.244286in}}%
\pgfpathlineto{\pgfqpoint{0.705292in}{1.244286in}}%
\pgfpathquadraticcurveto{\pgfqpoint{0.677514in}{1.244286in}}{\pgfqpoint{0.677514in}{1.216508in}}%
\pgfpathlineto{\pgfqpoint{0.677514in}{0.618826in}}%
\pgfpathquadraticcurveto{\pgfqpoint{0.677514in}{0.591048in}}{\pgfqpoint{0.705292in}{0.591048in}}%
\pgfpathlineto{\pgfqpoint{0.705292in}{0.591048in}}%
\pgfpathclose%
\pgfusepath{stroke,fill}%
\end{pgfscope}%
\begin{pgfscope}%
\pgfsetrectcap%
\pgfsetroundjoin%
\pgfsetlinewidth{1.505625pt}%
\definecolor{currentstroke}{rgb}{0.105882,0.619608,0.466667}%
\pgfsetstrokecolor{currentstroke}%
\pgfsetdash{}{0pt}%
\pgfpathmoveto{\pgfqpoint{0.733070in}{1.131819in}}%
\pgfpathlineto{\pgfqpoint{0.871959in}{1.131819in}}%
\pgfpathlineto{\pgfqpoint{1.010848in}{1.131819in}}%
\pgfusepath{stroke}%
\end{pgfscope}%
\begin{pgfscope}%
\definecolor{textcolor}{rgb}{0.000000,0.000000,0.000000}%
\pgfsetstrokecolor{textcolor}%
\pgfsetfillcolor{textcolor}%
\pgftext[x=1.121959in,y=1.083208in,left,base]{\color{textcolor}\sffamily\fontsize{10.000000}{12.000000}\selectfont \(\displaystyle \textsc{ProjectedRS}\)}%
\end{pgfscope}%
\begin{pgfscope}%
\pgfsetrectcap%
\pgfsetroundjoin%
\pgfsetlinewidth{1.505625pt}%
\definecolor{currentstroke}{rgb}{0.850980,0.372549,0.007843}%
\pgfsetstrokecolor{currentstroke}%
\pgfsetdash{}{0pt}%
\pgfpathmoveto{\pgfqpoint{0.733070in}{0.927961in}}%
\pgfpathlineto{\pgfqpoint{0.871959in}{0.927961in}}%
\pgfpathlineto{\pgfqpoint{1.010848in}{0.927961in}}%
\pgfusepath{stroke}%
\end{pgfscope}%
\begin{pgfscope}%
\definecolor{textcolor}{rgb}{0.000000,0.000000,0.000000}%
\pgfsetstrokecolor{textcolor}%
\pgfsetfillcolor{textcolor}%
\pgftext[x=1.121959in,y=0.879350in,left,base]{\color{textcolor}\sffamily\fontsize{10.000000}{12.000000}\selectfont \(\displaystyle \textsc{ProjectedRSRandom}\)}%
\end{pgfscope}%
\begin{pgfscope}%
\pgfsetrectcap%
\pgfsetroundjoin%
\pgfsetlinewidth{1.505625pt}%
\definecolor{currentstroke}{rgb}{0.458824,0.439216,0.701961}%
\pgfsetstrokecolor{currentstroke}%
\pgfsetdash{}{0pt}%
\pgfpathmoveto{\pgfqpoint{0.733070in}{0.724104in}}%
\pgfpathlineto{\pgfqpoint{0.871959in}{0.724104in}}%
\pgfpathlineto{\pgfqpoint{1.010848in}{0.724104in}}%
\pgfusepath{stroke}%
\end{pgfscope}%
\begin{pgfscope}%
\definecolor{textcolor}{rgb}{0.000000,0.000000,0.000000}%
\pgfsetstrokecolor{textcolor}%
\pgfsetfillcolor{textcolor}%
\pgftext[x=1.121959in,y=0.675493in,left,base]{\color{textcolor}\sffamily\fontsize{10.000000}{12.000000}\selectfont \(\displaystyle \textsc{RS}\)}%
\end{pgfscope}%
\end{pgfpicture}%
\makeatother%
\endgroup%

%% file: figs/cifar10_main_radii.pgf
%% Creator: Matplotlib, PGF backend
%%
%% To include the figure in your LaTeX document, write
%%   \input{<filename>.pgf}
%%
%% Make sure the required packages are loaded in your preamble
%%   \usepackage{pgf}
%%
%% Also ensure that all the required font packages are loaded; for instance,
%% the lmodern package is sometimes necessary when using math font.
%%   \usepackage{lmodern}
%%
%% Figures using additional raster images can only be included by \input if
%% they are in the same directory as the main LaTeX file. For loading figures
%% from other directories you can use the `import` package
%%   \usepackage{import}
%%
%% and then include the figures with
%%   \import{<path to file>}{<filename>.pgf}
%%
%% Matplotlib used the following preamble
%%
\begingroup%
\makeatletter%
\begin{pgfpicture}%
\pgfpathrectangle{\pgfpointorigin}{\pgfqpoint{3.521204in}{2.139691in}}%
\pgfusepath{use as bounding box, clip}%
\begin{pgfscope}%
\pgfsetbuttcap%
\pgfsetmiterjoin%
\pgfsetlinewidth{0.000000pt}%
\definecolor{currentstroke}{rgb}{0.000000,0.000000,0.000000}%
\pgfsetstrokecolor{currentstroke}%
\pgfsetstrokeopacity{0.000000}%
\pgfsetdash{}{0pt}%
\pgfpathmoveto{\pgfqpoint{0.000000in}{0.000000in}}%
\pgfpathlineto{\pgfqpoint{3.521204in}{0.000000in}}%
\pgfpathlineto{\pgfqpoint{3.521204in}{2.139691in}}%
\pgfpathlineto{\pgfqpoint{0.000000in}{2.139691in}}%
\pgfpathlineto{\pgfqpoint{0.000000in}{0.000000in}}%
\pgfpathclose%
\pgfusepath{}%
\end{pgfscope}%
\begin{pgfscope}%
\pgfsetbuttcap%
\pgfsetmiterjoin%
\pgfsetlinewidth{0.000000pt}%
\definecolor{currentstroke}{rgb}{0.000000,0.000000,0.000000}%
\pgfsetstrokecolor{currentstroke}%
\pgfsetstrokeopacity{0.000000}%
\pgfsetdash{}{0pt}%
\pgfpathmoveto{\pgfqpoint{0.553704in}{0.499691in}}%
\pgfpathlineto{\pgfqpoint{3.421204in}{0.499691in}}%
\pgfpathlineto{\pgfqpoint{3.421204in}{2.039691in}}%
\pgfpathlineto{\pgfqpoint{0.553704in}{2.039691in}}%
\pgfpathlineto{\pgfqpoint{0.553704in}{0.499691in}}%
\pgfpathclose%
\pgfusepath{}%
\end{pgfscope}%
\begin{pgfscope}%
\pgfsetbuttcap%
\pgfsetroundjoin%
\definecolor{currentfill}{rgb}{0.000000,0.000000,0.000000}%
\pgfsetfillcolor{currentfill}%
\pgfsetlinewidth{0.803000pt}%
\definecolor{currentstroke}{rgb}{0.000000,0.000000,0.000000}%
\pgfsetstrokecolor{currentstroke}%
\pgfsetdash{}{0pt}%
\pgfsys@defobject{currentmarker}{\pgfqpoint{0.000000in}{-0.048611in}}{\pgfqpoint{0.000000in}{0.000000in}}{%
\pgfpathmoveto{\pgfqpoint{0.000000in}{0.000000in}}%
\pgfpathlineto{\pgfqpoint{0.000000in}{-0.048611in}}%
\pgfusepath{stroke,fill}%
}%
\begin{pgfscope}%
\pgfsys@transformshift{0.553704in}{0.499691in}%
\pgfsys@useobject{currentmarker}{}%
\end{pgfscope}%
\end{pgfscope}%
\begin{pgfscope}%
\definecolor{textcolor}{rgb}{0.000000,0.000000,0.000000}%
\pgfsetstrokecolor{textcolor}%
\pgfsetfillcolor{textcolor}%
\pgftext[x=0.553704in,y=0.402469in,,top]{\color{textcolor}\rmfamily\fontsize{10.000000}{12.000000}\selectfont \(\displaystyle {0.00}\)}%
\end{pgfscope}%
\begin{pgfscope}%
\pgfsetbuttcap%
\pgfsetroundjoin%
\definecolor{currentfill}{rgb}{0.000000,0.000000,0.000000}%
\pgfsetfillcolor{currentfill}%
\pgfsetlinewidth{0.803000pt}%
\definecolor{currentstroke}{rgb}{0.000000,0.000000,0.000000}%
\pgfsetstrokecolor{currentstroke}%
\pgfsetdash{}{0pt}%
\pgfsys@defobject{currentmarker}{\pgfqpoint{0.000000in}{-0.048611in}}{\pgfqpoint{0.000000in}{0.000000in}}{%
\pgfpathmoveto{\pgfqpoint{0.000000in}{0.000000in}}%
\pgfpathlineto{\pgfqpoint{0.000000in}{-0.048611in}}%
\pgfusepath{stroke,fill}%
}%
\begin{pgfscope}%
\pgfsys@transformshift{1.085256in}{0.499691in}%
\pgfsys@useobject{currentmarker}{}%
\end{pgfscope}%
\end{pgfscope}%
\begin{pgfscope}%
\definecolor{textcolor}{rgb}{0.000000,0.000000,0.000000}%
\pgfsetstrokecolor{textcolor}%
\pgfsetfillcolor{textcolor}%
\pgftext[x=1.085256in,y=0.402469in,,top]{\color{textcolor}\rmfamily\fontsize{10.000000}{12.000000}\selectfont \(\displaystyle {0.25}\)}%
\end{pgfscope}%
\begin{pgfscope}%
\pgfsetbuttcap%
\pgfsetroundjoin%
\definecolor{currentfill}{rgb}{0.000000,0.000000,0.000000}%
\pgfsetfillcolor{currentfill}%
\pgfsetlinewidth{0.803000pt}%
\definecolor{currentstroke}{rgb}{0.000000,0.000000,0.000000}%
\pgfsetstrokecolor{currentstroke}%
\pgfsetdash{}{0pt}%
\pgfsys@defobject{currentmarker}{\pgfqpoint{0.000000in}{-0.048611in}}{\pgfqpoint{0.000000in}{0.000000in}}{%
\pgfpathmoveto{\pgfqpoint{0.000000in}{0.000000in}}%
\pgfpathlineto{\pgfqpoint{0.000000in}{-0.048611in}}%
\pgfusepath{stroke,fill}%
}%
\begin{pgfscope}%
\pgfsys@transformshift{1.616807in}{0.499691in}%
\pgfsys@useobject{currentmarker}{}%
\end{pgfscope}%
\end{pgfscope}%
\begin{pgfscope}%
\definecolor{textcolor}{rgb}{0.000000,0.000000,0.000000}%
\pgfsetstrokecolor{textcolor}%
\pgfsetfillcolor{textcolor}%
\pgftext[x=1.616807in,y=0.402469in,,top]{\color{textcolor}\rmfamily\fontsize{10.000000}{12.000000}\selectfont \(\displaystyle {0.50}\)}%
\end{pgfscope}%
\begin{pgfscope}%
\pgfsetbuttcap%
\pgfsetroundjoin%
\definecolor{currentfill}{rgb}{0.000000,0.000000,0.000000}%
\pgfsetfillcolor{currentfill}%
\pgfsetlinewidth{0.803000pt}%
\definecolor{currentstroke}{rgb}{0.000000,0.000000,0.000000}%
\pgfsetstrokecolor{currentstroke}%
\pgfsetdash{}{0pt}%
\pgfsys@defobject{currentmarker}{\pgfqpoint{0.000000in}{-0.048611in}}{\pgfqpoint{0.000000in}{0.000000in}}{%
\pgfpathmoveto{\pgfqpoint{0.000000in}{0.000000in}}%
\pgfpathlineto{\pgfqpoint{0.000000in}{-0.048611in}}%
\pgfusepath{stroke,fill}%
}%
\begin{pgfscope}%
\pgfsys@transformshift{2.148359in}{0.499691in}%
\pgfsys@useobject{currentmarker}{}%
\end{pgfscope}%
\end{pgfscope}%
\begin{pgfscope}%
\definecolor{textcolor}{rgb}{0.000000,0.000000,0.000000}%
\pgfsetstrokecolor{textcolor}%
\pgfsetfillcolor{textcolor}%
\pgftext[x=2.148359in,y=0.402469in,,top]{\color{textcolor}\rmfamily\fontsize{10.000000}{12.000000}\selectfont \(\displaystyle {0.75}\)}%
\end{pgfscope}%
\begin{pgfscope}%
\pgfsetbuttcap%
\pgfsetroundjoin%
\definecolor{currentfill}{rgb}{0.000000,0.000000,0.000000}%
\pgfsetfillcolor{currentfill}%
\pgfsetlinewidth{0.803000pt}%
\definecolor{currentstroke}{rgb}{0.000000,0.000000,0.000000}%
\pgfsetstrokecolor{currentstroke}%
\pgfsetdash{}{0pt}%
\pgfsys@defobject{currentmarker}{\pgfqpoint{0.000000in}{-0.048611in}}{\pgfqpoint{0.000000in}{0.000000in}}{%
\pgfpathmoveto{\pgfqpoint{0.000000in}{0.000000in}}%
\pgfpathlineto{\pgfqpoint{0.000000in}{-0.048611in}}%
\pgfusepath{stroke,fill}%
}%
\begin{pgfscope}%
\pgfsys@transformshift{2.679910in}{0.499691in}%
\pgfsys@useobject{currentmarker}{}%
\end{pgfscope}%
\end{pgfscope}%
\begin{pgfscope}%
\definecolor{textcolor}{rgb}{0.000000,0.000000,0.000000}%
\pgfsetstrokecolor{textcolor}%
\pgfsetfillcolor{textcolor}%
\pgftext[x=2.679910in,y=0.402469in,,top]{\color{textcolor}\rmfamily\fontsize{10.000000}{12.000000}\selectfont \(\displaystyle {1.00}\)}%
\end{pgfscope}%
\begin{pgfscope}%
\pgfsetbuttcap%
\pgfsetroundjoin%
\definecolor{currentfill}{rgb}{0.000000,0.000000,0.000000}%
\pgfsetfillcolor{currentfill}%
\pgfsetlinewidth{0.803000pt}%
\definecolor{currentstroke}{rgb}{0.000000,0.000000,0.000000}%
\pgfsetstrokecolor{currentstroke}%
\pgfsetdash{}{0pt}%
\pgfsys@defobject{currentmarker}{\pgfqpoint{0.000000in}{-0.048611in}}{\pgfqpoint{0.000000in}{0.000000in}}{%
\pgfpathmoveto{\pgfqpoint{0.000000in}{0.000000in}}%
\pgfpathlineto{\pgfqpoint{0.000000in}{-0.048611in}}%
\pgfusepath{stroke,fill}%
}%
\begin{pgfscope}%
\pgfsys@transformshift{3.211462in}{0.499691in}%
\pgfsys@useobject{currentmarker}{}%
\end{pgfscope}%
\end{pgfscope}%
\begin{pgfscope}%
\definecolor{textcolor}{rgb}{0.000000,0.000000,0.000000}%
\pgfsetstrokecolor{textcolor}%
\pgfsetfillcolor{textcolor}%
\pgftext[x=3.211462in,y=0.402469in,,top]{\color{textcolor}\rmfamily\fontsize{10.000000}{12.000000}\selectfont \(\displaystyle {1.25}\)}%
\end{pgfscope}%
\begin{pgfscope}%
\definecolor{textcolor}{rgb}{0.000000,0.000000,0.000000}%
\pgfsetstrokecolor{textcolor}%
\pgfsetfillcolor{textcolor}%
\pgftext[x=1.987454in,y=0.223457in,,top]{\color{textcolor}\rmfamily\fontsize{10.000000}{12.000000}\selectfont Radius}%
\end{pgfscope}%
\begin{pgfscope}%
\pgfsetbuttcap%
\pgfsetroundjoin%
\definecolor{currentfill}{rgb}{0.000000,0.000000,0.000000}%
\pgfsetfillcolor{currentfill}%
\pgfsetlinewidth{0.803000pt}%
\definecolor{currentstroke}{rgb}{0.000000,0.000000,0.000000}%
\pgfsetstrokecolor{currentstroke}%
\pgfsetdash{}{0pt}%
\pgfsys@defobject{currentmarker}{\pgfqpoint{-0.048611in}{0.000000in}}{\pgfqpoint{-0.000000in}{0.000000in}}{%
\pgfpathmoveto{\pgfqpoint{-0.000000in}{0.000000in}}%
\pgfpathlineto{\pgfqpoint{-0.048611in}{0.000000in}}%
\pgfusepath{stroke,fill}%
}%
\begin{pgfscope}%
\pgfsys@transformshift{0.553704in}{0.569691in}%
\pgfsys@useobject{currentmarker}{}%
\end{pgfscope}%
\end{pgfscope}%
\begin{pgfscope}%
\definecolor{textcolor}{rgb}{0.000000,0.000000,0.000000}%
\pgfsetstrokecolor{textcolor}%
\pgfsetfillcolor{textcolor}%
\pgftext[x=0.279012in, y=0.521466in, left, base]{\color{textcolor}\rmfamily\fontsize{10.000000}{12.000000}\selectfont \(\displaystyle {0.0}\)}%
\end{pgfscope}%
\begin{pgfscope}%
\pgfsetbuttcap%
\pgfsetroundjoin%
\definecolor{currentfill}{rgb}{0.000000,0.000000,0.000000}%
\pgfsetfillcolor{currentfill}%
\pgfsetlinewidth{0.803000pt}%
\definecolor{currentstroke}{rgb}{0.000000,0.000000,0.000000}%
\pgfsetstrokecolor{currentstroke}%
\pgfsetdash{}{0pt}%
\pgfsys@defobject{currentmarker}{\pgfqpoint{-0.048611in}{0.000000in}}{\pgfqpoint{-0.000000in}{0.000000in}}{%
\pgfpathmoveto{\pgfqpoint{-0.000000in}{0.000000in}}%
\pgfpathlineto{\pgfqpoint{-0.048611in}{0.000000in}}%
\pgfusepath{stroke,fill}%
}%
\begin{pgfscope}%
\pgfsys@transformshift{0.553704in}{0.888598in}%
\pgfsys@useobject{currentmarker}{}%
\end{pgfscope}%
\end{pgfscope}%
\begin{pgfscope}%
\definecolor{textcolor}{rgb}{0.000000,0.000000,0.000000}%
\pgfsetstrokecolor{textcolor}%
\pgfsetfillcolor{textcolor}%
\pgftext[x=0.279012in, y=0.840372in, left, base]{\color{textcolor}\rmfamily\fontsize{10.000000}{12.000000}\selectfont \(\displaystyle {0.2}\)}%
\end{pgfscope}%
\begin{pgfscope}%
\pgfsetbuttcap%
\pgfsetroundjoin%
\definecolor{currentfill}{rgb}{0.000000,0.000000,0.000000}%
\pgfsetfillcolor{currentfill}%
\pgfsetlinewidth{0.803000pt}%
\definecolor{currentstroke}{rgb}{0.000000,0.000000,0.000000}%
\pgfsetstrokecolor{currentstroke}%
\pgfsetdash{}{0pt}%
\pgfsys@defobject{currentmarker}{\pgfqpoint{-0.048611in}{0.000000in}}{\pgfqpoint{-0.000000in}{0.000000in}}{%
\pgfpathmoveto{\pgfqpoint{-0.000000in}{0.000000in}}%
\pgfpathlineto{\pgfqpoint{-0.048611in}{0.000000in}}%
\pgfusepath{stroke,fill}%
}%
\begin{pgfscope}%
\pgfsys@transformshift{0.553704in}{1.207504in}%
\pgfsys@useobject{currentmarker}{}%
\end{pgfscope}%
\end{pgfscope}%
\begin{pgfscope}%
\definecolor{textcolor}{rgb}{0.000000,0.000000,0.000000}%
\pgfsetstrokecolor{textcolor}%
\pgfsetfillcolor{textcolor}%
\pgftext[x=0.279012in, y=1.159279in, left, base]{\color{textcolor}\rmfamily\fontsize{10.000000}{12.000000}\selectfont \(\displaystyle {0.4}\)}%
\end{pgfscope}%
\begin{pgfscope}%
\pgfsetbuttcap%
\pgfsetroundjoin%
\definecolor{currentfill}{rgb}{0.000000,0.000000,0.000000}%
\pgfsetfillcolor{currentfill}%
\pgfsetlinewidth{0.803000pt}%
\definecolor{currentstroke}{rgb}{0.000000,0.000000,0.000000}%
\pgfsetstrokecolor{currentstroke}%
\pgfsetdash{}{0pt}%
\pgfsys@defobject{currentmarker}{\pgfqpoint{-0.048611in}{0.000000in}}{\pgfqpoint{-0.000000in}{0.000000in}}{%
\pgfpathmoveto{\pgfqpoint{-0.000000in}{0.000000in}}%
\pgfpathlineto{\pgfqpoint{-0.048611in}{0.000000in}}%
\pgfusepath{stroke,fill}%
}%
\begin{pgfscope}%
\pgfsys@transformshift{0.553704in}{1.526411in}%
\pgfsys@useobject{currentmarker}{}%
\end{pgfscope}%
\end{pgfscope}%
\begin{pgfscope}%
\definecolor{textcolor}{rgb}{0.000000,0.000000,0.000000}%
\pgfsetstrokecolor{textcolor}%
\pgfsetfillcolor{textcolor}%
\pgftext[x=0.279012in, y=1.478186in, left, base]{\color{textcolor}\rmfamily\fontsize{10.000000}{12.000000}\selectfont \(\displaystyle {0.6}\)}%
\end{pgfscope}%
\begin{pgfscope}%
\pgfsetbuttcap%
\pgfsetroundjoin%
\definecolor{currentfill}{rgb}{0.000000,0.000000,0.000000}%
\pgfsetfillcolor{currentfill}%
\pgfsetlinewidth{0.803000pt}%
\definecolor{currentstroke}{rgb}{0.000000,0.000000,0.000000}%
\pgfsetstrokecolor{currentstroke}%
\pgfsetdash{}{0pt}%
\pgfsys@defobject{currentmarker}{\pgfqpoint{-0.048611in}{0.000000in}}{\pgfqpoint{-0.000000in}{0.000000in}}{%
\pgfpathmoveto{\pgfqpoint{-0.000000in}{0.000000in}}%
\pgfpathlineto{\pgfqpoint{-0.048611in}{0.000000in}}%
\pgfusepath{stroke,fill}%
}%
\begin{pgfscope}%
\pgfsys@transformshift{0.553704in}{1.845318in}%
\pgfsys@useobject{currentmarker}{}%
\end{pgfscope}%
\end{pgfscope}%
\begin{pgfscope}%
\definecolor{textcolor}{rgb}{0.000000,0.000000,0.000000}%
\pgfsetstrokecolor{textcolor}%
\pgfsetfillcolor{textcolor}%
\pgftext[x=0.279012in, y=1.797092in, left, base]{\color{textcolor}\rmfamily\fontsize{10.000000}{12.000000}\selectfont \(\displaystyle {0.8}\)}%
\end{pgfscope}%
\begin{pgfscope}%
\definecolor{textcolor}{rgb}{0.000000,0.000000,0.000000}%
\pgfsetstrokecolor{textcolor}%
\pgfsetfillcolor{textcolor}%
\pgftext[x=0.223457in,y=1.269691in,,bottom,rotate=90.000000]{\color{textcolor}\rmfamily\fontsize{10.000000}{12.000000}\selectfont Certified accuracy}%
\end{pgfscope}%
\begin{pgfscope}%
\pgfpathrectangle{\pgfqpoint{0.553704in}{0.499691in}}{\pgfqpoint{2.867500in}{1.540000in}}%
\pgfusepath{clip}%
\pgfsetrectcap%
\pgfsetroundjoin%
\pgfsetlinewidth{1.505625pt}%
\definecolor{currentstroke}{rgb}{0.105882,0.619608,0.466667}%
\pgfsetstrokecolor{currentstroke}%
\pgfsetdash{}{0pt}%
\pgfpathmoveto{\pgfqpoint{0.553704in}{1.937800in}}%
\pgfpathlineto{\pgfqpoint{0.582475in}{1.928233in}}%
\pgfpathlineto{\pgfqpoint{0.601656in}{1.915477in}}%
\pgfpathlineto{\pgfqpoint{0.620836in}{1.915477in}}%
\pgfpathlineto{\pgfqpoint{0.640017in}{1.909099in}}%
\pgfpathlineto{\pgfqpoint{0.649607in}{1.909099in}}%
\pgfpathlineto{\pgfqpoint{0.659197in}{1.905910in}}%
\pgfpathlineto{\pgfqpoint{0.668788in}{1.899532in}}%
\pgfpathlineto{\pgfqpoint{0.678378in}{1.889964in}}%
\pgfpathlineto{\pgfqpoint{0.687968in}{1.886775in}}%
\pgfpathlineto{\pgfqpoint{0.697559in}{1.874019in}}%
\pgfpathlineto{\pgfqpoint{0.707149in}{1.870830in}}%
\pgfpathlineto{\pgfqpoint{0.716739in}{1.864452in}}%
\pgfpathlineto{\pgfqpoint{0.726330in}{1.861263in}}%
\pgfpathlineto{\pgfqpoint{0.735920in}{1.848507in}}%
\pgfpathlineto{\pgfqpoint{0.745510in}{1.848507in}}%
\pgfpathlineto{\pgfqpoint{0.755100in}{1.835750in}}%
\pgfpathlineto{\pgfqpoint{0.764691in}{1.835750in}}%
\pgfpathlineto{\pgfqpoint{0.793462in}{1.807049in}}%
\pgfpathlineto{\pgfqpoint{0.803052in}{1.807049in}}%
\pgfpathlineto{\pgfqpoint{0.812642in}{1.803860in}}%
\pgfpathlineto{\pgfqpoint{0.822233in}{1.797482in}}%
\pgfpathlineto{\pgfqpoint{0.831823in}{1.784725in}}%
\pgfpathlineto{\pgfqpoint{0.851003in}{1.778347in}}%
\pgfpathlineto{\pgfqpoint{0.860594in}{1.778347in}}%
\pgfpathlineto{\pgfqpoint{0.870184in}{1.775158in}}%
\pgfpathlineto{\pgfqpoint{0.879774in}{1.775158in}}%
\pgfpathlineto{\pgfqpoint{0.889365in}{1.768780in}}%
\pgfpathlineto{\pgfqpoint{0.918136in}{1.759213in}}%
\pgfpathlineto{\pgfqpoint{0.927726in}{1.746456in}}%
\pgfpathlineto{\pgfqpoint{0.937316in}{1.743267in}}%
\pgfpathlineto{\pgfqpoint{0.946907in}{1.743267in}}%
\pgfpathlineto{\pgfqpoint{0.956497in}{1.736889in}}%
\pgfpathlineto{\pgfqpoint{0.966087in}{1.733700in}}%
\pgfpathlineto{\pgfqpoint{0.975677in}{1.727322in}}%
\pgfpathlineto{\pgfqpoint{0.985268in}{1.724133in}}%
\pgfpathlineto{\pgfqpoint{0.994858in}{1.714566in}}%
\pgfpathlineto{\pgfqpoint{1.004448in}{1.708188in}}%
\pgfpathlineto{\pgfqpoint{1.014039in}{1.704999in}}%
\pgfpathlineto{\pgfqpoint{1.023629in}{1.692242in}}%
\pgfpathlineto{\pgfqpoint{1.033219in}{1.682675in}}%
\pgfpathlineto{\pgfqpoint{1.042810in}{1.679486in}}%
\pgfpathlineto{\pgfqpoint{1.061990in}{1.666730in}}%
\pgfpathlineto{\pgfqpoint{1.071580in}{1.663541in}}%
\pgfpathlineto{\pgfqpoint{1.081171in}{1.653974in}}%
\pgfpathlineto{\pgfqpoint{1.090761in}{1.650785in}}%
\pgfpathlineto{\pgfqpoint{1.100351in}{1.644406in}}%
\pgfpathlineto{\pgfqpoint{1.109942in}{1.641217in}}%
\pgfpathlineto{\pgfqpoint{1.119532in}{1.641217in}}%
\pgfpathlineto{\pgfqpoint{1.129122in}{1.631650in}}%
\pgfpathlineto{\pgfqpoint{1.148303in}{1.625272in}}%
\pgfpathlineto{\pgfqpoint{1.157893in}{1.609327in}}%
\pgfpathlineto{\pgfqpoint{1.167483in}{1.596570in}}%
\pgfpathlineto{\pgfqpoint{1.186664in}{1.583814in}}%
\pgfpathlineto{\pgfqpoint{1.215435in}{1.555113in}}%
\pgfpathlineto{\pgfqpoint{1.253796in}{1.529600in}}%
\pgfpathlineto{\pgfqpoint{1.272977in}{1.529600in}}%
\pgfpathlineto{\pgfqpoint{1.282567in}{1.523222in}}%
\pgfpathlineto{\pgfqpoint{1.292157in}{1.513655in}}%
\pgfpathlineto{\pgfqpoint{1.301748in}{1.510466in}}%
\pgfpathlineto{\pgfqpoint{1.311338in}{1.504087in}}%
\pgfpathlineto{\pgfqpoint{1.320928in}{1.504087in}}%
\pgfpathlineto{\pgfqpoint{1.330519in}{1.491331in}}%
\pgfpathlineto{\pgfqpoint{1.340109in}{1.488142in}}%
\pgfpathlineto{\pgfqpoint{1.349699in}{1.472197in}}%
\pgfpathlineto{\pgfqpoint{1.359289in}{1.465819in}}%
\pgfpathlineto{\pgfqpoint{1.368880in}{1.449873in}}%
\pgfpathlineto{\pgfqpoint{1.378470in}{1.446684in}}%
\pgfpathlineto{\pgfqpoint{1.388060in}{1.437117in}}%
\pgfpathlineto{\pgfqpoint{1.397651in}{1.433928in}}%
\pgfpathlineto{\pgfqpoint{1.407241in}{1.427550in}}%
\pgfpathlineto{\pgfqpoint{1.416831in}{1.411605in}}%
\pgfpathlineto{\pgfqpoint{1.426422in}{1.402037in}}%
\pgfpathlineto{\pgfqpoint{1.436012in}{1.398848in}}%
\pgfpathlineto{\pgfqpoint{1.455192in}{1.379714in}}%
\pgfpathlineto{\pgfqpoint{1.464783in}{1.373336in}}%
\pgfpathlineto{\pgfqpoint{1.474373in}{1.363769in}}%
\pgfpathlineto{\pgfqpoint{1.483963in}{1.357390in}}%
\pgfpathlineto{\pgfqpoint{1.493554in}{1.347823in}}%
\pgfpathlineto{\pgfqpoint{1.503144in}{1.344634in}}%
\pgfpathlineto{\pgfqpoint{1.512734in}{1.335067in}}%
\pgfpathlineto{\pgfqpoint{1.522325in}{1.328689in}}%
\pgfpathlineto{\pgfqpoint{1.531915in}{1.319122in}}%
\pgfpathlineto{\pgfqpoint{1.541505in}{1.306365in}}%
\pgfpathlineto{\pgfqpoint{1.551095in}{1.299987in}}%
\pgfpathlineto{\pgfqpoint{1.560686in}{1.290420in}}%
\pgfpathlineto{\pgfqpoint{1.589457in}{1.280853in}}%
\pgfpathlineto{\pgfqpoint{1.599047in}{1.271286in}}%
\pgfpathlineto{\pgfqpoint{1.618228in}{1.264908in}}%
\pgfpathlineto{\pgfqpoint{1.627818in}{1.252151in}}%
\pgfpathlineto{\pgfqpoint{1.637408in}{1.242584in}}%
\pgfpathlineto{\pgfqpoint{1.656589in}{1.229828in}}%
\pgfpathlineto{\pgfqpoint{1.666179in}{1.220261in}}%
\pgfpathlineto{\pgfqpoint{1.675769in}{1.201126in}}%
\pgfpathlineto{\pgfqpoint{1.685360in}{1.178803in}}%
\pgfpathlineto{\pgfqpoint{1.694950in}{1.166046in}}%
\pgfpathlineto{\pgfqpoint{1.704540in}{1.156479in}}%
\pgfpathlineto{\pgfqpoint{1.714131in}{1.143723in}}%
\pgfpathlineto{\pgfqpoint{1.723721in}{1.134156in}}%
\pgfpathlineto{\pgfqpoint{1.733311in}{1.118210in}}%
\pgfpathlineto{\pgfqpoint{1.742901in}{1.118210in}}%
\pgfpathlineto{\pgfqpoint{1.752492in}{1.067185in}}%
\pgfpathlineto{\pgfqpoint{1.762082in}{1.067185in}}%
\pgfpathlineto{\pgfqpoint{1.771672in}{0.569691in}}%
\pgfpathlineto{\pgfqpoint{3.421204in}{0.569691in}}%
\pgfpathlineto{\pgfqpoint{3.421204in}{0.569691in}}%
\pgfusepath{stroke}%
\end{pgfscope}%
\begin{pgfscope}%
\pgfpathrectangle{\pgfqpoint{0.553704in}{0.499691in}}{\pgfqpoint{2.867500in}{1.540000in}}%
\pgfusepath{clip}%
\pgfsetrectcap%
\pgfsetroundjoin%
\pgfsetlinewidth{1.505625pt}%
\definecolor{currentstroke}{rgb}{0.850980,0.372549,0.007843}%
\pgfsetstrokecolor{currentstroke}%
\pgfsetdash{}{0pt}%
\pgfpathmoveto{\pgfqpoint{0.553704in}{1.969691in}}%
\pgfpathlineto{\pgfqpoint{0.572885in}{1.963313in}}%
\pgfpathlineto{\pgfqpoint{0.582475in}{1.953746in}}%
\pgfpathlineto{\pgfqpoint{0.592065in}{1.947368in}}%
\pgfpathlineto{\pgfqpoint{0.601656in}{1.947368in}}%
\pgfpathlineto{\pgfqpoint{0.611246in}{1.944179in}}%
\pgfpathlineto{\pgfqpoint{0.630427in}{1.944179in}}%
\pgfpathlineto{\pgfqpoint{0.668788in}{1.918666in}}%
\pgfpathlineto{\pgfqpoint{0.678378in}{1.915477in}}%
\pgfpathlineto{\pgfqpoint{0.687968in}{1.909099in}}%
\pgfpathlineto{\pgfqpoint{0.697559in}{1.905910in}}%
\pgfpathlineto{\pgfqpoint{0.707149in}{1.899532in}}%
\pgfpathlineto{\pgfqpoint{0.716739in}{1.899532in}}%
\pgfpathlineto{\pgfqpoint{0.726330in}{1.896343in}}%
\pgfpathlineto{\pgfqpoint{0.735920in}{1.886775in}}%
\pgfpathlineto{\pgfqpoint{0.745510in}{1.880397in}}%
\pgfpathlineto{\pgfqpoint{0.755100in}{1.880397in}}%
\pgfpathlineto{\pgfqpoint{0.764691in}{1.870830in}}%
\pgfpathlineto{\pgfqpoint{0.774281in}{1.858074in}}%
\pgfpathlineto{\pgfqpoint{0.783871in}{1.848507in}}%
\pgfpathlineto{\pgfqpoint{0.793462in}{1.848507in}}%
\pgfpathlineto{\pgfqpoint{0.803052in}{1.835750in}}%
\pgfpathlineto{\pgfqpoint{0.812642in}{1.832561in}}%
\pgfpathlineto{\pgfqpoint{0.822233in}{1.826183in}}%
\pgfpathlineto{\pgfqpoint{0.831823in}{1.816616in}}%
\pgfpathlineto{\pgfqpoint{0.841413in}{1.813427in}}%
\pgfpathlineto{\pgfqpoint{0.851003in}{1.803860in}}%
\pgfpathlineto{\pgfqpoint{0.860594in}{1.803860in}}%
\pgfpathlineto{\pgfqpoint{0.889365in}{1.784725in}}%
\pgfpathlineto{\pgfqpoint{0.898955in}{1.781536in}}%
\pgfpathlineto{\pgfqpoint{0.908545in}{1.768780in}}%
\pgfpathlineto{\pgfqpoint{0.918136in}{1.762402in}}%
\pgfpathlineto{\pgfqpoint{0.927726in}{1.762402in}}%
\pgfpathlineto{\pgfqpoint{0.937316in}{1.756024in}}%
\pgfpathlineto{\pgfqpoint{0.946907in}{1.743267in}}%
\pgfpathlineto{\pgfqpoint{0.956497in}{1.743267in}}%
\pgfpathlineto{\pgfqpoint{0.985268in}{1.733700in}}%
\pgfpathlineto{\pgfqpoint{1.014039in}{1.733700in}}%
\pgfpathlineto{\pgfqpoint{1.033219in}{1.720944in}}%
\pgfpathlineto{\pgfqpoint{1.042810in}{1.720944in}}%
\pgfpathlineto{\pgfqpoint{1.052400in}{1.714566in}}%
\pgfpathlineto{\pgfqpoint{1.061990in}{1.711377in}}%
\pgfpathlineto{\pgfqpoint{1.071580in}{1.711377in}}%
\pgfpathlineto{\pgfqpoint{1.081171in}{1.698620in}}%
\pgfpathlineto{\pgfqpoint{1.090761in}{1.695431in}}%
\pgfpathlineto{\pgfqpoint{1.109942in}{1.682675in}}%
\pgfpathlineto{\pgfqpoint{1.119532in}{1.679486in}}%
\pgfpathlineto{\pgfqpoint{1.129122in}{1.666730in}}%
\pgfpathlineto{\pgfqpoint{1.138713in}{1.657163in}}%
\pgfpathlineto{\pgfqpoint{1.148303in}{1.650785in}}%
\pgfpathlineto{\pgfqpoint{1.157893in}{1.647595in}}%
\pgfpathlineto{\pgfqpoint{1.167483in}{1.641217in}}%
\pgfpathlineto{\pgfqpoint{1.177074in}{1.641217in}}%
\pgfpathlineto{\pgfqpoint{1.186664in}{1.631650in}}%
\pgfpathlineto{\pgfqpoint{1.196254in}{1.625272in}}%
\pgfpathlineto{\pgfqpoint{1.234616in}{1.612516in}}%
\pgfpathlineto{\pgfqpoint{1.244206in}{1.602949in}}%
\pgfpathlineto{\pgfqpoint{1.253796in}{1.596570in}}%
\pgfpathlineto{\pgfqpoint{1.263386in}{1.574247in}}%
\pgfpathlineto{\pgfqpoint{1.272977in}{1.561491in}}%
\pgfpathlineto{\pgfqpoint{1.282567in}{1.558302in}}%
\pgfpathlineto{\pgfqpoint{1.292157in}{1.551923in}}%
\pgfpathlineto{\pgfqpoint{1.311338in}{1.532789in}}%
\pgfpathlineto{\pgfqpoint{1.320928in}{1.520033in}}%
\pgfpathlineto{\pgfqpoint{1.340109in}{1.513655in}}%
\pgfpathlineto{\pgfqpoint{1.349699in}{1.513655in}}%
\pgfpathlineto{\pgfqpoint{1.368880in}{1.494520in}}%
\pgfpathlineto{\pgfqpoint{1.378470in}{1.481764in}}%
\pgfpathlineto{\pgfqpoint{1.397651in}{1.462630in}}%
\pgfpathlineto{\pgfqpoint{1.407241in}{1.459441in}}%
\pgfpathlineto{\pgfqpoint{1.426422in}{1.440306in}}%
\pgfpathlineto{\pgfqpoint{1.436012in}{1.433928in}}%
\pgfpathlineto{\pgfqpoint{1.445602in}{1.430739in}}%
\pgfpathlineto{\pgfqpoint{1.455192in}{1.430739in}}%
\pgfpathlineto{\pgfqpoint{1.464783in}{1.421172in}}%
\pgfpathlineto{\pgfqpoint{1.474373in}{1.408415in}}%
\pgfpathlineto{\pgfqpoint{1.483963in}{1.398848in}}%
\pgfpathlineto{\pgfqpoint{1.493554in}{1.386092in}}%
\pgfpathlineto{\pgfqpoint{1.503144in}{1.379714in}}%
\pgfpathlineto{\pgfqpoint{1.512734in}{1.376525in}}%
\pgfpathlineto{\pgfqpoint{1.522325in}{1.366958in}}%
\pgfpathlineto{\pgfqpoint{1.531915in}{1.366958in}}%
\pgfpathlineto{\pgfqpoint{1.541505in}{1.360579in}}%
\pgfpathlineto{\pgfqpoint{1.551095in}{1.357390in}}%
\pgfpathlineto{\pgfqpoint{1.570276in}{1.331878in}}%
\pgfpathlineto{\pgfqpoint{1.579866in}{1.315933in}}%
\pgfpathlineto{\pgfqpoint{1.589457in}{1.309554in}}%
\pgfpathlineto{\pgfqpoint{1.599047in}{1.293609in}}%
\pgfpathlineto{\pgfqpoint{1.618228in}{1.280853in}}%
\pgfpathlineto{\pgfqpoint{1.627818in}{1.268097in}}%
\pgfpathlineto{\pgfqpoint{1.646998in}{1.261718in}}%
\pgfpathlineto{\pgfqpoint{1.656589in}{1.248962in}}%
\pgfpathlineto{\pgfqpoint{1.666179in}{1.239395in}}%
\pgfpathlineto{\pgfqpoint{1.675769in}{1.220261in}}%
\pgfpathlineto{\pgfqpoint{1.685360in}{1.207504in}}%
\pgfpathlineto{\pgfqpoint{1.704540in}{1.194748in}}%
\pgfpathlineto{\pgfqpoint{1.714131in}{1.191559in}}%
\pgfpathlineto{\pgfqpoint{1.723721in}{1.178803in}}%
\pgfpathlineto{\pgfqpoint{1.733311in}{1.162857in}}%
\pgfpathlineto{\pgfqpoint{1.742901in}{1.162857in}}%
\pgfpathlineto{\pgfqpoint{1.752492in}{1.121400in}}%
\pgfpathlineto{\pgfqpoint{1.762082in}{1.121400in}}%
\pgfpathlineto{\pgfqpoint{1.771672in}{0.569691in}}%
\pgfpathlineto{\pgfqpoint{3.421204in}{0.569691in}}%
\pgfpathlineto{\pgfqpoint{3.421204in}{0.569691in}}%
\pgfusepath{stroke}%
\end{pgfscope}%
\begin{pgfscope}%
\pgfpathrectangle{\pgfqpoint{0.553704in}{0.499691in}}{\pgfqpoint{2.867500in}{1.540000in}}%
\pgfusepath{clip}%
\pgfsetrectcap%
\pgfsetroundjoin%
\pgfsetlinewidth{1.505625pt}%
\definecolor{currentstroke}{rgb}{0.458824,0.439216,0.701961}%
\pgfsetstrokecolor{currentstroke}%
\pgfsetdash{}{0pt}%
\pgfpathmoveto{\pgfqpoint{0.553704in}{1.963313in}}%
\pgfpathlineto{\pgfqpoint{1.340109in}{1.963313in}}%
\pgfpathlineto{\pgfqpoint{1.349699in}{1.960124in}}%
\pgfpathlineto{\pgfqpoint{1.368880in}{1.960124in}}%
\pgfpathlineto{\pgfqpoint{1.378470in}{1.956935in}}%
\pgfpathlineto{\pgfqpoint{1.407241in}{1.956935in}}%
\pgfpathlineto{\pgfqpoint{1.426422in}{1.950557in}}%
\pgfpathlineto{\pgfqpoint{1.436012in}{1.950557in}}%
\pgfpathlineto{\pgfqpoint{1.445602in}{1.944179in}}%
\pgfpathlineto{\pgfqpoint{1.455192in}{1.940990in}}%
\pgfpathlineto{\pgfqpoint{1.464783in}{1.940990in}}%
\pgfpathlineto{\pgfqpoint{1.474373in}{1.937800in}}%
\pgfpathlineto{\pgfqpoint{1.483963in}{1.937800in}}%
\pgfpathlineto{\pgfqpoint{1.503144in}{1.931422in}}%
\pgfpathlineto{\pgfqpoint{1.522325in}{1.931422in}}%
\pgfpathlineto{\pgfqpoint{1.531915in}{1.925044in}}%
\pgfpathlineto{\pgfqpoint{1.551095in}{1.925044in}}%
\pgfpathlineto{\pgfqpoint{1.560686in}{1.918666in}}%
\pgfpathlineto{\pgfqpoint{1.570276in}{1.918666in}}%
\pgfpathlineto{\pgfqpoint{1.579866in}{1.909099in}}%
\pgfpathlineto{\pgfqpoint{1.589457in}{1.909099in}}%
\pgfpathlineto{\pgfqpoint{1.608637in}{1.896343in}}%
\pgfpathlineto{\pgfqpoint{1.627818in}{1.889964in}}%
\pgfpathlineto{\pgfqpoint{1.637408in}{1.889964in}}%
\pgfpathlineto{\pgfqpoint{1.646998in}{1.886775in}}%
\pgfpathlineto{\pgfqpoint{1.656589in}{1.880397in}}%
\pgfpathlineto{\pgfqpoint{1.666179in}{1.870830in}}%
\pgfpathlineto{\pgfqpoint{1.675769in}{1.867641in}}%
\pgfpathlineto{\pgfqpoint{1.685360in}{1.851696in}}%
\pgfpathlineto{\pgfqpoint{1.694950in}{1.851696in}}%
\pgfpathlineto{\pgfqpoint{1.704540in}{1.848507in}}%
\pgfpathlineto{\pgfqpoint{1.714131in}{1.842128in}}%
\pgfpathlineto{\pgfqpoint{1.723721in}{1.832561in}}%
\pgfpathlineto{\pgfqpoint{1.733311in}{1.826183in}}%
\pgfpathlineto{\pgfqpoint{1.742901in}{1.816616in}}%
\pgfpathlineto{\pgfqpoint{1.752492in}{1.813427in}}%
\pgfpathlineto{\pgfqpoint{1.762082in}{1.807049in}}%
\pgfpathlineto{\pgfqpoint{1.771672in}{1.797482in}}%
\pgfpathlineto{\pgfqpoint{1.781263in}{1.791103in}}%
\pgfpathlineto{\pgfqpoint{1.790853in}{1.791103in}}%
\pgfpathlineto{\pgfqpoint{1.800443in}{1.787914in}}%
\pgfpathlineto{\pgfqpoint{1.810034in}{1.778347in}}%
\pgfpathlineto{\pgfqpoint{1.819624in}{1.771969in}}%
\pgfpathlineto{\pgfqpoint{1.829214in}{1.768780in}}%
\pgfpathlineto{\pgfqpoint{1.838805in}{1.768780in}}%
\pgfpathlineto{\pgfqpoint{1.848395in}{1.762402in}}%
\pgfpathlineto{\pgfqpoint{1.857985in}{1.762402in}}%
\pgfpathlineto{\pgfqpoint{1.867575in}{1.749646in}}%
\pgfpathlineto{\pgfqpoint{1.877166in}{1.749646in}}%
\pgfpathlineto{\pgfqpoint{1.886756in}{1.740078in}}%
\pgfpathlineto{\pgfqpoint{1.896346in}{1.724133in}}%
\pgfpathlineto{\pgfqpoint{1.905937in}{1.714566in}}%
\pgfpathlineto{\pgfqpoint{1.915527in}{1.701810in}}%
\pgfpathlineto{\pgfqpoint{1.925117in}{1.698620in}}%
\pgfpathlineto{\pgfqpoint{1.934708in}{1.692242in}}%
\pgfpathlineto{\pgfqpoint{1.944298in}{1.692242in}}%
\pgfpathlineto{\pgfqpoint{1.953888in}{1.663541in}}%
\pgfpathlineto{\pgfqpoint{1.963478in}{1.657163in}}%
\pgfpathlineto{\pgfqpoint{1.973069in}{1.641217in}}%
\pgfpathlineto{\pgfqpoint{1.982659in}{1.628461in}}%
\pgfpathlineto{\pgfqpoint{1.992249in}{1.622083in}}%
\pgfpathlineto{\pgfqpoint{2.001840in}{1.606138in}}%
\pgfpathlineto{\pgfqpoint{2.011430in}{1.599759in}}%
\pgfpathlineto{\pgfqpoint{2.021020in}{1.577436in}}%
\pgfpathlineto{\pgfqpoint{2.030611in}{1.551923in}}%
\pgfpathlineto{\pgfqpoint{2.040201in}{1.545545in}}%
\pgfpathlineto{\pgfqpoint{2.049791in}{1.532789in}}%
\pgfpathlineto{\pgfqpoint{2.059381in}{1.516844in}}%
\pgfpathlineto{\pgfqpoint{2.068972in}{1.494520in}}%
\pgfpathlineto{\pgfqpoint{2.078562in}{1.475386in}}%
\pgfpathlineto{\pgfqpoint{2.088152in}{1.459441in}}%
\pgfpathlineto{\pgfqpoint{2.097743in}{1.440306in}}%
\pgfpathlineto{\pgfqpoint{2.107333in}{1.430739in}}%
\pgfpathlineto{\pgfqpoint{2.126514in}{1.398848in}}%
\pgfpathlineto{\pgfqpoint{2.136104in}{1.392470in}}%
\pgfpathlineto{\pgfqpoint{2.145694in}{1.373336in}}%
\pgfpathlineto{\pgfqpoint{2.155284in}{1.357390in}}%
\pgfpathlineto{\pgfqpoint{2.174465in}{1.293609in}}%
\pgfpathlineto{\pgfqpoint{2.184055in}{1.274475in}}%
\pgfpathlineto{\pgfqpoint{2.193646in}{1.261718in}}%
\pgfpathlineto{\pgfqpoint{2.203236in}{1.233017in}}%
\pgfpathlineto{\pgfqpoint{2.222417in}{1.201126in}}%
\pgfpathlineto{\pgfqpoint{2.232007in}{1.188370in}}%
\pgfpathlineto{\pgfqpoint{2.241597in}{1.185181in}}%
\pgfpathlineto{\pgfqpoint{2.251187in}{1.166046in}}%
\pgfpathlineto{\pgfqpoint{2.260778in}{1.153290in}}%
\pgfpathlineto{\pgfqpoint{2.270368in}{1.143723in}}%
\pgfpathlineto{\pgfqpoint{2.279958in}{1.127778in}}%
\pgfpathlineto{\pgfqpoint{2.289549in}{1.118210in}}%
\pgfpathlineto{\pgfqpoint{2.299139in}{1.092698in}}%
\pgfpathlineto{\pgfqpoint{2.308729in}{1.073564in}}%
\pgfpathlineto{\pgfqpoint{2.318320in}{1.067185in}}%
\pgfpathlineto{\pgfqpoint{2.337500in}{1.041673in}}%
\pgfpathlineto{\pgfqpoint{2.347090in}{1.032106in}}%
\pgfpathlineto{\pgfqpoint{2.356681in}{1.019349in}}%
\pgfpathlineto{\pgfqpoint{2.366271in}{1.003404in}}%
\pgfpathlineto{\pgfqpoint{2.375861in}{0.993837in}}%
\pgfpathlineto{\pgfqpoint{2.385452in}{0.981081in}}%
\pgfpathlineto{\pgfqpoint{2.395042in}{0.971513in}}%
\pgfpathlineto{\pgfqpoint{2.404632in}{0.955568in}}%
\pgfpathlineto{\pgfqpoint{2.414223in}{0.946001in}}%
\pgfpathlineto{\pgfqpoint{2.423813in}{0.933245in}}%
\pgfpathlineto{\pgfqpoint{2.433403in}{0.923677in}}%
\pgfpathlineto{\pgfqpoint{2.442993in}{0.917299in}}%
\pgfpathlineto{\pgfqpoint{2.452584in}{0.907732in}}%
\pgfpathlineto{\pgfqpoint{2.462174in}{0.888598in}}%
\pgfpathlineto{\pgfqpoint{2.471764in}{0.872652in}}%
\pgfpathlineto{\pgfqpoint{2.481355in}{0.859896in}}%
\pgfpathlineto{\pgfqpoint{2.500535in}{0.828005in}}%
\pgfpathlineto{\pgfqpoint{2.519716in}{0.821627in}}%
\pgfpathlineto{\pgfqpoint{2.529306in}{0.812060in}}%
\pgfpathlineto{\pgfqpoint{2.538896in}{0.812060in}}%
\pgfpathlineto{\pgfqpoint{2.548487in}{0.805682in}}%
\pgfpathlineto{\pgfqpoint{2.558077in}{0.802493in}}%
\pgfpathlineto{\pgfqpoint{2.567667in}{0.789737in}}%
\pgfpathlineto{\pgfqpoint{2.577258in}{0.783359in}}%
\pgfpathlineto{\pgfqpoint{2.586848in}{0.783359in}}%
\pgfpathlineto{\pgfqpoint{2.596438in}{0.773791in}}%
\pgfpathlineto{\pgfqpoint{2.606029in}{0.770602in}}%
\pgfpathlineto{\pgfqpoint{2.615619in}{0.757846in}}%
\pgfpathlineto{\pgfqpoint{2.625209in}{0.741901in}}%
\pgfpathlineto{\pgfqpoint{2.634799in}{0.735523in}}%
\pgfpathlineto{\pgfqpoint{2.644390in}{0.722766in}}%
\pgfpathlineto{\pgfqpoint{2.653980in}{0.716388in}}%
\pgfpathlineto{\pgfqpoint{2.663570in}{0.700443in}}%
\pgfpathlineto{\pgfqpoint{2.673161in}{0.697254in}}%
\pgfpathlineto{\pgfqpoint{2.682751in}{0.697254in}}%
\pgfpathlineto{\pgfqpoint{2.711522in}{0.687687in}}%
\pgfpathlineto{\pgfqpoint{2.721112in}{0.681308in}}%
\pgfpathlineto{\pgfqpoint{2.749883in}{0.671741in}}%
\pgfpathlineto{\pgfqpoint{2.759473in}{0.671741in}}%
\pgfpathlineto{\pgfqpoint{2.769064in}{0.665363in}}%
\pgfpathlineto{\pgfqpoint{2.778654in}{0.665363in}}%
\pgfpathlineto{\pgfqpoint{2.788244in}{0.658985in}}%
\pgfpathlineto{\pgfqpoint{2.797835in}{0.655796in}}%
\pgfpathlineto{\pgfqpoint{2.817015in}{0.655796in}}%
\pgfpathlineto{\pgfqpoint{2.826606in}{0.649418in}}%
\pgfpathlineto{\pgfqpoint{2.836196in}{0.649418in}}%
\pgfpathlineto{\pgfqpoint{2.845786in}{0.643040in}}%
\pgfpathlineto{\pgfqpoint{2.864967in}{0.643040in}}%
\pgfpathlineto{\pgfqpoint{2.874557in}{0.639851in}}%
\pgfpathlineto{\pgfqpoint{2.884147in}{0.630283in}}%
\pgfpathlineto{\pgfqpoint{2.893738in}{0.630283in}}%
\pgfpathlineto{\pgfqpoint{2.903328in}{0.627094in}}%
\pgfpathlineto{\pgfqpoint{2.912918in}{0.620716in}}%
\pgfpathlineto{\pgfqpoint{2.922509in}{0.617527in}}%
\pgfpathlineto{\pgfqpoint{2.932099in}{0.611149in}}%
\pgfpathlineto{\pgfqpoint{2.970460in}{0.611149in}}%
\pgfpathlineto{\pgfqpoint{2.980050in}{0.607960in}}%
\pgfpathlineto{\pgfqpoint{2.989641in}{0.607960in}}%
\pgfpathlineto{\pgfqpoint{2.999231in}{0.604771in}}%
\pgfpathlineto{\pgfqpoint{3.008821in}{0.595204in}}%
\pgfpathlineto{\pgfqpoint{3.018412in}{0.592015in}}%
\pgfpathlineto{\pgfqpoint{3.104724in}{0.592015in}}%
\pgfpathlineto{\pgfqpoint{3.114315in}{0.588826in}}%
\pgfpathlineto{\pgfqpoint{3.123905in}{0.588826in}}%
\pgfpathlineto{\pgfqpoint{3.133495in}{0.585636in}}%
\pgfpathlineto{\pgfqpoint{3.286940in}{0.585636in}}%
\pgfpathlineto{\pgfqpoint{3.296530in}{0.582447in}}%
\pgfpathlineto{\pgfqpoint{3.325301in}{0.582447in}}%
\pgfpathlineto{\pgfqpoint{3.334891in}{0.579258in}}%
\pgfpathlineto{\pgfqpoint{3.354072in}{0.579258in}}%
\pgfpathlineto{\pgfqpoint{3.363662in}{0.576069in}}%
\pgfpathlineto{\pgfqpoint{3.392433in}{0.576069in}}%
\pgfpathlineto{\pgfqpoint{3.402024in}{0.572880in}}%
\pgfpathlineto{\pgfqpoint{3.411614in}{0.572880in}}%
\pgfpathlineto{\pgfqpoint{3.421204in}{0.569691in}}%
\pgfpathlineto{\pgfqpoint{3.421204in}{0.569691in}}%
\pgfusepath{stroke}%
\end{pgfscope}%
\begin{pgfscope}%
\pgfsetrectcap%
\pgfsetmiterjoin%
\pgfsetlinewidth{0.803000pt}%
\definecolor{currentstroke}{rgb}{0.000000,0.000000,0.000000}%
\pgfsetstrokecolor{currentstroke}%
\pgfsetdash{}{0pt}%
\pgfpathmoveto{\pgfqpoint{0.553704in}{0.499691in}}%
\pgfpathlineto{\pgfqpoint{0.553704in}{2.039691in}}%
\pgfusepath{stroke}%
\end{pgfscope}%
\begin{pgfscope}%
\pgfsetrectcap%
\pgfsetmiterjoin%
\pgfsetlinewidth{0.803000pt}%
\definecolor{currentstroke}{rgb}{0.000000,0.000000,0.000000}%
\pgfsetstrokecolor{currentstroke}%
\pgfsetdash{}{0pt}%
\pgfpathmoveto{\pgfqpoint{3.421204in}{0.499691in}}%
\pgfpathlineto{\pgfqpoint{3.421204in}{2.039691in}}%
\pgfusepath{stroke}%
\end{pgfscope}%
\begin{pgfscope}%
\pgfsetrectcap%
\pgfsetmiterjoin%
\pgfsetlinewidth{0.803000pt}%
\definecolor{currentstroke}{rgb}{0.000000,0.000000,0.000000}%
\pgfsetstrokecolor{currentstroke}%
\pgfsetdash{}{0pt}%
\pgfpathmoveto{\pgfqpoint{0.553704in}{0.499691in}}%
\pgfpathlineto{\pgfqpoint{3.421204in}{0.499691in}}%
\pgfusepath{stroke}%
\end{pgfscope}%
\begin{pgfscope}%
\pgfsetrectcap%
\pgfsetmiterjoin%
\pgfsetlinewidth{0.803000pt}%
\definecolor{currentstroke}{rgb}{0.000000,0.000000,0.000000}%
\pgfsetstrokecolor{currentstroke}%
\pgfsetdash{}{0pt}%
\pgfpathmoveto{\pgfqpoint{0.553704in}{2.039691in}}%
\pgfpathlineto{\pgfqpoint{3.421204in}{2.039691in}}%
\pgfusepath{stroke}%
\end{pgfscope}%
\begin{pgfscope}%
\pgfsetbuttcap%
\pgfsetmiterjoin%
\definecolor{currentfill}{rgb}{1.000000,1.000000,1.000000}%
\pgfsetfillcolor{currentfill}%
\pgfsetfillopacity{0.800000}%
\pgfsetlinewidth{1.003750pt}%
\definecolor{currentstroke}{rgb}{0.800000,0.800000,0.800000}%
\pgfsetstrokecolor{currentstroke}%
\pgfsetstrokeopacity{0.800000}%
\pgfsetdash{}{0pt}%
\pgfpathmoveto{\pgfqpoint{0.650926in}{0.569136in}}%
\pgfpathlineto{\pgfqpoint{2.091854in}{0.569136in}}%
\pgfpathquadraticcurveto{\pgfqpoint{2.119632in}{0.569136in}}{\pgfqpoint{2.119632in}{0.596913in}}%
\pgfpathlineto{\pgfqpoint{2.119632in}{1.164043in}}%
\pgfpathquadraticcurveto{\pgfqpoint{2.119632in}{1.191821in}}{\pgfqpoint{2.091854in}{1.191821in}}%
\pgfpathlineto{\pgfqpoint{0.650926in}{1.191821in}}%
\pgfpathquadraticcurveto{\pgfqpoint{0.623149in}{1.191821in}}{\pgfqpoint{0.623149in}{1.164043in}}%
\pgfpathlineto{\pgfqpoint{0.623149in}{0.596913in}}%
\pgfpathquadraticcurveto{\pgfqpoint{0.623149in}{0.569136in}}{\pgfqpoint{0.650926in}{0.569136in}}%
\pgfpathlineto{\pgfqpoint{0.650926in}{0.569136in}}%
\pgfpathclose%
\pgfusepath{stroke,fill}%
\end{pgfscope}%
\begin{pgfscope}%
\pgfsetrectcap%
\pgfsetroundjoin%
\pgfsetlinewidth{1.505625pt}%
\definecolor{currentstroke}{rgb}{0.105882,0.619608,0.466667}%
\pgfsetstrokecolor{currentstroke}%
\pgfsetdash{}{0pt}%
\pgfpathmoveto{\pgfqpoint{0.678704in}{1.087654in}}%
\pgfpathlineto{\pgfqpoint{0.817593in}{1.087654in}}%
\pgfpathlineto{\pgfqpoint{0.956482in}{1.087654in}}%
\pgfusepath{stroke}%
\end{pgfscope}%
\begin{pgfscope}%
\definecolor{textcolor}{rgb}{0.000000,0.000000,0.000000}%
\pgfsetstrokecolor{textcolor}%
\pgfsetfillcolor{textcolor}%
\pgftext[x=1.067593in,y=1.039043in,left,base]{\color{textcolor}\rmfamily\fontsize{10.000000}{12.000000}\selectfont \(\displaystyle \textsc{ProjectedRS}^*\)}%
\end{pgfscope}%
\begin{pgfscope}%
\pgfsetrectcap%
\pgfsetroundjoin%
\pgfsetlinewidth{1.505625pt}%
\definecolor{currentstroke}{rgb}{0.850980,0.372549,0.007843}%
\pgfsetstrokecolor{currentstroke}%
\pgfsetdash{}{0pt}%
\pgfpathmoveto{\pgfqpoint{0.678704in}{0.893981in}}%
\pgfpathlineto{\pgfqpoint{0.817593in}{0.893981in}}%
\pgfpathlineto{\pgfqpoint{0.956482in}{0.893981in}}%
\pgfusepath{stroke}%
\end{pgfscope}%
\begin{pgfscope}%
\definecolor{textcolor}{rgb}{0.000000,0.000000,0.000000}%
\pgfsetstrokecolor{textcolor}%
\pgfsetfillcolor{textcolor}%
\pgftext[x=1.067593in,y=0.845370in,left,base]{\color{textcolor}\rmfamily\fontsize{10.000000}{12.000000}\selectfont \(\displaystyle \textsc{RS}\)}%
\end{pgfscope}%
\begin{pgfscope}%
\pgfsetrectcap%
\pgfsetroundjoin%
\pgfsetlinewidth{1.505625pt}%
\definecolor{currentstroke}{rgb}{0.458824,0.439216,0.701961}%
\pgfsetstrokecolor{currentstroke}%
\pgfsetdash{}{0pt}%
\pgfpathmoveto{\pgfqpoint{0.678704in}{0.700308in}}%
\pgfpathlineto{\pgfqpoint{0.817593in}{0.700308in}}%
\pgfpathlineto{\pgfqpoint{0.956482in}{0.700308in}}%
\pgfusepath{stroke}%
\end{pgfscope}%
\begin{pgfscope}%
\definecolor{textcolor}{rgb}{0.000000,0.000000,0.000000}%
\pgfsetstrokecolor{textcolor}%
\pgfsetfillcolor{textcolor}%
\pgftext[x=1.067593in,y=0.651697in,left,base]{\color{textcolor}\rmfamily\fontsize{10.000000}{12.000000}\selectfont \(\displaystyle \textsc{ANCER}\)}%
\end{pgfscope}%
\end{pgfpicture}%
\makeatother%
\endgroup%

%% file: figs/cifar10_rs4al1sweep.pgf
%% Creator: Matplotlib, PGF backend
%%
%% To include the figure in your LaTeX document, write
%%   \input{<filename>.pgf}
%%
%% Make sure the required packages are loaded in your preamble
%%   \usepackage{pgf}
%%
%% Also ensure that all the required font packages are loaded; for instance,
%% the lmodern package is sometimes necessary when using math font.
%%   \usepackage{lmodern}
%%
%% Figures using additional raster images can only be included by \input if
%% they are in the same directory as the main LaTeX file. For loading figures
%% from other directories you can use the `import` package
%%   \usepackage{import}
%%
%% and then include the figures with
%%   \import{<path to file>}{<filename>.pgf}
%%
%% Matplotlib used the following preamble
%%
\begingroup%
\makeatletter%
\begin{pgfpicture}%
\pgfpathrectangle{\pgfpointorigin}{\pgfqpoint{2.668704in}{2.524691in}}%
\pgfusepath{use as bounding box, clip}%
\begin{pgfscope}%
\pgfsetbuttcap%
\pgfsetmiterjoin%
\pgfsetlinewidth{0.000000pt}%
\definecolor{currentstroke}{rgb}{0.000000,0.000000,0.000000}%
\pgfsetstrokecolor{currentstroke}%
\pgfsetstrokeopacity{0.000000}%
\pgfsetdash{}{0pt}%
\pgfpathmoveto{\pgfqpoint{0.000000in}{0.000000in}}%
\pgfpathlineto{\pgfqpoint{2.668704in}{0.000000in}}%
\pgfpathlineto{\pgfqpoint{2.668704in}{2.524691in}}%
\pgfpathlineto{\pgfqpoint{0.000000in}{2.524691in}}%
\pgfpathlineto{\pgfqpoint{0.000000in}{0.000000in}}%
\pgfpathclose%
\pgfusepath{}%
\end{pgfscope}%
\begin{pgfscope}%
\pgfsetbuttcap%
\pgfsetmiterjoin%
\pgfsetlinewidth{0.000000pt}%
\definecolor{currentstroke}{rgb}{0.000000,0.000000,0.000000}%
\pgfsetstrokecolor{currentstroke}%
\pgfsetstrokeopacity{0.000000}%
\pgfsetdash{}{0pt}%
\pgfpathmoveto{\pgfqpoint{0.553704in}{0.499691in}}%
\pgfpathlineto{\pgfqpoint{2.568704in}{0.499691in}}%
\pgfpathlineto{\pgfqpoint{2.568704in}{2.424691in}}%
\pgfpathlineto{\pgfqpoint{0.553704in}{2.424691in}}%
\pgfpathlineto{\pgfqpoint{0.553704in}{0.499691in}}%
\pgfpathclose%
\pgfusepath{}%
\end{pgfscope}%
\begin{pgfscope}%
\pgfsetbuttcap%
\pgfsetroundjoin%
\definecolor{currentfill}{rgb}{0.000000,0.000000,0.000000}%
\pgfsetfillcolor{currentfill}%
\pgfsetlinewidth{0.803000pt}%
\definecolor{currentstroke}{rgb}{0.000000,0.000000,0.000000}%
\pgfsetstrokecolor{currentstroke}%
\pgfsetdash{}{0pt}%
\pgfsys@defobject{currentmarker}{\pgfqpoint{0.000000in}{-0.048611in}}{\pgfqpoint{0.000000in}{0.000000in}}{%
\pgfpathmoveto{\pgfqpoint{0.000000in}{0.000000in}}%
\pgfpathlineto{\pgfqpoint{0.000000in}{-0.048611in}}%
\pgfusepath{stroke,fill}%
}%
\begin{pgfscope}%
\pgfsys@transformshift{1.118911in}{0.499691in}%
\pgfsys@useobject{currentmarker}{}%
\end{pgfscope}%
\end{pgfscope}%
\begin{pgfscope}%
\definecolor{textcolor}{rgb}{0.000000,0.000000,0.000000}%
\pgfsetstrokecolor{textcolor}%
\pgfsetfillcolor{textcolor}%
\pgftext[x=1.118911in,y=0.402469in,,top]{\color{textcolor}\rmfamily\fontsize{10.000000}{12.000000}\selectfont \(\displaystyle 10^{-4 \alpha}\)}%
\end{pgfscope}%
\begin{pgfscope}%
\pgfsetbuttcap%
\pgfsetroundjoin%
\definecolor{currentfill}{rgb}{0.000000,0.000000,0.000000}%
\pgfsetfillcolor{currentfill}%
\pgfsetlinewidth{0.803000pt}%
\definecolor{currentstroke}{rgb}{0.000000,0.000000,0.000000}%
\pgfsetstrokecolor{currentstroke}%
\pgfsetdash{}{0pt}%
\pgfsys@defobject{currentmarker}{\pgfqpoint{0.000000in}{-0.048611in}}{\pgfqpoint{0.000000in}{0.000000in}}{%
\pgfpathmoveto{\pgfqpoint{0.000000in}{0.000000in}}%
\pgfpathlineto{\pgfqpoint{0.000000in}{-0.048611in}}%
\pgfusepath{stroke,fill}%
}%
\begin{pgfscope}%
\pgfsys@transformshift{1.939871in}{0.499691in}%
\pgfsys@useobject{currentmarker}{}%
\end{pgfscope}%
\end{pgfscope}%
\begin{pgfscope}%
\definecolor{textcolor}{rgb}{0.000000,0.000000,0.000000}%
\pgfsetstrokecolor{textcolor}%
\pgfsetfillcolor{textcolor}%
\pgftext[x=1.939871in,y=0.402469in,,top]{\color{textcolor}\rmfamily\fontsize{10.000000}{12.000000}\selectfont \(\displaystyle 10^{-3 \alpha}\)}%
\end{pgfscope}%
\begin{pgfscope}%
\definecolor{textcolor}{rgb}{0.000000,0.000000,0.000000}%
\pgfsetstrokecolor{textcolor}%
\pgfsetfillcolor{textcolor}%
\pgftext[x=1.561204in,y=0.223457in,,top]{\color{textcolor}\rmfamily\fontsize{10.000000}{12.000000}\selectfont Volume}%
\end{pgfscope}%
\begin{pgfscope}%
\pgfsetbuttcap%
\pgfsetroundjoin%
\definecolor{currentfill}{rgb}{0.000000,0.000000,0.000000}%
\pgfsetfillcolor{currentfill}%
\pgfsetlinewidth{0.803000pt}%
\definecolor{currentstroke}{rgb}{0.000000,0.000000,0.000000}%
\pgfsetstrokecolor{currentstroke}%
\pgfsetdash{}{0pt}%
\pgfsys@defobject{currentmarker}{\pgfqpoint{-0.048611in}{0.000000in}}{\pgfqpoint{-0.000000in}{0.000000in}}{%
\pgfpathmoveto{\pgfqpoint{-0.000000in}{0.000000in}}%
\pgfpathlineto{\pgfqpoint{-0.048611in}{0.000000in}}%
\pgfusepath{stroke,fill}%
}%
\begin{pgfscope}%
\pgfsys@transformshift{0.553704in}{0.587191in}%
\pgfsys@useobject{currentmarker}{}%
\end{pgfscope}%
\end{pgfscope}%
\begin{pgfscope}%
\definecolor{textcolor}{rgb}{0.000000,0.000000,0.000000}%
\pgfsetstrokecolor{textcolor}%
\pgfsetfillcolor{textcolor}%
\pgftext[x=0.279012in, y=0.538966in, left, base]{\color{textcolor}\rmfamily\fontsize{10.000000}{12.000000}\selectfont \(\displaystyle {0.0}\)}%
\end{pgfscope}%
\begin{pgfscope}%
\pgfsetbuttcap%
\pgfsetroundjoin%
\definecolor{currentfill}{rgb}{0.000000,0.000000,0.000000}%
\pgfsetfillcolor{currentfill}%
\pgfsetlinewidth{0.803000pt}%
\definecolor{currentstroke}{rgb}{0.000000,0.000000,0.000000}%
\pgfsetstrokecolor{currentstroke}%
\pgfsetdash{}{0pt}%
\pgfsys@defobject{currentmarker}{\pgfqpoint{-0.048611in}{0.000000in}}{\pgfqpoint{-0.000000in}{0.000000in}}{%
\pgfpathmoveto{\pgfqpoint{-0.000000in}{0.000000in}}%
\pgfpathlineto{\pgfqpoint{-0.048611in}{0.000000in}}%
\pgfusepath{stroke,fill}%
}%
\begin{pgfscope}%
\pgfsys@transformshift{0.553704in}{1.007864in}%
\pgfsys@useobject{currentmarker}{}%
\end{pgfscope}%
\end{pgfscope}%
\begin{pgfscope}%
\definecolor{textcolor}{rgb}{0.000000,0.000000,0.000000}%
\pgfsetstrokecolor{textcolor}%
\pgfsetfillcolor{textcolor}%
\pgftext[x=0.279012in, y=0.959639in, left, base]{\color{textcolor}\rmfamily\fontsize{10.000000}{12.000000}\selectfont \(\displaystyle {0.2}\)}%
\end{pgfscope}%
\begin{pgfscope}%
\pgfsetbuttcap%
\pgfsetroundjoin%
\definecolor{currentfill}{rgb}{0.000000,0.000000,0.000000}%
\pgfsetfillcolor{currentfill}%
\pgfsetlinewidth{0.803000pt}%
\definecolor{currentstroke}{rgb}{0.000000,0.000000,0.000000}%
\pgfsetstrokecolor{currentstroke}%
\pgfsetdash{}{0pt}%
\pgfsys@defobject{currentmarker}{\pgfqpoint{-0.048611in}{0.000000in}}{\pgfqpoint{-0.000000in}{0.000000in}}{%
\pgfpathmoveto{\pgfqpoint{-0.000000in}{0.000000in}}%
\pgfpathlineto{\pgfqpoint{-0.048611in}{0.000000in}}%
\pgfusepath{stroke,fill}%
}%
\begin{pgfscope}%
\pgfsys@transformshift{0.553704in}{1.428537in}%
\pgfsys@useobject{currentmarker}{}%
\end{pgfscope}%
\end{pgfscope}%
\begin{pgfscope}%
\definecolor{textcolor}{rgb}{0.000000,0.000000,0.000000}%
\pgfsetstrokecolor{textcolor}%
\pgfsetfillcolor{textcolor}%
\pgftext[x=0.279012in, y=1.380312in, left, base]{\color{textcolor}\rmfamily\fontsize{10.000000}{12.000000}\selectfont \(\displaystyle {0.4}\)}%
\end{pgfscope}%
\begin{pgfscope}%
\pgfsetbuttcap%
\pgfsetroundjoin%
\definecolor{currentfill}{rgb}{0.000000,0.000000,0.000000}%
\pgfsetfillcolor{currentfill}%
\pgfsetlinewidth{0.803000pt}%
\definecolor{currentstroke}{rgb}{0.000000,0.000000,0.000000}%
\pgfsetstrokecolor{currentstroke}%
\pgfsetdash{}{0pt}%
\pgfsys@defobject{currentmarker}{\pgfqpoint{-0.048611in}{0.000000in}}{\pgfqpoint{-0.000000in}{0.000000in}}{%
\pgfpathmoveto{\pgfqpoint{-0.000000in}{0.000000in}}%
\pgfpathlineto{\pgfqpoint{-0.048611in}{0.000000in}}%
\pgfusepath{stroke,fill}%
}%
\begin{pgfscope}%
\pgfsys@transformshift{0.553704in}{1.849210in}%
\pgfsys@useobject{currentmarker}{}%
\end{pgfscope}%
\end{pgfscope}%
\begin{pgfscope}%
\definecolor{textcolor}{rgb}{0.000000,0.000000,0.000000}%
\pgfsetstrokecolor{textcolor}%
\pgfsetfillcolor{textcolor}%
\pgftext[x=0.279012in, y=1.800985in, left, base]{\color{textcolor}\rmfamily\fontsize{10.000000}{12.000000}\selectfont \(\displaystyle {0.6}\)}%
\end{pgfscope}%
\begin{pgfscope}%
\pgfsetbuttcap%
\pgfsetroundjoin%
\definecolor{currentfill}{rgb}{0.000000,0.000000,0.000000}%
\pgfsetfillcolor{currentfill}%
\pgfsetlinewidth{0.803000pt}%
\definecolor{currentstroke}{rgb}{0.000000,0.000000,0.000000}%
\pgfsetstrokecolor{currentstroke}%
\pgfsetdash{}{0pt}%
\pgfsys@defobject{currentmarker}{\pgfqpoint{-0.048611in}{0.000000in}}{\pgfqpoint{-0.000000in}{0.000000in}}{%
\pgfpathmoveto{\pgfqpoint{-0.000000in}{0.000000in}}%
\pgfpathlineto{\pgfqpoint{-0.048611in}{0.000000in}}%
\pgfusepath{stroke,fill}%
}%
\begin{pgfscope}%
\pgfsys@transformshift{0.553704in}{2.269883in}%
\pgfsys@useobject{currentmarker}{}%
\end{pgfscope}%
\end{pgfscope}%
\begin{pgfscope}%
\definecolor{textcolor}{rgb}{0.000000,0.000000,0.000000}%
\pgfsetstrokecolor{textcolor}%
\pgfsetfillcolor{textcolor}%
\pgftext[x=0.279012in, y=2.221658in, left, base]{\color{textcolor}\rmfamily\fontsize{10.000000}{12.000000}\selectfont \(\displaystyle {0.8}\)}%
\end{pgfscope}%
\begin{pgfscope}%
\definecolor{textcolor}{rgb}{0.000000,0.000000,0.000000}%
\pgfsetstrokecolor{textcolor}%
\pgfsetfillcolor{textcolor}%
\pgftext[x=0.223457in,y=1.462191in,,bottom,rotate=90.000000]{\color{textcolor}\rmfamily\fontsize{10.000000}{12.000000}\selectfont Certified accuracy}%
\end{pgfscope}%
\begin{pgfscope}%
\pgfpathrectangle{\pgfqpoint{0.553704in}{0.499691in}}{\pgfqpoint{2.015000in}{1.925000in}}%
\pgfusepath{clip}%
\pgfsetrectcap%
\pgfsetroundjoin%
\pgfsetlinewidth{1.505625pt}%
\definecolor{currentstroke}{rgb}{0.716186,0.833203,0.916155}%
\pgfsetstrokecolor{currentstroke}%
\pgfsetdash{}{0pt}%
\pgfpathmoveto{\pgfqpoint{0.553704in}{2.337191in}}%
\pgfpathlineto{\pgfqpoint{1.166965in}{2.337191in}}%
\pgfpathlineto{\pgfqpoint{1.173704in}{2.332984in}}%
\pgfpathlineto{\pgfqpoint{1.321965in}{2.332984in}}%
\pgfpathlineto{\pgfqpoint{1.328704in}{2.328778in}}%
\pgfpathlineto{\pgfqpoint{1.409574in}{2.328778in}}%
\pgfpathlineto{\pgfqpoint{1.416313in}{2.324571in}}%
\pgfpathlineto{\pgfqpoint{1.423052in}{2.324571in}}%
\pgfpathlineto{\pgfqpoint{1.429791in}{2.320364in}}%
\pgfpathlineto{\pgfqpoint{1.605009in}{2.320364in}}%
\pgfpathlineto{\pgfqpoint{1.625226in}{2.307744in}}%
\pgfpathlineto{\pgfqpoint{1.679139in}{2.307744in}}%
\pgfpathlineto{\pgfqpoint{1.685878in}{2.303537in}}%
\pgfpathlineto{\pgfqpoint{1.746530in}{2.303537in}}%
\pgfpathlineto{\pgfqpoint{1.753269in}{2.299331in}}%
\pgfpathlineto{\pgfqpoint{1.766748in}{2.299331in}}%
\pgfpathlineto{\pgfqpoint{1.773487in}{2.295124in}}%
\pgfpathlineto{\pgfqpoint{1.820661in}{2.295124in}}%
\pgfpathlineto{\pgfqpoint{1.827400in}{2.290917in}}%
\pgfpathlineto{\pgfqpoint{1.861095in}{2.290917in}}%
\pgfpathlineto{\pgfqpoint{1.867835in}{2.286710in}}%
\pgfpathlineto{\pgfqpoint{1.874574in}{2.278297in}}%
\pgfpathlineto{\pgfqpoint{1.881313in}{2.274090in}}%
\pgfpathlineto{\pgfqpoint{1.928487in}{2.274090in}}%
\pgfpathlineto{\pgfqpoint{1.935226in}{2.265677in}}%
\pgfpathlineto{\pgfqpoint{1.941965in}{2.261470in}}%
\pgfpathlineto{\pgfqpoint{1.962182in}{2.261470in}}%
\pgfpathlineto{\pgfqpoint{1.968922in}{2.257263in}}%
\pgfpathlineto{\pgfqpoint{1.989139in}{2.232023in}}%
\pgfpathlineto{\pgfqpoint{1.995878in}{2.232023in}}%
\pgfpathlineto{\pgfqpoint{2.002617in}{2.223609in}}%
\pgfpathlineto{\pgfqpoint{2.009356in}{2.219403in}}%
\pgfpathlineto{\pgfqpoint{2.016095in}{2.219403in}}%
\pgfpathlineto{\pgfqpoint{2.022835in}{2.202576in}}%
\pgfpathlineto{\pgfqpoint{2.029574in}{2.202576in}}%
\pgfpathlineto{\pgfqpoint{2.036313in}{2.189956in}}%
\pgfpathlineto{\pgfqpoint{2.056530in}{2.177335in}}%
\pgfpathlineto{\pgfqpoint{2.063269in}{2.152095in}}%
\pgfpathlineto{\pgfqpoint{2.070009in}{2.143681in}}%
\pgfpathlineto{\pgfqpoint{2.076748in}{2.139475in}}%
\pgfpathlineto{\pgfqpoint{2.083487in}{2.126855in}}%
\pgfpathlineto{\pgfqpoint{2.096965in}{2.072167in}}%
\pgfpathlineto{\pgfqpoint{2.103704in}{2.042720in}}%
\pgfpathlineto{\pgfqpoint{2.110443in}{2.034306in}}%
\pgfpathlineto{\pgfqpoint{2.117182in}{1.996446in}}%
\pgfpathlineto{\pgfqpoint{2.123922in}{1.941758in}}%
\pgfpathlineto{\pgfqpoint{2.130661in}{0.587191in}}%
\pgfpathlineto{\pgfqpoint{2.568704in}{0.587191in}}%
\pgfpathlineto{\pgfqpoint{2.568704in}{0.587191in}}%
\pgfusepath{stroke}%
\end{pgfscope}%
\begin{pgfscope}%
\pgfpathrectangle{\pgfqpoint{0.553704in}{0.499691in}}{\pgfqpoint{2.015000in}{1.925000in}}%
\pgfusepath{clip}%
\pgfsetrectcap%
\pgfsetroundjoin%
\pgfsetlinewidth{1.505625pt}%
\definecolor{currentstroke}{rgb}{0.376732,0.653072,0.822484}%
\pgfsetstrokecolor{currentstroke}%
\pgfsetdash{}{0pt}%
\pgfpathmoveto{\pgfqpoint{0.553704in}{2.210989in}}%
\pgfpathlineto{\pgfqpoint{1.166965in}{2.210989in}}%
\pgfpathlineto{\pgfqpoint{1.173704in}{2.206782in}}%
\pgfpathlineto{\pgfqpoint{1.254574in}{2.206782in}}%
\pgfpathlineto{\pgfqpoint{1.261313in}{2.202576in}}%
\pgfpathlineto{\pgfqpoint{1.510661in}{2.202576in}}%
\pgfpathlineto{\pgfqpoint{1.517400in}{2.198369in}}%
\pgfpathlineto{\pgfqpoint{1.625226in}{2.198369in}}%
\pgfpathlineto{\pgfqpoint{1.631965in}{2.194162in}}%
\pgfpathlineto{\pgfqpoint{1.638704in}{2.194162in}}%
\pgfpathlineto{\pgfqpoint{1.645443in}{2.189956in}}%
\pgfpathlineto{\pgfqpoint{1.685878in}{2.189956in}}%
\pgfpathlineto{\pgfqpoint{1.692617in}{2.185749in}}%
\pgfpathlineto{\pgfqpoint{1.780226in}{2.185749in}}%
\pgfpathlineto{\pgfqpoint{1.786965in}{2.181542in}}%
\pgfpathlineto{\pgfqpoint{1.834139in}{2.181542in}}%
\pgfpathlineto{\pgfqpoint{1.840878in}{2.173129in}}%
\pgfpathlineto{\pgfqpoint{1.867835in}{2.173129in}}%
\pgfpathlineto{\pgfqpoint{1.874574in}{2.168922in}}%
\pgfpathlineto{\pgfqpoint{1.908269in}{2.168922in}}%
\pgfpathlineto{\pgfqpoint{1.915009in}{2.160508in}}%
\pgfpathlineto{\pgfqpoint{1.928487in}{2.160508in}}%
\pgfpathlineto{\pgfqpoint{1.948704in}{2.147888in}}%
\pgfpathlineto{\pgfqpoint{1.995878in}{2.147888in}}%
\pgfpathlineto{\pgfqpoint{2.002617in}{2.143681in}}%
\pgfpathlineto{\pgfqpoint{2.029574in}{2.143681in}}%
\pgfpathlineto{\pgfqpoint{2.036313in}{2.131061in}}%
\pgfpathlineto{\pgfqpoint{2.043052in}{2.131061in}}%
\pgfpathlineto{\pgfqpoint{2.049791in}{2.122648in}}%
\pgfpathlineto{\pgfqpoint{2.056530in}{2.122648in}}%
\pgfpathlineto{\pgfqpoint{2.070009in}{2.114234in}}%
\pgfpathlineto{\pgfqpoint{2.083487in}{2.114234in}}%
\pgfpathlineto{\pgfqpoint{2.090226in}{2.110028in}}%
\pgfpathlineto{\pgfqpoint{2.096965in}{2.110028in}}%
\pgfpathlineto{\pgfqpoint{2.103704in}{2.097407in}}%
\pgfpathlineto{\pgfqpoint{2.123922in}{2.084787in}}%
\pgfpathlineto{\pgfqpoint{2.130661in}{2.076374in}}%
\pgfpathlineto{\pgfqpoint{2.137400in}{2.076374in}}%
\pgfpathlineto{\pgfqpoint{2.144139in}{2.067960in}}%
\pgfpathlineto{\pgfqpoint{2.164356in}{2.067960in}}%
\pgfpathlineto{\pgfqpoint{2.171095in}{2.046927in}}%
\pgfpathlineto{\pgfqpoint{2.191313in}{2.034306in}}%
\pgfpathlineto{\pgfqpoint{2.211530in}{2.009066in}}%
\pgfpathlineto{\pgfqpoint{2.218269in}{2.009066in}}%
\pgfpathlineto{\pgfqpoint{2.225009in}{2.004859in}}%
\pgfpathlineto{\pgfqpoint{2.238487in}{1.988032in}}%
\pgfpathlineto{\pgfqpoint{2.245226in}{1.971206in}}%
\pgfpathlineto{\pgfqpoint{2.258704in}{1.962792in}}%
\pgfpathlineto{\pgfqpoint{2.272182in}{1.937552in}}%
\pgfpathlineto{\pgfqpoint{2.278922in}{1.912311in}}%
\pgfpathlineto{\pgfqpoint{2.285661in}{1.908105in}}%
\pgfpathlineto{\pgfqpoint{2.292400in}{1.887071in}}%
\pgfpathlineto{\pgfqpoint{2.299139in}{1.874451in}}%
\pgfpathlineto{\pgfqpoint{2.305878in}{1.853417in}}%
\pgfpathlineto{\pgfqpoint{2.319356in}{1.798730in}}%
\pgfpathlineto{\pgfqpoint{2.332835in}{1.714595in}}%
\pgfpathlineto{\pgfqpoint{2.339574in}{1.655701in}}%
\pgfpathlineto{\pgfqpoint{2.346313in}{1.558946in}}%
\pgfpathlineto{\pgfqpoint{2.353052in}{0.587191in}}%
\pgfpathlineto{\pgfqpoint{2.568704in}{0.587191in}}%
\pgfpathlineto{\pgfqpoint{2.568704in}{0.587191in}}%
\pgfusepath{stroke}%
\end{pgfscope}%
\begin{pgfscope}%
\pgfpathrectangle{\pgfqpoint{0.553704in}{0.499691in}}{\pgfqpoint{2.015000in}{1.925000in}}%
\pgfusepath{clip}%
\pgfsetrectcap%
\pgfsetroundjoin%
\pgfsetlinewidth{1.505625pt}%
\definecolor{currentstroke}{rgb}{0.114802,0.424437,0.695194}%
\pgfsetstrokecolor{currentstroke}%
\pgfsetdash{}{0pt}%
\pgfpathmoveto{\pgfqpoint{0.553704in}{2.055340in}}%
\pgfpathlineto{\pgfqpoint{0.924356in}{2.055340in}}%
\pgfpathlineto{\pgfqpoint{0.931095in}{2.051133in}}%
\pgfpathlineto{\pgfqpoint{1.261313in}{2.051133in}}%
\pgfpathlineto{\pgfqpoint{1.268052in}{2.046927in}}%
\pgfpathlineto{\pgfqpoint{1.402835in}{2.046927in}}%
\pgfpathlineto{\pgfqpoint{1.409574in}{2.042720in}}%
\pgfpathlineto{\pgfqpoint{1.470226in}{2.042720in}}%
\pgfpathlineto{\pgfqpoint{1.476965in}{2.038513in}}%
\pgfpathlineto{\pgfqpoint{1.598269in}{2.038513in}}%
\pgfpathlineto{\pgfqpoint{1.605009in}{2.034306in}}%
\pgfpathlineto{\pgfqpoint{1.618487in}{2.034306in}}%
\pgfpathlineto{\pgfqpoint{1.625226in}{2.030100in}}%
\pgfpathlineto{\pgfqpoint{1.793704in}{2.030100in}}%
\pgfpathlineto{\pgfqpoint{1.800443in}{2.025893in}}%
\pgfpathlineto{\pgfqpoint{1.820661in}{2.025893in}}%
\pgfpathlineto{\pgfqpoint{1.834139in}{2.017480in}}%
\pgfpathlineto{\pgfqpoint{1.861095in}{2.017480in}}%
\pgfpathlineto{\pgfqpoint{1.867835in}{2.013273in}}%
\pgfpathlineto{\pgfqpoint{1.881313in}{2.013273in}}%
\pgfpathlineto{\pgfqpoint{1.888052in}{2.004859in}}%
\pgfpathlineto{\pgfqpoint{1.921748in}{2.004859in}}%
\pgfpathlineto{\pgfqpoint{1.935226in}{1.996446in}}%
\pgfpathlineto{\pgfqpoint{1.962182in}{1.996446in}}%
\pgfpathlineto{\pgfqpoint{1.968922in}{1.992239in}}%
\pgfpathlineto{\pgfqpoint{1.995878in}{1.992239in}}%
\pgfpathlineto{\pgfqpoint{2.002617in}{1.988032in}}%
\pgfpathlineto{\pgfqpoint{2.009356in}{1.988032in}}%
\pgfpathlineto{\pgfqpoint{2.036313in}{1.971206in}}%
\pgfpathlineto{\pgfqpoint{2.049791in}{1.971206in}}%
\pgfpathlineto{\pgfqpoint{2.056530in}{1.962792in}}%
\pgfpathlineto{\pgfqpoint{2.083487in}{1.962792in}}%
\pgfpathlineto{\pgfqpoint{2.090226in}{1.958585in}}%
\pgfpathlineto{\pgfqpoint{2.096965in}{1.958585in}}%
\pgfpathlineto{\pgfqpoint{2.110443in}{1.950172in}}%
\pgfpathlineto{\pgfqpoint{2.117182in}{1.950172in}}%
\pgfpathlineto{\pgfqpoint{2.123922in}{1.945965in}}%
\pgfpathlineto{\pgfqpoint{2.144139in}{1.945965in}}%
\pgfpathlineto{\pgfqpoint{2.150878in}{1.937552in}}%
\pgfpathlineto{\pgfqpoint{2.184574in}{1.937552in}}%
\pgfpathlineto{\pgfqpoint{2.191313in}{1.933345in}}%
\pgfpathlineto{\pgfqpoint{2.204791in}{1.933345in}}%
\pgfpathlineto{\pgfqpoint{2.218269in}{1.924931in}}%
\pgfpathlineto{\pgfqpoint{2.231748in}{1.924931in}}%
\pgfpathlineto{\pgfqpoint{2.245226in}{1.908105in}}%
\pgfpathlineto{\pgfqpoint{2.251965in}{1.903898in}}%
\pgfpathlineto{\pgfqpoint{2.258704in}{1.882864in}}%
\pgfpathlineto{\pgfqpoint{2.265443in}{1.878657in}}%
\pgfpathlineto{\pgfqpoint{2.272182in}{1.878657in}}%
\pgfpathlineto{\pgfqpoint{2.278922in}{1.861831in}}%
\pgfpathlineto{\pgfqpoint{2.292400in}{1.845004in}}%
\pgfpathlineto{\pgfqpoint{2.299139in}{1.845004in}}%
\pgfpathlineto{\pgfqpoint{2.305878in}{1.836590in}}%
\pgfpathlineto{\pgfqpoint{2.312617in}{1.836590in}}%
\pgfpathlineto{\pgfqpoint{2.332835in}{1.798730in}}%
\pgfpathlineto{\pgfqpoint{2.353052in}{1.786109in}}%
\pgfpathlineto{\pgfqpoint{2.366530in}{1.786109in}}%
\pgfpathlineto{\pgfqpoint{2.373269in}{1.773489in}}%
\pgfpathlineto{\pgfqpoint{2.386748in}{1.773489in}}%
\pgfpathlineto{\pgfqpoint{2.393487in}{1.769282in}}%
\pgfpathlineto{\pgfqpoint{2.400226in}{1.739835in}}%
\pgfpathlineto{\pgfqpoint{2.406965in}{1.735629in}}%
\pgfpathlineto{\pgfqpoint{2.413704in}{1.727215in}}%
\pgfpathlineto{\pgfqpoint{2.420443in}{1.710388in}}%
\pgfpathlineto{\pgfqpoint{2.427182in}{1.697768in}}%
\pgfpathlineto{\pgfqpoint{2.433922in}{1.668321in}}%
\pgfpathlineto{\pgfqpoint{2.440661in}{1.651494in}}%
\pgfpathlineto{\pgfqpoint{2.454139in}{1.601013in}}%
\pgfpathlineto{\pgfqpoint{2.460878in}{1.567359in}}%
\pgfpathlineto{\pgfqpoint{2.467617in}{1.483225in}}%
\pgfpathlineto{\pgfqpoint{2.474356in}{1.352816in}}%
\pgfpathlineto{\pgfqpoint{2.481095in}{0.587191in}}%
\pgfpathlineto{\pgfqpoint{2.568704in}{0.587191in}}%
\pgfpathlineto{\pgfqpoint{2.568704in}{0.587191in}}%
\pgfusepath{stroke}%
\end{pgfscope}%
\begin{pgfscope}%
\pgfpathrectangle{\pgfqpoint{0.553704in}{0.499691in}}{\pgfqpoint{2.015000in}{1.925000in}}%
\pgfusepath{clip}%
\pgfsetrectcap%
\pgfsetroundjoin%
\pgfsetlinewidth{1.505625pt}%
\definecolor{currentstroke}{rgb}{0.031373,0.188235,0.419608}%
\pgfsetstrokecolor{currentstroke}%
\pgfsetdash{}{0pt}%
\pgfpathmoveto{\pgfqpoint{0.553704in}{1.966999in}}%
\pgfpathlineto{\pgfqpoint{1.321965in}{1.966999in}}%
\pgfpathlineto{\pgfqpoint{1.328704in}{1.962792in}}%
\pgfpathlineto{\pgfqpoint{1.355661in}{1.962792in}}%
\pgfpathlineto{\pgfqpoint{1.362400in}{1.958585in}}%
\pgfpathlineto{\pgfqpoint{1.564574in}{1.958585in}}%
\pgfpathlineto{\pgfqpoint{1.571313in}{1.954379in}}%
\pgfpathlineto{\pgfqpoint{1.578052in}{1.954379in}}%
\pgfpathlineto{\pgfqpoint{1.584791in}{1.950172in}}%
\pgfpathlineto{\pgfqpoint{1.605009in}{1.950172in}}%
\pgfpathlineto{\pgfqpoint{1.611748in}{1.945965in}}%
\pgfpathlineto{\pgfqpoint{1.746530in}{1.945965in}}%
\pgfpathlineto{\pgfqpoint{1.753269in}{1.941758in}}%
\pgfpathlineto{\pgfqpoint{1.786965in}{1.941758in}}%
\pgfpathlineto{\pgfqpoint{1.793704in}{1.937552in}}%
\pgfpathlineto{\pgfqpoint{1.800443in}{1.937552in}}%
\pgfpathlineto{\pgfqpoint{1.807182in}{1.933345in}}%
\pgfpathlineto{\pgfqpoint{1.921748in}{1.933345in}}%
\pgfpathlineto{\pgfqpoint{1.928487in}{1.929138in}}%
\pgfpathlineto{\pgfqpoint{1.935226in}{1.929138in}}%
\pgfpathlineto{\pgfqpoint{1.941965in}{1.924931in}}%
\pgfpathlineto{\pgfqpoint{1.962182in}{1.924931in}}%
\pgfpathlineto{\pgfqpoint{1.968922in}{1.916518in}}%
\pgfpathlineto{\pgfqpoint{1.975661in}{1.912311in}}%
\pgfpathlineto{\pgfqpoint{1.982400in}{1.912311in}}%
\pgfpathlineto{\pgfqpoint{1.989139in}{1.903898in}}%
\pgfpathlineto{\pgfqpoint{2.029574in}{1.903898in}}%
\pgfpathlineto{\pgfqpoint{2.036313in}{1.899691in}}%
\pgfpathlineto{\pgfqpoint{2.083487in}{1.899691in}}%
\pgfpathlineto{\pgfqpoint{2.090226in}{1.895484in}}%
\pgfpathlineto{\pgfqpoint{2.096965in}{1.895484in}}%
\pgfpathlineto{\pgfqpoint{2.110443in}{1.887071in}}%
\pgfpathlineto{\pgfqpoint{2.117182in}{1.887071in}}%
\pgfpathlineto{\pgfqpoint{2.123922in}{1.882864in}}%
\pgfpathlineto{\pgfqpoint{2.144139in}{1.882864in}}%
\pgfpathlineto{\pgfqpoint{2.150878in}{1.878657in}}%
\pgfpathlineto{\pgfqpoint{2.164356in}{1.878657in}}%
\pgfpathlineto{\pgfqpoint{2.171095in}{1.866037in}}%
\pgfpathlineto{\pgfqpoint{2.177835in}{1.866037in}}%
\pgfpathlineto{\pgfqpoint{2.184574in}{1.853417in}}%
\pgfpathlineto{\pgfqpoint{2.204791in}{1.853417in}}%
\pgfpathlineto{\pgfqpoint{2.211530in}{1.845004in}}%
\pgfpathlineto{\pgfqpoint{2.218269in}{1.840797in}}%
\pgfpathlineto{\pgfqpoint{2.225009in}{1.832383in}}%
\pgfpathlineto{\pgfqpoint{2.272182in}{1.832383in}}%
\pgfpathlineto{\pgfqpoint{2.292400in}{1.807143in}}%
\pgfpathlineto{\pgfqpoint{2.299139in}{1.802936in}}%
\pgfpathlineto{\pgfqpoint{2.305878in}{1.794523in}}%
\pgfpathlineto{\pgfqpoint{2.319356in}{1.794523in}}%
\pgfpathlineto{\pgfqpoint{2.326095in}{1.786109in}}%
\pgfpathlineto{\pgfqpoint{2.339574in}{1.786109in}}%
\pgfpathlineto{\pgfqpoint{2.346313in}{1.781903in}}%
\pgfpathlineto{\pgfqpoint{2.359791in}{1.781903in}}%
\pgfpathlineto{\pgfqpoint{2.366530in}{1.777696in}}%
\pgfpathlineto{\pgfqpoint{2.373269in}{1.769282in}}%
\pgfpathlineto{\pgfqpoint{2.380009in}{1.765076in}}%
\pgfpathlineto{\pgfqpoint{2.386748in}{1.765076in}}%
\pgfpathlineto{\pgfqpoint{2.393487in}{1.760869in}}%
\pgfpathlineto{\pgfqpoint{2.400226in}{1.760869in}}%
\pgfpathlineto{\pgfqpoint{2.420443in}{1.748249in}}%
\pgfpathlineto{\pgfqpoint{2.427182in}{1.735629in}}%
\pgfpathlineto{\pgfqpoint{2.433922in}{1.735629in}}%
\pgfpathlineto{\pgfqpoint{2.440661in}{1.727215in}}%
\pgfpathlineto{\pgfqpoint{2.447400in}{1.714595in}}%
\pgfpathlineto{\pgfqpoint{2.454139in}{1.706181in}}%
\pgfpathlineto{\pgfqpoint{2.460878in}{1.689355in}}%
\pgfpathlineto{\pgfqpoint{2.467617in}{1.685148in}}%
\pgfpathlineto{\pgfqpoint{2.474356in}{1.672528in}}%
\pgfpathlineto{\pgfqpoint{2.481095in}{1.651494in}}%
\pgfpathlineto{\pgfqpoint{2.487835in}{1.643081in}}%
\pgfpathlineto{\pgfqpoint{2.494574in}{1.638874in}}%
\pgfpathlineto{\pgfqpoint{2.501313in}{1.630460in}}%
\pgfpathlineto{\pgfqpoint{2.508052in}{1.626254in}}%
\pgfpathlineto{\pgfqpoint{2.514791in}{1.609427in}}%
\pgfpathlineto{\pgfqpoint{2.521530in}{1.596806in}}%
\pgfpathlineto{\pgfqpoint{2.528269in}{1.575773in}}%
\pgfpathlineto{\pgfqpoint{2.535009in}{1.537912in}}%
\pgfpathlineto{\pgfqpoint{2.541748in}{1.516879in}}%
\pgfpathlineto{\pgfqpoint{2.555226in}{1.420124in}}%
\pgfpathlineto{\pgfqpoint{2.561965in}{1.335989in}}%
\pgfpathlineto{\pgfqpoint{2.568704in}{0.587191in}}%
\pgfpathlineto{\pgfqpoint{2.568704in}{0.587191in}}%
\pgfusepath{stroke}%
\end{pgfscope}%
\begin{pgfscope}%
\pgfsetrectcap%
\pgfsetmiterjoin%
\pgfsetlinewidth{0.803000pt}%
\definecolor{currentstroke}{rgb}{0.000000,0.000000,0.000000}%
\pgfsetstrokecolor{currentstroke}%
\pgfsetdash{}{0pt}%
\pgfpathmoveto{\pgfqpoint{0.553704in}{0.499691in}}%
\pgfpathlineto{\pgfqpoint{0.553704in}{2.424691in}}%
\pgfusepath{stroke}%
\end{pgfscope}%
\begin{pgfscope}%
\pgfsetrectcap%
\pgfsetmiterjoin%
\pgfsetlinewidth{0.803000pt}%
\definecolor{currentstroke}{rgb}{0.000000,0.000000,0.000000}%
\pgfsetstrokecolor{currentstroke}%
\pgfsetdash{}{0pt}%
\pgfpathmoveto{\pgfqpoint{2.568704in}{0.499691in}}%
\pgfpathlineto{\pgfqpoint{2.568704in}{2.424691in}}%
\pgfusepath{stroke}%
\end{pgfscope}%
\begin{pgfscope}%
\pgfsetrectcap%
\pgfsetmiterjoin%
\pgfsetlinewidth{0.803000pt}%
\definecolor{currentstroke}{rgb}{0.000000,0.000000,0.000000}%
\pgfsetstrokecolor{currentstroke}%
\pgfsetdash{}{0pt}%
\pgfpathmoveto{\pgfqpoint{0.553704in}{0.499691in}}%
\pgfpathlineto{\pgfqpoint{2.568704in}{0.499691in}}%
\pgfusepath{stroke}%
\end{pgfscope}%
\begin{pgfscope}%
\pgfsetrectcap%
\pgfsetmiterjoin%
\pgfsetlinewidth{0.803000pt}%
\definecolor{currentstroke}{rgb}{0.000000,0.000000,0.000000}%
\pgfsetstrokecolor{currentstroke}%
\pgfsetdash{}{0pt}%
\pgfpathmoveto{\pgfqpoint{0.553704in}{2.424691in}}%
\pgfpathlineto{\pgfqpoint{2.568704in}{2.424691in}}%
\pgfusepath{stroke}%
\end{pgfscope}%
\begin{pgfscope}%
\pgfsetbuttcap%
\pgfsetmiterjoin%
\definecolor{currentfill}{rgb}{1.000000,1.000000,1.000000}%
\pgfsetfillcolor{currentfill}%
\pgfsetfillopacity{0.800000}%
\pgfsetlinewidth{1.003750pt}%
\definecolor{currentstroke}{rgb}{0.800000,0.800000,0.800000}%
\pgfsetstrokecolor{currentstroke}%
\pgfsetstrokeopacity{0.800000}%
\pgfsetdash{}{0pt}%
\pgfpathmoveto{\pgfqpoint{0.650926in}{0.569136in}}%
\pgfpathlineto{\pgfqpoint{1.611815in}{0.569136in}}%
\pgfpathquadraticcurveto{\pgfqpoint{1.639593in}{0.569136in}}{\pgfqpoint{1.639593in}{0.596913in}}%
\pgfpathlineto{\pgfqpoint{1.639593in}{1.550617in}}%
\pgfpathquadraticcurveto{\pgfqpoint{1.639593in}{1.578394in}}{\pgfqpoint{1.611815in}{1.578394in}}%
\pgfpathlineto{\pgfqpoint{0.650926in}{1.578394in}}%
\pgfpathquadraticcurveto{\pgfqpoint{0.623149in}{1.578394in}}{\pgfqpoint{0.623149in}{1.550617in}}%
\pgfpathlineto{\pgfqpoint{0.623149in}{0.596913in}}%
\pgfpathquadraticcurveto{\pgfqpoint{0.623149in}{0.569136in}}{\pgfqpoint{0.650926in}{0.569136in}}%
\pgfpathlineto{\pgfqpoint{0.650926in}{0.569136in}}%
\pgfpathclose%
\pgfusepath{stroke,fill}%
\end{pgfscope}%
\begin{pgfscope}%
\definecolor{textcolor}{rgb}{0.000000,0.000000,0.000000}%
\pgfsetstrokecolor{textcolor}%
\pgfsetfillcolor{textcolor}%
\pgftext[x=0.793506in,y=1.426388in,left,base]{\color{textcolor}\rmfamily\fontsize{10.000000}{12.000000}\selectfont \(\displaystyle \textsc{RS4A}-\ell_{1}\)}%
\end{pgfscope}%
\begin{pgfscope}%
\pgfsetrectcap%
\pgfsetroundjoin%
\pgfsetlinewidth{1.505625pt}%
\definecolor{currentstroke}{rgb}{0.716186,0.833203,0.916155}%
\pgfsetstrokecolor{currentstroke}%
\pgfsetdash{}{0pt}%
\pgfpathmoveto{\pgfqpoint{0.678704in}{1.281327in}}%
\pgfpathlineto{\pgfqpoint{0.817593in}{1.281327in}}%
\pgfpathlineto{\pgfqpoint{0.956482in}{1.281327in}}%
\pgfusepath{stroke}%
\end{pgfscope}%
\begin{pgfscope}%
\definecolor{textcolor}{rgb}{0.000000,0.000000,0.000000}%
\pgfsetstrokecolor{textcolor}%
\pgfsetfillcolor{textcolor}%
\pgftext[x=1.067593in,y=1.232716in,left,base]{\color{textcolor}\rmfamily\fontsize{10.000000}{12.000000}\selectfont \(\displaystyle \sigma=0.25\)}%
\end{pgfscope}%
\begin{pgfscope}%
\pgfsetrectcap%
\pgfsetroundjoin%
\pgfsetlinewidth{1.505625pt}%
\definecolor{currentstroke}{rgb}{0.376732,0.653072,0.822484}%
\pgfsetstrokecolor{currentstroke}%
\pgfsetdash{}{0pt}%
\pgfpathmoveto{\pgfqpoint{0.678704in}{1.087654in}}%
\pgfpathlineto{\pgfqpoint{0.817593in}{1.087654in}}%
\pgfpathlineto{\pgfqpoint{0.956482in}{1.087654in}}%
\pgfusepath{stroke}%
\end{pgfscope}%
\begin{pgfscope}%
\definecolor{textcolor}{rgb}{0.000000,0.000000,0.000000}%
\pgfsetstrokecolor{textcolor}%
\pgfsetfillcolor{textcolor}%
\pgftext[x=1.067593in,y=1.039043in,left,base]{\color{textcolor}\rmfamily\fontsize{10.000000}{12.000000}\selectfont \(\displaystyle \sigma=0.50\)}%
\end{pgfscope}%
\begin{pgfscope}%
\pgfsetrectcap%
\pgfsetroundjoin%
\pgfsetlinewidth{1.505625pt}%
\definecolor{currentstroke}{rgb}{0.114802,0.424437,0.695194}%
\pgfsetstrokecolor{currentstroke}%
\pgfsetdash{}{0pt}%
\pgfpathmoveto{\pgfqpoint{0.678704in}{0.893981in}}%
\pgfpathlineto{\pgfqpoint{0.817593in}{0.893981in}}%
\pgfpathlineto{\pgfqpoint{0.956482in}{0.893981in}}%
\pgfusepath{stroke}%
\end{pgfscope}%
\begin{pgfscope}%
\definecolor{textcolor}{rgb}{0.000000,0.000000,0.000000}%
\pgfsetstrokecolor{textcolor}%
\pgfsetfillcolor{textcolor}%
\pgftext[x=1.067593in,y=0.845370in,left,base]{\color{textcolor}\rmfamily\fontsize{10.000000}{12.000000}\selectfont \(\displaystyle \sigma=0.75\)}%
\end{pgfscope}%
\begin{pgfscope}%
\pgfsetrectcap%
\pgfsetroundjoin%
\pgfsetlinewidth{1.505625pt}%
\definecolor{currentstroke}{rgb}{0.031373,0.188235,0.419608}%
\pgfsetstrokecolor{currentstroke}%
\pgfsetdash{}{0pt}%
\pgfpathmoveto{\pgfqpoint{0.678704in}{0.700308in}}%
\pgfpathlineto{\pgfqpoint{0.817593in}{0.700308in}}%
\pgfpathlineto{\pgfqpoint{0.956482in}{0.700308in}}%
\pgfusepath{stroke}%
\end{pgfscope}%
\begin{pgfscope}%
\definecolor{textcolor}{rgb}{0.000000,0.000000,0.000000}%
\pgfsetstrokecolor{textcolor}%
\pgfsetfillcolor{textcolor}%
\pgftext[x=1.067593in,y=0.651697in,left,base]{\color{textcolor}\rmfamily\fontsize{10.000000}{12.000000}\selectfont \(\displaystyle \sigma=1.00\)}%
\end{pgfscope}%
\end{pgfpicture}%
\makeatother%
\endgroup%

%% file: figs/cifar10_rs4alinfsweep.pgf
%% Creator: Matplotlib, PGF backend
%%
%% To include the figure in your LaTeX document, write
%%   \input{<filename>.pgf}
%%
%% Make sure the required packages are loaded in your preamble
%%   \usepackage{pgf}
%%
%% Also ensure that all the required font packages are loaded; for instance,
%% the lmodern package is sometimes necessary when using math font.
%%   \usepackage{lmodern}
%%
%% Figures using additional raster images can only be included by \input if
%% they are in the same directory as the main LaTeX file. For loading figures
%% from other directories you can use the `import` package
%%   \usepackage{import}
%%
%% and then include the figures with
%%   \import{<path to file>}{<filename>.pgf}
%%
%% Matplotlib used the following preamble
%%
\begingroup%
\makeatletter%
\begin{pgfpicture}%
\pgfpathrectangle{\pgfpointorigin}{\pgfqpoint{2.668704in}{2.524691in}}%
\pgfusepath{use as bounding box, clip}%
\begin{pgfscope}%
\pgfsetbuttcap%
\pgfsetmiterjoin%
\pgfsetlinewidth{0.000000pt}%
\definecolor{currentstroke}{rgb}{0.000000,0.000000,0.000000}%
\pgfsetstrokecolor{currentstroke}%
\pgfsetstrokeopacity{0.000000}%
\pgfsetdash{}{0pt}%
\pgfpathmoveto{\pgfqpoint{0.000000in}{0.000000in}}%
\pgfpathlineto{\pgfqpoint{2.668704in}{0.000000in}}%
\pgfpathlineto{\pgfqpoint{2.668704in}{2.524691in}}%
\pgfpathlineto{\pgfqpoint{0.000000in}{2.524691in}}%
\pgfpathlineto{\pgfqpoint{0.000000in}{0.000000in}}%
\pgfpathclose%
\pgfusepath{}%
\end{pgfscope}%
\begin{pgfscope}%
\pgfsetbuttcap%
\pgfsetmiterjoin%
\pgfsetlinewidth{0.000000pt}%
\definecolor{currentstroke}{rgb}{0.000000,0.000000,0.000000}%
\pgfsetstrokecolor{currentstroke}%
\pgfsetstrokeopacity{0.000000}%
\pgfsetdash{}{0pt}%
\pgfpathmoveto{\pgfqpoint{0.553704in}{0.499691in}}%
\pgfpathlineto{\pgfqpoint{2.568704in}{0.499691in}}%
\pgfpathlineto{\pgfqpoint{2.568704in}{2.424691in}}%
\pgfpathlineto{\pgfqpoint{0.553704in}{2.424691in}}%
\pgfpathlineto{\pgfqpoint{0.553704in}{0.499691in}}%
\pgfpathclose%
\pgfusepath{}%
\end{pgfscope}%
\begin{pgfscope}%
\pgfsetbuttcap%
\pgfsetroundjoin%
\definecolor{currentfill}{rgb}{0.000000,0.000000,0.000000}%
\pgfsetfillcolor{currentfill}%
\pgfsetlinewidth{0.803000pt}%
\definecolor{currentstroke}{rgb}{0.000000,0.000000,0.000000}%
\pgfsetstrokecolor{currentstroke}%
\pgfsetdash{}{0pt}%
\pgfsys@defobject{currentmarker}{\pgfqpoint{0.000000in}{-0.048611in}}{\pgfqpoint{0.000000in}{0.000000in}}{%
\pgfpathmoveto{\pgfqpoint{0.000000in}{0.000000in}}%
\pgfpathlineto{\pgfqpoint{0.000000in}{-0.048611in}}%
\pgfusepath{stroke,fill}%
}%
\begin{pgfscope}%
\pgfsys@transformshift{0.702229in}{0.499691in}%
\pgfsys@useobject{currentmarker}{}%
\end{pgfscope}%
\end{pgfscope}%
\begin{pgfscope}%
\definecolor{textcolor}{rgb}{0.000000,0.000000,0.000000}%
\pgfsetstrokecolor{textcolor}%
\pgfsetfillcolor{textcolor}%
\pgftext[x=0.702229in,y=0.402469in,,top]{\color{textcolor}\rmfamily\fontsize{10.000000}{12.000000}\selectfont \(\displaystyle 10^{-4 \alpha}\)}%
\end{pgfscope}%
\begin{pgfscope}%
\pgfsetbuttcap%
\pgfsetroundjoin%
\definecolor{currentfill}{rgb}{0.000000,0.000000,0.000000}%
\pgfsetfillcolor{currentfill}%
\pgfsetlinewidth{0.803000pt}%
\definecolor{currentstroke}{rgb}{0.000000,0.000000,0.000000}%
\pgfsetstrokecolor{currentstroke}%
\pgfsetdash{}{0pt}%
\pgfsys@defobject{currentmarker}{\pgfqpoint{0.000000in}{-0.048611in}}{\pgfqpoint{0.000000in}{0.000000in}}{%
\pgfpathmoveto{\pgfqpoint{0.000000in}{0.000000in}}%
\pgfpathlineto{\pgfqpoint{0.000000in}{-0.048611in}}%
\pgfusepath{stroke,fill}%
}%
\begin{pgfscope}%
\pgfsys@transformshift{1.438648in}{0.499691in}%
\pgfsys@useobject{currentmarker}{}%
\end{pgfscope}%
\end{pgfscope}%
\begin{pgfscope}%
\definecolor{textcolor}{rgb}{0.000000,0.000000,0.000000}%
\pgfsetstrokecolor{textcolor}%
\pgfsetfillcolor{textcolor}%
\pgftext[x=1.438648in,y=0.402469in,,top]{\color{textcolor}\rmfamily\fontsize{10.000000}{12.000000}\selectfont \(\displaystyle 10^{-3 \alpha}\)}%
\end{pgfscope}%
\begin{pgfscope}%
\pgfsetbuttcap%
\pgfsetroundjoin%
\definecolor{currentfill}{rgb}{0.000000,0.000000,0.000000}%
\pgfsetfillcolor{currentfill}%
\pgfsetlinewidth{0.803000pt}%
\definecolor{currentstroke}{rgb}{0.000000,0.000000,0.000000}%
\pgfsetstrokecolor{currentstroke}%
\pgfsetdash{}{0pt}%
\pgfsys@defobject{currentmarker}{\pgfqpoint{0.000000in}{-0.048611in}}{\pgfqpoint{0.000000in}{0.000000in}}{%
\pgfpathmoveto{\pgfqpoint{0.000000in}{0.000000in}}%
\pgfpathlineto{\pgfqpoint{0.000000in}{-0.048611in}}%
\pgfusepath{stroke,fill}%
}%
\begin{pgfscope}%
\pgfsys@transformshift{2.175067in}{0.499691in}%
\pgfsys@useobject{currentmarker}{}%
\end{pgfscope}%
\end{pgfscope}%
\begin{pgfscope}%
\definecolor{textcolor}{rgb}{0.000000,0.000000,0.000000}%
\pgfsetstrokecolor{textcolor}%
\pgfsetfillcolor{textcolor}%
\pgftext[x=2.175067in,y=0.402469in,,top]{\color{textcolor}\rmfamily\fontsize{10.000000}{12.000000}\selectfont \(\displaystyle 10^{-2 \alpha}\)}%
\end{pgfscope}%
\begin{pgfscope}%
\definecolor{textcolor}{rgb}{0.000000,0.000000,0.000000}%
\pgfsetstrokecolor{textcolor}%
\pgfsetfillcolor{textcolor}%
\pgftext[x=1.561204in,y=0.223457in,,top]{\color{textcolor}\rmfamily\fontsize{10.000000}{12.000000}\selectfont Volume}%
\end{pgfscope}%
\begin{pgfscope}%
\pgfsetbuttcap%
\pgfsetroundjoin%
\definecolor{currentfill}{rgb}{0.000000,0.000000,0.000000}%
\pgfsetfillcolor{currentfill}%
\pgfsetlinewidth{0.803000pt}%
\definecolor{currentstroke}{rgb}{0.000000,0.000000,0.000000}%
\pgfsetstrokecolor{currentstroke}%
\pgfsetdash{}{0pt}%
\pgfsys@defobject{currentmarker}{\pgfqpoint{-0.048611in}{0.000000in}}{\pgfqpoint{-0.000000in}{0.000000in}}{%
\pgfpathmoveto{\pgfqpoint{-0.000000in}{0.000000in}}%
\pgfpathlineto{\pgfqpoint{-0.048611in}{0.000000in}}%
\pgfusepath{stroke,fill}%
}%
\begin{pgfscope}%
\pgfsys@transformshift{0.553704in}{0.587191in}%
\pgfsys@useobject{currentmarker}{}%
\end{pgfscope}%
\end{pgfscope}%
\begin{pgfscope}%
\definecolor{textcolor}{rgb}{0.000000,0.000000,0.000000}%
\pgfsetstrokecolor{textcolor}%
\pgfsetfillcolor{textcolor}%
\pgftext[x=0.279012in, y=0.538966in, left, base]{\color{textcolor}\rmfamily\fontsize{10.000000}{12.000000}\selectfont \(\displaystyle {0.0}\)}%
\end{pgfscope}%
\begin{pgfscope}%
\pgfsetbuttcap%
\pgfsetroundjoin%
\definecolor{currentfill}{rgb}{0.000000,0.000000,0.000000}%
\pgfsetfillcolor{currentfill}%
\pgfsetlinewidth{0.803000pt}%
\definecolor{currentstroke}{rgb}{0.000000,0.000000,0.000000}%
\pgfsetstrokecolor{currentstroke}%
\pgfsetdash{}{0pt}%
\pgfsys@defobject{currentmarker}{\pgfqpoint{-0.048611in}{0.000000in}}{\pgfqpoint{-0.000000in}{0.000000in}}{%
\pgfpathmoveto{\pgfqpoint{-0.000000in}{0.000000in}}%
\pgfpathlineto{\pgfqpoint{-0.048611in}{0.000000in}}%
\pgfusepath{stroke,fill}%
}%
\begin{pgfscope}%
\pgfsys@transformshift{0.553704in}{1.003858in}%
\pgfsys@useobject{currentmarker}{}%
\end{pgfscope}%
\end{pgfscope}%
\begin{pgfscope}%
\definecolor{textcolor}{rgb}{0.000000,0.000000,0.000000}%
\pgfsetstrokecolor{textcolor}%
\pgfsetfillcolor{textcolor}%
\pgftext[x=0.279012in, y=0.955633in, left, base]{\color{textcolor}\rmfamily\fontsize{10.000000}{12.000000}\selectfont \(\displaystyle {0.2}\)}%
\end{pgfscope}%
\begin{pgfscope}%
\pgfsetbuttcap%
\pgfsetroundjoin%
\definecolor{currentfill}{rgb}{0.000000,0.000000,0.000000}%
\pgfsetfillcolor{currentfill}%
\pgfsetlinewidth{0.803000pt}%
\definecolor{currentstroke}{rgb}{0.000000,0.000000,0.000000}%
\pgfsetstrokecolor{currentstroke}%
\pgfsetdash{}{0pt}%
\pgfsys@defobject{currentmarker}{\pgfqpoint{-0.048611in}{0.000000in}}{\pgfqpoint{-0.000000in}{0.000000in}}{%
\pgfpathmoveto{\pgfqpoint{-0.000000in}{0.000000in}}%
\pgfpathlineto{\pgfqpoint{-0.048611in}{0.000000in}}%
\pgfusepath{stroke,fill}%
}%
\begin{pgfscope}%
\pgfsys@transformshift{0.553704in}{1.420524in}%
\pgfsys@useobject{currentmarker}{}%
\end{pgfscope}%
\end{pgfscope}%
\begin{pgfscope}%
\definecolor{textcolor}{rgb}{0.000000,0.000000,0.000000}%
\pgfsetstrokecolor{textcolor}%
\pgfsetfillcolor{textcolor}%
\pgftext[x=0.279012in, y=1.372299in, left, base]{\color{textcolor}\rmfamily\fontsize{10.000000}{12.000000}\selectfont \(\displaystyle {0.4}\)}%
\end{pgfscope}%
\begin{pgfscope}%
\pgfsetbuttcap%
\pgfsetroundjoin%
\definecolor{currentfill}{rgb}{0.000000,0.000000,0.000000}%
\pgfsetfillcolor{currentfill}%
\pgfsetlinewidth{0.803000pt}%
\definecolor{currentstroke}{rgb}{0.000000,0.000000,0.000000}%
\pgfsetstrokecolor{currentstroke}%
\pgfsetdash{}{0pt}%
\pgfsys@defobject{currentmarker}{\pgfqpoint{-0.048611in}{0.000000in}}{\pgfqpoint{-0.000000in}{0.000000in}}{%
\pgfpathmoveto{\pgfqpoint{-0.000000in}{0.000000in}}%
\pgfpathlineto{\pgfqpoint{-0.048611in}{0.000000in}}%
\pgfusepath{stroke,fill}%
}%
\begin{pgfscope}%
\pgfsys@transformshift{0.553704in}{1.837191in}%
\pgfsys@useobject{currentmarker}{}%
\end{pgfscope}%
\end{pgfscope}%
\begin{pgfscope}%
\definecolor{textcolor}{rgb}{0.000000,0.000000,0.000000}%
\pgfsetstrokecolor{textcolor}%
\pgfsetfillcolor{textcolor}%
\pgftext[x=0.279012in, y=1.788966in, left, base]{\color{textcolor}\rmfamily\fontsize{10.000000}{12.000000}\selectfont \(\displaystyle {0.6}\)}%
\end{pgfscope}%
\begin{pgfscope}%
\pgfsetbuttcap%
\pgfsetroundjoin%
\definecolor{currentfill}{rgb}{0.000000,0.000000,0.000000}%
\pgfsetfillcolor{currentfill}%
\pgfsetlinewidth{0.803000pt}%
\definecolor{currentstroke}{rgb}{0.000000,0.000000,0.000000}%
\pgfsetstrokecolor{currentstroke}%
\pgfsetdash{}{0pt}%
\pgfsys@defobject{currentmarker}{\pgfqpoint{-0.048611in}{0.000000in}}{\pgfqpoint{-0.000000in}{0.000000in}}{%
\pgfpathmoveto{\pgfqpoint{-0.000000in}{0.000000in}}%
\pgfpathlineto{\pgfqpoint{-0.048611in}{0.000000in}}%
\pgfusepath{stroke,fill}%
}%
\begin{pgfscope}%
\pgfsys@transformshift{0.553704in}{2.253858in}%
\pgfsys@useobject{currentmarker}{}%
\end{pgfscope}%
\end{pgfscope}%
\begin{pgfscope}%
\definecolor{textcolor}{rgb}{0.000000,0.000000,0.000000}%
\pgfsetstrokecolor{textcolor}%
\pgfsetfillcolor{textcolor}%
\pgftext[x=0.279012in, y=2.205633in, left, base]{\color{textcolor}\rmfamily\fontsize{10.000000}{12.000000}\selectfont \(\displaystyle {0.8}\)}%
\end{pgfscope}%
\begin{pgfscope}%
\definecolor{textcolor}{rgb}{0.000000,0.000000,0.000000}%
\pgfsetstrokecolor{textcolor}%
\pgfsetfillcolor{textcolor}%
\pgftext[x=0.223457in,y=1.462191in,,bottom,rotate=90.000000]{\color{textcolor}\rmfamily\fontsize{10.000000}{12.000000}\selectfont Certified accuracy}%
\end{pgfscope}%
\begin{pgfscope}%
\pgfpathrectangle{\pgfqpoint{0.553704in}{0.499691in}}{\pgfqpoint{2.015000in}{1.925000in}}%
\pgfusepath{clip}%
\pgfsetrectcap%
\pgfsetroundjoin%
\pgfsetlinewidth{1.505625pt}%
\definecolor{currentstroke}{rgb}{0.716186,0.833203,0.916155}%
\pgfsetstrokecolor{currentstroke}%
\pgfsetdash{}{0pt}%
\pgfpathmoveto{\pgfqpoint{0.553704in}{2.337191in}}%
\pgfpathlineto{\pgfqpoint{1.032182in}{2.337191in}}%
\pgfpathlineto{\pgfqpoint{1.038922in}{2.333024in}}%
\pgfpathlineto{\pgfqpoint{1.166965in}{2.333024in}}%
\pgfpathlineto{\pgfqpoint{1.173704in}{2.324691in}}%
\pgfpathlineto{\pgfqpoint{1.544356in}{2.324691in}}%
\pgfpathlineto{\pgfqpoint{1.551095in}{2.320524in}}%
\pgfpathlineto{\pgfqpoint{1.618487in}{2.320524in}}%
\pgfpathlineto{\pgfqpoint{1.625226in}{2.316358in}}%
\pgfpathlineto{\pgfqpoint{1.719574in}{2.316358in}}%
\pgfpathlineto{\pgfqpoint{1.733052in}{2.308024in}}%
\pgfpathlineto{\pgfqpoint{1.746530in}{2.308024in}}%
\pgfpathlineto{\pgfqpoint{1.753269in}{2.303858in}}%
\pgfpathlineto{\pgfqpoint{1.760009in}{2.295524in}}%
\pgfpathlineto{\pgfqpoint{1.793704in}{2.274691in}}%
\pgfpathlineto{\pgfqpoint{1.820661in}{2.274691in}}%
\pgfpathlineto{\pgfqpoint{1.827400in}{2.270524in}}%
\pgfpathlineto{\pgfqpoint{1.840878in}{2.270524in}}%
\pgfpathlineto{\pgfqpoint{1.847617in}{2.266358in}}%
\pgfpathlineto{\pgfqpoint{1.861095in}{2.266358in}}%
\pgfpathlineto{\pgfqpoint{1.867835in}{2.262191in}}%
\pgfpathlineto{\pgfqpoint{1.874574in}{2.253858in}}%
\pgfpathlineto{\pgfqpoint{1.881313in}{2.253858in}}%
\pgfpathlineto{\pgfqpoint{1.888052in}{2.245524in}}%
\pgfpathlineto{\pgfqpoint{1.894791in}{2.245524in}}%
\pgfpathlineto{\pgfqpoint{1.901530in}{2.237191in}}%
\pgfpathlineto{\pgfqpoint{1.908269in}{2.233024in}}%
\pgfpathlineto{\pgfqpoint{1.928487in}{2.233024in}}%
\pgfpathlineto{\pgfqpoint{1.935226in}{2.224691in}}%
\pgfpathlineto{\pgfqpoint{1.941965in}{2.220524in}}%
\pgfpathlineto{\pgfqpoint{1.948704in}{2.212191in}}%
\pgfpathlineto{\pgfqpoint{1.955443in}{2.208024in}}%
\pgfpathlineto{\pgfqpoint{1.962182in}{2.199691in}}%
\pgfpathlineto{\pgfqpoint{1.975661in}{2.199691in}}%
\pgfpathlineto{\pgfqpoint{1.989139in}{2.191358in}}%
\pgfpathlineto{\pgfqpoint{1.995878in}{2.178858in}}%
\pgfpathlineto{\pgfqpoint{2.002617in}{2.174691in}}%
\pgfpathlineto{\pgfqpoint{2.016095in}{2.174691in}}%
\pgfpathlineto{\pgfqpoint{2.022835in}{2.166358in}}%
\pgfpathlineto{\pgfqpoint{2.029574in}{2.162191in}}%
\pgfpathlineto{\pgfqpoint{2.049791in}{2.162191in}}%
\pgfpathlineto{\pgfqpoint{2.056530in}{2.158024in}}%
\pgfpathlineto{\pgfqpoint{2.063269in}{2.158024in}}%
\pgfpathlineto{\pgfqpoint{2.070009in}{2.141358in}}%
\pgfpathlineto{\pgfqpoint{2.076748in}{2.137191in}}%
\pgfpathlineto{\pgfqpoint{2.083487in}{2.137191in}}%
\pgfpathlineto{\pgfqpoint{2.096965in}{2.112191in}}%
\pgfpathlineto{\pgfqpoint{2.103704in}{2.112191in}}%
\pgfpathlineto{\pgfqpoint{2.110443in}{2.103858in}}%
\pgfpathlineto{\pgfqpoint{2.117182in}{2.099691in}}%
\pgfpathlineto{\pgfqpoint{2.130661in}{2.074691in}}%
\pgfpathlineto{\pgfqpoint{2.137400in}{2.066358in}}%
\pgfpathlineto{\pgfqpoint{2.144139in}{2.062191in}}%
\pgfpathlineto{\pgfqpoint{2.150878in}{2.053858in}}%
\pgfpathlineto{\pgfqpoint{2.164356in}{2.028858in}}%
\pgfpathlineto{\pgfqpoint{2.171095in}{2.020524in}}%
\pgfpathlineto{\pgfqpoint{2.177835in}{1.999691in}}%
\pgfpathlineto{\pgfqpoint{2.184574in}{1.995524in}}%
\pgfpathlineto{\pgfqpoint{2.191313in}{1.974691in}}%
\pgfpathlineto{\pgfqpoint{2.204791in}{1.941358in}}%
\pgfpathlineto{\pgfqpoint{2.218269in}{1.853858in}}%
\pgfpathlineto{\pgfqpoint{2.225009in}{1.853858in}}%
\pgfpathlineto{\pgfqpoint{2.231748in}{0.587191in}}%
\pgfpathlineto{\pgfqpoint{2.568704in}{0.587191in}}%
\pgfpathlineto{\pgfqpoint{2.568704in}{0.587191in}}%
\pgfusepath{stroke}%
\end{pgfscope}%
\begin{pgfscope}%
\pgfpathrectangle{\pgfqpoint{0.553704in}{0.499691in}}{\pgfqpoint{2.015000in}{1.925000in}}%
\pgfusepath{clip}%
\pgfsetrectcap%
\pgfsetroundjoin%
\pgfsetlinewidth{1.505625pt}%
\definecolor{currentstroke}{rgb}{0.231926,0.545652,0.762614}%
\pgfsetstrokecolor{currentstroke}%
\pgfsetdash{}{0pt}%
\pgfpathmoveto{\pgfqpoint{0.553704in}{2.195524in}}%
\pgfpathlineto{\pgfqpoint{0.782835in}{2.195524in}}%
\pgfpathlineto{\pgfqpoint{0.789574in}{2.191358in}}%
\pgfpathlineto{\pgfqpoint{1.032182in}{2.191358in}}%
\pgfpathlineto{\pgfqpoint{1.038922in}{2.187191in}}%
\pgfpathlineto{\pgfqpoint{1.160226in}{2.187191in}}%
\pgfpathlineto{\pgfqpoint{1.166965in}{2.178858in}}%
\pgfpathlineto{\pgfqpoint{1.214139in}{2.178858in}}%
\pgfpathlineto{\pgfqpoint{1.220878in}{2.174691in}}%
\pgfpathlineto{\pgfqpoint{1.483704in}{2.174691in}}%
\pgfpathlineto{\pgfqpoint{1.490443in}{2.170524in}}%
\pgfpathlineto{\pgfqpoint{1.497182in}{2.170524in}}%
\pgfpathlineto{\pgfqpoint{1.503922in}{2.166358in}}%
\pgfpathlineto{\pgfqpoint{1.557835in}{2.166358in}}%
\pgfpathlineto{\pgfqpoint{1.564574in}{2.162191in}}%
\pgfpathlineto{\pgfqpoint{1.571313in}{2.162191in}}%
\pgfpathlineto{\pgfqpoint{1.578052in}{2.158024in}}%
\pgfpathlineto{\pgfqpoint{1.584791in}{2.158024in}}%
\pgfpathlineto{\pgfqpoint{1.591530in}{2.153858in}}%
\pgfpathlineto{\pgfqpoint{1.645443in}{2.153858in}}%
\pgfpathlineto{\pgfqpoint{1.652182in}{2.149691in}}%
\pgfpathlineto{\pgfqpoint{1.692617in}{2.149691in}}%
\pgfpathlineto{\pgfqpoint{1.699356in}{2.145524in}}%
\pgfpathlineto{\pgfqpoint{1.753269in}{2.145524in}}%
\pgfpathlineto{\pgfqpoint{1.760009in}{2.141358in}}%
\pgfpathlineto{\pgfqpoint{1.793704in}{2.141358in}}%
\pgfpathlineto{\pgfqpoint{1.800443in}{2.133024in}}%
\pgfpathlineto{\pgfqpoint{1.861095in}{2.133024in}}%
\pgfpathlineto{\pgfqpoint{1.867835in}{2.124691in}}%
\pgfpathlineto{\pgfqpoint{1.901530in}{2.124691in}}%
\pgfpathlineto{\pgfqpoint{1.908269in}{2.120524in}}%
\pgfpathlineto{\pgfqpoint{1.915009in}{2.112191in}}%
\pgfpathlineto{\pgfqpoint{1.921748in}{2.112191in}}%
\pgfpathlineto{\pgfqpoint{1.928487in}{2.108024in}}%
\pgfpathlineto{\pgfqpoint{1.935226in}{2.108024in}}%
\pgfpathlineto{\pgfqpoint{1.948704in}{2.099691in}}%
\pgfpathlineto{\pgfqpoint{1.955443in}{2.099691in}}%
\pgfpathlineto{\pgfqpoint{1.962182in}{2.087191in}}%
\pgfpathlineto{\pgfqpoint{1.968922in}{2.078858in}}%
\pgfpathlineto{\pgfqpoint{1.982400in}{2.070524in}}%
\pgfpathlineto{\pgfqpoint{1.989139in}{2.070524in}}%
\pgfpathlineto{\pgfqpoint{1.995878in}{2.062191in}}%
\pgfpathlineto{\pgfqpoint{2.002617in}{2.062191in}}%
\pgfpathlineto{\pgfqpoint{2.009356in}{2.053858in}}%
\pgfpathlineto{\pgfqpoint{2.016095in}{2.049691in}}%
\pgfpathlineto{\pgfqpoint{2.022835in}{2.049691in}}%
\pgfpathlineto{\pgfqpoint{2.036313in}{2.041358in}}%
\pgfpathlineto{\pgfqpoint{2.043052in}{2.033024in}}%
\pgfpathlineto{\pgfqpoint{2.049791in}{2.016358in}}%
\pgfpathlineto{\pgfqpoint{2.056530in}{2.003858in}}%
\pgfpathlineto{\pgfqpoint{2.063269in}{1.995524in}}%
\pgfpathlineto{\pgfqpoint{2.070009in}{1.991358in}}%
\pgfpathlineto{\pgfqpoint{2.076748in}{1.983024in}}%
\pgfpathlineto{\pgfqpoint{2.083487in}{1.966358in}}%
\pgfpathlineto{\pgfqpoint{2.090226in}{1.966358in}}%
\pgfpathlineto{\pgfqpoint{2.096965in}{1.958024in}}%
\pgfpathlineto{\pgfqpoint{2.110443in}{1.958024in}}%
\pgfpathlineto{\pgfqpoint{2.117182in}{1.953858in}}%
\pgfpathlineto{\pgfqpoint{2.137400in}{1.928858in}}%
\pgfpathlineto{\pgfqpoint{2.144139in}{1.928858in}}%
\pgfpathlineto{\pgfqpoint{2.150878in}{1.924691in}}%
\pgfpathlineto{\pgfqpoint{2.157617in}{1.916358in}}%
\pgfpathlineto{\pgfqpoint{2.171095in}{1.908024in}}%
\pgfpathlineto{\pgfqpoint{2.177835in}{1.899691in}}%
\pgfpathlineto{\pgfqpoint{2.184574in}{1.895524in}}%
\pgfpathlineto{\pgfqpoint{2.191313in}{1.895524in}}%
\pgfpathlineto{\pgfqpoint{2.198052in}{1.891358in}}%
\pgfpathlineto{\pgfqpoint{2.225009in}{1.858024in}}%
\pgfpathlineto{\pgfqpoint{2.238487in}{1.833024in}}%
\pgfpathlineto{\pgfqpoint{2.245226in}{1.812191in}}%
\pgfpathlineto{\pgfqpoint{2.278922in}{1.749691in}}%
\pgfpathlineto{\pgfqpoint{2.285661in}{1.733024in}}%
\pgfpathlineto{\pgfqpoint{2.292400in}{1.728858in}}%
\pgfpathlineto{\pgfqpoint{2.299139in}{1.716358in}}%
\pgfpathlineto{\pgfqpoint{2.305878in}{1.683024in}}%
\pgfpathlineto{\pgfqpoint{2.312617in}{1.683024in}}%
\pgfpathlineto{\pgfqpoint{2.319356in}{1.653858in}}%
\pgfpathlineto{\pgfqpoint{2.326095in}{1.637191in}}%
\pgfpathlineto{\pgfqpoint{2.332835in}{1.591358in}}%
\pgfpathlineto{\pgfqpoint{2.339574in}{1.566358in}}%
\pgfpathlineto{\pgfqpoint{2.346313in}{1.528858in}}%
\pgfpathlineto{\pgfqpoint{2.353052in}{1.503858in}}%
\pgfpathlineto{\pgfqpoint{2.359791in}{1.458024in}}%
\pgfpathlineto{\pgfqpoint{2.366530in}{1.458024in}}%
\pgfpathlineto{\pgfqpoint{2.373269in}{0.587191in}}%
\pgfpathlineto{\pgfqpoint{2.568704in}{0.587191in}}%
\pgfpathlineto{\pgfqpoint{2.568704in}{0.587191in}}%
\pgfusepath{stroke}%
\end{pgfscope}%
\begin{pgfscope}%
\pgfpathrectangle{\pgfqpoint{0.553704in}{0.499691in}}{\pgfqpoint{2.015000in}{1.925000in}}%
\pgfusepath{clip}%
\pgfsetrectcap%
\pgfsetroundjoin%
\pgfsetlinewidth{1.505625pt}%
\definecolor{currentstroke}{rgb}{0.031373,0.188235,0.419608}%
\pgfsetstrokecolor{currentstroke}%
\pgfsetdash{}{0pt}%
\pgfpathmoveto{\pgfqpoint{0.553704in}{1.878858in}}%
\pgfpathlineto{\pgfqpoint{1.106313in}{1.878858in}}%
\pgfpathlineto{\pgfqpoint{1.113052in}{1.874691in}}%
\pgfpathlineto{\pgfqpoint{1.429791in}{1.874691in}}%
\pgfpathlineto{\pgfqpoint{1.436530in}{1.870524in}}%
\pgfpathlineto{\pgfqpoint{1.530878in}{1.870524in}}%
\pgfpathlineto{\pgfqpoint{1.537617in}{1.866358in}}%
\pgfpathlineto{\pgfqpoint{1.591530in}{1.866358in}}%
\pgfpathlineto{\pgfqpoint{1.598269in}{1.862191in}}%
\pgfpathlineto{\pgfqpoint{1.618487in}{1.862191in}}%
\pgfpathlineto{\pgfqpoint{1.625226in}{1.858024in}}%
\pgfpathlineto{\pgfqpoint{1.746530in}{1.858024in}}%
\pgfpathlineto{\pgfqpoint{1.753269in}{1.853858in}}%
\pgfpathlineto{\pgfqpoint{1.773487in}{1.853858in}}%
\pgfpathlineto{\pgfqpoint{1.780226in}{1.849691in}}%
\pgfpathlineto{\pgfqpoint{1.786965in}{1.849691in}}%
\pgfpathlineto{\pgfqpoint{1.793704in}{1.837191in}}%
\pgfpathlineto{\pgfqpoint{1.813922in}{1.837191in}}%
\pgfpathlineto{\pgfqpoint{1.820661in}{1.833024in}}%
\pgfpathlineto{\pgfqpoint{1.834139in}{1.833024in}}%
\pgfpathlineto{\pgfqpoint{1.840878in}{1.828858in}}%
\pgfpathlineto{\pgfqpoint{1.867835in}{1.828858in}}%
\pgfpathlineto{\pgfqpoint{1.874574in}{1.820524in}}%
\pgfpathlineto{\pgfqpoint{1.894791in}{1.820524in}}%
\pgfpathlineto{\pgfqpoint{1.908269in}{1.812191in}}%
\pgfpathlineto{\pgfqpoint{1.915009in}{1.803858in}}%
\pgfpathlineto{\pgfqpoint{1.921748in}{1.803858in}}%
\pgfpathlineto{\pgfqpoint{1.928487in}{1.799691in}}%
\pgfpathlineto{\pgfqpoint{1.935226in}{1.799691in}}%
\pgfpathlineto{\pgfqpoint{1.941965in}{1.795524in}}%
\pgfpathlineto{\pgfqpoint{1.955443in}{1.795524in}}%
\pgfpathlineto{\pgfqpoint{1.962182in}{1.791358in}}%
\pgfpathlineto{\pgfqpoint{2.002617in}{1.791358in}}%
\pgfpathlineto{\pgfqpoint{2.009356in}{1.787191in}}%
\pgfpathlineto{\pgfqpoint{2.022835in}{1.787191in}}%
\pgfpathlineto{\pgfqpoint{2.043052in}{1.774691in}}%
\pgfpathlineto{\pgfqpoint{2.056530in}{1.774691in}}%
\pgfpathlineto{\pgfqpoint{2.083487in}{1.758024in}}%
\pgfpathlineto{\pgfqpoint{2.096965in}{1.758024in}}%
\pgfpathlineto{\pgfqpoint{2.117182in}{1.745524in}}%
\pgfpathlineto{\pgfqpoint{2.130661in}{1.745524in}}%
\pgfpathlineto{\pgfqpoint{2.144139in}{1.728858in}}%
\pgfpathlineto{\pgfqpoint{2.150878in}{1.724691in}}%
\pgfpathlineto{\pgfqpoint{2.157617in}{1.712191in}}%
\pgfpathlineto{\pgfqpoint{2.164356in}{1.708024in}}%
\pgfpathlineto{\pgfqpoint{2.171095in}{1.708024in}}%
\pgfpathlineto{\pgfqpoint{2.177835in}{1.695524in}}%
\pgfpathlineto{\pgfqpoint{2.184574in}{1.695524in}}%
\pgfpathlineto{\pgfqpoint{2.191313in}{1.687191in}}%
\pgfpathlineto{\pgfqpoint{2.198052in}{1.687191in}}%
\pgfpathlineto{\pgfqpoint{2.204791in}{1.683024in}}%
\pgfpathlineto{\pgfqpoint{2.218269in}{1.666358in}}%
\pgfpathlineto{\pgfqpoint{2.231748in}{1.658024in}}%
\pgfpathlineto{\pgfqpoint{2.245226in}{1.641358in}}%
\pgfpathlineto{\pgfqpoint{2.258704in}{1.641358in}}%
\pgfpathlineto{\pgfqpoint{2.278922in}{1.616358in}}%
\pgfpathlineto{\pgfqpoint{2.285661in}{1.616358in}}%
\pgfpathlineto{\pgfqpoint{2.292400in}{1.612191in}}%
\pgfpathlineto{\pgfqpoint{2.299139in}{1.599691in}}%
\pgfpathlineto{\pgfqpoint{2.305878in}{1.599691in}}%
\pgfpathlineto{\pgfqpoint{2.312617in}{1.591358in}}%
\pgfpathlineto{\pgfqpoint{2.319356in}{1.591358in}}%
\pgfpathlineto{\pgfqpoint{2.326095in}{1.587191in}}%
\pgfpathlineto{\pgfqpoint{2.332835in}{1.587191in}}%
\pgfpathlineto{\pgfqpoint{2.339574in}{1.574691in}}%
\pgfpathlineto{\pgfqpoint{2.346313in}{1.570524in}}%
\pgfpathlineto{\pgfqpoint{2.353052in}{1.562191in}}%
\pgfpathlineto{\pgfqpoint{2.359791in}{1.558024in}}%
\pgfpathlineto{\pgfqpoint{2.366530in}{1.541358in}}%
\pgfpathlineto{\pgfqpoint{2.373269in}{1.528858in}}%
\pgfpathlineto{\pgfqpoint{2.380009in}{1.512191in}}%
\pgfpathlineto{\pgfqpoint{2.393487in}{1.495524in}}%
\pgfpathlineto{\pgfqpoint{2.400226in}{1.474691in}}%
\pgfpathlineto{\pgfqpoint{2.406965in}{1.462191in}}%
\pgfpathlineto{\pgfqpoint{2.413704in}{1.458024in}}%
\pgfpathlineto{\pgfqpoint{2.420443in}{1.441358in}}%
\pgfpathlineto{\pgfqpoint{2.427182in}{1.433024in}}%
\pgfpathlineto{\pgfqpoint{2.433922in}{1.420524in}}%
\pgfpathlineto{\pgfqpoint{2.447400in}{1.378858in}}%
\pgfpathlineto{\pgfqpoint{2.454139in}{1.378858in}}%
\pgfpathlineto{\pgfqpoint{2.467617in}{1.362191in}}%
\pgfpathlineto{\pgfqpoint{2.474356in}{1.345524in}}%
\pgfpathlineto{\pgfqpoint{2.487835in}{1.291358in}}%
\pgfpathlineto{\pgfqpoint{2.494574in}{1.287191in}}%
\pgfpathlineto{\pgfqpoint{2.501313in}{1.270524in}}%
\pgfpathlineto{\pgfqpoint{2.508052in}{1.237191in}}%
\pgfpathlineto{\pgfqpoint{2.514791in}{1.212191in}}%
\pgfpathlineto{\pgfqpoint{2.521530in}{1.199691in}}%
\pgfpathlineto{\pgfqpoint{2.528269in}{1.174691in}}%
\pgfpathlineto{\pgfqpoint{2.535009in}{1.128858in}}%
\pgfpathlineto{\pgfqpoint{2.541748in}{1.112191in}}%
\pgfpathlineto{\pgfqpoint{2.548487in}{1.074691in}}%
\pgfpathlineto{\pgfqpoint{2.555226in}{1.074691in}}%
\pgfpathlineto{\pgfqpoint{2.561965in}{1.016358in}}%
\pgfpathlineto{\pgfqpoint{2.568704in}{0.587191in}}%
\pgfpathlineto{\pgfqpoint{2.568704in}{0.587191in}}%
\pgfusepath{stroke}%
\end{pgfscope}%
\begin{pgfscope}%
\pgfsetrectcap%
\pgfsetmiterjoin%
\pgfsetlinewidth{0.803000pt}%
\definecolor{currentstroke}{rgb}{0.000000,0.000000,0.000000}%
\pgfsetstrokecolor{currentstroke}%
\pgfsetdash{}{0pt}%
\pgfpathmoveto{\pgfqpoint{0.553704in}{0.499691in}}%
\pgfpathlineto{\pgfqpoint{0.553704in}{2.424691in}}%
\pgfusepath{stroke}%
\end{pgfscope}%
\begin{pgfscope}%
\pgfsetrectcap%
\pgfsetmiterjoin%
\pgfsetlinewidth{0.803000pt}%
\definecolor{currentstroke}{rgb}{0.000000,0.000000,0.000000}%
\pgfsetstrokecolor{currentstroke}%
\pgfsetdash{}{0pt}%
\pgfpathmoveto{\pgfqpoint{2.568704in}{0.499691in}}%
\pgfpathlineto{\pgfqpoint{2.568704in}{2.424691in}}%
\pgfusepath{stroke}%
\end{pgfscope}%
\begin{pgfscope}%
\pgfsetrectcap%
\pgfsetmiterjoin%
\pgfsetlinewidth{0.803000pt}%
\definecolor{currentstroke}{rgb}{0.000000,0.000000,0.000000}%
\pgfsetstrokecolor{currentstroke}%
\pgfsetdash{}{0pt}%
\pgfpathmoveto{\pgfqpoint{0.553704in}{0.499691in}}%
\pgfpathlineto{\pgfqpoint{2.568704in}{0.499691in}}%
\pgfusepath{stroke}%
\end{pgfscope}%
\begin{pgfscope}%
\pgfsetrectcap%
\pgfsetmiterjoin%
\pgfsetlinewidth{0.803000pt}%
\definecolor{currentstroke}{rgb}{0.000000,0.000000,0.000000}%
\pgfsetstrokecolor{currentstroke}%
\pgfsetdash{}{0pt}%
\pgfpathmoveto{\pgfqpoint{0.553704in}{2.424691in}}%
\pgfpathlineto{\pgfqpoint{2.568704in}{2.424691in}}%
\pgfusepath{stroke}%
\end{pgfscope}%
\begin{pgfscope}%
\pgfsetbuttcap%
\pgfsetmiterjoin%
\definecolor{currentfill}{rgb}{1.000000,1.000000,1.000000}%
\pgfsetfillcolor{currentfill}%
\pgfsetfillopacity{0.800000}%
\pgfsetlinewidth{1.003750pt}%
\definecolor{currentstroke}{rgb}{0.800000,0.800000,0.800000}%
\pgfsetstrokecolor{currentstroke}%
\pgfsetstrokeopacity{0.800000}%
\pgfsetdash{}{0pt}%
\pgfpathmoveto{\pgfqpoint{0.650926in}{0.569136in}}%
\pgfpathlineto{\pgfqpoint{1.611815in}{0.569136in}}%
\pgfpathquadraticcurveto{\pgfqpoint{1.639593in}{0.569136in}}{\pgfqpoint{1.639593in}{0.596913in}}%
\pgfpathlineto{\pgfqpoint{1.639593in}{1.356944in}}%
\pgfpathquadraticcurveto{\pgfqpoint{1.639593in}{1.384722in}}{\pgfqpoint{1.611815in}{1.384722in}}%
\pgfpathlineto{\pgfqpoint{0.650926in}{1.384722in}}%
\pgfpathquadraticcurveto{\pgfqpoint{0.623149in}{1.384722in}}{\pgfqpoint{0.623149in}{1.356944in}}%
\pgfpathlineto{\pgfqpoint{0.623149in}{0.596913in}}%
\pgfpathquadraticcurveto{\pgfqpoint{0.623149in}{0.569136in}}{\pgfqpoint{0.650926in}{0.569136in}}%
\pgfpathlineto{\pgfqpoint{0.650926in}{0.569136in}}%
\pgfpathclose%
\pgfusepath{stroke,fill}%
\end{pgfscope}%
\begin{pgfscope}%
\definecolor{textcolor}{rgb}{0.000000,0.000000,0.000000}%
\pgfsetstrokecolor{textcolor}%
\pgfsetfillcolor{textcolor}%
\pgftext[x=0.765825in,y=1.232716in,left,base]{\color{textcolor}\rmfamily\fontsize{10.000000}{12.000000}\selectfont \(\displaystyle \textsc{RS4A}-\ell_{\infty}\)}%
\end{pgfscope}%
\begin{pgfscope}%
\pgfsetrectcap%
\pgfsetroundjoin%
\pgfsetlinewidth{1.505625pt}%
\definecolor{currentstroke}{rgb}{0.716186,0.833203,0.916155}%
\pgfsetstrokecolor{currentstroke}%
\pgfsetdash{}{0pt}%
\pgfpathmoveto{\pgfqpoint{0.678704in}{1.087654in}}%
\pgfpathlineto{\pgfqpoint{0.817593in}{1.087654in}}%
\pgfpathlineto{\pgfqpoint{0.956482in}{1.087654in}}%
\pgfusepath{stroke}%
\end{pgfscope}%
\begin{pgfscope}%
\definecolor{textcolor}{rgb}{0.000000,0.000000,0.000000}%
\pgfsetstrokecolor{textcolor}%
\pgfsetfillcolor{textcolor}%
\pgftext[x=1.067593in,y=1.039043in,left,base]{\color{textcolor}\rmfamily\fontsize{10.000000}{12.000000}\selectfont \(\displaystyle \sigma=0.15\)}%
\end{pgfscope}%
\begin{pgfscope}%
\pgfsetrectcap%
\pgfsetroundjoin%
\pgfsetlinewidth{1.505625pt}%
\definecolor{currentstroke}{rgb}{0.231926,0.545652,0.762614}%
\pgfsetstrokecolor{currentstroke}%
\pgfsetdash{}{0pt}%
\pgfpathmoveto{\pgfqpoint{0.678704in}{0.893981in}}%
\pgfpathlineto{\pgfqpoint{0.817593in}{0.893981in}}%
\pgfpathlineto{\pgfqpoint{0.956482in}{0.893981in}}%
\pgfusepath{stroke}%
\end{pgfscope}%
\begin{pgfscope}%
\definecolor{textcolor}{rgb}{0.000000,0.000000,0.000000}%
\pgfsetstrokecolor{textcolor}%
\pgfsetfillcolor{textcolor}%
\pgftext[x=1.067593in,y=0.845370in,left,base]{\color{textcolor}\rmfamily\fontsize{10.000000}{12.000000}\selectfont \(\displaystyle \sigma=0.25\)}%
\end{pgfscope}%
\begin{pgfscope}%
\pgfsetrectcap%
\pgfsetroundjoin%
\pgfsetlinewidth{1.505625pt}%
\definecolor{currentstroke}{rgb}{0.031373,0.188235,0.419608}%
\pgfsetstrokecolor{currentstroke}%
\pgfsetdash{}{0pt}%
\pgfpathmoveto{\pgfqpoint{0.678704in}{0.700308in}}%
\pgfpathlineto{\pgfqpoint{0.817593in}{0.700308in}}%
\pgfpathlineto{\pgfqpoint{0.956482in}{0.700308in}}%
\pgfusepath{stroke}%
\end{pgfscope}%
\begin{pgfscope}%
\definecolor{textcolor}{rgb}{0.000000,0.000000,0.000000}%
\pgfsetstrokecolor{textcolor}%
\pgfsetfillcolor{textcolor}%
\pgftext[x=1.067593in,y=0.651697in,left,base]{\color{textcolor}\rmfamily\fontsize{10.000000}{12.000000}\selectfont \(\displaystyle \sigma=0.50\)}%
\end{pgfscope}%
\end{pgfpicture}%
\makeatother%
\endgroup%

%% file: figs/cifar10_ancersweep.pgf
%% Creator: Matplotlib, PGF backend
%%
%% To include the figure in your LaTeX document, write
%%   \input{<filename>.pgf}
%%
%% Make sure the required packages are loaded in your preamble
%%   \usepackage{pgf}
%%
%% Also ensure that all the required font packages are loaded; for instance,
%% the lmodern package is sometimes necessary when using math font.
%%   \usepackage{lmodern}
%%
%% Figures using additional raster images can only be included by \input if
%% they are in the same directory as the main LaTeX file. For loading figures
%% from other directories you can use the `import` package
%%   \usepackage{import}
%%
%% and then include the figures with
%%   \import{<path to file>}{<filename>.pgf}
%%
%% Matplotlib used the following preamble
%%
\begingroup%
\makeatletter%
\begin{pgfpicture}%
\pgfpathrectangle{\pgfpointorigin}{\pgfqpoint{2.668704in}{2.524691in}}%
\pgfusepath{use as bounding box, clip}%
\begin{pgfscope}%
\pgfsetbuttcap%
\pgfsetmiterjoin%
\pgfsetlinewidth{0.000000pt}%
\definecolor{currentstroke}{rgb}{0.000000,0.000000,0.000000}%
\pgfsetstrokecolor{currentstroke}%
\pgfsetstrokeopacity{0.000000}%
\pgfsetdash{}{0pt}%
\pgfpathmoveto{\pgfqpoint{0.000000in}{0.000000in}}%
\pgfpathlineto{\pgfqpoint{2.668704in}{0.000000in}}%
\pgfpathlineto{\pgfqpoint{2.668704in}{2.524691in}}%
\pgfpathlineto{\pgfqpoint{0.000000in}{2.524691in}}%
\pgfpathlineto{\pgfqpoint{0.000000in}{0.000000in}}%
\pgfpathclose%
\pgfusepath{}%
\end{pgfscope}%
\begin{pgfscope}%
\pgfsetbuttcap%
\pgfsetmiterjoin%
\pgfsetlinewidth{0.000000pt}%
\definecolor{currentstroke}{rgb}{0.000000,0.000000,0.000000}%
\pgfsetstrokecolor{currentstroke}%
\pgfsetstrokeopacity{0.000000}%
\pgfsetdash{}{0pt}%
\pgfpathmoveto{\pgfqpoint{0.553704in}{0.499691in}}%
\pgfpathlineto{\pgfqpoint{2.568704in}{0.499691in}}%
\pgfpathlineto{\pgfqpoint{2.568704in}{2.424691in}}%
\pgfpathlineto{\pgfqpoint{0.553704in}{2.424691in}}%
\pgfpathlineto{\pgfqpoint{0.553704in}{0.499691in}}%
\pgfpathclose%
\pgfusepath{}%
\end{pgfscope}%
\begin{pgfscope}%
\pgfsetbuttcap%
\pgfsetroundjoin%
\definecolor{currentfill}{rgb}{0.000000,0.000000,0.000000}%
\pgfsetfillcolor{currentfill}%
\pgfsetlinewidth{0.803000pt}%
\definecolor{currentstroke}{rgb}{0.000000,0.000000,0.000000}%
\pgfsetstrokecolor{currentstroke}%
\pgfsetdash{}{0pt}%
\pgfsys@defobject{currentmarker}{\pgfqpoint{0.000000in}{-0.048611in}}{\pgfqpoint{0.000000in}{0.000000in}}{%
\pgfpathmoveto{\pgfqpoint{0.000000in}{0.000000in}}%
\pgfpathlineto{\pgfqpoint{0.000000in}{-0.048611in}}%
\pgfusepath{stroke,fill}%
}%
\begin{pgfscope}%
\pgfsys@transformshift{0.985019in}{0.499691in}%
\pgfsys@useobject{currentmarker}{}%
\end{pgfscope}%
\end{pgfscope}%
\begin{pgfscope}%
\definecolor{textcolor}{rgb}{0.000000,0.000000,0.000000}%
\pgfsetstrokecolor{textcolor}%
\pgfsetfillcolor{textcolor}%
\pgftext[x=0.985019in,y=0.402469in,,top]{\color{textcolor}\rmfamily\fontsize{10.000000}{12.000000}\selectfont \(\displaystyle 10^{-1.6 \alpha}\)}%
\end{pgfscope}%
\begin{pgfscope}%
\pgfsetbuttcap%
\pgfsetroundjoin%
\definecolor{currentfill}{rgb}{0.000000,0.000000,0.000000}%
\pgfsetfillcolor{currentfill}%
\pgfsetlinewidth{0.803000pt}%
\definecolor{currentstroke}{rgb}{0.000000,0.000000,0.000000}%
\pgfsetstrokecolor{currentstroke}%
\pgfsetdash{}{0pt}%
\pgfsys@defobject{currentmarker}{\pgfqpoint{0.000000in}{-0.048611in}}{\pgfqpoint{0.000000in}{0.000000in}}{%
\pgfpathmoveto{\pgfqpoint{0.000000in}{0.000000in}}%
\pgfpathlineto{\pgfqpoint{0.000000in}{-0.048611in}}%
\pgfusepath{stroke,fill}%
}%
\begin{pgfscope}%
\pgfsys@transformshift{1.420843in}{0.499691in}%
\pgfsys@useobject{currentmarker}{}%
\end{pgfscope}%
\end{pgfscope}%
\begin{pgfscope}%
\definecolor{textcolor}{rgb}{0.000000,0.000000,0.000000}%
\pgfsetstrokecolor{textcolor}%
\pgfsetfillcolor{textcolor}%
\pgftext[x=1.420843in,y=0.402469in,,top]{\color{textcolor}\rmfamily\fontsize{10.000000}{12.000000}\selectfont \(\displaystyle 10^{-1.4 \alpha}\)}%
\end{pgfscope}%
\begin{pgfscope}%
\pgfsetbuttcap%
\pgfsetroundjoin%
\definecolor{currentfill}{rgb}{0.000000,0.000000,0.000000}%
\pgfsetfillcolor{currentfill}%
\pgfsetlinewidth{0.803000pt}%
\definecolor{currentstroke}{rgb}{0.000000,0.000000,0.000000}%
\pgfsetstrokecolor{currentstroke}%
\pgfsetdash{}{0pt}%
\pgfsys@defobject{currentmarker}{\pgfqpoint{0.000000in}{-0.048611in}}{\pgfqpoint{0.000000in}{0.000000in}}{%
\pgfpathmoveto{\pgfqpoint{0.000000in}{0.000000in}}%
\pgfpathlineto{\pgfqpoint{0.000000in}{-0.048611in}}%
\pgfusepath{stroke,fill}%
}%
\begin{pgfscope}%
\pgfsys@transformshift{1.856667in}{0.499691in}%
\pgfsys@useobject{currentmarker}{}%
\end{pgfscope}%
\end{pgfscope}%
\begin{pgfscope}%
\definecolor{textcolor}{rgb}{0.000000,0.000000,0.000000}%
\pgfsetstrokecolor{textcolor}%
\pgfsetfillcolor{textcolor}%
\pgftext[x=1.856667in,y=0.402469in,,top]{\color{textcolor}\rmfamily\fontsize{10.000000}{12.000000}\selectfont \(\displaystyle 10^{-1.2 \alpha}\)}%
\end{pgfscope}%
\begin{pgfscope}%
\pgfsetbuttcap%
\pgfsetroundjoin%
\definecolor{currentfill}{rgb}{0.000000,0.000000,0.000000}%
\pgfsetfillcolor{currentfill}%
\pgfsetlinewidth{0.803000pt}%
\definecolor{currentstroke}{rgb}{0.000000,0.000000,0.000000}%
\pgfsetstrokecolor{currentstroke}%
\pgfsetdash{}{0pt}%
\pgfsys@defobject{currentmarker}{\pgfqpoint{0.000000in}{-0.048611in}}{\pgfqpoint{0.000000in}{0.000000in}}{%
\pgfpathmoveto{\pgfqpoint{0.000000in}{0.000000in}}%
\pgfpathlineto{\pgfqpoint{0.000000in}{-0.048611in}}%
\pgfusepath{stroke,fill}%
}%
\begin{pgfscope}%
\pgfsys@transformshift{2.292491in}{0.499691in}%
\pgfsys@useobject{currentmarker}{}%
\end{pgfscope}%
\end{pgfscope}%
\begin{pgfscope}%
\definecolor{textcolor}{rgb}{0.000000,0.000000,0.000000}%
\pgfsetstrokecolor{textcolor}%
\pgfsetfillcolor{textcolor}%
\pgftext[x=2.292491in,y=0.402469in,,top]{\color{textcolor}\rmfamily\fontsize{10.000000}{12.000000}\selectfont \(\displaystyle 10^{-1.0 \alpha}\)}%
\end{pgfscope}%
\begin{pgfscope}%
\definecolor{textcolor}{rgb}{0.000000,0.000000,0.000000}%
\pgfsetstrokecolor{textcolor}%
\pgfsetfillcolor{textcolor}%
\pgftext[x=1.561204in,y=0.223457in,,top]{\color{textcolor}\rmfamily\fontsize{10.000000}{12.000000}\selectfont Volume}%
\end{pgfscope}%
\begin{pgfscope}%
\pgfsetbuttcap%
\pgfsetroundjoin%
\definecolor{currentfill}{rgb}{0.000000,0.000000,0.000000}%
\pgfsetfillcolor{currentfill}%
\pgfsetlinewidth{0.803000pt}%
\definecolor{currentstroke}{rgb}{0.000000,0.000000,0.000000}%
\pgfsetstrokecolor{currentstroke}%
\pgfsetdash{}{0pt}%
\pgfsys@defobject{currentmarker}{\pgfqpoint{-0.048611in}{0.000000in}}{\pgfqpoint{-0.000000in}{0.000000in}}{%
\pgfpathmoveto{\pgfqpoint{-0.000000in}{0.000000in}}%
\pgfpathlineto{\pgfqpoint{-0.048611in}{0.000000in}}%
\pgfusepath{stroke,fill}%
}%
\begin{pgfscope}%
\pgfsys@transformshift{0.553704in}{0.587191in}%
\pgfsys@useobject{currentmarker}{}%
\end{pgfscope}%
\end{pgfscope}%
\begin{pgfscope}%
\definecolor{textcolor}{rgb}{0.000000,0.000000,0.000000}%
\pgfsetstrokecolor{textcolor}%
\pgfsetfillcolor{textcolor}%
\pgftext[x=0.279012in, y=0.538966in, left, base]{\color{textcolor}\rmfamily\fontsize{10.000000}{12.000000}\selectfont \(\displaystyle {0.0}\)}%
\end{pgfscope}%
\begin{pgfscope}%
\pgfsetbuttcap%
\pgfsetroundjoin%
\definecolor{currentfill}{rgb}{0.000000,0.000000,0.000000}%
\pgfsetfillcolor{currentfill}%
\pgfsetlinewidth{0.803000pt}%
\definecolor{currentstroke}{rgb}{0.000000,0.000000,0.000000}%
\pgfsetstrokecolor{currentstroke}%
\pgfsetdash{}{0pt}%
\pgfsys@defobject{currentmarker}{\pgfqpoint{-0.048611in}{0.000000in}}{\pgfqpoint{-0.000000in}{0.000000in}}{%
\pgfpathmoveto{\pgfqpoint{-0.000000in}{0.000000in}}%
\pgfpathlineto{\pgfqpoint{-0.048611in}{0.000000in}}%
\pgfusepath{stroke,fill}%
}%
\begin{pgfscope}%
\pgfsys@transformshift{0.553704in}{0.971806in}%
\pgfsys@useobject{currentmarker}{}%
\end{pgfscope}%
\end{pgfscope}%
\begin{pgfscope}%
\definecolor{textcolor}{rgb}{0.000000,0.000000,0.000000}%
\pgfsetstrokecolor{textcolor}%
\pgfsetfillcolor{textcolor}%
\pgftext[x=0.279012in, y=0.923581in, left, base]{\color{textcolor}\rmfamily\fontsize{10.000000}{12.000000}\selectfont \(\displaystyle {0.2}\)}%
\end{pgfscope}%
\begin{pgfscope}%
\pgfsetbuttcap%
\pgfsetroundjoin%
\definecolor{currentfill}{rgb}{0.000000,0.000000,0.000000}%
\pgfsetfillcolor{currentfill}%
\pgfsetlinewidth{0.803000pt}%
\definecolor{currentstroke}{rgb}{0.000000,0.000000,0.000000}%
\pgfsetstrokecolor{currentstroke}%
\pgfsetdash{}{0pt}%
\pgfsys@defobject{currentmarker}{\pgfqpoint{-0.048611in}{0.000000in}}{\pgfqpoint{-0.000000in}{0.000000in}}{%
\pgfpathmoveto{\pgfqpoint{-0.000000in}{0.000000in}}%
\pgfpathlineto{\pgfqpoint{-0.048611in}{0.000000in}}%
\pgfusepath{stroke,fill}%
}%
\begin{pgfscope}%
\pgfsys@transformshift{0.553704in}{1.356422in}%
\pgfsys@useobject{currentmarker}{}%
\end{pgfscope}%
\end{pgfscope}%
\begin{pgfscope}%
\definecolor{textcolor}{rgb}{0.000000,0.000000,0.000000}%
\pgfsetstrokecolor{textcolor}%
\pgfsetfillcolor{textcolor}%
\pgftext[x=0.279012in, y=1.308197in, left, base]{\color{textcolor}\rmfamily\fontsize{10.000000}{12.000000}\selectfont \(\displaystyle {0.4}\)}%
\end{pgfscope}%
\begin{pgfscope}%
\pgfsetbuttcap%
\pgfsetroundjoin%
\definecolor{currentfill}{rgb}{0.000000,0.000000,0.000000}%
\pgfsetfillcolor{currentfill}%
\pgfsetlinewidth{0.803000pt}%
\definecolor{currentstroke}{rgb}{0.000000,0.000000,0.000000}%
\pgfsetstrokecolor{currentstroke}%
\pgfsetdash{}{0pt}%
\pgfsys@defobject{currentmarker}{\pgfqpoint{-0.048611in}{0.000000in}}{\pgfqpoint{-0.000000in}{0.000000in}}{%
\pgfpathmoveto{\pgfqpoint{-0.000000in}{0.000000in}}%
\pgfpathlineto{\pgfqpoint{-0.048611in}{0.000000in}}%
\pgfusepath{stroke,fill}%
}%
\begin{pgfscope}%
\pgfsys@transformshift{0.553704in}{1.741037in}%
\pgfsys@useobject{currentmarker}{}%
\end{pgfscope}%
\end{pgfscope}%
\begin{pgfscope}%
\definecolor{textcolor}{rgb}{0.000000,0.000000,0.000000}%
\pgfsetstrokecolor{textcolor}%
\pgfsetfillcolor{textcolor}%
\pgftext[x=0.279012in, y=1.692812in, left, base]{\color{textcolor}\rmfamily\fontsize{10.000000}{12.000000}\selectfont \(\displaystyle {0.6}\)}%
\end{pgfscope}%
\begin{pgfscope}%
\pgfsetbuttcap%
\pgfsetroundjoin%
\definecolor{currentfill}{rgb}{0.000000,0.000000,0.000000}%
\pgfsetfillcolor{currentfill}%
\pgfsetlinewidth{0.803000pt}%
\definecolor{currentstroke}{rgb}{0.000000,0.000000,0.000000}%
\pgfsetstrokecolor{currentstroke}%
\pgfsetdash{}{0pt}%
\pgfsys@defobject{currentmarker}{\pgfqpoint{-0.048611in}{0.000000in}}{\pgfqpoint{-0.000000in}{0.000000in}}{%
\pgfpathmoveto{\pgfqpoint{-0.000000in}{0.000000in}}%
\pgfpathlineto{\pgfqpoint{-0.048611in}{0.000000in}}%
\pgfusepath{stroke,fill}%
}%
\begin{pgfscope}%
\pgfsys@transformshift{0.553704in}{2.125653in}%
\pgfsys@useobject{currentmarker}{}%
\end{pgfscope}%
\end{pgfscope}%
\begin{pgfscope}%
\definecolor{textcolor}{rgb}{0.000000,0.000000,0.000000}%
\pgfsetstrokecolor{textcolor}%
\pgfsetfillcolor{textcolor}%
\pgftext[x=0.279012in, y=2.077427in, left, base]{\color{textcolor}\rmfamily\fontsize{10.000000}{12.000000}\selectfont \(\displaystyle {0.8}\)}%
\end{pgfscope}%
\begin{pgfscope}%
\definecolor{textcolor}{rgb}{0.000000,0.000000,0.000000}%
\pgfsetstrokecolor{textcolor}%
\pgfsetfillcolor{textcolor}%
\pgftext[x=0.223457in,y=1.462191in,,bottom,rotate=90.000000]{\color{textcolor}\rmfamily\fontsize{10.000000}{12.000000}\selectfont Certified accuracy}%
\end{pgfscope}%
\begin{pgfscope}%
\pgfpathrectangle{\pgfqpoint{0.553704in}{0.499691in}}{\pgfqpoint{2.015000in}{1.925000in}}%
\pgfusepath{clip}%
\pgfsetrectcap%
\pgfsetroundjoin%
\pgfsetlinewidth{1.505625pt}%
\definecolor{currentstroke}{rgb}{0.716186,0.833203,0.916155}%
\pgfsetstrokecolor{currentstroke}%
\pgfsetdash{}{0pt}%
\pgfpathmoveto{\pgfqpoint{0.553704in}{1.587191in}}%
\pgfpathlineto{\pgfqpoint{1.402835in}{1.587191in}}%
\pgfpathlineto{\pgfqpoint{1.416313in}{1.548730in}}%
\pgfpathlineto{\pgfqpoint{1.470226in}{1.548730in}}%
\pgfpathlineto{\pgfqpoint{1.476965in}{1.529499in}}%
\pgfpathlineto{\pgfqpoint{1.605009in}{1.529499in}}%
\pgfpathlineto{\pgfqpoint{1.611748in}{1.510268in}}%
\pgfpathlineto{\pgfqpoint{1.645443in}{1.510268in}}%
\pgfpathlineto{\pgfqpoint{1.652182in}{1.471806in}}%
\pgfpathlineto{\pgfqpoint{1.665661in}{1.433345in}}%
\pgfpathlineto{\pgfqpoint{1.692617in}{1.433345in}}%
\pgfpathlineto{\pgfqpoint{1.699356in}{1.414114in}}%
\pgfpathlineto{\pgfqpoint{1.706095in}{1.414114in}}%
\pgfpathlineto{\pgfqpoint{1.712835in}{1.394883in}}%
\pgfpathlineto{\pgfqpoint{1.726313in}{1.394883in}}%
\pgfpathlineto{\pgfqpoint{1.733052in}{1.375653in}}%
\pgfpathlineto{\pgfqpoint{1.746530in}{1.375653in}}%
\pgfpathlineto{\pgfqpoint{1.753269in}{1.337191in}}%
\pgfpathlineto{\pgfqpoint{1.766748in}{1.337191in}}%
\pgfpathlineto{\pgfqpoint{1.773487in}{1.279499in}}%
\pgfpathlineto{\pgfqpoint{1.780226in}{1.279499in}}%
\pgfpathlineto{\pgfqpoint{1.793704in}{1.241037in}}%
\pgfpathlineto{\pgfqpoint{1.807182in}{1.241037in}}%
\pgfpathlineto{\pgfqpoint{1.820661in}{1.202576in}}%
\pgfpathlineto{\pgfqpoint{1.827400in}{1.202576in}}%
\pgfpathlineto{\pgfqpoint{1.834139in}{1.183345in}}%
\pgfpathlineto{\pgfqpoint{1.840878in}{1.144883in}}%
\pgfpathlineto{\pgfqpoint{1.847617in}{1.144883in}}%
\pgfpathlineto{\pgfqpoint{1.854356in}{1.125653in}}%
\pgfpathlineto{\pgfqpoint{1.881313in}{1.125653in}}%
\pgfpathlineto{\pgfqpoint{1.908269in}{1.048730in}}%
\pgfpathlineto{\pgfqpoint{1.941965in}{1.048730in}}%
\pgfpathlineto{\pgfqpoint{1.948704in}{0.991037in}}%
\pgfpathlineto{\pgfqpoint{1.955443in}{0.971806in}}%
\pgfpathlineto{\pgfqpoint{1.962182in}{0.933345in}}%
\pgfpathlineto{\pgfqpoint{1.989139in}{0.933345in}}%
\pgfpathlineto{\pgfqpoint{1.995878in}{0.914114in}}%
\pgfpathlineto{\pgfqpoint{2.016095in}{0.914114in}}%
\pgfpathlineto{\pgfqpoint{2.022835in}{0.894883in}}%
\pgfpathlineto{\pgfqpoint{2.029574in}{0.894883in}}%
\pgfpathlineto{\pgfqpoint{2.036313in}{0.856422in}}%
\pgfpathlineto{\pgfqpoint{2.043052in}{0.837191in}}%
\pgfpathlineto{\pgfqpoint{2.090226in}{0.837191in}}%
\pgfpathlineto{\pgfqpoint{2.096965in}{0.798730in}}%
\pgfpathlineto{\pgfqpoint{2.110443in}{0.798730in}}%
\pgfpathlineto{\pgfqpoint{2.117182in}{0.779499in}}%
\pgfpathlineto{\pgfqpoint{2.130661in}{0.779499in}}%
\pgfpathlineto{\pgfqpoint{2.137400in}{0.760268in}}%
\pgfpathlineto{\pgfqpoint{2.150878in}{0.760268in}}%
\pgfpathlineto{\pgfqpoint{2.157617in}{0.721806in}}%
\pgfpathlineto{\pgfqpoint{2.164356in}{0.721806in}}%
\pgfpathlineto{\pgfqpoint{2.171095in}{0.702576in}}%
\pgfpathlineto{\pgfqpoint{2.211530in}{0.702576in}}%
\pgfpathlineto{\pgfqpoint{2.225009in}{0.664114in}}%
\pgfpathlineto{\pgfqpoint{2.238487in}{0.664114in}}%
\pgfpathlineto{\pgfqpoint{2.245226in}{0.644883in}}%
\pgfpathlineto{\pgfqpoint{2.501313in}{0.644883in}}%
\pgfpathlineto{\pgfqpoint{2.508052in}{0.625653in}}%
\pgfpathlineto{\pgfqpoint{2.541748in}{0.625653in}}%
\pgfpathlineto{\pgfqpoint{2.548487in}{0.606422in}}%
\pgfpathlineto{\pgfqpoint{2.561965in}{0.606422in}}%
\pgfpathlineto{\pgfqpoint{2.568704in}{0.587191in}}%
\pgfpathlineto{\pgfqpoint{2.568704in}{0.587191in}}%
\pgfusepath{stroke}%
\end{pgfscope}%
\begin{pgfscope}%
\pgfpathrectangle{\pgfqpoint{0.553704in}{0.499691in}}{\pgfqpoint{2.015000in}{1.925000in}}%
\pgfusepath{clip}%
\pgfsetrectcap%
\pgfsetroundjoin%
\pgfsetlinewidth{1.505625pt}%
\definecolor{currentstroke}{rgb}{0.376732,0.653072,0.822484}%
\pgfsetstrokecolor{currentstroke}%
\pgfsetdash{}{0pt}%
\pgfpathmoveto{\pgfqpoint{0.553704in}{2.317960in}}%
\pgfpathlineto{\pgfqpoint{1.173704in}{2.317960in}}%
\pgfpathlineto{\pgfqpoint{1.180443in}{2.298730in}}%
\pgfpathlineto{\pgfqpoint{1.328704in}{2.298730in}}%
\pgfpathlineto{\pgfqpoint{1.335443in}{2.279499in}}%
\pgfpathlineto{\pgfqpoint{1.375878in}{2.279499in}}%
\pgfpathlineto{\pgfqpoint{1.382617in}{2.260268in}}%
\pgfpathlineto{\pgfqpoint{1.416313in}{2.260268in}}%
\pgfpathlineto{\pgfqpoint{1.436530in}{2.202576in}}%
\pgfpathlineto{\pgfqpoint{1.450009in}{2.202576in}}%
\pgfpathlineto{\pgfqpoint{1.456748in}{2.183345in}}%
\pgfpathlineto{\pgfqpoint{1.490443in}{2.183345in}}%
\pgfpathlineto{\pgfqpoint{1.497182in}{2.164114in}}%
\pgfpathlineto{\pgfqpoint{1.510661in}{2.164114in}}%
\pgfpathlineto{\pgfqpoint{1.517400in}{2.125653in}}%
\pgfpathlineto{\pgfqpoint{1.530878in}{2.087191in}}%
\pgfpathlineto{\pgfqpoint{1.557835in}{2.087191in}}%
\pgfpathlineto{\pgfqpoint{1.564574in}{2.067960in}}%
\pgfpathlineto{\pgfqpoint{1.591530in}{2.067960in}}%
\pgfpathlineto{\pgfqpoint{1.625226in}{1.971806in}}%
\pgfpathlineto{\pgfqpoint{1.631965in}{1.971806in}}%
\pgfpathlineto{\pgfqpoint{1.638704in}{1.952576in}}%
\pgfpathlineto{\pgfqpoint{1.645443in}{1.952576in}}%
\pgfpathlineto{\pgfqpoint{1.658922in}{1.914114in}}%
\pgfpathlineto{\pgfqpoint{1.665661in}{1.914114in}}%
\pgfpathlineto{\pgfqpoint{1.672400in}{1.894883in}}%
\pgfpathlineto{\pgfqpoint{1.679139in}{1.856422in}}%
\pgfpathlineto{\pgfqpoint{1.692617in}{1.817960in}}%
\pgfpathlineto{\pgfqpoint{1.699356in}{1.817960in}}%
\pgfpathlineto{\pgfqpoint{1.706095in}{1.779499in}}%
\pgfpathlineto{\pgfqpoint{1.712835in}{1.779499in}}%
\pgfpathlineto{\pgfqpoint{1.719574in}{1.760268in}}%
\pgfpathlineto{\pgfqpoint{1.726313in}{1.721806in}}%
\pgfpathlineto{\pgfqpoint{1.733052in}{1.721806in}}%
\pgfpathlineto{\pgfqpoint{1.746530in}{1.683345in}}%
\pgfpathlineto{\pgfqpoint{1.760009in}{1.606422in}}%
\pgfpathlineto{\pgfqpoint{1.766748in}{1.529499in}}%
\pgfpathlineto{\pgfqpoint{1.773487in}{1.510268in}}%
\pgfpathlineto{\pgfqpoint{1.793704in}{1.394883in}}%
\pgfpathlineto{\pgfqpoint{1.800443in}{1.375653in}}%
\pgfpathlineto{\pgfqpoint{1.807182in}{1.298730in}}%
\pgfpathlineto{\pgfqpoint{1.813922in}{1.260268in}}%
\pgfpathlineto{\pgfqpoint{1.820661in}{1.241037in}}%
\pgfpathlineto{\pgfqpoint{1.834139in}{1.241037in}}%
\pgfpathlineto{\pgfqpoint{1.854356in}{1.183345in}}%
\pgfpathlineto{\pgfqpoint{1.861095in}{1.144883in}}%
\pgfpathlineto{\pgfqpoint{1.867835in}{1.144883in}}%
\pgfpathlineto{\pgfqpoint{1.874574in}{1.106422in}}%
\pgfpathlineto{\pgfqpoint{1.894791in}{1.106422in}}%
\pgfpathlineto{\pgfqpoint{1.901530in}{1.087191in}}%
\pgfpathlineto{\pgfqpoint{1.915009in}{1.087191in}}%
\pgfpathlineto{\pgfqpoint{1.921748in}{1.067960in}}%
\pgfpathlineto{\pgfqpoint{1.935226in}{1.067960in}}%
\pgfpathlineto{\pgfqpoint{1.941965in}{1.029499in}}%
\pgfpathlineto{\pgfqpoint{1.955443in}{0.991037in}}%
\pgfpathlineto{\pgfqpoint{1.962182in}{0.952576in}}%
\pgfpathlineto{\pgfqpoint{1.968922in}{0.952576in}}%
\pgfpathlineto{\pgfqpoint{1.975661in}{0.933345in}}%
\pgfpathlineto{\pgfqpoint{1.982400in}{0.875653in}}%
\pgfpathlineto{\pgfqpoint{1.989139in}{0.837191in}}%
\pgfpathlineto{\pgfqpoint{2.022835in}{0.837191in}}%
\pgfpathlineto{\pgfqpoint{2.029574in}{0.798730in}}%
\pgfpathlineto{\pgfqpoint{2.056530in}{0.798730in}}%
\pgfpathlineto{\pgfqpoint{2.070009in}{0.721806in}}%
\pgfpathlineto{\pgfqpoint{2.090226in}{0.721806in}}%
\pgfpathlineto{\pgfqpoint{2.096965in}{0.683345in}}%
\pgfpathlineto{\pgfqpoint{2.110443in}{0.683345in}}%
\pgfpathlineto{\pgfqpoint{2.123922in}{0.644883in}}%
\pgfpathlineto{\pgfqpoint{2.130661in}{0.644883in}}%
\pgfpathlineto{\pgfqpoint{2.137400in}{0.625653in}}%
\pgfpathlineto{\pgfqpoint{2.305878in}{0.625653in}}%
\pgfpathlineto{\pgfqpoint{2.312617in}{0.606422in}}%
\pgfpathlineto{\pgfqpoint{2.433922in}{0.606422in}}%
\pgfpathlineto{\pgfqpoint{2.440661in}{0.587191in}}%
\pgfpathlineto{\pgfqpoint{2.568704in}{0.587191in}}%
\pgfpathlineto{\pgfqpoint{2.568704in}{0.587191in}}%
\pgfusepath{stroke}%
\end{pgfscope}%
\begin{pgfscope}%
\pgfpathrectangle{\pgfqpoint{0.553704in}{0.499691in}}{\pgfqpoint{2.015000in}{1.925000in}}%
\pgfusepath{clip}%
\pgfsetrectcap%
\pgfsetroundjoin%
\pgfsetlinewidth{1.505625pt}%
\definecolor{currentstroke}{rgb}{0.114802,0.424437,0.695194}%
\pgfsetstrokecolor{currentstroke}%
\pgfsetdash{}{0pt}%
\pgfpathmoveto{\pgfqpoint{0.553704in}{2.337191in}}%
\pgfpathlineto{\pgfqpoint{0.675009in}{2.337191in}}%
\pgfpathlineto{\pgfqpoint{0.681748in}{2.317960in}}%
\pgfpathlineto{\pgfqpoint{0.951313in}{2.317960in}}%
\pgfpathlineto{\pgfqpoint{0.958052in}{2.298730in}}%
\pgfpathlineto{\pgfqpoint{1.220878in}{2.298730in}}%
\pgfpathlineto{\pgfqpoint{1.227617in}{2.279499in}}%
\pgfpathlineto{\pgfqpoint{1.254574in}{2.279499in}}%
\pgfpathlineto{\pgfqpoint{1.268052in}{2.241037in}}%
\pgfpathlineto{\pgfqpoint{1.288269in}{2.241037in}}%
\pgfpathlineto{\pgfqpoint{1.301748in}{2.202576in}}%
\pgfpathlineto{\pgfqpoint{1.328704in}{2.202576in}}%
\pgfpathlineto{\pgfqpoint{1.342182in}{2.164114in}}%
\pgfpathlineto{\pgfqpoint{1.362400in}{2.164114in}}%
\pgfpathlineto{\pgfqpoint{1.375878in}{2.125653in}}%
\pgfpathlineto{\pgfqpoint{1.382617in}{2.125653in}}%
\pgfpathlineto{\pgfqpoint{1.389356in}{2.106422in}}%
\pgfpathlineto{\pgfqpoint{1.416313in}{2.106422in}}%
\pgfpathlineto{\pgfqpoint{1.423052in}{2.067960in}}%
\pgfpathlineto{\pgfqpoint{1.443269in}{2.067960in}}%
\pgfpathlineto{\pgfqpoint{1.456748in}{2.029499in}}%
\pgfpathlineto{\pgfqpoint{1.463487in}{2.029499in}}%
\pgfpathlineto{\pgfqpoint{1.476965in}{1.991037in}}%
\pgfpathlineto{\pgfqpoint{1.483704in}{1.991037in}}%
\pgfpathlineto{\pgfqpoint{1.490443in}{1.971806in}}%
\pgfpathlineto{\pgfqpoint{1.497182in}{1.971806in}}%
\pgfpathlineto{\pgfqpoint{1.503922in}{1.933345in}}%
\pgfpathlineto{\pgfqpoint{1.517400in}{1.894883in}}%
\pgfpathlineto{\pgfqpoint{1.524139in}{1.894883in}}%
\pgfpathlineto{\pgfqpoint{1.537617in}{1.856422in}}%
\pgfpathlineto{\pgfqpoint{1.544356in}{1.856422in}}%
\pgfpathlineto{\pgfqpoint{1.551095in}{1.817960in}}%
\pgfpathlineto{\pgfqpoint{1.564574in}{1.817960in}}%
\pgfpathlineto{\pgfqpoint{1.578052in}{1.779499in}}%
\pgfpathlineto{\pgfqpoint{1.584791in}{1.779499in}}%
\pgfpathlineto{\pgfqpoint{1.598269in}{1.741037in}}%
\pgfpathlineto{\pgfqpoint{1.611748in}{1.741037in}}%
\pgfpathlineto{\pgfqpoint{1.618487in}{1.702576in}}%
\pgfpathlineto{\pgfqpoint{1.625226in}{1.644883in}}%
\pgfpathlineto{\pgfqpoint{1.638704in}{1.606422in}}%
\pgfpathlineto{\pgfqpoint{1.645443in}{1.606422in}}%
\pgfpathlineto{\pgfqpoint{1.665661in}{1.548730in}}%
\pgfpathlineto{\pgfqpoint{1.672400in}{1.510268in}}%
\pgfpathlineto{\pgfqpoint{1.679139in}{1.510268in}}%
\pgfpathlineto{\pgfqpoint{1.692617in}{1.471806in}}%
\pgfpathlineto{\pgfqpoint{1.706095in}{1.471806in}}%
\pgfpathlineto{\pgfqpoint{1.726313in}{1.414114in}}%
\pgfpathlineto{\pgfqpoint{1.733052in}{1.375653in}}%
\pgfpathlineto{\pgfqpoint{1.746530in}{1.337191in}}%
\pgfpathlineto{\pgfqpoint{1.760009in}{1.260268in}}%
\pgfpathlineto{\pgfqpoint{1.766748in}{1.241037in}}%
\pgfpathlineto{\pgfqpoint{1.773487in}{1.164114in}}%
\pgfpathlineto{\pgfqpoint{1.807182in}{1.164114in}}%
\pgfpathlineto{\pgfqpoint{1.813922in}{1.106422in}}%
\pgfpathlineto{\pgfqpoint{1.827400in}{1.067960in}}%
\pgfpathlineto{\pgfqpoint{1.834139in}{1.067960in}}%
\pgfpathlineto{\pgfqpoint{1.840878in}{1.048730in}}%
\pgfpathlineto{\pgfqpoint{1.847617in}{1.048730in}}%
\pgfpathlineto{\pgfqpoint{1.867835in}{0.991037in}}%
\pgfpathlineto{\pgfqpoint{1.881313in}{0.991037in}}%
\pgfpathlineto{\pgfqpoint{1.901530in}{0.875653in}}%
\pgfpathlineto{\pgfqpoint{1.915009in}{0.837191in}}%
\pgfpathlineto{\pgfqpoint{1.941965in}{0.837191in}}%
\pgfpathlineto{\pgfqpoint{1.948704in}{0.817960in}}%
\pgfpathlineto{\pgfqpoint{1.975661in}{0.817960in}}%
\pgfpathlineto{\pgfqpoint{1.982400in}{0.760268in}}%
\pgfpathlineto{\pgfqpoint{1.989139in}{0.760268in}}%
\pgfpathlineto{\pgfqpoint{1.995878in}{0.683345in}}%
\pgfpathlineto{\pgfqpoint{2.009356in}{0.683345in}}%
\pgfpathlineto{\pgfqpoint{2.016095in}{0.644883in}}%
\pgfpathlineto{\pgfqpoint{2.022835in}{0.625653in}}%
\pgfpathlineto{\pgfqpoint{2.076748in}{0.625653in}}%
\pgfpathlineto{\pgfqpoint{2.083487in}{0.606422in}}%
\pgfpathlineto{\pgfqpoint{2.177835in}{0.606422in}}%
\pgfpathlineto{\pgfqpoint{2.184574in}{0.587191in}}%
\pgfpathlineto{\pgfqpoint{2.568704in}{0.587191in}}%
\pgfpathlineto{\pgfqpoint{2.568704in}{0.587191in}}%
\pgfusepath{stroke}%
\end{pgfscope}%
\begin{pgfscope}%
\pgfpathrectangle{\pgfqpoint{0.553704in}{0.499691in}}{\pgfqpoint{2.015000in}{1.925000in}}%
\pgfusepath{clip}%
\pgfsetrectcap%
\pgfsetroundjoin%
\pgfsetlinewidth{1.505625pt}%
\definecolor{currentstroke}{rgb}{0.031373,0.188235,0.419608}%
\pgfsetstrokecolor{currentstroke}%
\pgfsetdash{}{0pt}%
\pgfpathmoveto{\pgfqpoint{0.553704in}{2.317960in}}%
\pgfpathlineto{\pgfqpoint{0.560443in}{2.298730in}}%
\pgfpathlineto{\pgfqpoint{0.789574in}{2.298730in}}%
\pgfpathlineto{\pgfqpoint{0.796313in}{2.279499in}}%
\pgfpathlineto{\pgfqpoint{0.890661in}{2.279499in}}%
\pgfpathlineto{\pgfqpoint{0.897400in}{2.260268in}}%
\pgfpathlineto{\pgfqpoint{0.964791in}{2.260268in}}%
\pgfpathlineto{\pgfqpoint{0.971530in}{2.241037in}}%
\pgfpathlineto{\pgfqpoint{1.005226in}{2.241037in}}%
\pgfpathlineto{\pgfqpoint{1.011965in}{2.221806in}}%
\pgfpathlineto{\pgfqpoint{1.025443in}{2.221806in}}%
\pgfpathlineto{\pgfqpoint{1.038922in}{2.183345in}}%
\pgfpathlineto{\pgfqpoint{1.052400in}{2.183345in}}%
\pgfpathlineto{\pgfqpoint{1.065878in}{2.144883in}}%
\pgfpathlineto{\pgfqpoint{1.099574in}{2.144883in}}%
\pgfpathlineto{\pgfqpoint{1.106313in}{2.125653in}}%
\pgfpathlineto{\pgfqpoint{1.119791in}{2.125653in}}%
\pgfpathlineto{\pgfqpoint{1.133269in}{2.087191in}}%
\pgfpathlineto{\pgfqpoint{1.268052in}{2.087191in}}%
\pgfpathlineto{\pgfqpoint{1.274791in}{2.067960in}}%
\pgfpathlineto{\pgfqpoint{1.301748in}{2.067960in}}%
\pgfpathlineto{\pgfqpoint{1.308487in}{2.029499in}}%
\pgfpathlineto{\pgfqpoint{1.362400in}{2.029499in}}%
\pgfpathlineto{\pgfqpoint{1.369139in}{2.010268in}}%
\pgfpathlineto{\pgfqpoint{1.375878in}{1.971806in}}%
\pgfpathlineto{\pgfqpoint{1.382617in}{1.952576in}}%
\pgfpathlineto{\pgfqpoint{1.402835in}{1.952576in}}%
\pgfpathlineto{\pgfqpoint{1.409574in}{1.914114in}}%
\pgfpathlineto{\pgfqpoint{1.429791in}{1.914114in}}%
\pgfpathlineto{\pgfqpoint{1.436530in}{1.856422in}}%
\pgfpathlineto{\pgfqpoint{1.456748in}{1.798730in}}%
\pgfpathlineto{\pgfqpoint{1.483704in}{1.798730in}}%
\pgfpathlineto{\pgfqpoint{1.490443in}{1.779499in}}%
\pgfpathlineto{\pgfqpoint{1.503922in}{1.779499in}}%
\pgfpathlineto{\pgfqpoint{1.510661in}{1.760268in}}%
\pgfpathlineto{\pgfqpoint{1.517400in}{1.721806in}}%
\pgfpathlineto{\pgfqpoint{1.530878in}{1.721806in}}%
\pgfpathlineto{\pgfqpoint{1.537617in}{1.664114in}}%
\pgfpathlineto{\pgfqpoint{1.544356in}{1.644883in}}%
\pgfpathlineto{\pgfqpoint{1.564574in}{1.644883in}}%
\pgfpathlineto{\pgfqpoint{1.584791in}{1.587191in}}%
\pgfpathlineto{\pgfqpoint{1.598269in}{1.587191in}}%
\pgfpathlineto{\pgfqpoint{1.605009in}{1.548730in}}%
\pgfpathlineto{\pgfqpoint{1.618487in}{1.510268in}}%
\pgfpathlineto{\pgfqpoint{1.625226in}{1.471806in}}%
\pgfpathlineto{\pgfqpoint{1.631965in}{1.452576in}}%
\pgfpathlineto{\pgfqpoint{1.638704in}{1.452576in}}%
\pgfpathlineto{\pgfqpoint{1.645443in}{1.433345in}}%
\pgfpathlineto{\pgfqpoint{1.652182in}{1.433345in}}%
\pgfpathlineto{\pgfqpoint{1.672400in}{1.317960in}}%
\pgfpathlineto{\pgfqpoint{1.706095in}{1.317960in}}%
\pgfpathlineto{\pgfqpoint{1.712835in}{1.279499in}}%
\pgfpathlineto{\pgfqpoint{1.733052in}{1.221806in}}%
\pgfpathlineto{\pgfqpoint{1.739791in}{1.164114in}}%
\pgfpathlineto{\pgfqpoint{1.753269in}{1.164114in}}%
\pgfpathlineto{\pgfqpoint{1.760009in}{1.144883in}}%
\pgfpathlineto{\pgfqpoint{1.766748in}{1.106422in}}%
\pgfpathlineto{\pgfqpoint{1.773487in}{1.087191in}}%
\pgfpathlineto{\pgfqpoint{1.786965in}{1.087191in}}%
\pgfpathlineto{\pgfqpoint{1.800443in}{1.010268in}}%
\pgfpathlineto{\pgfqpoint{1.813922in}{1.010268in}}%
\pgfpathlineto{\pgfqpoint{1.820661in}{0.971806in}}%
\pgfpathlineto{\pgfqpoint{1.827400in}{0.971806in}}%
\pgfpathlineto{\pgfqpoint{1.840878in}{0.933345in}}%
\pgfpathlineto{\pgfqpoint{1.847617in}{0.933345in}}%
\pgfpathlineto{\pgfqpoint{1.854356in}{0.914114in}}%
\pgfpathlineto{\pgfqpoint{1.874574in}{0.914114in}}%
\pgfpathlineto{\pgfqpoint{1.881313in}{0.894883in}}%
\pgfpathlineto{\pgfqpoint{1.894791in}{0.894883in}}%
\pgfpathlineto{\pgfqpoint{1.901530in}{0.875653in}}%
\pgfpathlineto{\pgfqpoint{1.908269in}{0.875653in}}%
\pgfpathlineto{\pgfqpoint{1.915009in}{0.837191in}}%
\pgfpathlineto{\pgfqpoint{1.921748in}{0.837191in}}%
\pgfpathlineto{\pgfqpoint{1.941965in}{0.779499in}}%
\pgfpathlineto{\pgfqpoint{1.948704in}{0.779499in}}%
\pgfpathlineto{\pgfqpoint{1.955443in}{0.760268in}}%
\pgfpathlineto{\pgfqpoint{1.975661in}{0.644883in}}%
\pgfpathlineto{\pgfqpoint{1.982400in}{0.644883in}}%
\pgfpathlineto{\pgfqpoint{1.989139in}{0.625653in}}%
\pgfpathlineto{\pgfqpoint{2.056530in}{0.625653in}}%
\pgfpathlineto{\pgfqpoint{2.063269in}{0.606422in}}%
\pgfpathlineto{\pgfqpoint{2.157617in}{0.606422in}}%
\pgfpathlineto{\pgfqpoint{2.164356in}{0.587191in}}%
\pgfpathlineto{\pgfqpoint{2.568704in}{0.587191in}}%
\pgfpathlineto{\pgfqpoint{2.568704in}{0.587191in}}%
\pgfusepath{stroke}%
\end{pgfscope}%
\begin{pgfscope}%
\pgfsetrectcap%
\pgfsetmiterjoin%
\pgfsetlinewidth{0.803000pt}%
\definecolor{currentstroke}{rgb}{0.000000,0.000000,0.000000}%
\pgfsetstrokecolor{currentstroke}%
\pgfsetdash{}{0pt}%
\pgfpathmoveto{\pgfqpoint{0.553704in}{0.499691in}}%
\pgfpathlineto{\pgfqpoint{0.553704in}{2.424691in}}%
\pgfusepath{stroke}%
\end{pgfscope}%
\begin{pgfscope}%
\pgfsetrectcap%
\pgfsetmiterjoin%
\pgfsetlinewidth{0.803000pt}%
\definecolor{currentstroke}{rgb}{0.000000,0.000000,0.000000}%
\pgfsetstrokecolor{currentstroke}%
\pgfsetdash{}{0pt}%
\pgfpathmoveto{\pgfqpoint{2.568704in}{0.499691in}}%
\pgfpathlineto{\pgfqpoint{2.568704in}{2.424691in}}%
\pgfusepath{stroke}%
\end{pgfscope}%
\begin{pgfscope}%
\pgfsetrectcap%
\pgfsetmiterjoin%
\pgfsetlinewidth{0.803000pt}%
\definecolor{currentstroke}{rgb}{0.000000,0.000000,0.000000}%
\pgfsetstrokecolor{currentstroke}%
\pgfsetdash{}{0pt}%
\pgfpathmoveto{\pgfqpoint{0.553704in}{0.499691in}}%
\pgfpathlineto{\pgfqpoint{2.568704in}{0.499691in}}%
\pgfusepath{stroke}%
\end{pgfscope}%
\begin{pgfscope}%
\pgfsetrectcap%
\pgfsetmiterjoin%
\pgfsetlinewidth{0.803000pt}%
\definecolor{currentstroke}{rgb}{0.000000,0.000000,0.000000}%
\pgfsetstrokecolor{currentstroke}%
\pgfsetdash{}{0pt}%
\pgfpathmoveto{\pgfqpoint{0.553704in}{2.424691in}}%
\pgfpathlineto{\pgfqpoint{2.568704in}{2.424691in}}%
\pgfusepath{stroke}%
\end{pgfscope}%
\begin{pgfscope}%
\pgfsetbuttcap%
\pgfsetmiterjoin%
\definecolor{currentfill}{rgb}{1.000000,1.000000,1.000000}%
\pgfsetfillcolor{currentfill}%
\pgfsetfillopacity{0.800000}%
\pgfsetlinewidth{1.003750pt}%
\definecolor{currentstroke}{rgb}{0.800000,0.800000,0.800000}%
\pgfsetstrokecolor{currentstroke}%
\pgfsetstrokeopacity{0.800000}%
\pgfsetdash{}{0pt}%
\pgfpathmoveto{\pgfqpoint{0.650926in}{0.569136in}}%
\pgfpathlineto{\pgfqpoint{1.707607in}{0.569136in}}%
\pgfpathquadraticcurveto{\pgfqpoint{1.735385in}{0.569136in}}{\pgfqpoint{1.735385in}{0.596913in}}%
\pgfpathlineto{\pgfqpoint{1.735385in}{1.550617in}}%
\pgfpathquadraticcurveto{\pgfqpoint{1.735385in}{1.578394in}}{\pgfqpoint{1.707607in}{1.578394in}}%
\pgfpathlineto{\pgfqpoint{0.650926in}{1.578394in}}%
\pgfpathquadraticcurveto{\pgfqpoint{0.623149in}{1.578394in}}{\pgfqpoint{0.623149in}{1.550617in}}%
\pgfpathlineto{\pgfqpoint{0.623149in}{0.596913in}}%
\pgfpathquadraticcurveto{\pgfqpoint{0.623149in}{0.569136in}}{\pgfqpoint{0.650926in}{0.569136in}}%
\pgfpathlineto{\pgfqpoint{0.650926in}{0.569136in}}%
\pgfpathclose%
\pgfusepath{stroke,fill}%
\end{pgfscope}%
\begin{pgfscope}%
\definecolor{textcolor}{rgb}{0.000000,0.000000,0.000000}%
\pgfsetstrokecolor{textcolor}%
\pgfsetfillcolor{textcolor}%
\pgftext[x=0.904578in,y=1.426388in,left,base]{\color{textcolor}\rmfamily\fontsize{10.000000}{12.000000}\selectfont \(\displaystyle \textsc{ANCER}\)}%
\end{pgfscope}%
\begin{pgfscope}%
\pgfsetrectcap%
\pgfsetroundjoin%
\pgfsetlinewidth{1.505625pt}%
\definecolor{currentstroke}{rgb}{0.716186,0.833203,0.916155}%
\pgfsetstrokecolor{currentstroke}%
\pgfsetdash{}{0pt}%
\pgfpathmoveto{\pgfqpoint{0.678704in}{1.281327in}}%
\pgfpathlineto{\pgfqpoint{0.817593in}{1.281327in}}%
\pgfpathlineto{\pgfqpoint{0.956482in}{1.281327in}}%
\pgfusepath{stroke}%
\end{pgfscope}%
\begin{pgfscope}%
\definecolor{textcolor}{rgb}{0.000000,0.000000,0.000000}%
\pgfsetstrokecolor{textcolor}%
\pgfsetfillcolor{textcolor}%
\pgftext[x=1.067593in,y=1.232716in,left,base]{\color{textcolor}\rmfamily\fontsize{10.000000}{12.000000}\selectfont \(\displaystyle lr=0.03\)}%
\end{pgfscope}%
\begin{pgfscope}%
\pgfsetrectcap%
\pgfsetroundjoin%
\pgfsetlinewidth{1.505625pt}%
\definecolor{currentstroke}{rgb}{0.376732,0.653072,0.822484}%
\pgfsetstrokecolor{currentstroke}%
\pgfsetdash{}{0pt}%
\pgfpathmoveto{\pgfqpoint{0.678704in}{1.087654in}}%
\pgfpathlineto{\pgfqpoint{0.817593in}{1.087654in}}%
\pgfpathlineto{\pgfqpoint{0.956482in}{1.087654in}}%
\pgfusepath{stroke}%
\end{pgfscope}%
\begin{pgfscope}%
\definecolor{textcolor}{rgb}{0.000000,0.000000,0.000000}%
\pgfsetstrokecolor{textcolor}%
\pgfsetfillcolor{textcolor}%
\pgftext[x=1.067593in,y=1.039043in,left,base]{\color{textcolor}\rmfamily\fontsize{10.000000}{12.000000}\selectfont \(\displaystyle lr=0.01\)}%
\end{pgfscope}%
\begin{pgfscope}%
\pgfsetrectcap%
\pgfsetroundjoin%
\pgfsetlinewidth{1.505625pt}%
\definecolor{currentstroke}{rgb}{0.114802,0.424437,0.695194}%
\pgfsetstrokecolor{currentstroke}%
\pgfsetdash{}{0pt}%
\pgfpathmoveto{\pgfqpoint{0.678704in}{0.893981in}}%
\pgfpathlineto{\pgfqpoint{0.817593in}{0.893981in}}%
\pgfpathlineto{\pgfqpoint{0.956482in}{0.893981in}}%
\pgfusepath{stroke}%
\end{pgfscope}%
\begin{pgfscope}%
\definecolor{textcolor}{rgb}{0.000000,0.000000,0.000000}%
\pgfsetstrokecolor{textcolor}%
\pgfsetfillcolor{textcolor}%
\pgftext[x=1.067593in,y=0.845370in,left,base]{\color{textcolor}\rmfamily\fontsize{10.000000}{12.000000}\selectfont \(\displaystyle lr=0.003\)}%
\end{pgfscope}%
\begin{pgfscope}%
\pgfsetrectcap%
\pgfsetroundjoin%
\pgfsetlinewidth{1.505625pt}%
\definecolor{currentstroke}{rgb}{0.031373,0.188235,0.419608}%
\pgfsetstrokecolor{currentstroke}%
\pgfsetdash{}{0pt}%
\pgfpathmoveto{\pgfqpoint{0.678704in}{0.700308in}}%
\pgfpathlineto{\pgfqpoint{0.817593in}{0.700308in}}%
\pgfpathlineto{\pgfqpoint{0.956482in}{0.700308in}}%
\pgfusepath{stroke}%
\end{pgfscope}%
\begin{pgfscope}%
\definecolor{textcolor}{rgb}{0.000000,0.000000,0.000000}%
\pgfsetstrokecolor{textcolor}%
\pgfsetfillcolor{textcolor}%
\pgftext[x=1.067593in,y=0.651697in,left,base]{\color{textcolor}\rmfamily\fontsize{10.000000}{12.000000}\selectfont \(\displaystyle lr=0.001\)}%
\end{pgfscope}%
\end{pgfpicture}%
\makeatother%
\endgroup%

%% file: figs/svhn_rs4al1sweep.pgf
%% Creator: Matplotlib, PGF backend
%%
%% To include the figure in your LaTeX document, write
%%   \input{<filename>.pgf}
%%
%% Make sure the required packages are loaded in your preamble
%%   \usepackage{pgf}
%%
%% Also ensure that all the required font packages are loaded; for instance,
%% the lmodern package is sometimes necessary when using math font.
%%   \usepackage{lmodern}
%%
%% Figures using additional raster images can only be included by \input if
%% they are in the same directory as the main LaTeX file. For loading figures
%% from other directories you can use the `import` package
%%   \usepackage{import}
%%
%% and then include the figures with
%%   \import{<path to file>}{<filename>.pgf}
%%
%% Matplotlib used the following preamble
%%
\begingroup%
\makeatletter%
\begin{pgfpicture}%
\pgfpathrectangle{\pgfpointorigin}{\pgfqpoint{2.668704in}{2.524691in}}%
\pgfusepath{use as bounding box, clip}%
\begin{pgfscope}%
\pgfsetbuttcap%
\pgfsetmiterjoin%
\pgfsetlinewidth{0.000000pt}%
\definecolor{currentstroke}{rgb}{0.000000,0.000000,0.000000}%
\pgfsetstrokecolor{currentstroke}%
\pgfsetstrokeopacity{0.000000}%
\pgfsetdash{}{0pt}%
\pgfpathmoveto{\pgfqpoint{0.000000in}{0.000000in}}%
\pgfpathlineto{\pgfqpoint{2.668704in}{0.000000in}}%
\pgfpathlineto{\pgfqpoint{2.668704in}{2.524691in}}%
\pgfpathlineto{\pgfqpoint{0.000000in}{2.524691in}}%
\pgfpathlineto{\pgfqpoint{0.000000in}{0.000000in}}%
\pgfpathclose%
\pgfusepath{}%
\end{pgfscope}%
\begin{pgfscope}%
\pgfsetbuttcap%
\pgfsetmiterjoin%
\pgfsetlinewidth{0.000000pt}%
\definecolor{currentstroke}{rgb}{0.000000,0.000000,0.000000}%
\pgfsetstrokecolor{currentstroke}%
\pgfsetstrokeopacity{0.000000}%
\pgfsetdash{}{0pt}%
\pgfpathmoveto{\pgfqpoint{0.553704in}{0.499691in}}%
\pgfpathlineto{\pgfqpoint{2.568704in}{0.499691in}}%
\pgfpathlineto{\pgfqpoint{2.568704in}{2.424691in}}%
\pgfpathlineto{\pgfqpoint{0.553704in}{2.424691in}}%
\pgfpathlineto{\pgfqpoint{0.553704in}{0.499691in}}%
\pgfpathclose%
\pgfusepath{}%
\end{pgfscope}%
\begin{pgfscope}%
\pgfsetbuttcap%
\pgfsetroundjoin%
\definecolor{currentfill}{rgb}{0.000000,0.000000,0.000000}%
\pgfsetfillcolor{currentfill}%
\pgfsetlinewidth{0.803000pt}%
\definecolor{currentstroke}{rgb}{0.000000,0.000000,0.000000}%
\pgfsetstrokecolor{currentstroke}%
\pgfsetdash{}{0pt}%
\pgfsys@defobject{currentmarker}{\pgfqpoint{0.000000in}{-0.048611in}}{\pgfqpoint{0.000000in}{0.000000in}}{%
\pgfpathmoveto{\pgfqpoint{0.000000in}{0.000000in}}%
\pgfpathlineto{\pgfqpoint{0.000000in}{-0.048611in}}%
\pgfusepath{stroke,fill}%
}%
\begin{pgfscope}%
\pgfsys@transformshift{1.020448in}{0.499691in}%
\pgfsys@useobject{currentmarker}{}%
\end{pgfscope}%
\end{pgfscope}%
\begin{pgfscope}%
\definecolor{textcolor}{rgb}{0.000000,0.000000,0.000000}%
\pgfsetstrokecolor{textcolor}%
\pgfsetfillcolor{textcolor}%
\pgftext[x=1.020448in,y=0.402469in,,top]{\color{textcolor}\rmfamily\fontsize{10.000000}{12.000000}\selectfont \(\displaystyle 10^{-4 \alpha}\)}%
\end{pgfscope}%
\begin{pgfscope}%
\pgfsetbuttcap%
\pgfsetroundjoin%
\definecolor{currentfill}{rgb}{0.000000,0.000000,0.000000}%
\pgfsetfillcolor{currentfill}%
\pgfsetlinewidth{0.803000pt}%
\definecolor{currentstroke}{rgb}{0.000000,0.000000,0.000000}%
\pgfsetstrokecolor{currentstroke}%
\pgfsetdash{}{0pt}%
\pgfsys@defobject{currentmarker}{\pgfqpoint{0.000000in}{-0.048611in}}{\pgfqpoint{0.000000in}{0.000000in}}{%
\pgfpathmoveto{\pgfqpoint{0.000000in}{0.000000in}}%
\pgfpathlineto{\pgfqpoint{0.000000in}{-0.048611in}}%
\pgfusepath{stroke,fill}%
}%
\begin{pgfscope}%
\pgfsys@transformshift{1.897164in}{0.499691in}%
\pgfsys@useobject{currentmarker}{}%
\end{pgfscope}%
\end{pgfscope}%
\begin{pgfscope}%
\definecolor{textcolor}{rgb}{0.000000,0.000000,0.000000}%
\pgfsetstrokecolor{textcolor}%
\pgfsetfillcolor{textcolor}%
\pgftext[x=1.897164in,y=0.402469in,,top]{\color{textcolor}\rmfamily\fontsize{10.000000}{12.000000}\selectfont \(\displaystyle 10^{-3 \alpha}\)}%
\end{pgfscope}%
\begin{pgfscope}%
\definecolor{textcolor}{rgb}{0.000000,0.000000,0.000000}%
\pgfsetstrokecolor{textcolor}%
\pgfsetfillcolor{textcolor}%
\pgftext[x=1.561204in,y=0.223457in,,top]{\color{textcolor}\rmfamily\fontsize{10.000000}{12.000000}\selectfont Volume}%
\end{pgfscope}%
\begin{pgfscope}%
\pgfsetbuttcap%
\pgfsetroundjoin%
\definecolor{currentfill}{rgb}{0.000000,0.000000,0.000000}%
\pgfsetfillcolor{currentfill}%
\pgfsetlinewidth{0.803000pt}%
\definecolor{currentstroke}{rgb}{0.000000,0.000000,0.000000}%
\pgfsetstrokecolor{currentstroke}%
\pgfsetdash{}{0pt}%
\pgfsys@defobject{currentmarker}{\pgfqpoint{-0.048611in}{0.000000in}}{\pgfqpoint{-0.000000in}{0.000000in}}{%
\pgfpathmoveto{\pgfqpoint{-0.000000in}{0.000000in}}%
\pgfpathlineto{\pgfqpoint{-0.048611in}{0.000000in}}%
\pgfusepath{stroke,fill}%
}%
\begin{pgfscope}%
\pgfsys@transformshift{0.553704in}{0.587191in}%
\pgfsys@useobject{currentmarker}{}%
\end{pgfscope}%
\end{pgfscope}%
\begin{pgfscope}%
\definecolor{textcolor}{rgb}{0.000000,0.000000,0.000000}%
\pgfsetstrokecolor{textcolor}%
\pgfsetfillcolor{textcolor}%
\pgftext[x=0.279012in, y=0.538966in, left, base]{\color{textcolor}\rmfamily\fontsize{10.000000}{12.000000}\selectfont \(\displaystyle {0.0}\)}%
\end{pgfscope}%
\begin{pgfscope}%
\pgfsetbuttcap%
\pgfsetroundjoin%
\definecolor{currentfill}{rgb}{0.000000,0.000000,0.000000}%
\pgfsetfillcolor{currentfill}%
\pgfsetlinewidth{0.803000pt}%
\definecolor{currentstroke}{rgb}{0.000000,0.000000,0.000000}%
\pgfsetstrokecolor{currentstroke}%
\pgfsetdash{}{0pt}%
\pgfsys@defobject{currentmarker}{\pgfqpoint{-0.048611in}{0.000000in}}{\pgfqpoint{-0.000000in}{0.000000in}}{%
\pgfpathmoveto{\pgfqpoint{-0.000000in}{0.000000in}}%
\pgfpathlineto{\pgfqpoint{-0.048611in}{0.000000in}}%
\pgfusepath{stroke,fill}%
}%
\begin{pgfscope}%
\pgfsys@transformshift{0.553704in}{0.965161in}%
\pgfsys@useobject{currentmarker}{}%
\end{pgfscope}%
\end{pgfscope}%
\begin{pgfscope}%
\definecolor{textcolor}{rgb}{0.000000,0.000000,0.000000}%
\pgfsetstrokecolor{textcolor}%
\pgfsetfillcolor{textcolor}%
\pgftext[x=0.279012in, y=0.916936in, left, base]{\color{textcolor}\rmfamily\fontsize{10.000000}{12.000000}\selectfont \(\displaystyle {0.2}\)}%
\end{pgfscope}%
\begin{pgfscope}%
\pgfsetbuttcap%
\pgfsetroundjoin%
\definecolor{currentfill}{rgb}{0.000000,0.000000,0.000000}%
\pgfsetfillcolor{currentfill}%
\pgfsetlinewidth{0.803000pt}%
\definecolor{currentstroke}{rgb}{0.000000,0.000000,0.000000}%
\pgfsetstrokecolor{currentstroke}%
\pgfsetdash{}{0pt}%
\pgfsys@defobject{currentmarker}{\pgfqpoint{-0.048611in}{0.000000in}}{\pgfqpoint{-0.000000in}{0.000000in}}{%
\pgfpathmoveto{\pgfqpoint{-0.000000in}{0.000000in}}%
\pgfpathlineto{\pgfqpoint{-0.048611in}{0.000000in}}%
\pgfusepath{stroke,fill}%
}%
\begin{pgfscope}%
\pgfsys@transformshift{0.553704in}{1.343131in}%
\pgfsys@useobject{currentmarker}{}%
\end{pgfscope}%
\end{pgfscope}%
\begin{pgfscope}%
\definecolor{textcolor}{rgb}{0.000000,0.000000,0.000000}%
\pgfsetstrokecolor{textcolor}%
\pgfsetfillcolor{textcolor}%
\pgftext[x=0.279012in, y=1.294905in, left, base]{\color{textcolor}\rmfamily\fontsize{10.000000}{12.000000}\selectfont \(\displaystyle {0.4}\)}%
\end{pgfscope}%
\begin{pgfscope}%
\pgfsetbuttcap%
\pgfsetroundjoin%
\definecolor{currentfill}{rgb}{0.000000,0.000000,0.000000}%
\pgfsetfillcolor{currentfill}%
\pgfsetlinewidth{0.803000pt}%
\definecolor{currentstroke}{rgb}{0.000000,0.000000,0.000000}%
\pgfsetstrokecolor{currentstroke}%
\pgfsetdash{}{0pt}%
\pgfsys@defobject{currentmarker}{\pgfqpoint{-0.048611in}{0.000000in}}{\pgfqpoint{-0.000000in}{0.000000in}}{%
\pgfpathmoveto{\pgfqpoint{-0.000000in}{0.000000in}}%
\pgfpathlineto{\pgfqpoint{-0.048611in}{0.000000in}}%
\pgfusepath{stroke,fill}%
}%
\begin{pgfscope}%
\pgfsys@transformshift{0.553704in}{1.721100in}%
\pgfsys@useobject{currentmarker}{}%
\end{pgfscope}%
\end{pgfscope}%
\begin{pgfscope}%
\definecolor{textcolor}{rgb}{0.000000,0.000000,0.000000}%
\pgfsetstrokecolor{textcolor}%
\pgfsetfillcolor{textcolor}%
\pgftext[x=0.279012in, y=1.672875in, left, base]{\color{textcolor}\rmfamily\fontsize{10.000000}{12.000000}\selectfont \(\displaystyle {0.6}\)}%
\end{pgfscope}%
\begin{pgfscope}%
\pgfsetbuttcap%
\pgfsetroundjoin%
\definecolor{currentfill}{rgb}{0.000000,0.000000,0.000000}%
\pgfsetfillcolor{currentfill}%
\pgfsetlinewidth{0.803000pt}%
\definecolor{currentstroke}{rgb}{0.000000,0.000000,0.000000}%
\pgfsetstrokecolor{currentstroke}%
\pgfsetdash{}{0pt}%
\pgfsys@defobject{currentmarker}{\pgfqpoint{-0.048611in}{0.000000in}}{\pgfqpoint{-0.000000in}{0.000000in}}{%
\pgfpathmoveto{\pgfqpoint{-0.000000in}{0.000000in}}%
\pgfpathlineto{\pgfqpoint{-0.048611in}{0.000000in}}%
\pgfusepath{stroke,fill}%
}%
\begin{pgfscope}%
\pgfsys@transformshift{0.553704in}{2.099070in}%
\pgfsys@useobject{currentmarker}{}%
\end{pgfscope}%
\end{pgfscope}%
\begin{pgfscope}%
\definecolor{textcolor}{rgb}{0.000000,0.000000,0.000000}%
\pgfsetstrokecolor{textcolor}%
\pgfsetfillcolor{textcolor}%
\pgftext[x=0.279012in, y=2.050845in, left, base]{\color{textcolor}\rmfamily\fontsize{10.000000}{12.000000}\selectfont \(\displaystyle {0.8}\)}%
\end{pgfscope}%
\begin{pgfscope}%
\definecolor{textcolor}{rgb}{0.000000,0.000000,0.000000}%
\pgfsetstrokecolor{textcolor}%
\pgfsetfillcolor{textcolor}%
\pgftext[x=0.223457in,y=1.462191in,,bottom,rotate=90.000000]{\color{textcolor}\rmfamily\fontsize{10.000000}{12.000000}\selectfont Certified accuracy}%
\end{pgfscope}%
\begin{pgfscope}%
\pgfpathrectangle{\pgfqpoint{0.553704in}{0.499691in}}{\pgfqpoint{2.015000in}{1.925000in}}%
\pgfusepath{clip}%
\pgfsetrectcap%
\pgfsetroundjoin%
\pgfsetlinewidth{1.505625pt}%
\definecolor{currentstroke}{rgb}{0.716186,0.833203,0.916155}%
\pgfsetstrokecolor{currentstroke}%
\pgfsetdash{}{0pt}%
\pgfpathmoveto{\pgfqpoint{0.553704in}{2.337191in}}%
\pgfpathlineto{\pgfqpoint{0.910878in}{2.337191in}}%
\pgfpathlineto{\pgfqpoint{0.917617in}{2.333411in}}%
\pgfpathlineto{\pgfqpoint{1.126530in}{2.333411in}}%
\pgfpathlineto{\pgfqpoint{1.133269in}{2.329632in}}%
\pgfpathlineto{\pgfqpoint{1.153487in}{2.329632in}}%
\pgfpathlineto{\pgfqpoint{1.160226in}{2.325852in}}%
\pgfpathlineto{\pgfqpoint{1.200661in}{2.325852in}}%
\pgfpathlineto{\pgfqpoint{1.207400in}{2.322072in}}%
\pgfpathlineto{\pgfqpoint{1.247835in}{2.322072in}}%
\pgfpathlineto{\pgfqpoint{1.254574in}{2.318293in}}%
\pgfpathlineto{\pgfqpoint{1.443269in}{2.318293in}}%
\pgfpathlineto{\pgfqpoint{1.450009in}{2.314513in}}%
\pgfpathlineto{\pgfqpoint{1.557835in}{2.314513in}}%
\pgfpathlineto{\pgfqpoint{1.571313in}{2.306954in}}%
\pgfpathlineto{\pgfqpoint{1.625226in}{2.306954in}}%
\pgfpathlineto{\pgfqpoint{1.631965in}{2.303174in}}%
\pgfpathlineto{\pgfqpoint{1.679139in}{2.303174in}}%
\pgfpathlineto{\pgfqpoint{1.685878in}{2.299394in}}%
\pgfpathlineto{\pgfqpoint{1.793704in}{2.299394in}}%
\pgfpathlineto{\pgfqpoint{1.800443in}{2.295614in}}%
\pgfpathlineto{\pgfqpoint{1.827400in}{2.295614in}}%
\pgfpathlineto{\pgfqpoint{1.834139in}{2.291835in}}%
\pgfpathlineto{\pgfqpoint{1.861095in}{2.291835in}}%
\pgfpathlineto{\pgfqpoint{1.867835in}{2.288055in}}%
\pgfpathlineto{\pgfqpoint{1.888052in}{2.288055in}}%
\pgfpathlineto{\pgfqpoint{1.894791in}{2.276716in}}%
\pgfpathlineto{\pgfqpoint{1.908269in}{2.276716in}}%
\pgfpathlineto{\pgfqpoint{1.915009in}{2.272936in}}%
\pgfpathlineto{\pgfqpoint{1.928487in}{2.272936in}}%
\pgfpathlineto{\pgfqpoint{1.935226in}{2.269157in}}%
\pgfpathlineto{\pgfqpoint{1.948704in}{2.269157in}}%
\pgfpathlineto{\pgfqpoint{1.955443in}{2.265377in}}%
\pgfpathlineto{\pgfqpoint{1.968922in}{2.250258in}}%
\pgfpathlineto{\pgfqpoint{1.975661in}{2.250258in}}%
\pgfpathlineto{\pgfqpoint{1.982400in}{2.242699in}}%
\pgfpathlineto{\pgfqpoint{1.989139in}{2.223800in}}%
\pgfpathlineto{\pgfqpoint{2.002617in}{2.216241in}}%
\pgfpathlineto{\pgfqpoint{2.009356in}{2.216241in}}%
\pgfpathlineto{\pgfqpoint{2.016095in}{2.212461in}}%
\pgfpathlineto{\pgfqpoint{2.022835in}{2.204902in}}%
\pgfpathlineto{\pgfqpoint{2.036313in}{2.197342in}}%
\pgfpathlineto{\pgfqpoint{2.049791in}{2.182224in}}%
\pgfpathlineto{\pgfqpoint{2.056530in}{2.170884in}}%
\pgfpathlineto{\pgfqpoint{2.063269in}{2.155766in}}%
\pgfpathlineto{\pgfqpoint{2.070009in}{2.144427in}}%
\pgfpathlineto{\pgfqpoint{2.076748in}{2.114189in}}%
\pgfpathlineto{\pgfqpoint{2.083487in}{2.099070in}}%
\pgfpathlineto{\pgfqpoint{2.090226in}{2.034815in}}%
\pgfpathlineto{\pgfqpoint{2.096965in}{1.944103in}}%
\pgfpathlineto{\pgfqpoint{2.103704in}{0.587191in}}%
\pgfpathlineto{\pgfqpoint{2.568704in}{0.587191in}}%
\pgfpathlineto{\pgfqpoint{2.568704in}{0.587191in}}%
\pgfusepath{stroke}%
\end{pgfscope}%
\begin{pgfscope}%
\pgfpathrectangle{\pgfqpoint{0.553704in}{0.499691in}}{\pgfqpoint{2.015000in}{1.925000in}}%
\pgfusepath{clip}%
\pgfsetrectcap%
\pgfsetroundjoin%
\pgfsetlinewidth{1.505625pt}%
\definecolor{currentstroke}{rgb}{0.376732,0.653072,0.822484}%
\pgfsetstrokecolor{currentstroke}%
\pgfsetdash{}{0pt}%
\pgfpathmoveto{\pgfqpoint{0.553704in}{2.238919in}}%
\pgfpathlineto{\pgfqpoint{0.762617in}{2.238919in}}%
\pgfpathlineto{\pgfqpoint{0.769356in}{2.235139in}}%
\pgfpathlineto{\pgfqpoint{0.823269in}{2.235139in}}%
\pgfpathlineto{\pgfqpoint{0.830009in}{2.231360in}}%
\pgfpathlineto{\pgfqpoint{1.146748in}{2.231360in}}%
\pgfpathlineto{\pgfqpoint{1.153487in}{2.227580in}}%
\pgfpathlineto{\pgfqpoint{1.321965in}{2.227580in}}%
\pgfpathlineto{\pgfqpoint{1.328704in}{2.223800in}}%
\pgfpathlineto{\pgfqpoint{1.382617in}{2.223800in}}%
\pgfpathlineto{\pgfqpoint{1.389356in}{2.220020in}}%
\pgfpathlineto{\pgfqpoint{1.571313in}{2.220020in}}%
\pgfpathlineto{\pgfqpoint{1.578052in}{2.216241in}}%
\pgfpathlineto{\pgfqpoint{1.591530in}{2.216241in}}%
\pgfpathlineto{\pgfqpoint{1.598269in}{2.212461in}}%
\pgfpathlineto{\pgfqpoint{1.625226in}{2.212461in}}%
\pgfpathlineto{\pgfqpoint{1.631965in}{2.208681in}}%
\pgfpathlineto{\pgfqpoint{1.658922in}{2.208681in}}%
\pgfpathlineto{\pgfqpoint{1.665661in}{2.204902in}}%
\pgfpathlineto{\pgfqpoint{1.706095in}{2.204902in}}%
\pgfpathlineto{\pgfqpoint{1.712835in}{2.201122in}}%
\pgfpathlineto{\pgfqpoint{1.719574in}{2.201122in}}%
\pgfpathlineto{\pgfqpoint{1.726313in}{2.197342in}}%
\pgfpathlineto{\pgfqpoint{1.733052in}{2.197342in}}%
\pgfpathlineto{\pgfqpoint{1.739791in}{2.193563in}}%
\pgfpathlineto{\pgfqpoint{1.753269in}{2.193563in}}%
\pgfpathlineto{\pgfqpoint{1.760009in}{2.189783in}}%
\pgfpathlineto{\pgfqpoint{1.807182in}{2.189783in}}%
\pgfpathlineto{\pgfqpoint{1.820661in}{2.182224in}}%
\pgfpathlineto{\pgfqpoint{1.834139in}{2.182224in}}%
\pgfpathlineto{\pgfqpoint{1.840878in}{2.178444in}}%
\pgfpathlineto{\pgfqpoint{1.888052in}{2.178444in}}%
\pgfpathlineto{\pgfqpoint{1.894791in}{2.174664in}}%
\pgfpathlineto{\pgfqpoint{1.915009in}{2.174664in}}%
\pgfpathlineto{\pgfqpoint{1.921748in}{2.170884in}}%
\pgfpathlineto{\pgfqpoint{1.928487in}{2.170884in}}%
\pgfpathlineto{\pgfqpoint{1.935226in}{2.167105in}}%
\pgfpathlineto{\pgfqpoint{1.948704in}{2.167105in}}%
\pgfpathlineto{\pgfqpoint{1.962182in}{2.159545in}}%
\pgfpathlineto{\pgfqpoint{1.982400in}{2.159545in}}%
\pgfpathlineto{\pgfqpoint{1.989139in}{2.151986in}}%
\pgfpathlineto{\pgfqpoint{1.995878in}{2.151986in}}%
\pgfpathlineto{\pgfqpoint{2.002617in}{2.148206in}}%
\pgfpathlineto{\pgfqpoint{2.016095in}{2.148206in}}%
\pgfpathlineto{\pgfqpoint{2.022835in}{2.136867in}}%
\pgfpathlineto{\pgfqpoint{2.029574in}{2.133087in}}%
\pgfpathlineto{\pgfqpoint{2.049791in}{2.133087in}}%
\pgfpathlineto{\pgfqpoint{2.070009in}{2.121748in}}%
\pgfpathlineto{\pgfqpoint{2.076748in}{2.110409in}}%
\pgfpathlineto{\pgfqpoint{2.083487in}{2.110409in}}%
\pgfpathlineto{\pgfqpoint{2.090226in}{2.106630in}}%
\pgfpathlineto{\pgfqpoint{2.096965in}{2.099070in}}%
\pgfpathlineto{\pgfqpoint{2.110443in}{2.099070in}}%
\pgfpathlineto{\pgfqpoint{2.117182in}{2.091511in}}%
\pgfpathlineto{\pgfqpoint{2.130661in}{2.083951in}}%
\pgfpathlineto{\pgfqpoint{2.137400in}{2.083951in}}%
\pgfpathlineto{\pgfqpoint{2.144139in}{2.080172in}}%
\pgfpathlineto{\pgfqpoint{2.157617in}{2.065053in}}%
\pgfpathlineto{\pgfqpoint{2.164356in}{2.061273in}}%
\pgfpathlineto{\pgfqpoint{2.177835in}{2.061273in}}%
\pgfpathlineto{\pgfqpoint{2.184574in}{2.057493in}}%
\pgfpathlineto{\pgfqpoint{2.204791in}{2.012137in}}%
\pgfpathlineto{\pgfqpoint{2.218269in}{1.997018in}}%
\pgfpathlineto{\pgfqpoint{2.225009in}{1.993239in}}%
\pgfpathlineto{\pgfqpoint{2.231748in}{1.974340in}}%
\pgfpathlineto{\pgfqpoint{2.238487in}{1.966781in}}%
\pgfpathlineto{\pgfqpoint{2.245226in}{1.947882in}}%
\pgfpathlineto{\pgfqpoint{2.251965in}{1.925204in}}%
\pgfpathlineto{\pgfqpoint{2.258704in}{1.910085in}}%
\pgfpathlineto{\pgfqpoint{2.265443in}{1.883627in}}%
\pgfpathlineto{\pgfqpoint{2.272182in}{1.864729in}}%
\pgfpathlineto{\pgfqpoint{2.278922in}{1.860949in}}%
\pgfpathlineto{\pgfqpoint{2.285661in}{1.842051in}}%
\pgfpathlineto{\pgfqpoint{2.292400in}{1.834491in}}%
\pgfpathlineto{\pgfqpoint{2.299139in}{1.800474in}}%
\pgfpathlineto{\pgfqpoint{2.305878in}{1.774016in}}%
\pgfpathlineto{\pgfqpoint{2.319356in}{1.664405in}}%
\pgfpathlineto{\pgfqpoint{2.326095in}{1.528336in}}%
\pgfpathlineto{\pgfqpoint{2.332835in}{1.331792in}}%
\pgfpathlineto{\pgfqpoint{2.339574in}{0.587191in}}%
\pgfpathlineto{\pgfqpoint{2.568704in}{0.587191in}}%
\pgfpathlineto{\pgfqpoint{2.568704in}{0.587191in}}%
\pgfusepath{stroke}%
\end{pgfscope}%
\begin{pgfscope}%
\pgfpathrectangle{\pgfqpoint{0.553704in}{0.499691in}}{\pgfqpoint{2.015000in}{1.925000in}}%
\pgfusepath{clip}%
\pgfsetrectcap%
\pgfsetroundjoin%
\pgfsetlinewidth{1.505625pt}%
\definecolor{currentstroke}{rgb}{0.114802,0.424437,0.695194}%
\pgfsetstrokecolor{currentstroke}%
\pgfsetdash{}{0pt}%
\pgfpathmoveto{\pgfqpoint{0.553704in}{2.227580in}}%
\pgfpathlineto{\pgfqpoint{1.416313in}{2.227580in}}%
\pgfpathlineto{\pgfqpoint{1.423052in}{2.223800in}}%
\pgfpathlineto{\pgfqpoint{1.429791in}{2.223800in}}%
\pgfpathlineto{\pgfqpoint{1.436530in}{2.220020in}}%
\pgfpathlineto{\pgfqpoint{1.456748in}{2.220020in}}%
\pgfpathlineto{\pgfqpoint{1.463487in}{2.216241in}}%
\pgfpathlineto{\pgfqpoint{1.470226in}{2.216241in}}%
\pgfpathlineto{\pgfqpoint{1.476965in}{2.212461in}}%
\pgfpathlineto{\pgfqpoint{1.510661in}{2.212461in}}%
\pgfpathlineto{\pgfqpoint{1.517400in}{2.208681in}}%
\pgfpathlineto{\pgfqpoint{1.544356in}{2.208681in}}%
\pgfpathlineto{\pgfqpoint{1.551095in}{2.204902in}}%
\pgfpathlineto{\pgfqpoint{1.571313in}{2.204902in}}%
\pgfpathlineto{\pgfqpoint{1.578052in}{2.201122in}}%
\pgfpathlineto{\pgfqpoint{1.584791in}{2.201122in}}%
\pgfpathlineto{\pgfqpoint{1.591530in}{2.197342in}}%
\pgfpathlineto{\pgfqpoint{1.605009in}{2.197342in}}%
\pgfpathlineto{\pgfqpoint{1.611748in}{2.193563in}}%
\pgfpathlineto{\pgfqpoint{1.699356in}{2.193563in}}%
\pgfpathlineto{\pgfqpoint{1.706095in}{2.189783in}}%
\pgfpathlineto{\pgfqpoint{1.719574in}{2.189783in}}%
\pgfpathlineto{\pgfqpoint{1.726313in}{2.186003in}}%
\pgfpathlineto{\pgfqpoint{1.786965in}{2.186003in}}%
\pgfpathlineto{\pgfqpoint{1.793704in}{2.182224in}}%
\pgfpathlineto{\pgfqpoint{1.813922in}{2.182224in}}%
\pgfpathlineto{\pgfqpoint{1.820661in}{2.178444in}}%
\pgfpathlineto{\pgfqpoint{1.827400in}{2.178444in}}%
\pgfpathlineto{\pgfqpoint{1.840878in}{2.170884in}}%
\pgfpathlineto{\pgfqpoint{1.854356in}{2.170884in}}%
\pgfpathlineto{\pgfqpoint{1.861095in}{2.167105in}}%
\pgfpathlineto{\pgfqpoint{1.881313in}{2.167105in}}%
\pgfpathlineto{\pgfqpoint{1.888052in}{2.163325in}}%
\pgfpathlineto{\pgfqpoint{1.908269in}{2.163325in}}%
\pgfpathlineto{\pgfqpoint{1.915009in}{2.159545in}}%
\pgfpathlineto{\pgfqpoint{1.948704in}{2.159545in}}%
\pgfpathlineto{\pgfqpoint{1.955443in}{2.155766in}}%
\pgfpathlineto{\pgfqpoint{1.968922in}{2.155766in}}%
\pgfpathlineto{\pgfqpoint{1.975661in}{2.148206in}}%
\pgfpathlineto{\pgfqpoint{1.982400in}{2.148206in}}%
\pgfpathlineto{\pgfqpoint{1.989139in}{2.144427in}}%
\pgfpathlineto{\pgfqpoint{2.009356in}{2.144427in}}%
\pgfpathlineto{\pgfqpoint{2.016095in}{2.140647in}}%
\pgfpathlineto{\pgfqpoint{2.022835in}{2.133087in}}%
\pgfpathlineto{\pgfqpoint{2.036313in}{2.125528in}}%
\pgfpathlineto{\pgfqpoint{2.043052in}{2.114189in}}%
\pgfpathlineto{\pgfqpoint{2.049791in}{2.106630in}}%
\pgfpathlineto{\pgfqpoint{2.056530in}{2.106630in}}%
\pgfpathlineto{\pgfqpoint{2.070009in}{2.099070in}}%
\pgfpathlineto{\pgfqpoint{2.076748in}{2.083951in}}%
\pgfpathlineto{\pgfqpoint{2.090226in}{2.083951in}}%
\pgfpathlineto{\pgfqpoint{2.103704in}{2.076392in}}%
\pgfpathlineto{\pgfqpoint{2.110443in}{2.076392in}}%
\pgfpathlineto{\pgfqpoint{2.117182in}{2.072612in}}%
\pgfpathlineto{\pgfqpoint{2.123922in}{2.061273in}}%
\pgfpathlineto{\pgfqpoint{2.130661in}{2.053714in}}%
\pgfpathlineto{\pgfqpoint{2.137400in}{2.053714in}}%
\pgfpathlineto{\pgfqpoint{2.144139in}{2.049934in}}%
\pgfpathlineto{\pgfqpoint{2.150878in}{2.049934in}}%
\pgfpathlineto{\pgfqpoint{2.157617in}{2.046154in}}%
\pgfpathlineto{\pgfqpoint{2.164356in}{2.034815in}}%
\pgfpathlineto{\pgfqpoint{2.177835in}{2.034815in}}%
\pgfpathlineto{\pgfqpoint{2.184574in}{2.031036in}}%
\pgfpathlineto{\pgfqpoint{2.191313in}{2.031036in}}%
\pgfpathlineto{\pgfqpoint{2.198052in}{2.027256in}}%
\pgfpathlineto{\pgfqpoint{2.204791in}{2.019697in}}%
\pgfpathlineto{\pgfqpoint{2.231748in}{2.019697in}}%
\pgfpathlineto{\pgfqpoint{2.238487in}{2.015917in}}%
\pgfpathlineto{\pgfqpoint{2.251965in}{2.000798in}}%
\pgfpathlineto{\pgfqpoint{2.258704in}{1.997018in}}%
\pgfpathlineto{\pgfqpoint{2.272182in}{1.974340in}}%
\pgfpathlineto{\pgfqpoint{2.285661in}{1.959221in}}%
\pgfpathlineto{\pgfqpoint{2.292400in}{1.944103in}}%
\pgfpathlineto{\pgfqpoint{2.299139in}{1.940323in}}%
\pgfpathlineto{\pgfqpoint{2.305878in}{1.932763in}}%
\pgfpathlineto{\pgfqpoint{2.312617in}{1.928984in}}%
\pgfpathlineto{\pgfqpoint{2.319356in}{1.917645in}}%
\pgfpathlineto{\pgfqpoint{2.332835in}{1.902526in}}%
\pgfpathlineto{\pgfqpoint{2.339574in}{1.879848in}}%
\pgfpathlineto{\pgfqpoint{2.346313in}{1.864729in}}%
\pgfpathlineto{\pgfqpoint{2.353052in}{1.845830in}}%
\pgfpathlineto{\pgfqpoint{2.359791in}{1.834491in}}%
\pgfpathlineto{\pgfqpoint{2.366530in}{1.819373in}}%
\pgfpathlineto{\pgfqpoint{2.373269in}{1.800474in}}%
\pgfpathlineto{\pgfqpoint{2.400226in}{1.713541in}}%
\pgfpathlineto{\pgfqpoint{2.413704in}{1.690863in}}%
\pgfpathlineto{\pgfqpoint{2.420443in}{1.641727in}}%
\pgfpathlineto{\pgfqpoint{2.427182in}{1.611489in}}%
\pgfpathlineto{\pgfqpoint{2.440661in}{1.520776in}}%
\pgfpathlineto{\pgfqpoint{2.454139in}{1.377148in}}%
\pgfpathlineto{\pgfqpoint{2.460878in}{1.256198in}}%
\pgfpathlineto{\pgfqpoint{2.467617in}{1.059653in}}%
\pgfpathlineto{\pgfqpoint{2.474356in}{0.587191in}}%
\pgfpathlineto{\pgfqpoint{2.568704in}{0.587191in}}%
\pgfpathlineto{\pgfqpoint{2.568704in}{0.587191in}}%
\pgfusepath{stroke}%
\end{pgfscope}%
\begin{pgfscope}%
\pgfpathrectangle{\pgfqpoint{0.553704in}{0.499691in}}{\pgfqpoint{2.015000in}{1.925000in}}%
\pgfusepath{clip}%
\pgfsetrectcap%
\pgfsetroundjoin%
\pgfsetlinewidth{1.505625pt}%
\definecolor{currentstroke}{rgb}{0.031373,0.188235,0.419608}%
\pgfsetstrokecolor{currentstroke}%
\pgfsetdash{}{0pt}%
\pgfpathmoveto{\pgfqpoint{0.553704in}{2.133087in}}%
\pgfpathlineto{\pgfqpoint{0.910878in}{2.133087in}}%
\pgfpathlineto{\pgfqpoint{0.917617in}{2.129308in}}%
\pgfpathlineto{\pgfqpoint{1.113052in}{2.129308in}}%
\pgfpathlineto{\pgfqpoint{1.119791in}{2.125528in}}%
\pgfpathlineto{\pgfqpoint{1.463487in}{2.125528in}}%
\pgfpathlineto{\pgfqpoint{1.470226in}{2.121748in}}%
\pgfpathlineto{\pgfqpoint{1.571313in}{2.121748in}}%
\pgfpathlineto{\pgfqpoint{1.578052in}{2.114189in}}%
\pgfpathlineto{\pgfqpoint{1.598269in}{2.114189in}}%
\pgfpathlineto{\pgfqpoint{1.611748in}{2.106630in}}%
\pgfpathlineto{\pgfqpoint{1.618487in}{2.106630in}}%
\pgfpathlineto{\pgfqpoint{1.625226in}{2.102850in}}%
\pgfpathlineto{\pgfqpoint{1.652182in}{2.102850in}}%
\pgfpathlineto{\pgfqpoint{1.658922in}{2.099070in}}%
\pgfpathlineto{\pgfqpoint{1.739791in}{2.099070in}}%
\pgfpathlineto{\pgfqpoint{1.746530in}{2.095290in}}%
\pgfpathlineto{\pgfqpoint{1.766748in}{2.095290in}}%
\pgfpathlineto{\pgfqpoint{1.773487in}{2.091511in}}%
\pgfpathlineto{\pgfqpoint{1.780226in}{2.091511in}}%
\pgfpathlineto{\pgfqpoint{1.786965in}{2.087731in}}%
\pgfpathlineto{\pgfqpoint{1.800443in}{2.087731in}}%
\pgfpathlineto{\pgfqpoint{1.813922in}{2.080172in}}%
\pgfpathlineto{\pgfqpoint{1.867835in}{2.080172in}}%
\pgfpathlineto{\pgfqpoint{1.874574in}{2.076392in}}%
\pgfpathlineto{\pgfqpoint{1.908269in}{2.076392in}}%
\pgfpathlineto{\pgfqpoint{1.915009in}{2.068833in}}%
\pgfpathlineto{\pgfqpoint{1.935226in}{2.068833in}}%
\pgfpathlineto{\pgfqpoint{1.941965in}{2.065053in}}%
\pgfpathlineto{\pgfqpoint{2.009356in}{2.065053in}}%
\pgfpathlineto{\pgfqpoint{2.036313in}{2.049934in}}%
\pgfpathlineto{\pgfqpoint{2.049791in}{2.049934in}}%
\pgfpathlineto{\pgfqpoint{2.090226in}{2.027256in}}%
\pgfpathlineto{\pgfqpoint{2.110443in}{2.027256in}}%
\pgfpathlineto{\pgfqpoint{2.117182in}{2.023476in}}%
\pgfpathlineto{\pgfqpoint{2.130661in}{2.023476in}}%
\pgfpathlineto{\pgfqpoint{2.137400in}{2.019697in}}%
\pgfpathlineto{\pgfqpoint{2.144139in}{2.012137in}}%
\pgfpathlineto{\pgfqpoint{2.150878in}{2.008357in}}%
\pgfpathlineto{\pgfqpoint{2.157617in}{2.000798in}}%
\pgfpathlineto{\pgfqpoint{2.164356in}{2.000798in}}%
\pgfpathlineto{\pgfqpoint{2.171095in}{1.989459in}}%
\pgfpathlineto{\pgfqpoint{2.177835in}{1.985679in}}%
\pgfpathlineto{\pgfqpoint{2.184574in}{1.985679in}}%
\pgfpathlineto{\pgfqpoint{2.191313in}{1.981900in}}%
\pgfpathlineto{\pgfqpoint{2.198052in}{1.981900in}}%
\pgfpathlineto{\pgfqpoint{2.204791in}{1.978120in}}%
\pgfpathlineto{\pgfqpoint{2.211530in}{1.966781in}}%
\pgfpathlineto{\pgfqpoint{2.225009in}{1.951662in}}%
\pgfpathlineto{\pgfqpoint{2.231748in}{1.947882in}}%
\pgfpathlineto{\pgfqpoint{2.238487in}{1.932763in}}%
\pgfpathlineto{\pgfqpoint{2.258704in}{1.921424in}}%
\pgfpathlineto{\pgfqpoint{2.265443in}{1.913865in}}%
\pgfpathlineto{\pgfqpoint{2.272182in}{1.902526in}}%
\pgfpathlineto{\pgfqpoint{2.278922in}{1.894966in}}%
\pgfpathlineto{\pgfqpoint{2.285661in}{1.891187in}}%
\pgfpathlineto{\pgfqpoint{2.292400in}{1.891187in}}%
\pgfpathlineto{\pgfqpoint{2.299139in}{1.887407in}}%
\pgfpathlineto{\pgfqpoint{2.305878in}{1.872288in}}%
\pgfpathlineto{\pgfqpoint{2.312617in}{1.864729in}}%
\pgfpathlineto{\pgfqpoint{2.319356in}{1.864729in}}%
\pgfpathlineto{\pgfqpoint{2.326095in}{1.853390in}}%
\pgfpathlineto{\pgfqpoint{2.332835in}{1.845830in}}%
\pgfpathlineto{\pgfqpoint{2.339574in}{1.834491in}}%
\pgfpathlineto{\pgfqpoint{2.346313in}{1.826932in}}%
\pgfpathlineto{\pgfqpoint{2.359791in}{1.819373in}}%
\pgfpathlineto{\pgfqpoint{2.366530in}{1.811813in}}%
\pgfpathlineto{\pgfqpoint{2.373269in}{1.792915in}}%
\pgfpathlineto{\pgfqpoint{2.380009in}{1.785355in}}%
\pgfpathlineto{\pgfqpoint{2.386748in}{1.770236in}}%
\pgfpathlineto{\pgfqpoint{2.393487in}{1.758897in}}%
\pgfpathlineto{\pgfqpoint{2.400226in}{1.751338in}}%
\pgfpathlineto{\pgfqpoint{2.413704in}{1.728660in}}%
\pgfpathlineto{\pgfqpoint{2.420443in}{1.702202in}}%
\pgfpathlineto{\pgfqpoint{2.427182in}{1.702202in}}%
\pgfpathlineto{\pgfqpoint{2.433922in}{1.694643in}}%
\pgfpathlineto{\pgfqpoint{2.447400in}{1.671964in}}%
\pgfpathlineto{\pgfqpoint{2.454139in}{1.653066in}}%
\pgfpathlineto{\pgfqpoint{2.460878in}{1.641727in}}%
\pgfpathlineto{\pgfqpoint{2.467617in}{1.626608in}}%
\pgfpathlineto{\pgfqpoint{2.481095in}{1.588811in}}%
\pgfpathlineto{\pgfqpoint{2.487835in}{1.577472in}}%
\pgfpathlineto{\pgfqpoint{2.494574in}{1.539675in}}%
\pgfpathlineto{\pgfqpoint{2.501313in}{1.516997in}}%
\pgfpathlineto{\pgfqpoint{2.521530in}{1.411165in}}%
\pgfpathlineto{\pgfqpoint{2.528269in}{1.354470in}}%
\pgfpathlineto{\pgfqpoint{2.541748in}{1.256198in}}%
\pgfpathlineto{\pgfqpoint{2.555226in}{1.097450in}}%
\pgfpathlineto{\pgfqpoint{2.561965in}{0.968941in}}%
\pgfpathlineto{\pgfqpoint{2.568704in}{0.587191in}}%
\pgfpathlineto{\pgfqpoint{2.568704in}{0.587191in}}%
\pgfusepath{stroke}%
\end{pgfscope}%
\begin{pgfscope}%
\pgfsetrectcap%
\pgfsetmiterjoin%
\pgfsetlinewidth{0.803000pt}%
\definecolor{currentstroke}{rgb}{0.000000,0.000000,0.000000}%
\pgfsetstrokecolor{currentstroke}%
\pgfsetdash{}{0pt}%
\pgfpathmoveto{\pgfqpoint{0.553704in}{0.499691in}}%
\pgfpathlineto{\pgfqpoint{0.553704in}{2.424691in}}%
\pgfusepath{stroke}%
\end{pgfscope}%
\begin{pgfscope}%
\pgfsetrectcap%
\pgfsetmiterjoin%
\pgfsetlinewidth{0.803000pt}%
\definecolor{currentstroke}{rgb}{0.000000,0.000000,0.000000}%
\pgfsetstrokecolor{currentstroke}%
\pgfsetdash{}{0pt}%
\pgfpathmoveto{\pgfqpoint{2.568704in}{0.499691in}}%
\pgfpathlineto{\pgfqpoint{2.568704in}{2.424691in}}%
\pgfusepath{stroke}%
\end{pgfscope}%
\begin{pgfscope}%
\pgfsetrectcap%
\pgfsetmiterjoin%
\pgfsetlinewidth{0.803000pt}%
\definecolor{currentstroke}{rgb}{0.000000,0.000000,0.000000}%
\pgfsetstrokecolor{currentstroke}%
\pgfsetdash{}{0pt}%
\pgfpathmoveto{\pgfqpoint{0.553704in}{0.499691in}}%
\pgfpathlineto{\pgfqpoint{2.568704in}{0.499691in}}%
\pgfusepath{stroke}%
\end{pgfscope}%
\begin{pgfscope}%
\pgfsetrectcap%
\pgfsetmiterjoin%
\pgfsetlinewidth{0.803000pt}%
\definecolor{currentstroke}{rgb}{0.000000,0.000000,0.000000}%
\pgfsetstrokecolor{currentstroke}%
\pgfsetdash{}{0pt}%
\pgfpathmoveto{\pgfqpoint{0.553704in}{2.424691in}}%
\pgfpathlineto{\pgfqpoint{2.568704in}{2.424691in}}%
\pgfusepath{stroke}%
\end{pgfscope}%
\begin{pgfscope}%
\pgfsetbuttcap%
\pgfsetmiterjoin%
\definecolor{currentfill}{rgb}{1.000000,1.000000,1.000000}%
\pgfsetfillcolor{currentfill}%
\pgfsetfillopacity{0.800000}%
\pgfsetlinewidth{1.003750pt}%
\definecolor{currentstroke}{rgb}{0.800000,0.800000,0.800000}%
\pgfsetstrokecolor{currentstroke}%
\pgfsetstrokeopacity{0.800000}%
\pgfsetdash{}{0pt}%
\pgfpathmoveto{\pgfqpoint{0.650926in}{0.569136in}}%
\pgfpathlineto{\pgfqpoint{1.611815in}{0.569136in}}%
\pgfpathquadraticcurveto{\pgfqpoint{1.639593in}{0.569136in}}{\pgfqpoint{1.639593in}{0.596913in}}%
\pgfpathlineto{\pgfqpoint{1.639593in}{1.550617in}}%
\pgfpathquadraticcurveto{\pgfqpoint{1.639593in}{1.578394in}}{\pgfqpoint{1.611815in}{1.578394in}}%
\pgfpathlineto{\pgfqpoint{0.650926in}{1.578394in}}%
\pgfpathquadraticcurveto{\pgfqpoint{0.623149in}{1.578394in}}{\pgfqpoint{0.623149in}{1.550617in}}%
\pgfpathlineto{\pgfqpoint{0.623149in}{0.596913in}}%
\pgfpathquadraticcurveto{\pgfqpoint{0.623149in}{0.569136in}}{\pgfqpoint{0.650926in}{0.569136in}}%
\pgfpathlineto{\pgfqpoint{0.650926in}{0.569136in}}%
\pgfpathclose%
\pgfusepath{stroke,fill}%
\end{pgfscope}%
\begin{pgfscope}%
\definecolor{textcolor}{rgb}{0.000000,0.000000,0.000000}%
\pgfsetstrokecolor{textcolor}%
\pgfsetfillcolor{textcolor}%
\pgftext[x=0.793506in,y=1.426388in,left,base]{\color{textcolor}\rmfamily\fontsize{10.000000}{12.000000}\selectfont \(\displaystyle \textsc{RS4A}-\ell_{1}\)}%
\end{pgfscope}%
\begin{pgfscope}%
\pgfsetrectcap%
\pgfsetroundjoin%
\pgfsetlinewidth{1.505625pt}%
\definecolor{currentstroke}{rgb}{0.716186,0.833203,0.916155}%
\pgfsetstrokecolor{currentstroke}%
\pgfsetdash{}{0pt}%
\pgfpathmoveto{\pgfqpoint{0.678704in}{1.281327in}}%
\pgfpathlineto{\pgfqpoint{0.817593in}{1.281327in}}%
\pgfpathlineto{\pgfqpoint{0.956482in}{1.281327in}}%
\pgfusepath{stroke}%
\end{pgfscope}%
\begin{pgfscope}%
\definecolor{textcolor}{rgb}{0.000000,0.000000,0.000000}%
\pgfsetstrokecolor{textcolor}%
\pgfsetfillcolor{textcolor}%
\pgftext[x=1.067593in,y=1.232716in,left,base]{\color{textcolor}\rmfamily\fontsize{10.000000}{12.000000}\selectfont \(\displaystyle \sigma=0.25\)}%
\end{pgfscope}%
\begin{pgfscope}%
\pgfsetrectcap%
\pgfsetroundjoin%
\pgfsetlinewidth{1.505625pt}%
\definecolor{currentstroke}{rgb}{0.376732,0.653072,0.822484}%
\pgfsetstrokecolor{currentstroke}%
\pgfsetdash{}{0pt}%
\pgfpathmoveto{\pgfqpoint{0.678704in}{1.087654in}}%
\pgfpathlineto{\pgfqpoint{0.817593in}{1.087654in}}%
\pgfpathlineto{\pgfqpoint{0.956482in}{1.087654in}}%
\pgfusepath{stroke}%
\end{pgfscope}%
\begin{pgfscope}%
\definecolor{textcolor}{rgb}{0.000000,0.000000,0.000000}%
\pgfsetstrokecolor{textcolor}%
\pgfsetfillcolor{textcolor}%
\pgftext[x=1.067593in,y=1.039043in,left,base]{\color{textcolor}\rmfamily\fontsize{10.000000}{12.000000}\selectfont \(\displaystyle \sigma=0.50\)}%
\end{pgfscope}%
\begin{pgfscope}%
\pgfsetrectcap%
\pgfsetroundjoin%
\pgfsetlinewidth{1.505625pt}%
\definecolor{currentstroke}{rgb}{0.114802,0.424437,0.695194}%
\pgfsetstrokecolor{currentstroke}%
\pgfsetdash{}{0pt}%
\pgfpathmoveto{\pgfqpoint{0.678704in}{0.893981in}}%
\pgfpathlineto{\pgfqpoint{0.817593in}{0.893981in}}%
\pgfpathlineto{\pgfqpoint{0.956482in}{0.893981in}}%
\pgfusepath{stroke}%
\end{pgfscope}%
\begin{pgfscope}%
\definecolor{textcolor}{rgb}{0.000000,0.000000,0.000000}%
\pgfsetstrokecolor{textcolor}%
\pgfsetfillcolor{textcolor}%
\pgftext[x=1.067593in,y=0.845370in,left,base]{\color{textcolor}\rmfamily\fontsize{10.000000}{12.000000}\selectfont \(\displaystyle \sigma=0.75\)}%
\end{pgfscope}%
\begin{pgfscope}%
\pgfsetrectcap%
\pgfsetroundjoin%
\pgfsetlinewidth{1.505625pt}%
\definecolor{currentstroke}{rgb}{0.031373,0.188235,0.419608}%
\pgfsetstrokecolor{currentstroke}%
\pgfsetdash{}{0pt}%
\pgfpathmoveto{\pgfqpoint{0.678704in}{0.700308in}}%
\pgfpathlineto{\pgfqpoint{0.817593in}{0.700308in}}%
\pgfpathlineto{\pgfqpoint{0.956482in}{0.700308in}}%
\pgfusepath{stroke}%
\end{pgfscope}%
\begin{pgfscope}%
\definecolor{textcolor}{rgb}{0.000000,0.000000,0.000000}%
\pgfsetstrokecolor{textcolor}%
\pgfsetfillcolor{textcolor}%
\pgftext[x=1.067593in,y=0.651697in,left,base]{\color{textcolor}\rmfamily\fontsize{10.000000}{12.000000}\selectfont \(\displaystyle \sigma=1.00\)}%
\end{pgfscope}%
\end{pgfpicture}%
\makeatother%
\endgroup%

%% file: figs/svhn_rs4alinfsweep.pgf
%% Creator: Matplotlib, PGF backend
%%
%% To include the figure in your LaTeX document, write
%%   \input{<filename>.pgf}
%%
%% Make sure the required packages are loaded in your preamble
%%   \usepackage{pgf}
%%
%% Also ensure that all the required font packages are loaded; for instance,
%% the lmodern package is sometimes necessary when using math font.
%%   \usepackage{lmodern}
%%
%% Figures using additional raster images can only be included by \input if
%% they are in the same directory as the main LaTeX file. For loading figures
%% from other directories you can use the `import` package
%%   \usepackage{import}
%%
%% and then include the figures with
%%   \import{<path to file>}{<filename>.pgf}
%%
%% Matplotlib used the following preamble
%%
\begingroup%
\makeatletter%
\begin{pgfpicture}%
\pgfpathrectangle{\pgfpointorigin}{\pgfqpoint{2.668704in}{2.524691in}}%
\pgfusepath{use as bounding box, clip}%
\begin{pgfscope}%
\pgfsetbuttcap%
\pgfsetmiterjoin%
\pgfsetlinewidth{0.000000pt}%
\definecolor{currentstroke}{rgb}{0.000000,0.000000,0.000000}%
\pgfsetstrokecolor{currentstroke}%
\pgfsetstrokeopacity{0.000000}%
\pgfsetdash{}{0pt}%
\pgfpathmoveto{\pgfqpoint{0.000000in}{0.000000in}}%
\pgfpathlineto{\pgfqpoint{2.668704in}{0.000000in}}%
\pgfpathlineto{\pgfqpoint{2.668704in}{2.524691in}}%
\pgfpathlineto{\pgfqpoint{0.000000in}{2.524691in}}%
\pgfpathlineto{\pgfqpoint{0.000000in}{0.000000in}}%
\pgfpathclose%
\pgfusepath{}%
\end{pgfscope}%
\begin{pgfscope}%
\pgfsetbuttcap%
\pgfsetmiterjoin%
\pgfsetlinewidth{0.000000pt}%
\definecolor{currentstroke}{rgb}{0.000000,0.000000,0.000000}%
\pgfsetstrokecolor{currentstroke}%
\pgfsetstrokeopacity{0.000000}%
\pgfsetdash{}{0pt}%
\pgfpathmoveto{\pgfqpoint{0.553704in}{0.499691in}}%
\pgfpathlineto{\pgfqpoint{2.568704in}{0.499691in}}%
\pgfpathlineto{\pgfqpoint{2.568704in}{2.424691in}}%
\pgfpathlineto{\pgfqpoint{0.553704in}{2.424691in}}%
\pgfpathlineto{\pgfqpoint{0.553704in}{0.499691in}}%
\pgfpathclose%
\pgfusepath{}%
\end{pgfscope}%
\begin{pgfscope}%
\pgfsetbuttcap%
\pgfsetroundjoin%
\definecolor{currentfill}{rgb}{0.000000,0.000000,0.000000}%
\pgfsetfillcolor{currentfill}%
\pgfsetlinewidth{0.803000pt}%
\definecolor{currentstroke}{rgb}{0.000000,0.000000,0.000000}%
\pgfsetstrokecolor{currentstroke}%
\pgfsetdash{}{0pt}%
\pgfsys@defobject{currentmarker}{\pgfqpoint{0.000000in}{-0.048611in}}{\pgfqpoint{0.000000in}{0.000000in}}{%
\pgfpathmoveto{\pgfqpoint{0.000000in}{0.000000in}}%
\pgfpathlineto{\pgfqpoint{0.000000in}{-0.048611in}}%
\pgfusepath{stroke,fill}%
}%
\begin{pgfscope}%
\pgfsys@transformshift{0.557697in}{0.499691in}%
\pgfsys@useobject{currentmarker}{}%
\end{pgfscope}%
\end{pgfscope}%
\begin{pgfscope}%
\definecolor{textcolor}{rgb}{0.000000,0.000000,0.000000}%
\pgfsetstrokecolor{textcolor}%
\pgfsetfillcolor{textcolor}%
\pgftext[x=0.557697in,y=0.402469in,,top]{\color{textcolor}\rmfamily\fontsize{10.000000}{12.000000}\selectfont \(\displaystyle 10^{-4 \alpha}\)}%
\end{pgfscope}%
\begin{pgfscope}%
\pgfsetbuttcap%
\pgfsetroundjoin%
\definecolor{currentfill}{rgb}{0.000000,0.000000,0.000000}%
\pgfsetfillcolor{currentfill}%
\pgfsetlinewidth{0.803000pt}%
\definecolor{currentstroke}{rgb}{0.000000,0.000000,0.000000}%
\pgfsetstrokecolor{currentstroke}%
\pgfsetdash{}{0pt}%
\pgfsys@defobject{currentmarker}{\pgfqpoint{0.000000in}{-0.048611in}}{\pgfqpoint{0.000000in}{0.000000in}}{%
\pgfpathmoveto{\pgfqpoint{0.000000in}{0.000000in}}%
\pgfpathlineto{\pgfqpoint{0.000000in}{-0.048611in}}%
\pgfusepath{stroke,fill}%
}%
\begin{pgfscope}%
\pgfsys@transformshift{1.351141in}{0.499691in}%
\pgfsys@useobject{currentmarker}{}%
\end{pgfscope}%
\end{pgfscope}%
\begin{pgfscope}%
\definecolor{textcolor}{rgb}{0.000000,0.000000,0.000000}%
\pgfsetstrokecolor{textcolor}%
\pgfsetfillcolor{textcolor}%
\pgftext[x=1.351141in,y=0.402469in,,top]{\color{textcolor}\rmfamily\fontsize{10.000000}{12.000000}\selectfont \(\displaystyle 10^{-3 \alpha}\)}%
\end{pgfscope}%
\begin{pgfscope}%
\pgfsetbuttcap%
\pgfsetroundjoin%
\definecolor{currentfill}{rgb}{0.000000,0.000000,0.000000}%
\pgfsetfillcolor{currentfill}%
\pgfsetlinewidth{0.803000pt}%
\definecolor{currentstroke}{rgb}{0.000000,0.000000,0.000000}%
\pgfsetstrokecolor{currentstroke}%
\pgfsetdash{}{0pt}%
\pgfsys@defobject{currentmarker}{\pgfqpoint{0.000000in}{-0.048611in}}{\pgfqpoint{0.000000in}{0.000000in}}{%
\pgfpathmoveto{\pgfqpoint{0.000000in}{0.000000in}}%
\pgfpathlineto{\pgfqpoint{0.000000in}{-0.048611in}}%
\pgfusepath{stroke,fill}%
}%
\begin{pgfscope}%
\pgfsys@transformshift{2.144585in}{0.499691in}%
\pgfsys@useobject{currentmarker}{}%
\end{pgfscope}%
\end{pgfscope}%
\begin{pgfscope}%
\definecolor{textcolor}{rgb}{0.000000,0.000000,0.000000}%
\pgfsetstrokecolor{textcolor}%
\pgfsetfillcolor{textcolor}%
\pgftext[x=2.144585in,y=0.402469in,,top]{\color{textcolor}\rmfamily\fontsize{10.000000}{12.000000}\selectfont \(\displaystyle 10^{-2 \alpha}\)}%
\end{pgfscope}%
\begin{pgfscope}%
\definecolor{textcolor}{rgb}{0.000000,0.000000,0.000000}%
\pgfsetstrokecolor{textcolor}%
\pgfsetfillcolor{textcolor}%
\pgftext[x=1.561204in,y=0.223457in,,top]{\color{textcolor}\rmfamily\fontsize{10.000000}{12.000000}\selectfont Volume}%
\end{pgfscope}%
\begin{pgfscope}%
\pgfsetbuttcap%
\pgfsetroundjoin%
\definecolor{currentfill}{rgb}{0.000000,0.000000,0.000000}%
\pgfsetfillcolor{currentfill}%
\pgfsetlinewidth{0.803000pt}%
\definecolor{currentstroke}{rgb}{0.000000,0.000000,0.000000}%
\pgfsetstrokecolor{currentstroke}%
\pgfsetdash{}{0pt}%
\pgfsys@defobject{currentmarker}{\pgfqpoint{-0.048611in}{0.000000in}}{\pgfqpoint{-0.000000in}{0.000000in}}{%
\pgfpathmoveto{\pgfqpoint{-0.000000in}{0.000000in}}%
\pgfpathlineto{\pgfqpoint{-0.048611in}{0.000000in}}%
\pgfusepath{stroke,fill}%
}%
\begin{pgfscope}%
\pgfsys@transformshift{0.553704in}{0.587191in}%
\pgfsys@useobject{currentmarker}{}%
\end{pgfscope}%
\end{pgfscope}%
\begin{pgfscope}%
\definecolor{textcolor}{rgb}{0.000000,0.000000,0.000000}%
\pgfsetstrokecolor{textcolor}%
\pgfsetfillcolor{textcolor}%
\pgftext[x=0.279012in, y=0.538966in, left, base]{\color{textcolor}\rmfamily\fontsize{10.000000}{12.000000}\selectfont \(\displaystyle {0.0}\)}%
\end{pgfscope}%
\begin{pgfscope}%
\pgfsetbuttcap%
\pgfsetroundjoin%
\definecolor{currentfill}{rgb}{0.000000,0.000000,0.000000}%
\pgfsetfillcolor{currentfill}%
\pgfsetlinewidth{0.803000pt}%
\definecolor{currentstroke}{rgb}{0.000000,0.000000,0.000000}%
\pgfsetstrokecolor{currentstroke}%
\pgfsetdash{}{0pt}%
\pgfsys@defobject{currentmarker}{\pgfqpoint{-0.048611in}{0.000000in}}{\pgfqpoint{-0.000000in}{0.000000in}}{%
\pgfpathmoveto{\pgfqpoint{-0.000000in}{0.000000in}}%
\pgfpathlineto{\pgfqpoint{-0.048611in}{0.000000in}}%
\pgfusepath{stroke,fill}%
}%
\begin{pgfscope}%
\pgfsys@transformshift{0.553704in}{0.966801in}%
\pgfsys@useobject{currentmarker}{}%
\end{pgfscope}%
\end{pgfscope}%
\begin{pgfscope}%
\definecolor{textcolor}{rgb}{0.000000,0.000000,0.000000}%
\pgfsetstrokecolor{textcolor}%
\pgfsetfillcolor{textcolor}%
\pgftext[x=0.279012in, y=0.918575in, left, base]{\color{textcolor}\rmfamily\fontsize{10.000000}{12.000000}\selectfont \(\displaystyle {0.2}\)}%
\end{pgfscope}%
\begin{pgfscope}%
\pgfsetbuttcap%
\pgfsetroundjoin%
\definecolor{currentfill}{rgb}{0.000000,0.000000,0.000000}%
\pgfsetfillcolor{currentfill}%
\pgfsetlinewidth{0.803000pt}%
\definecolor{currentstroke}{rgb}{0.000000,0.000000,0.000000}%
\pgfsetstrokecolor{currentstroke}%
\pgfsetdash{}{0pt}%
\pgfsys@defobject{currentmarker}{\pgfqpoint{-0.048611in}{0.000000in}}{\pgfqpoint{-0.000000in}{0.000000in}}{%
\pgfpathmoveto{\pgfqpoint{-0.000000in}{0.000000in}}%
\pgfpathlineto{\pgfqpoint{-0.048611in}{0.000000in}}%
\pgfusepath{stroke,fill}%
}%
\begin{pgfscope}%
\pgfsys@transformshift{0.553704in}{1.346410in}%
\pgfsys@useobject{currentmarker}{}%
\end{pgfscope}%
\end{pgfscope}%
\begin{pgfscope}%
\definecolor{textcolor}{rgb}{0.000000,0.000000,0.000000}%
\pgfsetstrokecolor{textcolor}%
\pgfsetfillcolor{textcolor}%
\pgftext[x=0.279012in, y=1.298185in, left, base]{\color{textcolor}\rmfamily\fontsize{10.000000}{12.000000}\selectfont \(\displaystyle {0.4}\)}%
\end{pgfscope}%
\begin{pgfscope}%
\pgfsetbuttcap%
\pgfsetroundjoin%
\definecolor{currentfill}{rgb}{0.000000,0.000000,0.000000}%
\pgfsetfillcolor{currentfill}%
\pgfsetlinewidth{0.803000pt}%
\definecolor{currentstroke}{rgb}{0.000000,0.000000,0.000000}%
\pgfsetstrokecolor{currentstroke}%
\pgfsetdash{}{0pt}%
\pgfsys@defobject{currentmarker}{\pgfqpoint{-0.048611in}{0.000000in}}{\pgfqpoint{-0.000000in}{0.000000in}}{%
\pgfpathmoveto{\pgfqpoint{-0.000000in}{0.000000in}}%
\pgfpathlineto{\pgfqpoint{-0.048611in}{0.000000in}}%
\pgfusepath{stroke,fill}%
}%
\begin{pgfscope}%
\pgfsys@transformshift{0.553704in}{1.726020in}%
\pgfsys@useobject{currentmarker}{}%
\end{pgfscope}%
\end{pgfscope}%
\begin{pgfscope}%
\definecolor{textcolor}{rgb}{0.000000,0.000000,0.000000}%
\pgfsetstrokecolor{textcolor}%
\pgfsetfillcolor{textcolor}%
\pgftext[x=0.279012in, y=1.677794in, left, base]{\color{textcolor}\rmfamily\fontsize{10.000000}{12.000000}\selectfont \(\displaystyle {0.6}\)}%
\end{pgfscope}%
\begin{pgfscope}%
\pgfsetbuttcap%
\pgfsetroundjoin%
\definecolor{currentfill}{rgb}{0.000000,0.000000,0.000000}%
\pgfsetfillcolor{currentfill}%
\pgfsetlinewidth{0.803000pt}%
\definecolor{currentstroke}{rgb}{0.000000,0.000000,0.000000}%
\pgfsetstrokecolor{currentstroke}%
\pgfsetdash{}{0pt}%
\pgfsys@defobject{currentmarker}{\pgfqpoint{-0.048611in}{0.000000in}}{\pgfqpoint{-0.000000in}{0.000000in}}{%
\pgfpathmoveto{\pgfqpoint{-0.000000in}{0.000000in}}%
\pgfpathlineto{\pgfqpoint{-0.048611in}{0.000000in}}%
\pgfusepath{stroke,fill}%
}%
\begin{pgfscope}%
\pgfsys@transformshift{0.553704in}{2.105629in}%
\pgfsys@useobject{currentmarker}{}%
\end{pgfscope}%
\end{pgfscope}%
\begin{pgfscope}%
\definecolor{textcolor}{rgb}{0.000000,0.000000,0.000000}%
\pgfsetstrokecolor{textcolor}%
\pgfsetfillcolor{textcolor}%
\pgftext[x=0.279012in, y=2.057404in, left, base]{\color{textcolor}\rmfamily\fontsize{10.000000}{12.000000}\selectfont \(\displaystyle {0.8}\)}%
\end{pgfscope}%
\begin{pgfscope}%
\definecolor{textcolor}{rgb}{0.000000,0.000000,0.000000}%
\pgfsetstrokecolor{textcolor}%
\pgfsetfillcolor{textcolor}%
\pgftext[x=0.223457in,y=1.462191in,,bottom,rotate=90.000000]{\color{textcolor}\rmfamily\fontsize{10.000000}{12.000000}\selectfont Certified accuracy}%
\end{pgfscope}%
\begin{pgfscope}%
\pgfpathrectangle{\pgfqpoint{0.553704in}{0.499691in}}{\pgfqpoint{2.015000in}{1.925000in}}%
\pgfusepath{clip}%
\pgfsetrectcap%
\pgfsetroundjoin%
\pgfsetlinewidth{1.505625pt}%
\definecolor{currentstroke}{rgb}{0.716186,0.833203,0.916155}%
\pgfsetstrokecolor{currentstroke}%
\pgfsetdash{}{0pt}%
\pgfpathmoveto{\pgfqpoint{0.553704in}{2.337191in}}%
\pgfpathlineto{\pgfqpoint{0.991748in}{2.337191in}}%
\pgfpathlineto{\pgfqpoint{0.998487in}{2.333395in}}%
\pgfpathlineto{\pgfqpoint{1.153487in}{2.333395in}}%
\pgfpathlineto{\pgfqpoint{1.160226in}{2.329599in}}%
\pgfpathlineto{\pgfqpoint{1.193922in}{2.329599in}}%
\pgfpathlineto{\pgfqpoint{1.200661in}{2.325803in}}%
\pgfpathlineto{\pgfqpoint{1.268052in}{2.325803in}}%
\pgfpathlineto{\pgfqpoint{1.281530in}{2.318211in}}%
\pgfpathlineto{\pgfqpoint{1.301748in}{2.318211in}}%
\pgfpathlineto{\pgfqpoint{1.308487in}{2.314415in}}%
\pgfpathlineto{\pgfqpoint{1.369139in}{2.314415in}}%
\pgfpathlineto{\pgfqpoint{1.375878in}{2.310618in}}%
\pgfpathlineto{\pgfqpoint{1.490443in}{2.310618in}}%
\pgfpathlineto{\pgfqpoint{1.497182in}{2.306822in}}%
\pgfpathlineto{\pgfqpoint{1.510661in}{2.306822in}}%
\pgfpathlineto{\pgfqpoint{1.517400in}{2.303026in}}%
\pgfpathlineto{\pgfqpoint{1.564574in}{2.303026in}}%
\pgfpathlineto{\pgfqpoint{1.571313in}{2.299230in}}%
\pgfpathlineto{\pgfqpoint{1.578052in}{2.287842in}}%
\pgfpathlineto{\pgfqpoint{1.631965in}{2.287842in}}%
\pgfpathlineto{\pgfqpoint{1.638704in}{2.284046in}}%
\pgfpathlineto{\pgfqpoint{1.699356in}{2.284046in}}%
\pgfpathlineto{\pgfqpoint{1.706095in}{2.276454in}}%
\pgfpathlineto{\pgfqpoint{1.719574in}{2.268861in}}%
\pgfpathlineto{\pgfqpoint{1.733052in}{2.268861in}}%
\pgfpathlineto{\pgfqpoint{1.746530in}{2.261269in}}%
\pgfpathlineto{\pgfqpoint{1.753269in}{2.253677in}}%
\pgfpathlineto{\pgfqpoint{1.760009in}{2.249881in}}%
\pgfpathlineto{\pgfqpoint{1.786965in}{2.249881in}}%
\pgfpathlineto{\pgfqpoint{1.793704in}{2.242289in}}%
\pgfpathlineto{\pgfqpoint{1.807182in}{2.234697in}}%
\pgfpathlineto{\pgfqpoint{1.813922in}{2.234697in}}%
\pgfpathlineto{\pgfqpoint{1.820661in}{2.227104in}}%
\pgfpathlineto{\pgfqpoint{1.827400in}{2.223308in}}%
\pgfpathlineto{\pgfqpoint{1.834139in}{2.223308in}}%
\pgfpathlineto{\pgfqpoint{1.840878in}{2.215716in}}%
\pgfpathlineto{\pgfqpoint{1.847617in}{2.215716in}}%
\pgfpathlineto{\pgfqpoint{1.861095in}{2.208124in}}%
\pgfpathlineto{\pgfqpoint{1.888052in}{2.208124in}}%
\pgfpathlineto{\pgfqpoint{1.894791in}{2.200532in}}%
\pgfpathlineto{\pgfqpoint{1.901530in}{2.196736in}}%
\pgfpathlineto{\pgfqpoint{1.908269in}{2.185347in}}%
\pgfpathlineto{\pgfqpoint{1.915009in}{2.185347in}}%
\pgfpathlineto{\pgfqpoint{1.921748in}{2.181551in}}%
\pgfpathlineto{\pgfqpoint{1.928487in}{2.181551in}}%
\pgfpathlineto{\pgfqpoint{1.948704in}{2.170163in}}%
\pgfpathlineto{\pgfqpoint{1.955443in}{2.162571in}}%
\pgfpathlineto{\pgfqpoint{1.968922in}{2.154979in}}%
\pgfpathlineto{\pgfqpoint{1.975661in}{2.143590in}}%
\pgfpathlineto{\pgfqpoint{1.982400in}{2.120814in}}%
\pgfpathlineto{\pgfqpoint{1.995878in}{2.105629in}}%
\pgfpathlineto{\pgfqpoint{2.002617in}{2.105629in}}%
\pgfpathlineto{\pgfqpoint{2.009356in}{2.090445in}}%
\pgfpathlineto{\pgfqpoint{2.016095in}{2.090445in}}%
\pgfpathlineto{\pgfqpoint{2.022835in}{2.079057in}}%
\pgfpathlineto{\pgfqpoint{2.029574in}{2.079057in}}%
\pgfpathlineto{\pgfqpoint{2.036313in}{2.075261in}}%
\pgfpathlineto{\pgfqpoint{2.043052in}{2.060076in}}%
\pgfpathlineto{\pgfqpoint{2.049791in}{2.052484in}}%
\pgfpathlineto{\pgfqpoint{2.056530in}{2.052484in}}%
\pgfpathlineto{\pgfqpoint{2.063269in}{2.041096in}}%
\pgfpathlineto{\pgfqpoint{2.070009in}{2.041096in}}%
\pgfpathlineto{\pgfqpoint{2.076748in}{2.037300in}}%
\pgfpathlineto{\pgfqpoint{2.083487in}{2.029707in}}%
\pgfpathlineto{\pgfqpoint{2.090226in}{2.025911in}}%
\pgfpathlineto{\pgfqpoint{2.103704in}{1.987950in}}%
\pgfpathlineto{\pgfqpoint{2.110443in}{1.972766in}}%
\pgfpathlineto{\pgfqpoint{2.123922in}{1.949989in}}%
\pgfpathlineto{\pgfqpoint{2.130661in}{1.949989in}}%
\pgfpathlineto{\pgfqpoint{2.137400in}{1.942397in}}%
\pgfpathlineto{\pgfqpoint{2.144139in}{1.919621in}}%
\pgfpathlineto{\pgfqpoint{2.150878in}{1.908232in}}%
\pgfpathlineto{\pgfqpoint{2.157617in}{1.862679in}}%
\pgfpathlineto{\pgfqpoint{2.164356in}{1.855087in}}%
\pgfpathlineto{\pgfqpoint{2.171095in}{1.828514in}}%
\pgfpathlineto{\pgfqpoint{2.177835in}{1.828514in}}%
\pgfpathlineto{\pgfqpoint{2.184574in}{1.775369in}}%
\pgfpathlineto{\pgfqpoint{2.191313in}{1.699447in}}%
\pgfpathlineto{\pgfqpoint{2.198052in}{1.699447in}}%
\pgfpathlineto{\pgfqpoint{2.204791in}{0.587191in}}%
\pgfpathlineto{\pgfqpoint{2.568704in}{0.587191in}}%
\pgfpathlineto{\pgfqpoint{2.568704in}{0.587191in}}%
\pgfusepath{stroke}%
\end{pgfscope}%
\begin{pgfscope}%
\pgfpathrectangle{\pgfqpoint{0.553704in}{0.499691in}}{\pgfqpoint{2.015000in}{1.925000in}}%
\pgfusepath{clip}%
\pgfsetrectcap%
\pgfsetroundjoin%
\pgfsetlinewidth{1.505625pt}%
\definecolor{currentstroke}{rgb}{0.231926,0.545652,0.762614}%
\pgfsetstrokecolor{currentstroke}%
\pgfsetdash{}{0pt}%
\pgfpathmoveto{\pgfqpoint{0.553704in}{2.234697in}}%
\pgfpathlineto{\pgfqpoint{0.648052in}{2.234697in}}%
\pgfpathlineto{\pgfqpoint{0.654791in}{2.230900in}}%
\pgfpathlineto{\pgfqpoint{0.998487in}{2.230900in}}%
\pgfpathlineto{\pgfqpoint{1.005226in}{2.227104in}}%
\pgfpathlineto{\pgfqpoint{1.113052in}{2.227104in}}%
\pgfpathlineto{\pgfqpoint{1.119791in}{2.223308in}}%
\pgfpathlineto{\pgfqpoint{1.173704in}{2.223308in}}%
\pgfpathlineto{\pgfqpoint{1.180443in}{2.219512in}}%
\pgfpathlineto{\pgfqpoint{1.227617in}{2.219512in}}%
\pgfpathlineto{\pgfqpoint{1.234356in}{2.215716in}}%
\pgfpathlineto{\pgfqpoint{1.416313in}{2.215716in}}%
\pgfpathlineto{\pgfqpoint{1.423052in}{2.211920in}}%
\pgfpathlineto{\pgfqpoint{1.443269in}{2.211920in}}%
\pgfpathlineto{\pgfqpoint{1.450009in}{2.208124in}}%
\pgfpathlineto{\pgfqpoint{1.557835in}{2.208124in}}%
\pgfpathlineto{\pgfqpoint{1.564574in}{2.204328in}}%
\pgfpathlineto{\pgfqpoint{1.625226in}{2.204328in}}%
\pgfpathlineto{\pgfqpoint{1.631965in}{2.200532in}}%
\pgfpathlineto{\pgfqpoint{1.645443in}{2.200532in}}%
\pgfpathlineto{\pgfqpoint{1.652182in}{2.196736in}}%
\pgfpathlineto{\pgfqpoint{1.658922in}{2.189143in}}%
\pgfpathlineto{\pgfqpoint{1.685878in}{2.189143in}}%
\pgfpathlineto{\pgfqpoint{1.692617in}{2.181551in}}%
\pgfpathlineto{\pgfqpoint{1.699356in}{2.177755in}}%
\pgfpathlineto{\pgfqpoint{1.719574in}{2.177755in}}%
\pgfpathlineto{\pgfqpoint{1.726313in}{2.170163in}}%
\pgfpathlineto{\pgfqpoint{1.733052in}{2.166367in}}%
\pgfpathlineto{\pgfqpoint{1.739791in}{2.166367in}}%
\pgfpathlineto{\pgfqpoint{1.746530in}{2.162571in}}%
\pgfpathlineto{\pgfqpoint{1.753269in}{2.162571in}}%
\pgfpathlineto{\pgfqpoint{1.766748in}{2.147386in}}%
\pgfpathlineto{\pgfqpoint{1.773487in}{2.143590in}}%
\pgfpathlineto{\pgfqpoint{1.786965in}{2.143590in}}%
\pgfpathlineto{\pgfqpoint{1.793704in}{2.139794in}}%
\pgfpathlineto{\pgfqpoint{1.807182in}{2.139794in}}%
\pgfpathlineto{\pgfqpoint{1.813922in}{2.128406in}}%
\pgfpathlineto{\pgfqpoint{1.820661in}{2.124610in}}%
\pgfpathlineto{\pgfqpoint{1.827400in}{2.113221in}}%
\pgfpathlineto{\pgfqpoint{1.840878in}{2.113221in}}%
\pgfpathlineto{\pgfqpoint{1.847617in}{2.105629in}}%
\pgfpathlineto{\pgfqpoint{1.867835in}{2.105629in}}%
\pgfpathlineto{\pgfqpoint{1.874574in}{2.098037in}}%
\pgfpathlineto{\pgfqpoint{1.894791in}{2.098037in}}%
\pgfpathlineto{\pgfqpoint{1.908269in}{2.082853in}}%
\pgfpathlineto{\pgfqpoint{1.915009in}{2.079057in}}%
\pgfpathlineto{\pgfqpoint{1.921748in}{2.071464in}}%
\pgfpathlineto{\pgfqpoint{1.948704in}{2.071464in}}%
\pgfpathlineto{\pgfqpoint{1.962182in}{2.063872in}}%
\pgfpathlineto{\pgfqpoint{1.968922in}{2.056280in}}%
\pgfpathlineto{\pgfqpoint{1.975661in}{2.052484in}}%
\pgfpathlineto{\pgfqpoint{1.982400in}{2.044892in}}%
\pgfpathlineto{\pgfqpoint{1.989139in}{2.044892in}}%
\pgfpathlineto{\pgfqpoint{1.995878in}{2.037300in}}%
\pgfpathlineto{\pgfqpoint{2.002617in}{2.033503in}}%
\pgfpathlineto{\pgfqpoint{2.009356in}{2.033503in}}%
\pgfpathlineto{\pgfqpoint{2.016095in}{2.029707in}}%
\pgfpathlineto{\pgfqpoint{2.022835in}{2.029707in}}%
\pgfpathlineto{\pgfqpoint{2.029574in}{2.022115in}}%
\pgfpathlineto{\pgfqpoint{2.043052in}{2.014523in}}%
\pgfpathlineto{\pgfqpoint{2.049791in}{2.014523in}}%
\pgfpathlineto{\pgfqpoint{2.056530in}{2.010727in}}%
\pgfpathlineto{\pgfqpoint{2.063269in}{1.999339in}}%
\pgfpathlineto{\pgfqpoint{2.070009in}{1.991746in}}%
\pgfpathlineto{\pgfqpoint{2.076748in}{1.987950in}}%
\pgfpathlineto{\pgfqpoint{2.083487in}{1.980358in}}%
\pgfpathlineto{\pgfqpoint{2.090226in}{1.976562in}}%
\pgfpathlineto{\pgfqpoint{2.096965in}{1.965174in}}%
\pgfpathlineto{\pgfqpoint{2.117182in}{1.942397in}}%
\pgfpathlineto{\pgfqpoint{2.123922in}{1.938601in}}%
\pgfpathlineto{\pgfqpoint{2.130661in}{1.919621in}}%
\pgfpathlineto{\pgfqpoint{2.137400in}{1.904436in}}%
\pgfpathlineto{\pgfqpoint{2.144139in}{1.904436in}}%
\pgfpathlineto{\pgfqpoint{2.150878in}{1.896844in}}%
\pgfpathlineto{\pgfqpoint{2.157617in}{1.885456in}}%
\pgfpathlineto{\pgfqpoint{2.164356in}{1.877864in}}%
\pgfpathlineto{\pgfqpoint{2.171095in}{1.862679in}}%
\pgfpathlineto{\pgfqpoint{2.177835in}{1.855087in}}%
\pgfpathlineto{\pgfqpoint{2.184574in}{1.843699in}}%
\pgfpathlineto{\pgfqpoint{2.191313in}{1.836107in}}%
\pgfpathlineto{\pgfqpoint{2.198052in}{1.820922in}}%
\pgfpathlineto{\pgfqpoint{2.204791in}{1.801942in}}%
\pgfpathlineto{\pgfqpoint{2.211530in}{1.786757in}}%
\pgfpathlineto{\pgfqpoint{2.218269in}{1.763981in}}%
\pgfpathlineto{\pgfqpoint{2.231748in}{1.748796in}}%
\pgfpathlineto{\pgfqpoint{2.238487in}{1.733612in}}%
\pgfpathlineto{\pgfqpoint{2.245226in}{1.707039in}}%
\pgfpathlineto{\pgfqpoint{2.251965in}{1.695651in}}%
\pgfpathlineto{\pgfqpoint{2.258704in}{1.680467in}}%
\pgfpathlineto{\pgfqpoint{2.265443in}{1.669078in}}%
\pgfpathlineto{\pgfqpoint{2.272182in}{1.653894in}}%
\pgfpathlineto{\pgfqpoint{2.278922in}{1.608341in}}%
\pgfpathlineto{\pgfqpoint{2.285661in}{1.593156in}}%
\pgfpathlineto{\pgfqpoint{2.292400in}{1.574176in}}%
\pgfpathlineto{\pgfqpoint{2.299139in}{1.505846in}}%
\pgfpathlineto{\pgfqpoint{2.305878in}{1.498254in}}%
\pgfpathlineto{\pgfqpoint{2.312617in}{1.483070in}}%
\pgfpathlineto{\pgfqpoint{2.319356in}{1.429924in}}%
\pgfpathlineto{\pgfqpoint{2.326095in}{1.391963in}}%
\pgfpathlineto{\pgfqpoint{2.332835in}{1.391963in}}%
\pgfpathlineto{\pgfqpoint{2.339574in}{1.331226in}}%
\pgfpathlineto{\pgfqpoint{2.346313in}{1.232527in}}%
\pgfpathlineto{\pgfqpoint{2.353052in}{1.232527in}}%
\pgfpathlineto{\pgfqpoint{2.359791in}{0.587191in}}%
\pgfpathlineto{\pgfqpoint{2.568704in}{0.587191in}}%
\pgfpathlineto{\pgfqpoint{2.568704in}{0.587191in}}%
\pgfusepath{stroke}%
\end{pgfscope}%
\begin{pgfscope}%
\pgfpathrectangle{\pgfqpoint{0.553704in}{0.499691in}}{\pgfqpoint{2.015000in}{1.925000in}}%
\pgfusepath{clip}%
\pgfsetrectcap%
\pgfsetroundjoin%
\pgfsetlinewidth{1.505625pt}%
\definecolor{currentstroke}{rgb}{0.031373,0.188235,0.419608}%
\pgfsetstrokecolor{currentstroke}%
\pgfsetdash{}{0pt}%
\pgfpathmoveto{\pgfqpoint{0.553704in}{1.991746in}}%
\pgfpathlineto{\pgfqpoint{0.621095in}{1.991746in}}%
\pgfpathlineto{\pgfqpoint{0.627835in}{1.987950in}}%
\pgfpathlineto{\pgfqpoint{0.991748in}{1.987950in}}%
\pgfpathlineto{\pgfqpoint{0.998487in}{1.984154in}}%
\pgfpathlineto{\pgfqpoint{1.038922in}{1.984154in}}%
\pgfpathlineto{\pgfqpoint{1.045661in}{1.980358in}}%
\pgfpathlineto{\pgfqpoint{1.355661in}{1.980358in}}%
\pgfpathlineto{\pgfqpoint{1.362400in}{1.976562in}}%
\pgfpathlineto{\pgfqpoint{1.369139in}{1.976562in}}%
\pgfpathlineto{\pgfqpoint{1.375878in}{1.972766in}}%
\pgfpathlineto{\pgfqpoint{1.389356in}{1.972766in}}%
\pgfpathlineto{\pgfqpoint{1.396095in}{1.968970in}}%
\pgfpathlineto{\pgfqpoint{1.537617in}{1.968970in}}%
\pgfpathlineto{\pgfqpoint{1.551095in}{1.961378in}}%
\pgfpathlineto{\pgfqpoint{1.571313in}{1.961378in}}%
\pgfpathlineto{\pgfqpoint{1.578052in}{1.957582in}}%
\pgfpathlineto{\pgfqpoint{1.605009in}{1.957582in}}%
\pgfpathlineto{\pgfqpoint{1.611748in}{1.949989in}}%
\pgfpathlineto{\pgfqpoint{1.618487in}{1.949989in}}%
\pgfpathlineto{\pgfqpoint{1.625226in}{1.946193in}}%
\pgfpathlineto{\pgfqpoint{1.665661in}{1.946193in}}%
\pgfpathlineto{\pgfqpoint{1.672400in}{1.942397in}}%
\pgfpathlineto{\pgfqpoint{1.699356in}{1.942397in}}%
\pgfpathlineto{\pgfqpoint{1.706095in}{1.938601in}}%
\pgfpathlineto{\pgfqpoint{1.712835in}{1.938601in}}%
\pgfpathlineto{\pgfqpoint{1.719574in}{1.934805in}}%
\pgfpathlineto{\pgfqpoint{1.780226in}{1.934805in}}%
\pgfpathlineto{\pgfqpoint{1.786965in}{1.927213in}}%
\pgfpathlineto{\pgfqpoint{1.800443in}{1.927213in}}%
\pgfpathlineto{\pgfqpoint{1.807182in}{1.919621in}}%
\pgfpathlineto{\pgfqpoint{1.813922in}{1.919621in}}%
\pgfpathlineto{\pgfqpoint{1.820661in}{1.908232in}}%
\pgfpathlineto{\pgfqpoint{1.840878in}{1.896844in}}%
\pgfpathlineto{\pgfqpoint{1.847617in}{1.896844in}}%
\pgfpathlineto{\pgfqpoint{1.874574in}{1.881660in}}%
\pgfpathlineto{\pgfqpoint{1.888052in}{1.881660in}}%
\pgfpathlineto{\pgfqpoint{1.894791in}{1.874067in}}%
\pgfpathlineto{\pgfqpoint{1.901530in}{1.862679in}}%
\pgfpathlineto{\pgfqpoint{1.915009in}{1.855087in}}%
\pgfpathlineto{\pgfqpoint{1.921748in}{1.847495in}}%
\pgfpathlineto{\pgfqpoint{1.928487in}{1.843699in}}%
\pgfpathlineto{\pgfqpoint{1.935226in}{1.843699in}}%
\pgfpathlineto{\pgfqpoint{1.941965in}{1.839903in}}%
\pgfpathlineto{\pgfqpoint{1.948704in}{1.832310in}}%
\pgfpathlineto{\pgfqpoint{1.955443in}{1.832310in}}%
\pgfpathlineto{\pgfqpoint{1.962182in}{1.828514in}}%
\pgfpathlineto{\pgfqpoint{1.968922in}{1.817126in}}%
\pgfpathlineto{\pgfqpoint{1.975661in}{1.817126in}}%
\pgfpathlineto{\pgfqpoint{1.989139in}{1.809534in}}%
\pgfpathlineto{\pgfqpoint{2.002617in}{1.809534in}}%
\pgfpathlineto{\pgfqpoint{2.009356in}{1.805738in}}%
\pgfpathlineto{\pgfqpoint{2.016095in}{1.805738in}}%
\pgfpathlineto{\pgfqpoint{2.022835in}{1.801942in}}%
\pgfpathlineto{\pgfqpoint{2.029574in}{1.790553in}}%
\pgfpathlineto{\pgfqpoint{2.043052in}{1.782961in}}%
\pgfpathlineto{\pgfqpoint{2.049791in}{1.775369in}}%
\pgfpathlineto{\pgfqpoint{2.056530in}{1.763981in}}%
\pgfpathlineto{\pgfqpoint{2.070009in}{1.748796in}}%
\pgfpathlineto{\pgfqpoint{2.083487in}{1.741204in}}%
\pgfpathlineto{\pgfqpoint{2.090226in}{1.733612in}}%
\pgfpathlineto{\pgfqpoint{2.096965in}{1.733612in}}%
\pgfpathlineto{\pgfqpoint{2.110443in}{1.718428in}}%
\pgfpathlineto{\pgfqpoint{2.117182in}{1.714631in}}%
\pgfpathlineto{\pgfqpoint{2.123922in}{1.714631in}}%
\pgfpathlineto{\pgfqpoint{2.130661in}{1.703243in}}%
\pgfpathlineto{\pgfqpoint{2.137400in}{1.672874in}}%
\pgfpathlineto{\pgfqpoint{2.144139in}{1.665282in}}%
\pgfpathlineto{\pgfqpoint{2.150878in}{1.661486in}}%
\pgfpathlineto{\pgfqpoint{2.157617in}{1.642506in}}%
\pgfpathlineto{\pgfqpoint{2.164356in}{1.634913in}}%
\pgfpathlineto{\pgfqpoint{2.171095in}{1.631117in}}%
\pgfpathlineto{\pgfqpoint{2.177835in}{1.623525in}}%
\pgfpathlineto{\pgfqpoint{2.191313in}{1.585564in}}%
\pgfpathlineto{\pgfqpoint{2.198052in}{1.570380in}}%
\pgfpathlineto{\pgfqpoint{2.204791in}{1.558992in}}%
\pgfpathlineto{\pgfqpoint{2.218269in}{1.551399in}}%
\pgfpathlineto{\pgfqpoint{2.225009in}{1.543807in}}%
\pgfpathlineto{\pgfqpoint{2.231748in}{1.540011in}}%
\pgfpathlineto{\pgfqpoint{2.238487in}{1.528623in}}%
\pgfpathlineto{\pgfqpoint{2.245226in}{1.524827in}}%
\pgfpathlineto{\pgfqpoint{2.251965in}{1.490662in}}%
\pgfpathlineto{\pgfqpoint{2.258704in}{1.486866in}}%
\pgfpathlineto{\pgfqpoint{2.285661in}{1.456497in}}%
\pgfpathlineto{\pgfqpoint{2.292400in}{1.441313in}}%
\pgfpathlineto{\pgfqpoint{2.299139in}{1.422332in}}%
\pgfpathlineto{\pgfqpoint{2.305878in}{1.407148in}}%
\pgfpathlineto{\pgfqpoint{2.312617in}{1.380575in}}%
\pgfpathlineto{\pgfqpoint{2.319356in}{1.361595in}}%
\pgfpathlineto{\pgfqpoint{2.326095in}{1.338818in}}%
\pgfpathlineto{\pgfqpoint{2.332835in}{1.335022in}}%
\pgfpathlineto{\pgfqpoint{2.339574in}{1.316041in}}%
\pgfpathlineto{\pgfqpoint{2.353052in}{1.266692in}}%
\pgfpathlineto{\pgfqpoint{2.359791in}{1.262896in}}%
\pgfpathlineto{\pgfqpoint{2.366530in}{1.251508in}}%
\pgfpathlineto{\pgfqpoint{2.373269in}{1.232527in}}%
\pgfpathlineto{\pgfqpoint{2.380009in}{1.217343in}}%
\pgfpathlineto{\pgfqpoint{2.386748in}{1.213547in}}%
\pgfpathlineto{\pgfqpoint{2.393487in}{1.202159in}}%
\pgfpathlineto{\pgfqpoint{2.400226in}{1.183178in}}%
\pgfpathlineto{\pgfqpoint{2.406965in}{1.167994in}}%
\pgfpathlineto{\pgfqpoint{2.413704in}{1.156605in}}%
\pgfpathlineto{\pgfqpoint{2.427182in}{1.111052in}}%
\pgfpathlineto{\pgfqpoint{2.433922in}{1.103460in}}%
\pgfpathlineto{\pgfqpoint{2.440661in}{1.088276in}}%
\pgfpathlineto{\pgfqpoint{2.447400in}{1.065499in}}%
\pgfpathlineto{\pgfqpoint{2.454139in}{1.050315in}}%
\pgfpathlineto{\pgfqpoint{2.460878in}{1.038926in}}%
\pgfpathlineto{\pgfqpoint{2.467617in}{1.019946in}}%
\pgfpathlineto{\pgfqpoint{2.474356in}{1.004762in}}%
\pgfpathlineto{\pgfqpoint{2.481095in}{0.970597in}}%
\pgfpathlineto{\pgfqpoint{2.487835in}{0.947820in}}%
\pgfpathlineto{\pgfqpoint{2.494574in}{0.932636in}}%
\pgfpathlineto{\pgfqpoint{2.501313in}{0.932636in}}%
\pgfpathlineto{\pgfqpoint{2.508052in}{0.917451in}}%
\pgfpathlineto{\pgfqpoint{2.514791in}{0.894675in}}%
\pgfpathlineto{\pgfqpoint{2.521530in}{0.890879in}}%
\pgfpathlineto{\pgfqpoint{2.528269in}{0.875694in}}%
\pgfpathlineto{\pgfqpoint{2.535009in}{0.856714in}}%
\pgfpathlineto{\pgfqpoint{2.541748in}{0.833937in}}%
\pgfpathlineto{\pgfqpoint{2.548487in}{0.792180in}}%
\pgfpathlineto{\pgfqpoint{2.555226in}{0.792180in}}%
\pgfpathlineto{\pgfqpoint{2.561965in}{0.723851in}}%
\pgfpathlineto{\pgfqpoint{2.568704in}{0.587191in}}%
\pgfpathlineto{\pgfqpoint{2.568704in}{0.587191in}}%
\pgfusepath{stroke}%
\end{pgfscope}%
\begin{pgfscope}%
\pgfsetrectcap%
\pgfsetmiterjoin%
\pgfsetlinewidth{0.803000pt}%
\definecolor{currentstroke}{rgb}{0.000000,0.000000,0.000000}%
\pgfsetstrokecolor{currentstroke}%
\pgfsetdash{}{0pt}%
\pgfpathmoveto{\pgfqpoint{0.553704in}{0.499691in}}%
\pgfpathlineto{\pgfqpoint{0.553704in}{2.424691in}}%
\pgfusepath{stroke}%
\end{pgfscope}%
\begin{pgfscope}%
\pgfsetrectcap%
\pgfsetmiterjoin%
\pgfsetlinewidth{0.803000pt}%
\definecolor{currentstroke}{rgb}{0.000000,0.000000,0.000000}%
\pgfsetstrokecolor{currentstroke}%
\pgfsetdash{}{0pt}%
\pgfpathmoveto{\pgfqpoint{2.568704in}{0.499691in}}%
\pgfpathlineto{\pgfqpoint{2.568704in}{2.424691in}}%
\pgfusepath{stroke}%
\end{pgfscope}%
\begin{pgfscope}%
\pgfsetrectcap%
\pgfsetmiterjoin%
\pgfsetlinewidth{0.803000pt}%
\definecolor{currentstroke}{rgb}{0.000000,0.000000,0.000000}%
\pgfsetstrokecolor{currentstroke}%
\pgfsetdash{}{0pt}%
\pgfpathmoveto{\pgfqpoint{0.553704in}{0.499691in}}%
\pgfpathlineto{\pgfqpoint{2.568704in}{0.499691in}}%
\pgfusepath{stroke}%
\end{pgfscope}%
\begin{pgfscope}%
\pgfsetrectcap%
\pgfsetmiterjoin%
\pgfsetlinewidth{0.803000pt}%
\definecolor{currentstroke}{rgb}{0.000000,0.000000,0.000000}%
\pgfsetstrokecolor{currentstroke}%
\pgfsetdash{}{0pt}%
\pgfpathmoveto{\pgfqpoint{0.553704in}{2.424691in}}%
\pgfpathlineto{\pgfqpoint{2.568704in}{2.424691in}}%
\pgfusepath{stroke}%
\end{pgfscope}%
\begin{pgfscope}%
\pgfsetbuttcap%
\pgfsetmiterjoin%
\definecolor{currentfill}{rgb}{1.000000,1.000000,1.000000}%
\pgfsetfillcolor{currentfill}%
\pgfsetfillopacity{0.800000}%
\pgfsetlinewidth{1.003750pt}%
\definecolor{currentstroke}{rgb}{0.800000,0.800000,0.800000}%
\pgfsetstrokecolor{currentstroke}%
\pgfsetstrokeopacity{0.800000}%
\pgfsetdash{}{0pt}%
\pgfpathmoveto{\pgfqpoint{0.650926in}{0.569136in}}%
\pgfpathlineto{\pgfqpoint{1.611815in}{0.569136in}}%
\pgfpathquadraticcurveto{\pgfqpoint{1.639593in}{0.569136in}}{\pgfqpoint{1.639593in}{0.596913in}}%
\pgfpathlineto{\pgfqpoint{1.639593in}{1.356944in}}%
\pgfpathquadraticcurveto{\pgfqpoint{1.639593in}{1.384722in}}{\pgfqpoint{1.611815in}{1.384722in}}%
\pgfpathlineto{\pgfqpoint{0.650926in}{1.384722in}}%
\pgfpathquadraticcurveto{\pgfqpoint{0.623149in}{1.384722in}}{\pgfqpoint{0.623149in}{1.356944in}}%
\pgfpathlineto{\pgfqpoint{0.623149in}{0.596913in}}%
\pgfpathquadraticcurveto{\pgfqpoint{0.623149in}{0.569136in}}{\pgfqpoint{0.650926in}{0.569136in}}%
\pgfpathlineto{\pgfqpoint{0.650926in}{0.569136in}}%
\pgfpathclose%
\pgfusepath{stroke,fill}%
\end{pgfscope}%
\begin{pgfscope}%
\definecolor{textcolor}{rgb}{0.000000,0.000000,0.000000}%
\pgfsetstrokecolor{textcolor}%
\pgfsetfillcolor{textcolor}%
\pgftext[x=0.765825in,y=1.232716in,left,base]{\color{textcolor}\rmfamily\fontsize{10.000000}{12.000000}\selectfont \(\displaystyle \textsc{RS4A}-\ell_{\infty}\)}%
\end{pgfscope}%
\begin{pgfscope}%
\pgfsetrectcap%
\pgfsetroundjoin%
\pgfsetlinewidth{1.505625pt}%
\definecolor{currentstroke}{rgb}{0.716186,0.833203,0.916155}%
\pgfsetstrokecolor{currentstroke}%
\pgfsetdash{}{0pt}%
\pgfpathmoveto{\pgfqpoint{0.678704in}{1.087654in}}%
\pgfpathlineto{\pgfqpoint{0.817593in}{1.087654in}}%
\pgfpathlineto{\pgfqpoint{0.956482in}{1.087654in}}%
\pgfusepath{stroke}%
\end{pgfscope}%
\begin{pgfscope}%
\definecolor{textcolor}{rgb}{0.000000,0.000000,0.000000}%
\pgfsetstrokecolor{textcolor}%
\pgfsetfillcolor{textcolor}%
\pgftext[x=1.067593in,y=1.039043in,left,base]{\color{textcolor}\rmfamily\fontsize{10.000000}{12.000000}\selectfont \(\displaystyle \sigma=0.15\)}%
\end{pgfscope}%
\begin{pgfscope}%
\pgfsetrectcap%
\pgfsetroundjoin%
\pgfsetlinewidth{1.505625pt}%
\definecolor{currentstroke}{rgb}{0.231926,0.545652,0.762614}%
\pgfsetstrokecolor{currentstroke}%
\pgfsetdash{}{0pt}%
\pgfpathmoveto{\pgfqpoint{0.678704in}{0.893981in}}%
\pgfpathlineto{\pgfqpoint{0.817593in}{0.893981in}}%
\pgfpathlineto{\pgfqpoint{0.956482in}{0.893981in}}%
\pgfusepath{stroke}%
\end{pgfscope}%
\begin{pgfscope}%
\definecolor{textcolor}{rgb}{0.000000,0.000000,0.000000}%
\pgfsetstrokecolor{textcolor}%
\pgfsetfillcolor{textcolor}%
\pgftext[x=1.067593in,y=0.845370in,left,base]{\color{textcolor}\rmfamily\fontsize{10.000000}{12.000000}\selectfont \(\displaystyle \sigma=0.25\)}%
\end{pgfscope}%
\begin{pgfscope}%
\pgfsetrectcap%
\pgfsetroundjoin%
\pgfsetlinewidth{1.505625pt}%
\definecolor{currentstroke}{rgb}{0.031373,0.188235,0.419608}%
\pgfsetstrokecolor{currentstroke}%
\pgfsetdash{}{0pt}%
\pgfpathmoveto{\pgfqpoint{0.678704in}{0.700308in}}%
\pgfpathlineto{\pgfqpoint{0.817593in}{0.700308in}}%
\pgfpathlineto{\pgfqpoint{0.956482in}{0.700308in}}%
\pgfusepath{stroke}%
\end{pgfscope}%
\begin{pgfscope}%
\definecolor{textcolor}{rgb}{0.000000,0.000000,0.000000}%
\pgfsetstrokecolor{textcolor}%
\pgfsetfillcolor{textcolor}%
\pgftext[x=1.067593in,y=0.651697in,left,base]{\color{textcolor}\rmfamily\fontsize{10.000000}{12.000000}\selectfont \(\displaystyle \sigma=0.50\)}%
\end{pgfscope}%
\end{pgfpicture}%
\makeatother%
\endgroup%

%% file: figs/svhn_ancersweep.pgf
%% Creator: Matplotlib, PGF backend
%%
%% To include the figure in your LaTeX document, write
%%   \input{<filename>.pgf}
%%
%% Make sure the required packages are loaded in your preamble
%%   \usepackage{pgf}
%%
%% Also ensure that all the required font packages are loaded; for instance,
%% the lmodern package is sometimes necessary when using math font.
%%   \usepackage{lmodern}
%%
%% Figures using additional raster images can only be included by \input if
%% they are in the same directory as the main LaTeX file. For loading figures
%% from other directories you can use the `import` package
%%   \usepackage{import}
%%
%% and then include the figures with
%%   \import{<path to file>}{<filename>.pgf}
%%
%% Matplotlib used the following preamble
%%
\begingroup%
\makeatletter%
\begin{pgfpicture}%
\pgfpathrectangle{\pgfpointorigin}{\pgfqpoint{2.721693in}{2.524691in}}%
\pgfusepath{use as bounding box, clip}%
\begin{pgfscope}%
\pgfsetbuttcap%
\pgfsetmiterjoin%
\pgfsetlinewidth{0.000000pt}%
\definecolor{currentstroke}{rgb}{0.000000,0.000000,0.000000}%
\pgfsetstrokecolor{currentstroke}%
\pgfsetstrokeopacity{0.000000}%
\pgfsetdash{}{0pt}%
\pgfpathmoveto{\pgfqpoint{0.000000in}{0.000000in}}%
\pgfpathlineto{\pgfqpoint{2.721693in}{0.000000in}}%
\pgfpathlineto{\pgfqpoint{2.721693in}{2.524691in}}%
\pgfpathlineto{\pgfqpoint{0.000000in}{2.524691in}}%
\pgfpathlineto{\pgfqpoint{0.000000in}{0.000000in}}%
\pgfpathclose%
\pgfusepath{}%
\end{pgfscope}%
\begin{pgfscope}%
\pgfsetbuttcap%
\pgfsetmiterjoin%
\pgfsetlinewidth{0.000000pt}%
\definecolor{currentstroke}{rgb}{0.000000,0.000000,0.000000}%
\pgfsetstrokecolor{currentstroke}%
\pgfsetstrokeopacity{0.000000}%
\pgfsetdash{}{0pt}%
\pgfpathmoveto{\pgfqpoint{0.553704in}{0.499691in}}%
\pgfpathlineto{\pgfqpoint{2.568704in}{0.499691in}}%
\pgfpathlineto{\pgfqpoint{2.568704in}{2.424691in}}%
\pgfpathlineto{\pgfqpoint{0.553704in}{2.424691in}}%
\pgfpathlineto{\pgfqpoint{0.553704in}{0.499691in}}%
\pgfpathclose%
\pgfusepath{}%
\end{pgfscope}%
\begin{pgfscope}%
\pgfsetbuttcap%
\pgfsetroundjoin%
\definecolor{currentfill}{rgb}{0.000000,0.000000,0.000000}%
\pgfsetfillcolor{currentfill}%
\pgfsetlinewidth{0.803000pt}%
\definecolor{currentstroke}{rgb}{0.000000,0.000000,0.000000}%
\pgfsetstrokecolor{currentstroke}%
\pgfsetdash{}{0pt}%
\pgfsys@defobject{currentmarker}{\pgfqpoint{0.000000in}{-0.048611in}}{\pgfqpoint{0.000000in}{0.000000in}}{%
\pgfpathmoveto{\pgfqpoint{0.000000in}{0.000000in}}%
\pgfpathlineto{\pgfqpoint{0.000000in}{-0.048611in}}%
\pgfusepath{stroke,fill}%
}%
\begin{pgfscope}%
\pgfsys@transformshift{0.705854in}{0.499691in}%
\pgfsys@useobject{currentmarker}{}%
\end{pgfscope}%
\end{pgfscope}%
\begin{pgfscope}%
\definecolor{textcolor}{rgb}{0.000000,0.000000,0.000000}%
\pgfsetstrokecolor{textcolor}%
\pgfsetfillcolor{textcolor}%
\pgftext[x=0.705854in,y=0.402469in,,top]{\color{textcolor}\rmfamily\fontsize{10.000000}{12.000000}\selectfont \(\displaystyle 10^{-2 \alpha}\)}%
\end{pgfscope}%
\begin{pgfscope}%
\pgfsetbuttcap%
\pgfsetroundjoin%
\definecolor{currentfill}{rgb}{0.000000,0.000000,0.000000}%
\pgfsetfillcolor{currentfill}%
\pgfsetlinewidth{0.803000pt}%
\definecolor{currentstroke}{rgb}{0.000000,0.000000,0.000000}%
\pgfsetstrokecolor{currentstroke}%
\pgfsetdash{}{0pt}%
\pgfsys@defobject{currentmarker}{\pgfqpoint{0.000000in}{-0.048611in}}{\pgfqpoint{0.000000in}{0.000000in}}{%
\pgfpathmoveto{\pgfqpoint{0.000000in}{0.000000in}}%
\pgfpathlineto{\pgfqpoint{0.000000in}{-0.048611in}}%
\pgfusepath{stroke,fill}%
}%
\begin{pgfscope}%
\pgfsys@transformshift{2.441589in}{0.499691in}%
\pgfsys@useobject{currentmarker}{}%
\end{pgfscope}%
\end{pgfscope}%
\begin{pgfscope}%
\definecolor{textcolor}{rgb}{0.000000,0.000000,0.000000}%
\pgfsetstrokecolor{textcolor}%
\pgfsetfillcolor{textcolor}%
\pgftext[x=2.441589in,y=0.402469in,,top]{\color{textcolor}\rmfamily\fontsize{10.000000}{12.000000}\selectfont \(\displaystyle 10^{-1 \alpha}\)}%
\end{pgfscope}%
\begin{pgfscope}%
\definecolor{textcolor}{rgb}{0.000000,0.000000,0.000000}%
\pgfsetstrokecolor{textcolor}%
\pgfsetfillcolor{textcolor}%
\pgftext[x=1.561204in,y=0.223457in,,top]{\color{textcolor}\rmfamily\fontsize{10.000000}{12.000000}\selectfont Volume}%
\end{pgfscope}%
\begin{pgfscope}%
\pgfsetbuttcap%
\pgfsetroundjoin%
\definecolor{currentfill}{rgb}{0.000000,0.000000,0.000000}%
\pgfsetfillcolor{currentfill}%
\pgfsetlinewidth{0.803000pt}%
\definecolor{currentstroke}{rgb}{0.000000,0.000000,0.000000}%
\pgfsetstrokecolor{currentstroke}%
\pgfsetdash{}{0pt}%
\pgfsys@defobject{currentmarker}{\pgfqpoint{-0.048611in}{0.000000in}}{\pgfqpoint{-0.000000in}{0.000000in}}{%
\pgfpathmoveto{\pgfqpoint{-0.000000in}{0.000000in}}%
\pgfpathlineto{\pgfqpoint{-0.048611in}{0.000000in}}%
\pgfusepath{stroke,fill}%
}%
\begin{pgfscope}%
\pgfsys@transformshift{0.553704in}{0.587191in}%
\pgfsys@useobject{currentmarker}{}%
\end{pgfscope}%
\end{pgfscope}%
\begin{pgfscope}%
\definecolor{textcolor}{rgb}{0.000000,0.000000,0.000000}%
\pgfsetstrokecolor{textcolor}%
\pgfsetfillcolor{textcolor}%
\pgftext[x=0.279012in, y=0.538966in, left, base]{\color{textcolor}\rmfamily\fontsize{10.000000}{12.000000}\selectfont \(\displaystyle {0.0}\)}%
\end{pgfscope}%
\begin{pgfscope}%
\pgfsetbuttcap%
\pgfsetroundjoin%
\definecolor{currentfill}{rgb}{0.000000,0.000000,0.000000}%
\pgfsetfillcolor{currentfill}%
\pgfsetlinewidth{0.803000pt}%
\definecolor{currentstroke}{rgb}{0.000000,0.000000,0.000000}%
\pgfsetstrokecolor{currentstroke}%
\pgfsetdash{}{0pt}%
\pgfsys@defobject{currentmarker}{\pgfqpoint{-0.048611in}{0.000000in}}{\pgfqpoint{-0.000000in}{0.000000in}}{%
\pgfpathmoveto{\pgfqpoint{-0.000000in}{0.000000in}}%
\pgfpathlineto{\pgfqpoint{-0.048611in}{0.000000in}}%
\pgfusepath{stroke,fill}%
}%
\begin{pgfscope}%
\pgfsys@transformshift{0.553704in}{0.980450in}%
\pgfsys@useobject{currentmarker}{}%
\end{pgfscope}%
\end{pgfscope}%
\begin{pgfscope}%
\definecolor{textcolor}{rgb}{0.000000,0.000000,0.000000}%
\pgfsetstrokecolor{textcolor}%
\pgfsetfillcolor{textcolor}%
\pgftext[x=0.279012in, y=0.932224in, left, base]{\color{textcolor}\rmfamily\fontsize{10.000000}{12.000000}\selectfont \(\displaystyle {0.2}\)}%
\end{pgfscope}%
\begin{pgfscope}%
\pgfsetbuttcap%
\pgfsetroundjoin%
\definecolor{currentfill}{rgb}{0.000000,0.000000,0.000000}%
\pgfsetfillcolor{currentfill}%
\pgfsetlinewidth{0.803000pt}%
\definecolor{currentstroke}{rgb}{0.000000,0.000000,0.000000}%
\pgfsetstrokecolor{currentstroke}%
\pgfsetdash{}{0pt}%
\pgfsys@defobject{currentmarker}{\pgfqpoint{-0.048611in}{0.000000in}}{\pgfqpoint{-0.000000in}{0.000000in}}{%
\pgfpathmoveto{\pgfqpoint{-0.000000in}{0.000000in}}%
\pgfpathlineto{\pgfqpoint{-0.048611in}{0.000000in}}%
\pgfusepath{stroke,fill}%
}%
\begin{pgfscope}%
\pgfsys@transformshift{0.553704in}{1.373708in}%
\pgfsys@useobject{currentmarker}{}%
\end{pgfscope}%
\end{pgfscope}%
\begin{pgfscope}%
\definecolor{textcolor}{rgb}{0.000000,0.000000,0.000000}%
\pgfsetstrokecolor{textcolor}%
\pgfsetfillcolor{textcolor}%
\pgftext[x=0.279012in, y=1.325483in, left, base]{\color{textcolor}\rmfamily\fontsize{10.000000}{12.000000}\selectfont \(\displaystyle {0.4}\)}%
\end{pgfscope}%
\begin{pgfscope}%
\pgfsetbuttcap%
\pgfsetroundjoin%
\definecolor{currentfill}{rgb}{0.000000,0.000000,0.000000}%
\pgfsetfillcolor{currentfill}%
\pgfsetlinewidth{0.803000pt}%
\definecolor{currentstroke}{rgb}{0.000000,0.000000,0.000000}%
\pgfsetstrokecolor{currentstroke}%
\pgfsetdash{}{0pt}%
\pgfsys@defobject{currentmarker}{\pgfqpoint{-0.048611in}{0.000000in}}{\pgfqpoint{-0.000000in}{0.000000in}}{%
\pgfpathmoveto{\pgfqpoint{-0.000000in}{0.000000in}}%
\pgfpathlineto{\pgfqpoint{-0.048611in}{0.000000in}}%
\pgfusepath{stroke,fill}%
}%
\begin{pgfscope}%
\pgfsys@transformshift{0.553704in}{1.766966in}%
\pgfsys@useobject{currentmarker}{}%
\end{pgfscope}%
\end{pgfscope}%
\begin{pgfscope}%
\definecolor{textcolor}{rgb}{0.000000,0.000000,0.000000}%
\pgfsetstrokecolor{textcolor}%
\pgfsetfillcolor{textcolor}%
\pgftext[x=0.279012in, y=1.718741in, left, base]{\color{textcolor}\rmfamily\fontsize{10.000000}{12.000000}\selectfont \(\displaystyle {0.6}\)}%
\end{pgfscope}%
\begin{pgfscope}%
\pgfsetbuttcap%
\pgfsetroundjoin%
\definecolor{currentfill}{rgb}{0.000000,0.000000,0.000000}%
\pgfsetfillcolor{currentfill}%
\pgfsetlinewidth{0.803000pt}%
\definecolor{currentstroke}{rgb}{0.000000,0.000000,0.000000}%
\pgfsetstrokecolor{currentstroke}%
\pgfsetdash{}{0pt}%
\pgfsys@defobject{currentmarker}{\pgfqpoint{-0.048611in}{0.000000in}}{\pgfqpoint{-0.000000in}{0.000000in}}{%
\pgfpathmoveto{\pgfqpoint{-0.000000in}{0.000000in}}%
\pgfpathlineto{\pgfqpoint{-0.048611in}{0.000000in}}%
\pgfusepath{stroke,fill}%
}%
\begin{pgfscope}%
\pgfsys@transformshift{0.553704in}{2.160225in}%
\pgfsys@useobject{currentmarker}{}%
\end{pgfscope}%
\end{pgfscope}%
\begin{pgfscope}%
\definecolor{textcolor}{rgb}{0.000000,0.000000,0.000000}%
\pgfsetstrokecolor{textcolor}%
\pgfsetfillcolor{textcolor}%
\pgftext[x=0.279012in, y=2.112000in, left, base]{\color{textcolor}\rmfamily\fontsize{10.000000}{12.000000}\selectfont \(\displaystyle {0.8}\)}%
\end{pgfscope}%
\begin{pgfscope}%
\definecolor{textcolor}{rgb}{0.000000,0.000000,0.000000}%
\pgfsetstrokecolor{textcolor}%
\pgfsetfillcolor{textcolor}%
\pgftext[x=0.223457in,y=1.462191in,,bottom,rotate=90.000000]{\color{textcolor}\rmfamily\fontsize{10.000000}{12.000000}\selectfont Certified accuracy}%
\end{pgfscope}%
\begin{pgfscope}%
\pgfpathrectangle{\pgfqpoint{0.553704in}{0.499691in}}{\pgfqpoint{2.015000in}{1.925000in}}%
\pgfusepath{clip}%
\pgfsetrectcap%
\pgfsetroundjoin%
\pgfsetlinewidth{1.505625pt}%
\definecolor{currentstroke}{rgb}{0.716186,0.833203,0.916155}%
\pgfsetstrokecolor{currentstroke}%
\pgfsetdash{}{0pt}%
\pgfpathmoveto{\pgfqpoint{0.553704in}{2.199551in}}%
\pgfpathlineto{\pgfqpoint{0.843487in}{2.199551in}}%
\pgfpathlineto{\pgfqpoint{0.850226in}{2.179888in}}%
\pgfpathlineto{\pgfqpoint{0.890661in}{2.179888in}}%
\pgfpathlineto{\pgfqpoint{0.904139in}{2.140562in}}%
\pgfpathlineto{\pgfqpoint{1.369139in}{2.140562in}}%
\pgfpathlineto{\pgfqpoint{1.375878in}{2.120899in}}%
\pgfpathlineto{\pgfqpoint{1.429791in}{2.120899in}}%
\pgfpathlineto{\pgfqpoint{1.443269in}{2.081573in}}%
\pgfpathlineto{\pgfqpoint{1.450009in}{2.081573in}}%
\pgfpathlineto{\pgfqpoint{1.456748in}{2.061910in}}%
\pgfpathlineto{\pgfqpoint{1.470226in}{2.061910in}}%
\pgfpathlineto{\pgfqpoint{1.476965in}{2.042247in}}%
\pgfpathlineto{\pgfqpoint{1.497182in}{2.042247in}}%
\pgfpathlineto{\pgfqpoint{1.503922in}{2.022584in}}%
\pgfpathlineto{\pgfqpoint{1.524139in}{2.022584in}}%
\pgfpathlineto{\pgfqpoint{1.530878in}{2.002921in}}%
\pgfpathlineto{\pgfqpoint{1.672400in}{2.002921in}}%
\pgfpathlineto{\pgfqpoint{1.679139in}{1.963596in}}%
\pgfpathlineto{\pgfqpoint{1.685878in}{1.943933in}}%
\pgfpathlineto{\pgfqpoint{1.692617in}{1.943933in}}%
\pgfpathlineto{\pgfqpoint{1.699356in}{1.924270in}}%
\pgfpathlineto{\pgfqpoint{1.753269in}{1.924270in}}%
\pgfpathlineto{\pgfqpoint{1.760009in}{1.884944in}}%
\pgfpathlineto{\pgfqpoint{1.773487in}{1.884944in}}%
\pgfpathlineto{\pgfqpoint{1.780226in}{1.865281in}}%
\pgfpathlineto{\pgfqpoint{1.793704in}{1.865281in}}%
\pgfpathlineto{\pgfqpoint{1.800443in}{1.845618in}}%
\pgfpathlineto{\pgfqpoint{1.834139in}{1.845618in}}%
\pgfpathlineto{\pgfqpoint{1.847617in}{1.806292in}}%
\pgfpathlineto{\pgfqpoint{1.854356in}{1.806292in}}%
\pgfpathlineto{\pgfqpoint{1.867835in}{1.766966in}}%
\pgfpathlineto{\pgfqpoint{1.874574in}{1.766966in}}%
\pgfpathlineto{\pgfqpoint{1.881313in}{1.707978in}}%
\pgfpathlineto{\pgfqpoint{1.894791in}{1.707978in}}%
\pgfpathlineto{\pgfqpoint{1.901530in}{1.688315in}}%
\pgfpathlineto{\pgfqpoint{1.921748in}{1.688315in}}%
\pgfpathlineto{\pgfqpoint{1.928487in}{1.668652in}}%
\pgfpathlineto{\pgfqpoint{1.941965in}{1.668652in}}%
\pgfpathlineto{\pgfqpoint{1.948704in}{1.648989in}}%
\pgfpathlineto{\pgfqpoint{1.955443in}{1.648989in}}%
\pgfpathlineto{\pgfqpoint{1.968922in}{1.609663in}}%
\pgfpathlineto{\pgfqpoint{1.982400in}{1.609663in}}%
\pgfpathlineto{\pgfqpoint{1.989139in}{1.570337in}}%
\pgfpathlineto{\pgfqpoint{2.002617in}{1.570337in}}%
\pgfpathlineto{\pgfqpoint{2.009356in}{1.511348in}}%
\pgfpathlineto{\pgfqpoint{2.016095in}{1.491685in}}%
\pgfpathlineto{\pgfqpoint{2.022835in}{1.491685in}}%
\pgfpathlineto{\pgfqpoint{2.043052in}{1.432697in}}%
\pgfpathlineto{\pgfqpoint{2.049791in}{1.393371in}}%
\pgfpathlineto{\pgfqpoint{2.063269in}{1.393371in}}%
\pgfpathlineto{\pgfqpoint{2.076748in}{1.354045in}}%
\pgfpathlineto{\pgfqpoint{2.083487in}{1.354045in}}%
\pgfpathlineto{\pgfqpoint{2.090226in}{1.334382in}}%
\pgfpathlineto{\pgfqpoint{2.096965in}{1.334382in}}%
\pgfpathlineto{\pgfqpoint{2.123922in}{1.255730in}}%
\pgfpathlineto{\pgfqpoint{2.130661in}{1.255730in}}%
\pgfpathlineto{\pgfqpoint{2.137400in}{1.236068in}}%
\pgfpathlineto{\pgfqpoint{2.150878in}{1.236068in}}%
\pgfpathlineto{\pgfqpoint{2.157617in}{1.216405in}}%
\pgfpathlineto{\pgfqpoint{2.164356in}{1.216405in}}%
\pgfpathlineto{\pgfqpoint{2.171095in}{1.177079in}}%
\pgfpathlineto{\pgfqpoint{2.177835in}{1.157416in}}%
\pgfpathlineto{\pgfqpoint{2.191313in}{1.157416in}}%
\pgfpathlineto{\pgfqpoint{2.198052in}{1.137753in}}%
\pgfpathlineto{\pgfqpoint{2.204791in}{1.137753in}}%
\pgfpathlineto{\pgfqpoint{2.211530in}{1.118090in}}%
\pgfpathlineto{\pgfqpoint{2.218269in}{1.059101in}}%
\pgfpathlineto{\pgfqpoint{2.225009in}{1.039438in}}%
\pgfpathlineto{\pgfqpoint{2.231748in}{1.039438in}}%
\pgfpathlineto{\pgfqpoint{2.238487in}{1.019775in}}%
\pgfpathlineto{\pgfqpoint{2.245226in}{0.980450in}}%
\pgfpathlineto{\pgfqpoint{2.265443in}{0.921461in}}%
\pgfpathlineto{\pgfqpoint{2.272182in}{0.882135in}}%
\pgfpathlineto{\pgfqpoint{2.278922in}{0.862472in}}%
\pgfpathlineto{\pgfqpoint{2.299139in}{0.862472in}}%
\pgfpathlineto{\pgfqpoint{2.305878in}{0.842809in}}%
\pgfpathlineto{\pgfqpoint{2.326095in}{0.842809in}}%
\pgfpathlineto{\pgfqpoint{2.339574in}{0.764157in}}%
\pgfpathlineto{\pgfqpoint{2.346313in}{0.744494in}}%
\pgfpathlineto{\pgfqpoint{2.400226in}{0.744494in}}%
\pgfpathlineto{\pgfqpoint{2.420443in}{0.685506in}}%
\pgfpathlineto{\pgfqpoint{2.454139in}{0.685506in}}%
\pgfpathlineto{\pgfqpoint{2.460878in}{0.646180in}}%
\pgfpathlineto{\pgfqpoint{2.474356in}{0.646180in}}%
\pgfpathlineto{\pgfqpoint{2.481095in}{0.626517in}}%
\pgfpathlineto{\pgfqpoint{2.487835in}{0.626517in}}%
\pgfpathlineto{\pgfqpoint{2.494574in}{0.606854in}}%
\pgfpathlineto{\pgfqpoint{2.561965in}{0.606854in}}%
\pgfpathlineto{\pgfqpoint{2.568704in}{0.587191in}}%
\pgfpathlineto{\pgfqpoint{2.568704in}{0.587191in}}%
\pgfusepath{stroke}%
\end{pgfscope}%
\begin{pgfscope}%
\pgfpathrectangle{\pgfqpoint{0.553704in}{0.499691in}}{\pgfqpoint{2.015000in}{1.925000in}}%
\pgfusepath{clip}%
\pgfsetrectcap%
\pgfsetroundjoin%
\pgfsetlinewidth{1.505625pt}%
\definecolor{currentstroke}{rgb}{0.376732,0.653072,0.822484}%
\pgfsetstrokecolor{currentstroke}%
\pgfsetdash{}{0pt}%
\pgfpathmoveto{\pgfqpoint{0.553704in}{2.337191in}}%
\pgfpathlineto{\pgfqpoint{1.018704in}{2.337191in}}%
\pgfpathlineto{\pgfqpoint{1.025443in}{2.278202in}}%
\pgfpathlineto{\pgfqpoint{1.227617in}{2.278202in}}%
\pgfpathlineto{\pgfqpoint{1.234356in}{2.258539in}}%
\pgfpathlineto{\pgfqpoint{1.281530in}{2.258539in}}%
\pgfpathlineto{\pgfqpoint{1.288269in}{2.219214in}}%
\pgfpathlineto{\pgfqpoint{1.382617in}{2.219214in}}%
\pgfpathlineto{\pgfqpoint{1.389356in}{2.199551in}}%
\pgfpathlineto{\pgfqpoint{1.483704in}{2.199551in}}%
\pgfpathlineto{\pgfqpoint{1.490443in}{2.179888in}}%
\pgfpathlineto{\pgfqpoint{1.578052in}{2.179888in}}%
\pgfpathlineto{\pgfqpoint{1.584791in}{2.160225in}}%
\pgfpathlineto{\pgfqpoint{1.598269in}{2.160225in}}%
\pgfpathlineto{\pgfqpoint{1.605009in}{2.140562in}}%
\pgfpathlineto{\pgfqpoint{1.692617in}{2.140562in}}%
\pgfpathlineto{\pgfqpoint{1.712835in}{2.081573in}}%
\pgfpathlineto{\pgfqpoint{1.739791in}{2.081573in}}%
\pgfpathlineto{\pgfqpoint{1.753269in}{2.042247in}}%
\pgfpathlineto{\pgfqpoint{1.760009in}{2.042247in}}%
\pgfpathlineto{\pgfqpoint{1.766748in}{2.022584in}}%
\pgfpathlineto{\pgfqpoint{1.780226in}{2.022584in}}%
\pgfpathlineto{\pgfqpoint{1.786965in}{2.002921in}}%
\pgfpathlineto{\pgfqpoint{1.793704in}{2.002921in}}%
\pgfpathlineto{\pgfqpoint{1.800443in}{1.983259in}}%
\pgfpathlineto{\pgfqpoint{1.807182in}{1.983259in}}%
\pgfpathlineto{\pgfqpoint{1.820661in}{1.943933in}}%
\pgfpathlineto{\pgfqpoint{1.827400in}{1.943933in}}%
\pgfpathlineto{\pgfqpoint{1.834139in}{1.924270in}}%
\pgfpathlineto{\pgfqpoint{1.840878in}{1.884944in}}%
\pgfpathlineto{\pgfqpoint{1.854356in}{1.884944in}}%
\pgfpathlineto{\pgfqpoint{1.861095in}{1.865281in}}%
\pgfpathlineto{\pgfqpoint{1.867835in}{1.865281in}}%
\pgfpathlineto{\pgfqpoint{1.881313in}{1.825955in}}%
\pgfpathlineto{\pgfqpoint{1.894791in}{1.825955in}}%
\pgfpathlineto{\pgfqpoint{1.901530in}{1.766966in}}%
\pgfpathlineto{\pgfqpoint{1.908269in}{1.766966in}}%
\pgfpathlineto{\pgfqpoint{1.921748in}{1.727641in}}%
\pgfpathlineto{\pgfqpoint{1.928487in}{1.727641in}}%
\pgfpathlineto{\pgfqpoint{1.941965in}{1.688315in}}%
\pgfpathlineto{\pgfqpoint{1.948704in}{1.648989in}}%
\pgfpathlineto{\pgfqpoint{1.955443in}{1.629326in}}%
\pgfpathlineto{\pgfqpoint{1.968922in}{1.550674in}}%
\pgfpathlineto{\pgfqpoint{1.975661in}{1.550674in}}%
\pgfpathlineto{\pgfqpoint{1.982400in}{1.531011in}}%
\pgfpathlineto{\pgfqpoint{1.989139in}{1.452360in}}%
\pgfpathlineto{\pgfqpoint{1.995878in}{1.432697in}}%
\pgfpathlineto{\pgfqpoint{2.002617in}{1.432697in}}%
\pgfpathlineto{\pgfqpoint{2.029574in}{1.354045in}}%
\pgfpathlineto{\pgfqpoint{2.056530in}{1.354045in}}%
\pgfpathlineto{\pgfqpoint{2.063269in}{1.334382in}}%
\pgfpathlineto{\pgfqpoint{2.070009in}{1.255730in}}%
\pgfpathlineto{\pgfqpoint{2.083487in}{1.255730in}}%
\pgfpathlineto{\pgfqpoint{2.090226in}{1.216405in}}%
\pgfpathlineto{\pgfqpoint{2.103704in}{1.177079in}}%
\pgfpathlineto{\pgfqpoint{2.110443in}{1.177079in}}%
\pgfpathlineto{\pgfqpoint{2.117182in}{1.157416in}}%
\pgfpathlineto{\pgfqpoint{2.130661in}{1.078764in}}%
\pgfpathlineto{\pgfqpoint{2.137400in}{1.059101in}}%
\pgfpathlineto{\pgfqpoint{2.144139in}{1.019775in}}%
\pgfpathlineto{\pgfqpoint{2.198052in}{1.019775in}}%
\pgfpathlineto{\pgfqpoint{2.204791in}{0.960787in}}%
\pgfpathlineto{\pgfqpoint{2.225009in}{0.901798in}}%
\pgfpathlineto{\pgfqpoint{2.245226in}{0.901798in}}%
\pgfpathlineto{\pgfqpoint{2.251965in}{0.882135in}}%
\pgfpathlineto{\pgfqpoint{2.265443in}{0.882135in}}%
\pgfpathlineto{\pgfqpoint{2.272182in}{0.862472in}}%
\pgfpathlineto{\pgfqpoint{2.278922in}{0.783820in}}%
\pgfpathlineto{\pgfqpoint{2.285661in}{0.764157in}}%
\pgfpathlineto{\pgfqpoint{2.326095in}{0.764157in}}%
\pgfpathlineto{\pgfqpoint{2.332835in}{0.744494in}}%
\pgfpathlineto{\pgfqpoint{2.339574in}{0.744494in}}%
\pgfpathlineto{\pgfqpoint{2.346313in}{0.705169in}}%
\pgfpathlineto{\pgfqpoint{2.359791in}{0.665843in}}%
\pgfpathlineto{\pgfqpoint{2.380009in}{0.665843in}}%
\pgfpathlineto{\pgfqpoint{2.400226in}{0.606854in}}%
\pgfpathlineto{\pgfqpoint{2.433922in}{0.606854in}}%
\pgfpathlineto{\pgfqpoint{2.440661in}{0.587191in}}%
\pgfpathlineto{\pgfqpoint{2.568704in}{0.587191in}}%
\pgfpathlineto{\pgfqpoint{2.568704in}{0.587191in}}%
\pgfusepath{stroke}%
\end{pgfscope}%
\begin{pgfscope}%
\pgfpathrectangle{\pgfqpoint{0.553704in}{0.499691in}}{\pgfqpoint{2.015000in}{1.925000in}}%
\pgfusepath{clip}%
\pgfsetrectcap%
\pgfsetroundjoin%
\pgfsetlinewidth{1.505625pt}%
\definecolor{currentstroke}{rgb}{0.114802,0.424437,0.695194}%
\pgfsetstrokecolor{currentstroke}%
\pgfsetdash{}{0pt}%
\pgfpathmoveto{\pgfqpoint{0.553704in}{2.337191in}}%
\pgfpathlineto{\pgfqpoint{0.695226in}{2.337191in}}%
\pgfpathlineto{\pgfqpoint{0.701965in}{2.317528in}}%
\pgfpathlineto{\pgfqpoint{0.715443in}{2.317528in}}%
\pgfpathlineto{\pgfqpoint{0.722182in}{2.297865in}}%
\pgfpathlineto{\pgfqpoint{0.816530in}{2.297865in}}%
\pgfpathlineto{\pgfqpoint{0.823269in}{2.278202in}}%
\pgfpathlineto{\pgfqpoint{0.843487in}{2.278202in}}%
\pgfpathlineto{\pgfqpoint{0.850226in}{2.258539in}}%
\pgfpathlineto{\pgfqpoint{0.870443in}{2.258539in}}%
\pgfpathlineto{\pgfqpoint{0.877182in}{2.238877in}}%
\pgfpathlineto{\pgfqpoint{0.964791in}{2.238877in}}%
\pgfpathlineto{\pgfqpoint{0.971530in}{2.219214in}}%
\pgfpathlineto{\pgfqpoint{1.227617in}{2.219214in}}%
\pgfpathlineto{\pgfqpoint{1.234356in}{2.199551in}}%
\pgfpathlineto{\pgfqpoint{1.315226in}{2.199551in}}%
\pgfpathlineto{\pgfqpoint{1.328704in}{2.160225in}}%
\pgfpathlineto{\pgfqpoint{1.409574in}{2.160225in}}%
\pgfpathlineto{\pgfqpoint{1.416313in}{2.140562in}}%
\pgfpathlineto{\pgfqpoint{1.429791in}{2.140562in}}%
\pgfpathlineto{\pgfqpoint{1.436530in}{2.120899in}}%
\pgfpathlineto{\pgfqpoint{1.544356in}{2.120899in}}%
\pgfpathlineto{\pgfqpoint{1.551095in}{2.101236in}}%
\pgfpathlineto{\pgfqpoint{1.591530in}{2.101236in}}%
\pgfpathlineto{\pgfqpoint{1.598269in}{2.081573in}}%
\pgfpathlineto{\pgfqpoint{1.625226in}{2.081573in}}%
\pgfpathlineto{\pgfqpoint{1.631965in}{2.061910in}}%
\pgfpathlineto{\pgfqpoint{1.685878in}{2.061910in}}%
\pgfpathlineto{\pgfqpoint{1.692617in}{2.002921in}}%
\pgfpathlineto{\pgfqpoint{1.699356in}{2.002921in}}%
\pgfpathlineto{\pgfqpoint{1.706095in}{1.983259in}}%
\pgfpathlineto{\pgfqpoint{1.712835in}{1.943933in}}%
\pgfpathlineto{\pgfqpoint{1.733052in}{1.943933in}}%
\pgfpathlineto{\pgfqpoint{1.739791in}{1.924270in}}%
\pgfpathlineto{\pgfqpoint{1.760009in}{1.924270in}}%
\pgfpathlineto{\pgfqpoint{1.766748in}{1.904607in}}%
\pgfpathlineto{\pgfqpoint{1.807182in}{1.904607in}}%
\pgfpathlineto{\pgfqpoint{1.813922in}{1.884944in}}%
\pgfpathlineto{\pgfqpoint{1.834139in}{1.884944in}}%
\pgfpathlineto{\pgfqpoint{1.840878in}{1.865281in}}%
\pgfpathlineto{\pgfqpoint{1.847617in}{1.865281in}}%
\pgfpathlineto{\pgfqpoint{1.861095in}{1.786629in}}%
\pgfpathlineto{\pgfqpoint{1.881313in}{1.786629in}}%
\pgfpathlineto{\pgfqpoint{1.888052in}{1.766966in}}%
\pgfpathlineto{\pgfqpoint{1.894791in}{1.707978in}}%
\pgfpathlineto{\pgfqpoint{1.908269in}{1.629326in}}%
\pgfpathlineto{\pgfqpoint{1.915009in}{1.629326in}}%
\pgfpathlineto{\pgfqpoint{1.921748in}{1.570337in}}%
\pgfpathlineto{\pgfqpoint{1.948704in}{1.491685in}}%
\pgfpathlineto{\pgfqpoint{1.968922in}{1.491685in}}%
\pgfpathlineto{\pgfqpoint{1.975661in}{1.472023in}}%
\pgfpathlineto{\pgfqpoint{1.982400in}{1.432697in}}%
\pgfpathlineto{\pgfqpoint{1.995878in}{1.393371in}}%
\pgfpathlineto{\pgfqpoint{2.009356in}{1.314719in}}%
\pgfpathlineto{\pgfqpoint{2.022835in}{1.275393in}}%
\pgfpathlineto{\pgfqpoint{2.029574in}{1.216405in}}%
\pgfpathlineto{\pgfqpoint{2.036313in}{1.216405in}}%
\pgfpathlineto{\pgfqpoint{2.043052in}{1.157416in}}%
\pgfpathlineto{\pgfqpoint{2.049791in}{1.137753in}}%
\pgfpathlineto{\pgfqpoint{2.056530in}{1.137753in}}%
\pgfpathlineto{\pgfqpoint{2.063269in}{1.098427in}}%
\pgfpathlineto{\pgfqpoint{2.083487in}{1.098427in}}%
\pgfpathlineto{\pgfqpoint{2.096965in}{1.059101in}}%
\pgfpathlineto{\pgfqpoint{2.110443in}{1.059101in}}%
\pgfpathlineto{\pgfqpoint{2.117182in}{1.039438in}}%
\pgfpathlineto{\pgfqpoint{2.123922in}{1.000112in}}%
\pgfpathlineto{\pgfqpoint{2.130661in}{0.980450in}}%
\pgfpathlineto{\pgfqpoint{2.137400in}{0.980450in}}%
\pgfpathlineto{\pgfqpoint{2.150878in}{0.941124in}}%
\pgfpathlineto{\pgfqpoint{2.157617in}{0.941124in}}%
\pgfpathlineto{\pgfqpoint{2.164356in}{0.921461in}}%
\pgfpathlineto{\pgfqpoint{2.171095in}{0.921461in}}%
\pgfpathlineto{\pgfqpoint{2.177835in}{0.862472in}}%
\pgfpathlineto{\pgfqpoint{2.184574in}{0.823146in}}%
\pgfpathlineto{\pgfqpoint{2.204791in}{0.823146in}}%
\pgfpathlineto{\pgfqpoint{2.211530in}{0.803483in}}%
\pgfpathlineto{\pgfqpoint{2.245226in}{0.803483in}}%
\pgfpathlineto{\pgfqpoint{2.251965in}{0.764157in}}%
\pgfpathlineto{\pgfqpoint{2.285661in}{0.764157in}}%
\pgfpathlineto{\pgfqpoint{2.299139in}{0.724832in}}%
\pgfpathlineto{\pgfqpoint{2.332835in}{0.724832in}}%
\pgfpathlineto{\pgfqpoint{2.339574in}{0.685506in}}%
\pgfpathlineto{\pgfqpoint{2.353052in}{0.646180in}}%
\pgfpathlineto{\pgfqpoint{2.366530in}{0.646180in}}%
\pgfpathlineto{\pgfqpoint{2.373269in}{0.626517in}}%
\pgfpathlineto{\pgfqpoint{2.393487in}{0.626517in}}%
\pgfpathlineto{\pgfqpoint{2.400226in}{0.606854in}}%
\pgfpathlineto{\pgfqpoint{2.454139in}{0.606854in}}%
\pgfpathlineto{\pgfqpoint{2.460878in}{0.587191in}}%
\pgfpathlineto{\pgfqpoint{2.568704in}{0.587191in}}%
\pgfpathlineto{\pgfqpoint{2.568704in}{0.587191in}}%
\pgfusepath{stroke}%
\end{pgfscope}%
\begin{pgfscope}%
\pgfpathrectangle{\pgfqpoint{0.553704in}{0.499691in}}{\pgfqpoint{2.015000in}{1.925000in}}%
\pgfusepath{clip}%
\pgfsetrectcap%
\pgfsetroundjoin%
\pgfsetlinewidth{1.505625pt}%
\definecolor{currentstroke}{rgb}{0.031373,0.188235,0.419608}%
\pgfsetstrokecolor{currentstroke}%
\pgfsetdash{}{0pt}%
\pgfpathmoveto{\pgfqpoint{0.553704in}{2.337191in}}%
\pgfpathlineto{\pgfqpoint{0.594139in}{2.337191in}}%
\pgfpathlineto{\pgfqpoint{0.607617in}{2.297865in}}%
\pgfpathlineto{\pgfqpoint{0.675009in}{2.297865in}}%
\pgfpathlineto{\pgfqpoint{0.688487in}{2.258539in}}%
\pgfpathlineto{\pgfqpoint{0.816530in}{2.258539in}}%
\pgfpathlineto{\pgfqpoint{0.823269in}{2.238877in}}%
\pgfpathlineto{\pgfqpoint{0.870443in}{2.238877in}}%
\pgfpathlineto{\pgfqpoint{0.877182in}{2.219214in}}%
\pgfpathlineto{\pgfqpoint{0.991748in}{2.219214in}}%
\pgfpathlineto{\pgfqpoint{0.998487in}{2.199551in}}%
\pgfpathlineto{\pgfqpoint{1.173704in}{2.199551in}}%
\pgfpathlineto{\pgfqpoint{1.180443in}{2.179888in}}%
\pgfpathlineto{\pgfqpoint{1.247835in}{2.179888in}}%
\pgfpathlineto{\pgfqpoint{1.254574in}{2.160225in}}%
\pgfpathlineto{\pgfqpoint{1.301748in}{2.160225in}}%
\pgfpathlineto{\pgfqpoint{1.308487in}{2.140562in}}%
\pgfpathlineto{\pgfqpoint{1.389356in}{2.140562in}}%
\pgfpathlineto{\pgfqpoint{1.396095in}{2.120899in}}%
\pgfpathlineto{\pgfqpoint{1.537617in}{2.120899in}}%
\pgfpathlineto{\pgfqpoint{1.544356in}{2.101236in}}%
\pgfpathlineto{\pgfqpoint{1.564574in}{2.101236in}}%
\pgfpathlineto{\pgfqpoint{1.584791in}{2.042247in}}%
\pgfpathlineto{\pgfqpoint{1.605009in}{2.042247in}}%
\pgfpathlineto{\pgfqpoint{1.618487in}{2.002921in}}%
\pgfpathlineto{\pgfqpoint{1.658922in}{2.002921in}}%
\pgfpathlineto{\pgfqpoint{1.665661in}{1.983259in}}%
\pgfpathlineto{\pgfqpoint{1.672400in}{1.983259in}}%
\pgfpathlineto{\pgfqpoint{1.679139in}{1.963596in}}%
\pgfpathlineto{\pgfqpoint{1.712835in}{1.963596in}}%
\pgfpathlineto{\pgfqpoint{1.726313in}{1.924270in}}%
\pgfpathlineto{\pgfqpoint{1.766748in}{1.924270in}}%
\pgfpathlineto{\pgfqpoint{1.773487in}{1.904607in}}%
\pgfpathlineto{\pgfqpoint{1.786965in}{1.904607in}}%
\pgfpathlineto{\pgfqpoint{1.793704in}{1.884944in}}%
\pgfpathlineto{\pgfqpoint{1.800443in}{1.845618in}}%
\pgfpathlineto{\pgfqpoint{1.807182in}{1.845618in}}%
\pgfpathlineto{\pgfqpoint{1.813922in}{1.786629in}}%
\pgfpathlineto{\pgfqpoint{1.827400in}{1.747303in}}%
\pgfpathlineto{\pgfqpoint{1.840878in}{1.747303in}}%
\pgfpathlineto{\pgfqpoint{1.847617in}{1.727641in}}%
\pgfpathlineto{\pgfqpoint{1.854356in}{1.688315in}}%
\pgfpathlineto{\pgfqpoint{1.861095in}{1.668652in}}%
\pgfpathlineto{\pgfqpoint{1.867835in}{1.668652in}}%
\pgfpathlineto{\pgfqpoint{1.874574in}{1.648989in}}%
\pgfpathlineto{\pgfqpoint{1.881313in}{1.609663in}}%
\pgfpathlineto{\pgfqpoint{1.888052in}{1.531011in}}%
\pgfpathlineto{\pgfqpoint{1.894791in}{1.511348in}}%
\pgfpathlineto{\pgfqpoint{1.908269in}{1.511348in}}%
\pgfpathlineto{\pgfqpoint{1.915009in}{1.491685in}}%
\pgfpathlineto{\pgfqpoint{1.935226in}{1.491685in}}%
\pgfpathlineto{\pgfqpoint{1.948704in}{1.452360in}}%
\pgfpathlineto{\pgfqpoint{1.962182in}{1.452360in}}%
\pgfpathlineto{\pgfqpoint{1.968922in}{1.393371in}}%
\pgfpathlineto{\pgfqpoint{2.002617in}{1.295056in}}%
\pgfpathlineto{\pgfqpoint{2.009356in}{1.236068in}}%
\pgfpathlineto{\pgfqpoint{2.016095in}{1.236068in}}%
\pgfpathlineto{\pgfqpoint{2.022835in}{1.216405in}}%
\pgfpathlineto{\pgfqpoint{2.029574in}{1.157416in}}%
\pgfpathlineto{\pgfqpoint{2.036313in}{1.118090in}}%
\pgfpathlineto{\pgfqpoint{2.043052in}{1.118090in}}%
\pgfpathlineto{\pgfqpoint{2.049791in}{1.078764in}}%
\pgfpathlineto{\pgfqpoint{2.056530in}{1.078764in}}%
\pgfpathlineto{\pgfqpoint{2.070009in}{1.039438in}}%
\pgfpathlineto{\pgfqpoint{2.076748in}{1.039438in}}%
\pgfpathlineto{\pgfqpoint{2.083487in}{1.019775in}}%
\pgfpathlineto{\pgfqpoint{2.096965in}{1.019775in}}%
\pgfpathlineto{\pgfqpoint{2.103704in}{1.000112in}}%
\pgfpathlineto{\pgfqpoint{2.117182in}{1.000112in}}%
\pgfpathlineto{\pgfqpoint{2.130661in}{0.960787in}}%
\pgfpathlineto{\pgfqpoint{2.144139in}{0.960787in}}%
\pgfpathlineto{\pgfqpoint{2.157617in}{0.921461in}}%
\pgfpathlineto{\pgfqpoint{2.171095in}{0.921461in}}%
\pgfpathlineto{\pgfqpoint{2.184574in}{0.882135in}}%
\pgfpathlineto{\pgfqpoint{2.191313in}{0.842809in}}%
\pgfpathlineto{\pgfqpoint{2.198052in}{0.842809in}}%
\pgfpathlineto{\pgfqpoint{2.204791in}{0.803483in}}%
\pgfpathlineto{\pgfqpoint{2.211530in}{0.783820in}}%
\pgfpathlineto{\pgfqpoint{2.218269in}{0.783820in}}%
\pgfpathlineto{\pgfqpoint{2.225009in}{0.764157in}}%
\pgfpathlineto{\pgfqpoint{2.231748in}{0.764157in}}%
\pgfpathlineto{\pgfqpoint{2.238487in}{0.744494in}}%
\pgfpathlineto{\pgfqpoint{2.245226in}{0.744494in}}%
\pgfpathlineto{\pgfqpoint{2.251965in}{0.724832in}}%
\pgfpathlineto{\pgfqpoint{2.285661in}{0.724832in}}%
\pgfpathlineto{\pgfqpoint{2.292400in}{0.705169in}}%
\pgfpathlineto{\pgfqpoint{2.305878in}{0.705169in}}%
\pgfpathlineto{\pgfqpoint{2.312617in}{0.685506in}}%
\pgfpathlineto{\pgfqpoint{2.339574in}{0.685506in}}%
\pgfpathlineto{\pgfqpoint{2.346313in}{0.665843in}}%
\pgfpathlineto{\pgfqpoint{2.353052in}{0.626517in}}%
\pgfpathlineto{\pgfqpoint{2.366530in}{0.626517in}}%
\pgfpathlineto{\pgfqpoint{2.373269in}{0.606854in}}%
\pgfpathlineto{\pgfqpoint{2.433922in}{0.606854in}}%
\pgfpathlineto{\pgfqpoint{2.440661in}{0.587191in}}%
\pgfpathlineto{\pgfqpoint{2.568704in}{0.587191in}}%
\pgfpathlineto{\pgfqpoint{2.568704in}{0.587191in}}%
\pgfusepath{stroke}%
\end{pgfscope}%
\begin{pgfscope}%
\pgfsetrectcap%
\pgfsetmiterjoin%
\pgfsetlinewidth{0.803000pt}%
\definecolor{currentstroke}{rgb}{0.000000,0.000000,0.000000}%
\pgfsetstrokecolor{currentstroke}%
\pgfsetdash{}{0pt}%
\pgfpathmoveto{\pgfqpoint{0.553704in}{0.499691in}}%
\pgfpathlineto{\pgfqpoint{0.553704in}{2.424691in}}%
\pgfusepath{stroke}%
\end{pgfscope}%
\begin{pgfscope}%
\pgfsetrectcap%
\pgfsetmiterjoin%
\pgfsetlinewidth{0.803000pt}%
\definecolor{currentstroke}{rgb}{0.000000,0.000000,0.000000}%
\pgfsetstrokecolor{currentstroke}%
\pgfsetdash{}{0pt}%
\pgfpathmoveto{\pgfqpoint{2.568704in}{0.499691in}}%
\pgfpathlineto{\pgfqpoint{2.568704in}{2.424691in}}%
\pgfusepath{stroke}%
\end{pgfscope}%
\begin{pgfscope}%
\pgfsetrectcap%
\pgfsetmiterjoin%
\pgfsetlinewidth{0.803000pt}%
\definecolor{currentstroke}{rgb}{0.000000,0.000000,0.000000}%
\pgfsetstrokecolor{currentstroke}%
\pgfsetdash{}{0pt}%
\pgfpathmoveto{\pgfqpoint{0.553704in}{0.499691in}}%
\pgfpathlineto{\pgfqpoint{2.568704in}{0.499691in}}%
\pgfusepath{stroke}%
\end{pgfscope}%
\begin{pgfscope}%
\pgfsetrectcap%
\pgfsetmiterjoin%
\pgfsetlinewidth{0.803000pt}%
\definecolor{currentstroke}{rgb}{0.000000,0.000000,0.000000}%
\pgfsetstrokecolor{currentstroke}%
\pgfsetdash{}{0pt}%
\pgfpathmoveto{\pgfqpoint{0.553704in}{2.424691in}}%
\pgfpathlineto{\pgfqpoint{2.568704in}{2.424691in}}%
\pgfusepath{stroke}%
\end{pgfscope}%
\begin{pgfscope}%
\pgfsetbuttcap%
\pgfsetmiterjoin%
\definecolor{currentfill}{rgb}{1.000000,1.000000,1.000000}%
\pgfsetfillcolor{currentfill}%
\pgfsetfillopacity{0.800000}%
\pgfsetlinewidth{1.003750pt}%
\definecolor{currentstroke}{rgb}{0.800000,0.800000,0.800000}%
\pgfsetstrokecolor{currentstroke}%
\pgfsetstrokeopacity{0.800000}%
\pgfsetdash{}{0pt}%
\pgfpathmoveto{\pgfqpoint{0.650926in}{0.569136in}}%
\pgfpathlineto{\pgfqpoint{1.707607in}{0.569136in}}%
\pgfpathquadraticcurveto{\pgfqpoint{1.735385in}{0.569136in}}{\pgfqpoint{1.735385in}{0.596913in}}%
\pgfpathlineto{\pgfqpoint{1.735385in}{1.550617in}}%
\pgfpathquadraticcurveto{\pgfqpoint{1.735385in}{1.578394in}}{\pgfqpoint{1.707607in}{1.578394in}}%
\pgfpathlineto{\pgfqpoint{0.650926in}{1.578394in}}%
\pgfpathquadraticcurveto{\pgfqpoint{0.623149in}{1.578394in}}{\pgfqpoint{0.623149in}{1.550617in}}%
\pgfpathlineto{\pgfqpoint{0.623149in}{0.596913in}}%
\pgfpathquadraticcurveto{\pgfqpoint{0.623149in}{0.569136in}}{\pgfqpoint{0.650926in}{0.569136in}}%
\pgfpathlineto{\pgfqpoint{0.650926in}{0.569136in}}%
\pgfpathclose%
\pgfusepath{stroke,fill}%
\end{pgfscope}%
\begin{pgfscope}%
\definecolor{textcolor}{rgb}{0.000000,0.000000,0.000000}%
\pgfsetstrokecolor{textcolor}%
\pgfsetfillcolor{textcolor}%
\pgftext[x=0.904578in,y=1.426388in,left,base]{\color{textcolor}\rmfamily\fontsize{10.000000}{12.000000}\selectfont \(\displaystyle \textsc{ANCER}\)}%
\end{pgfscope}%
\begin{pgfscope}%
\pgfsetrectcap%
\pgfsetroundjoin%
\pgfsetlinewidth{1.505625pt}%
\definecolor{currentstroke}{rgb}{0.716186,0.833203,0.916155}%
\pgfsetstrokecolor{currentstroke}%
\pgfsetdash{}{0pt}%
\pgfpathmoveto{\pgfqpoint{0.678704in}{1.281327in}}%
\pgfpathlineto{\pgfqpoint{0.817593in}{1.281327in}}%
\pgfpathlineto{\pgfqpoint{0.956482in}{1.281327in}}%
\pgfusepath{stroke}%
\end{pgfscope}%
\begin{pgfscope}%
\definecolor{textcolor}{rgb}{0.000000,0.000000,0.000000}%
\pgfsetstrokecolor{textcolor}%
\pgfsetfillcolor{textcolor}%
\pgftext[x=1.067593in,y=1.232716in,left,base]{\color{textcolor}\rmfamily\fontsize{10.000000}{12.000000}\selectfont \(\displaystyle lr=0.03\)}%
\end{pgfscope}%
\begin{pgfscope}%
\pgfsetrectcap%
\pgfsetroundjoin%
\pgfsetlinewidth{1.505625pt}%
\definecolor{currentstroke}{rgb}{0.376732,0.653072,0.822484}%
\pgfsetstrokecolor{currentstroke}%
\pgfsetdash{}{0pt}%
\pgfpathmoveto{\pgfqpoint{0.678704in}{1.087654in}}%
\pgfpathlineto{\pgfqpoint{0.817593in}{1.087654in}}%
\pgfpathlineto{\pgfqpoint{0.956482in}{1.087654in}}%
\pgfusepath{stroke}%
\end{pgfscope}%
\begin{pgfscope}%
\definecolor{textcolor}{rgb}{0.000000,0.000000,0.000000}%
\pgfsetstrokecolor{textcolor}%
\pgfsetfillcolor{textcolor}%
\pgftext[x=1.067593in,y=1.039043in,left,base]{\color{textcolor}\rmfamily\fontsize{10.000000}{12.000000}\selectfont \(\displaystyle lr=0.01\)}%
\end{pgfscope}%
\begin{pgfscope}%
\pgfsetrectcap%
\pgfsetroundjoin%
\pgfsetlinewidth{1.505625pt}%
\definecolor{currentstroke}{rgb}{0.114802,0.424437,0.695194}%
\pgfsetstrokecolor{currentstroke}%
\pgfsetdash{}{0pt}%
\pgfpathmoveto{\pgfqpoint{0.678704in}{0.893981in}}%
\pgfpathlineto{\pgfqpoint{0.817593in}{0.893981in}}%
\pgfpathlineto{\pgfqpoint{0.956482in}{0.893981in}}%
\pgfusepath{stroke}%
\end{pgfscope}%
\begin{pgfscope}%
\definecolor{textcolor}{rgb}{0.000000,0.000000,0.000000}%
\pgfsetstrokecolor{textcolor}%
\pgfsetfillcolor{textcolor}%
\pgftext[x=1.067593in,y=0.845370in,left,base]{\color{textcolor}\rmfamily\fontsize{10.000000}{12.000000}\selectfont \(\displaystyle lr=0.003\)}%
\end{pgfscope}%
\begin{pgfscope}%
\pgfsetrectcap%
\pgfsetroundjoin%
\pgfsetlinewidth{1.505625pt}%
\definecolor{currentstroke}{rgb}{0.031373,0.188235,0.419608}%
\pgfsetstrokecolor{currentstroke}%
\pgfsetdash{}{0pt}%
\pgfpathmoveto{\pgfqpoint{0.678704in}{0.700308in}}%
\pgfpathlineto{\pgfqpoint{0.817593in}{0.700308in}}%
\pgfpathlineto{\pgfqpoint{0.956482in}{0.700308in}}%
\pgfusepath{stroke}%
\end{pgfscope}%
\begin{pgfscope}%
\definecolor{textcolor}{rgb}{0.000000,0.000000,0.000000}%
\pgfsetstrokecolor{textcolor}%
\pgfsetfillcolor{textcolor}%
\pgftext[x=1.067593in,y=0.651697in,left,base]{\color{textcolor}\rmfamily\fontsize{10.000000}{12.000000}\selectfont \(\displaystyle lr=0.001\)}%
\end{pgfscope}%
\end{pgfpicture}%
\makeatother%
\endgroup%

%% file: conclusion.tex
\section{Conclusion}

Motivated by the manifold hypothesis, we consider a classifier architecture that first projects onto a principal component approximation of the data manifold and then applies randomized smoothing in the low-dimensional projected space. This yields a precise characterization of the input-space certified region as capturing disturbances in the projection nullspace. We interpret this as a certifiable robustification against vulnerable features that are irrelevant to the dataset information content as they are normal to the data manifold. We show that unprotected classifiers, unlike our method, are vulnerable to such perturbations by explicitly constructing adversarial examples in the span of the low-variance principal components. We prove a volumetric lower bound on the intersection of our certified region with the unit cube of feasible inputs and derive two additional ways to tighten the bound: \hl{one which involves solving an $\ell_{\infty}$-regression problem and another which is a closed-form radius adjustment.}

Comparing against state-of-the-art $\ell_1$-, $\ell_2$-, $\ell_{\infty}$-, and anisotropic baselines shows that our classifier produces certified regions with many orders of magnitude greater volume. This confirms an asymptotic analysis that shows that our method's certified volumes decay factorially in the low dimension of the \textit{projected space}, while competing methods decay factorially in the high dimension of the \textit{input space}.
\hl{Future research directions include examining more sophisticated dimensionality reduction techniques while maintaining certified guarantees for projected points in the original input space.}

\clearpage

%\begin{toappendix}
    %\subsection{Societal impact} \label{app:impact}
    %Improving neural network robustness is critical for ensuring that machine learning models are safe to deploy in real-world applications such as autonomous driving and medical diagnostics. However, potential side effects of improving robustness are not well understood, with some research suggesting that robust networks may be more biased \citep{chang2020adversarial}. While our work focuses on certifiable robustness, similar concerns may apply and are an important topic of future research.
%\end{toappendix}